\documentclass{article}
\usepackage[utf8]{inputenc}

\usepackage{boxinz-macros}
\usepackage{dsfont}
\usepackage{authblk}

\usepackage{graphicx}
\usepackage{multirow}
\usepackage{hhline}
\usepackage[colorlinks=true]{hyperref}
\usepackage{mathrsfs}
\graphicspath{ {./image/} }
\usepackage{xcolor}
\usepackage{subcaption} 
\usepackage{caption}

\usepackage{longtable}

\usepackage{placeins}

\usepackage{algorithm2e}
\usepackage{algorithmic}

\usepackage{centernot}

\newcommand{\Cov}{\text{Cov}}
\DeclareMathOperator{\tr}{tr}
\DeclareMathOperator{\Var}{Var}
\newcommand{\vertiii}[1]{{\left\vert\kern-0.25ex\left\vert\kern-0.25ex\left\vert #1 \right\vert\kern-0.25ex\right\vert\kern-0.25ex\right\vert}}
\newcommand\independent{\protect\mathpalette{\protect\independenT}{\perp}}
\def\independenT#1#2{\mathrel{\rlap{$#1#2$}\mkern2mu{#1#2}}}

\newcommand{\lfb}{ \mathcal{Q}^1_{n, \delta} }

\newcommand{\newText}{\color{.}}
\newcommand{\newTextTwo}{\color{.}}

\title{High-dimensional Functional Graphical Model Structure Learning via Neighborhood Selection Approach}

\author[1]{Boxin Zhao}
\author[1]{Percy S. Zhai}
\author[2]{Y. Samuel Wang}
\author[3]{Mladen Kolar}

\affil[1]{Booth School of Business, The University of Chicago}
\affil[2]{Department of Statistics and Data Science, Cornell University}
\affil[3]{Department of Data Sciences and Operations, Marshall School of Business, University of Southern California}

\date{}

\begin{document}

\maketitle

\begin{abstract}
Undirected graphical models are widely used to model the conditional independence structure of vector-valued data. However, in many modern applications, for example those involving EEG and fMRI data, observations are more appropriately modeled as multivariate random functions rather than vectors. Functional graphical models have been proposed to model the conditional independence structure of such functional data. We propose a neighborhood selection approach to estimate the structure of Gaussian functional graphical models, where we first estimate the neighborhood of each node via a function-on-function regression and subsequently recover the entire graph structure by combining the estimated neighborhoods. Our approach only requires assumptions on the conditional distributions of random functions, and we estimate the conditional independence structure directly. We thus circumvent the need for a well-defined precision operator that may not exist when the functions are infinite dimensional. Additionally, the neighborhood selection approach is computationally efficient and can be easily parallelized. The statistical consistency of the proposed method in the high-dimensional setting is supported by both theory and experimental results. In addition, we study the effect of the choice of the function basis used for dimensionality reduction in an intermediate step. We give a heuristic criterion for choosing a function basis and motivate two practically useful choices, which we justify by both theory and experiments. 
\end{abstract}

\section{Introduction}

Multivariate functional data are collected in applications such as neuroscience, medical science, traffic monitoring, and finance. 
Although each observation is typically only recorded at a discrete set of time points, the underlying process may be interpreted as a realization of a multivariate stochastic process in continuous time. Such interpretation can provide a unifying approach to the analysis of classical functional data and longitudinal data, where functional data can be used to deal with sparsely observed measurements with noise \citep{chiou2016pairwise}.

{\newText
Our work is centered around elucidating the conditional independence structure of multivariate random functions. Gaining a robust understanding of such a structure can yield extensive applications, including the interpretation of time course gene expression data in genomics~\citep{wei2008hidden}, multivariate time series data in finance~\citep{tsay2017modelling}, and electroencephalography (EEG) data in neuroscience~\citep{Qiao2015Functional,Qiao2017Doubly}. This paper is motivated by the analysis of data gathered from fMRI scans conducted on $116$ distinct brain regions, with a time-series signal recorded for each region \citep{milham2012adhd}. The sample comprises two groups: one group of individuals diagnosed with Attention Deficit Hyperactivity Disorder (ADHD) and a control group. Our aim is to comprehend the functional connectivity patterns between these brain regions for both the ADHD and control groups. Such functional connectivity can be uncovered by determining the conditional independence structure across the $116$ random functions.
}

Graphical models are widely used to represent the conditional independence structure of multivariate random variables \citep{Lauritzen1996Graphical}. Let $G = \{V, E\}$ denote an undirected graph where $V$ is the set of vertices and $E \subset V^2$ is the set of edges. When the data consist of random vectors $\bm{X}=(X_1,\dots,X_p)^{\top}$, we say that $\bm{X}$ satisfies the pairwise Markov property with respect to $G$ if $X_{v} \independent X_{w} \mid \{X_u\}_{u \in V\setminus \{v,w\}}$ holds if and only if $\{v,w\} \not\in E$. This notion has been extended to \emph{functional graphical models}---where each node represents a random function rather than a random scalar---in order to characterize the conditional independence relationship of multivariate random functions.

We propose a procedure to estimate the functional graphical model when the random functions follow a \emph{multivariate Gaussian processes} (MGP). This setting was considered in \citet{Qiao2015Functional}, who proposed the functional graphical lasso to estimate the structure of the graph. Their procedure first obtains a finite dimensional representation for the observed multivariate functions using functional principal component analysis (FPCA). Subsequently, a precision matrix is computed from the projection scores of the finite dimensional representation using a graphical lasso objective with a group penalty. The graph structure is finally obtained from the non-zero blocks of the estimated precision matrix. When the underlying random functions are infinite dimensional, the corresponding covariance operator is a compact operator, and its inverse, the precision operator, is ill-defined \citep{Hsing2015Theoretical}. As a result, \citet{Qiao2015Functional} ensure their estimand is well defined by requiring that the random functions lie in a finite dimensional space. However, that assumption is restrictive and excludes infinite dimensional functional data.

In contrast to the functional graphical lasso proposed by \citet{Qiao2015Functional}, we propose a neighborhood selection approach to estimate Gaussian functional graphical models. For vector-valued Gaussian graphical models, \citet{Meinshausen2006High} proposed a neighborhood selection procedure that estimates the neighborhood---the set of adjacent nodes in a conditional independence graph---for each node separately by sparse regression. The entire graph structure is then estimated by combining estimates of node-specific neighborhoods. We extend their approach to the functional data setting. This allows us to avoid defining the precision operator, and, as a result, our theory extends to truly infinite dimensional functional data. 

We cast the neighborhood selection procedure as a function-on-function regression problem. Due to the infinite dimensional nature of the functional data, we first project all observed random functions onto a finite dimensional basis. Thus, we approximate the function-on-function regression with a vector-on-vector regression problem that is solved by minimizing a squared error loss with a group lasso penalty. We do not require a specific choice of function basis for our methodology and the corresponding theory, and we provide a theoretically guided intuition on the choice of function basis under different conditions. Specifically, when estimating the neighborhood of a target node, we project all functions onto a single subspace instead of projecting each function onto its own subspace. In Section~\ref{sec:choice-fun-basis} we provide intuition for why this may be preferable to projecting each function onto its own subspace. This is supported by both the theory in Section~\ref{sec:th-ch-fun-basis} and the simulations in Section~\ref{sec:experimemts}.

In addition to the methodology, we also provide nontrivial theoretical contributions. Most importantly, by directly estimating the conditional independence structure without reference to a population ``precision operator,'' we do not require that the functional data are finite dimensional. However, in the infinite dimensional setting, there will be a residual term due to using a finite dimensional approximation, and deriving error bounds requires a careful analysis of this residual term. Finally, our theory is non-asymptotic in nature, and we derive finite-sample guarantees for graph recovery.

In summary, the neighborhood selection approach yields at least three benefits. First, it allows us to define functional graphical models directly from the conditional distribution and does not require the notion of a precision operator. As a result, we can estimate the graph structure even from infinite dimensional data, rather than restricting the data to finite dimensional functions. Second, by estimating the neighborhood of each node separately, we have increased flexibility in choosing the function basis used to represent the random functions, and tailoring the function basis for the specific task at hand results in empirically better estimation results. Finally, when estimating the neighborhood of a node, we only need to handle $p$ individual $M \times M$ matrices. These neighborhood estimation procedures can be performed in parallel, leading to a highly efficient estimation procedure. In comparison, fglasso \citep{Qiao2015Functional} needs to estimate a $pM \times pM$ matrix, which is computationally much more expensive. We demonstrate the practical value of our neighborhood selection method on the motivating ADHD fMRI data set, and also on another ASD fMRI data set.

\subsection{Related Work}

Our paper contributes to the growing literature on modeling multivariate functional data. We study the estimation of the conditional independence structure from multivariate functional data in the setting of MGPs \citep{Qiao2015Functional}. For each component of the multivariate functional data, \citet{Qiao2015Functional} projected observed functions on the corresponding function basis estimated by FPCA. Subsequently, the structure of the graph is estimated from the projection scores using the multitask extension of the graphical lasso \citep{kolar13multiatticml, Kolar2014Graph}, which estimates the precision matrix with a block structure. However, the precision operator is ill-defined when the functional data is infinite dimensional, and their method is computationally expensive when the number of nodes is large.

In the same setting, \citet{Qiao2017Doubly} proposed a dynamic functional graphical model that allows the graph structure to change over time.  \citet{Zhu2016Bayesian} proposed a Bayesian approach to functional graphical models. \citet{zapata2021partial} studied estimation of the graph structure under the assumption of partial separability. Roughly speaking, partial separability assumes that the time-varying covariance of the MGP can be decomposed node-wise into a time-varying component and a constant. However, this assumption can be restrictive and may not hold in many settings; \citet{lynch2018test} proposed a test to verify the partial separability assumption.  \citet{Zhao2019Direct} and \citet{zhao2022fudge} discussed direct estimation of the difference between two Gaussian functional graphical models without the need to estimate each individual structure. 
\citet{tsai2023latent} studied a latent multi-modal functional graphical model.
\citet{solea2020copula} extended the Gaussian functional graphical model to a copula version by allowing monotonic transformations of the FPCA scores. In addition, \citet{Li2018nonparametric} and~\citet{lee2022nonparametric} discussed a nonparametric functional graphical model; however, the graph therein is defined based on the additive conditional independence (ACI) relationship~\citep{Li2014additive,lee2016variable,Lee2016additive}, which is not equivalent to the conditional independence (CI) relationship and is thus not directly comparable to our paper.

Our paper is also related to the literature on \emph{function-on-function regression} that studies regression problems in which both the response and predictor variables are functions. \citet{Luo2017Function} and \citet{Qi2018Function} studied function-on-function linear regression, where each predictor function is transformed by a corresponding integral operator defined by the bivariate coefficient function, and the addition of all transformed predictor functions is defined as the signal function. The response function is then assumed to be a simple addition of the signal, intercept, and noise functions. To estimate the coefficient functions, they used the FPCA basis of the signal function to expand the observed functions and transform the original function-on-function regression problem into a function-on-scalar regression problem with uncorrelated predictors. Ultimately, this regression is solved by a penalized least-squares method. In contrast, we focus on variable selection, rather than prediction, and develop non-asymptotic theory. \citet{Luo2016Functional} also converted the function-on-function regression problem to function-on-scalar regression by projecting predictor functions with basis functions, but they restricted their choice of function basis to wavelet transformation, and thus can be considered as a special case of our approach. \citet{Ivanescu2015Penalized} discussed a method similar to what we propose, but they did not give any guidance on how to choose a function basis or offer any theoretical guarantees. These approaches can be treated as special cases within our framework with a specific choice of function basis. In addition to linear function-on-function regression, \citet{Xin2019Nonlinear} and \citet{Scheipl2015Functional} studied functional additive models, but did not give theoretical results on variable selection in the high-dimensional setting.

While finishing this paper, we were made aware of concurrent work by \citet{solea2021nonparametric} that also uses a neighborhood selection approach to estimate a graphical model from functional data. There are several differences between our approaches. First, when estimating the neighborhood of a particular node, \citet{solea2021nonparametric} project each observed function to its own eigenfunction basis. We consider several alternative choices for the functional basis expansion and suggest a different approach: when estimating the neighborhood of a node $j$, we project all observed functions onto the same basis---typically the eigenbasis of the $j$-th node. However, we may use different bases when estimating different neighborhoods. Second, \citet{solea2021nonparametric} assume a slightly more general setting where the FPCA scores arise from a non-parametric additive regression model, whereas the scores arise from a linear model in our work. However, it is unclear what the class of joint distributions is with conditional means that satisfy the non-linear additive structure. The Gaussian setting considered by our work seems like the only commonly used distribution to have a conditional mean that satisfies the additive structure. 
Finally, to show the consistency of graph recovery,  \citet{solea2021nonparametric} require that the truncation dimension grows with $n$; while our result shows non-asymptotic error bounds for any fixed truncation dimension and chooses the truncation dimension to satisfy a criterion that is independent of the sample size. 
{\newText
Our theoretical analysis also relies on weaker assumptions and has better convergence rates. See Section~\ref{sec:fix-fun-basis} for a detailed comparison between our theoretical results and the ones of~\citet{solea2021nonparametric}.
}

\subsection{Notation}

Let $[n]$ denote the set $\{1,2,\dots,n\}$. For a set $S$, we use $\vert S \vert$ to denote its cardinality. 

We use bold lower case letters (e.g., $\bm{a}$ and $\bm{b}$) to denote vectors and bold upper case letters (e.g., $\bm{A}$ and $\bm{B}$) to denote matrices. For a vector $\bm{a} \in\mathbb{R}^{n}$, let $\Vert \bm{a} \Vert_{q}$ denote its $l_q$-norm, $q\in[1,\infty)$, with the usual extension to $q=0$ and $q=\infty$. For a set of indices $I\subseteq [n]$, we use $\bm{a}_I$ to denote the vector in $\mathbb{R}^{n}$ with $a_{I,i}=v_i$ for all $i\in I$ and $a_{I,i}=0$ for all $i \notin I$. Let $\mathcal{G}=\{G_1,G_2,\dots,G_{ N_{\mathcal{G}} }\}$ be a partitioning of the set $[n]$ into a set of $N_{\mathcal{G}}$ disjoint groups. The mixed norm $\Vert \cdot \Vert_{1,q}$ is defined as $\Vert \bm{a} \Vert_{1,q}=\sum^{N_{\mathcal{G}}}_{t=1}\Vert \bm{a}_{G_{t}} \Vert_{q}$. For two vectors $\bm{a},\bm{b}\in\mathbb{R}^{n}$, we use $\langle \bm{a},\bm{b} \rangle =\sum^{n}_{i=1} a_i b_i$ to denote their inner product.

For a symmetric matrix $\bm{B}$, we use $\rho_{\max}(\bm{B})$ to denote its largest eigenvalue, $\rho_{\min}(\bm{B})$ to denote its smallest eigenvalue, and $\tr(\bm{B})$ to denote its trace. For a matrix $\bm{A} \in \mathbb{R}^{n_1 \times n_2}$, we use $\text{vec}(\bm{A})$ to denote the vector in $\mathbb{R}^{n_1 n_2}$ formed by stacking the columns of $\bm{A}$. For two matrices $\bm{A} \in \mathbb{R}^{n\times m}$ and $\bm{B} \in \mathbb{R}^{r \times s}$, $\bm{A} \otimes \bm{B} \in \mathbb{R}^{nr \times ms}$ denotes their Kronecker product, with $(\bm{A} \otimes \bm{B})_{ik,jl}=A_{ij} B_{kl}$. For a matrix $\bm{A} \in \mathbb{R}^{n_1 \times n_2}$, we use $\Vert \bm{A} \Vert_{q}$ to denote its operator norm, $q\in[1,\infty]$, that is, $\Vert \bm{A} \Vert_{q}=\sup_{\bm{v}\in\mathbb{R}^{n_2}:\Vert \bm{v} \Vert_q=1}{\Vert \bm{Av} \Vert_q}$. Thus, $\Vert \bm{A} \Vert_2$ denotes the maximum singular value of $\bm{A}$, $\Vert \bm{A} \Vert_1=\max_{1\leq j \leq n_2}\sum^{n_1}_{i=1}\vert A_{ij} \vert$, and $\Vert \bm{A} \Vert_\infty=\max_{1\leq i \leq n_1}\sum^{n_1}_{j=1}\vert A_{ij} \vert$. We use $\Vert \bm{A} \Vert_{F}$ to denote the Frobenius norm of $\bm{A}$, that is, $\Vert \bm{A} \Vert^{2}_{F}=\tr(\bm{A}^{\top}\bm{A})$. We use $\vertiii{\bm{A}}_{\infty}$ to denote the elementwise maximum absolute value of $\bm{A}$, that is, $\vertiii{\bm{A}}_{\infty}=\max_{i,j}\vert A_{ij}\vert$.

For a real-valued differentiable function $f:\mathbb{R}^{n}\rightarrow\mathbb{R}$, we use $\nabla f(\bm{x}) \in\mathbb{R}^{n}$ to denote its gradient at a point $\bm{x} \in\mathbb{R}^{n}$. 
For a closed interval $\mathcal{T} \subseteq \mathbb{R}$, we define $\mathcal{L}^2(\mathcal{T})$ to be the Hilbert space of square-integrable real-valued functions defined on domain $\mathcal{T}$, where for $f,g \in \mathcal{L}^2(\mathcal{T})$, we use $\langle f,g \rangle=\int_{\mathcal{T}} f(t) g(t) dt$ to denote their inner-product and $\Vert f \Vert = ( \int_{\mathcal{T}} f^2(t) dt)^{1/2}$ to denote the $L_2$-norm of $f$. For a bivariate function $B(t^{\prime},t)$ defined on $\mathcal{T} \times \mathcal{T}$, we use $\Vert B \Vert_{\text{HS}}=\Vert B(t^{\prime},t) \Vert_{\text{HS}} = (\int_{ \mathcal{T} \times \mathcal{T} } B^2(t^{\prime},t) dt^{\prime} dt)^{1/2}$ to denote its Hilbert-Schmidt norm. We use $\bm{f}(\cdot)=(f_1(\cdot),f_2(\cdot),\dots,f_n(\cdot))^{\top}$ to denote a vector with function entries.

For any two Hilbert spaces $\mathbb{H}$ and $\mathbb{G}$, we define $\mathcal{B}(\mathbb{H},\mathbb{G})$ as the set of bounded linear operators and $\mathcal{B}_{\text{HS}}(\mathbb{H},\mathbb{G})$ as the set of Hilbert-Schmidt operators from $\mathbb{H}$ to $\mathbb{G}$. Thus, $\mathcal{B}_{\text{HS}}(\mathbb{H},\mathbb{G}) \subseteq \mathcal{B}(\mathbb{H},\mathbb{G})$.
For $h \in \mathbb{H}$ and $g \in \mathbb{G}$, the outer product $g \otimes h:\mathbb{H}\mapsto\mathbb{G}$ is the rank-one linear operator $(g \otimes h)(h^{\prime}) \coloneqq \langle h,h^{\prime} \rangle_{\mathbb{H}}\, g$.
When $\mathbb{H}=\mathbb{G}$, we let $\mathcal{B}(\mathbb{H})=\mathcal{B}(\mathbb{H},\mathbb{H})$ and $\mathcal{B}_{\text{HS}}(\mathbb{H})=\mathcal{B}_{\text{HS}}(\mathbb{H},\mathbb{H})$.
For any operator $\mathscr{T}:\mathbb{H}\mapsto\mathbb{G}$, we use $\text{ran}(\mathscr{T})=\{\mathscr{T}(h):h\in \mathbb{H}\} \subseteq \mathbb{G}$ to denote its range. We denote the adjoint operator of $\mathscr{T}$~\citep[Definition 3.3.2]{Hsing2015Theoretical} by $\mathscr{T}^{*}$ and the Moore–Penrose inverse or pseudo inverse~\citep[Definition 3.5.7]{Hsing2015Theoretical} of $\mathscr{T}$ by $\mathscr{T}^{\dagger}$. 
We say that an orthonormal sequence $\{e_n\}_{n \geq 1}$ in a Hilbert space $\mathbb{H}$ is called an orthonormal basis or a complete orthonormal system (CONS) if $\overline{\text{span}\{e_n\}}=\mathbb{H}$.

For any finite number of Hilbert spaces $\mathbb{H}_1$--$\mathbb{H}_n$, we define their Cartesian product space $\mathbb{H}_1 \oplus \ldots \oplus \mathbb{H}_n \overset{\Delta}{=} \oplus^n_{i=1} \mathbb{H}_i$ as $\{ (f_1,\ldots,f_n) : f_i \in \mathbb{H}_i \}$, with endowed inner product defined by $\langle \bm{f},\bm{g} \rangle=\sum^n_{i=1} \langle f_i, g_i \rangle$ for all $\bm{f}=(f_1,\ldots,f_n),\bm{g}=(g_1,\ldots,g_n) \in \oplus^n_{i=1} \mathbb{H}_i$. 
We then have $\oplus^n_{i=1} \mathbb{H}_i$ to also be a Hilbert space.
For $\bm{f}=(f_1,\ldots,f_n) \in \oplus^n_{i=1} \mathbb{H}_i$, we use $\bm{f}_{-j}$ to denote $(f_1,\ldots,f_{j-1},f_{j+1},\ldots,f_n) \in \oplus^n_{i=1,i \neq j} \mathbb{H}_i$ for any $j \in [n]$.
When $\mathbb{H}_i=\mathbb{H}$ for all $i \in [n]$, we denote $\oplus^n_{i=1} \mathbb{H}_i$ by $\mathbb{H}^n$. 

For any two sequences $\{a_n\}_{n\geq 1}$ and $\{b_n\}_{n\geq 1}$, we use $a_n \lesssim b_n$ or $a_n=O(b_n)$ ($a_n \gtrsim b_n$ or $a_n=\Omega(b_n)$) to denote that there exists a constant $c \geq 0$ such that $a_n \leq c \cdot b_n$ ($a_n \geq c \cdot b_n$) for $n$ large enough. Similarly, we use $a_n=\tilde{O}(b_n)$ to ignore any $\log$ terms asymptotically, that is, $a_n=\tilde{O} (b_n)$ if $a_n=O(b_n \log^k b_n)$ for some $k\geq 0$.
In this paper, we use $\tilde{O}(\cdot)$ to ignore log terms of sample size, but we keep log terms of other quantities such as the number of vertices and the dimension of a truncated function.

\subsection{Outline of the Paper}

The rest of the paper is organized as follows. In Section~\ref{sec:methodology}, we introduce the functional graphical model and our methodology to estimate the graph structure. In Section~\ref{sec:ADMM}, we discuss the optimization algorithm used to compute the estimator. We develop theoretical guarantees for our approach in Section~\ref{sec:thm-prop}. Results on simulated and real data are reported in Section~\ref{sec:experimemts} and Section~\ref{sec:real_data}, respectively. We conclude the paper with a discussion in  Section~\ref{sec:conclusion}.
Code to replicate the results in this paper is available at: \url{https://github.com/PercyZhai/FGM_Neighborhood}.

\section{Methodology}
\label{sec:methodology}

In this section, we briefly review the functional graphical model in Section~\ref{sec:fgm}. We introduce a neighborhood selection procedure for estimating the functional graphical in Section~\ref{sec:funNbSelec} and discuss a practical implementation in subsequent subsections.

\subsection{Functional Graphical Model}\label{sec:fgm}

Let $\mathcal{T} \subseteq \mathbb{R}$ be a closed interval, and $\mathbb{H} \subseteq \mathcal{L}^2(\mathcal{T})$ be a Hilbert subspace of $\mathcal{L}^2(\mathcal{T})$. 
Since $\mathcal{L}^2(\mathcal{T})$ is a separable Hilbert space, $\mathbb{H}$ and $\mathbb{H}^n$ are also seperable Hilbert spaces for any $1 \leq n < \infty$.
Let $(\Omega, \mathcal{F}, \mathbb{P})$ be a probability space and $\bm{g}:\Omega \mapsto \mathbb{H}^{p}$ be a Gaussian random element; that is, for any $\bm{h} \in \mathbb{H}$, we have $\langle \bm{g}, \bm{h} \rangle$ be a real-valued Gaussian random variable\footnote{The existence and construction of Gaussian random elements taking values in any separable Hilbert space is shown as following: By Example 1.25 of~\cite{janson_1997}, we can construct Gaussian random elements taking values in $l^2$ space, that is, the space of square summable sequences. The desired conclusion then follows from the fact that any separable Hilbert space is isometrically isomorphic to $l^2$~\citep[Theorem 2.4.17]{Hsing2015Theoretical}.}. We can express $\bm{g}$ as $\bm{g}(\omega,t)=(g_1(\omega,t),\ldots,g_p(\omega,t))$, where $(\omega,t) \in \Omega \times \mathcal{T}$, and for all $\omega \in \Omega$ and $j \in [p]$, we have $g_j(\omega,\cdot) \in \mathbb{H}$ to be a function with domain $\mathcal{T}$. In the rest of the paper, we suppress the dependency on $\omega$, and denote $\bm{g}(t) = (g_1(t),\ldots, g_p(t) )$ for $t \in \mathcal{T}$, or sometimes even suppress the dependency on $t$ and let $\bm{g} = (g_1, \ldots, g_p )$.

To simplify the discussion, we assume that $\bm{g}$ is zero mean, that is, $\mathbb{E}[g_j]=0$ for all $j \in [p]$.
Furthermore, based on the Gaussian property, we have $\mathbb{E} [ \Vert \bm{g} \Vert^2] < \infty$. Thus, for all $j \in [p]$, we can define the covariance operator of $g_j$ as
\begin{equation}
\label{eq:cov-operator}
\mathscr{K}_j \coloneqq \mathbb{E} \left[ g_j \otimes g_j \right],
\end{equation}
and we have $\mathscr{K}_j \in \mathcal{B}_{\text{HS}}(\mathbb{H})$. In addition, for any index set $I,I_1,I_2 \subseteq [n]$, we define
\begin{equation*}
\mathscr{K}_I \coloneqq \mathbb{E} \left[ \left( g_j \right)_{j \in I} \otimes \left( g_j \right)_{j \in I} \right], \quad \mathscr{K}_{I_1,I_2} \coloneqq \mathbb{E} \left[ \left( g_j \right)_{j \in I_1} \otimes \left( g_j \right)_{j \in I_2} \right].
\end{equation*}
Furthermore, following \citet{Qiao2015Functional}, we define the conditional cross-covariance function as
\begin{equation*}
C_{jl}(t^{\prime},t) = \Cov \left( g_j(t^{\prime}), g_l(t) \mid g_k(\cdot), k \neq j,l \right).
\end{equation*}


Let $G=(V,E)$ denote an undirected graph where $V=[p]$ is the set of vertices or nodes, and $E \subset V^2$ is the set of edges. The edge set $E$ encodes the pairwise Markov property of $\bm{g}$~\citep{Lauritzen1996Graphical} if
\begin{equation}\label{eq:edgeDef}
E = \left\{ (j,l) \in V^2: j \neq l \text{ and } g_j \not \independent g_l \mid \{g_k\}_{k \neq, j,l} \right\}.
\end{equation}
Let $\bm{g_i}(\cdot)=(g_{i1}(\cdot), \ldots, g_{ip}(\cdot) )$ be a random copy of $\bm{g}(\cdot)$. The goal of this work is to estimate the set of edges $E$ when given $n$ i.i.d.~random copies $\{\bm{g_i}(\cdot)\}_{i=1}^n$. \citet{Qiao2015Functional} proposed to estimate $E$ using a functional graphical lasso procedure. In contrast, we propose a neighborhood selection approach detailed in the following section. In the following, we define the neighborhood of node $j$ as 
\begin{equation}
\label{eq:nighborhood-definition}
\mathscr{N}_j \coloneqq \{k \, : \, (j,k) \in E\}.
\end{equation}

\subsection{Functional Neighborhood Selection}
\label{sec:funNbSelec}

We develop a neighborhood selection procedure to estimate the functional graphical model. The neighborhood selection approach can be traced back to \citet{Besag1975Statistical} and was further developed for Gaussian graphical models in a high-dimensional setting by \citet{Meinshausen2006High}. Specifically, \citet{Meinshausen2006High} estimated the conditional independence graph for vector-valued data $\mathbf{X} = (X_k)_{k \in [p]}$ drawn from a multivariate Gaussian. Properties of the multivariate Gaussian ensure that for each $j \in [p]$, there exist $\{\beta_{jk}\}_{k \neq j}$ such that
\begin{equation}
\label{eq:linear-repre-scalar}
X_j = \sum_{k \neq j} \beta_{jk} X_k + \varepsilon_j,
\end{equation}
where $\varepsilon_j$ is normally distributed and independent of all $X_k$, $k \neq j$. By~\eqref{eq:linear-repre-scalar}, it is clear that $\mathscr{N}_j$ is equivalent to the set $\{k \in [p]\backslash\{j\} \,: \, \beta_{jk} \neq 0\}$. Thus, \citet{Meinshausen2006High} use the variables selected from a penalized regression of $X_j$ onto all other variables to estimate $\mathscr{N}_j$; specifically, $\mathscr{\hat N}_j = \{k \in [p]\backslash\{j\} \, : \, \hat \beta_{jk} \neq 0\}$. After estimating each neighborhood, they combine the estimates into a single estimate of the entire graph $G$. 

Our first contribution is to show that an analogous representation to~\eqref{eq:linear-repre-scalar} also holds for $\bm{g}$ under mild conditions. We start by considering the conditional expectation $\mathbb{E}[g_j \mid \bm{g}_{-j}]$ for $j \in [p]$. By Doob–Dynkin representation~\citep[Lemma 1.13]{kallenberg1997foundations}, we have a measurable map $\mathscr{B}_j:\mathbb{H}^{p-1} \mapsto \mathbb{H}$ such that $\mathbb{E}[g_j \mid \bm{g}_{-j}]=\mathscr{B}_j(\bm{g}_{-j})$ almost surely. Due to the Gaussianity of $\bm{g}$, we have $\mathscr{B}_j \in \mathcal{B}(\mathbb{H}^{p-1},\mathbb{H})$, and $e_j \coloneqq g_j-\mathbb{E}[g_j \mid \bm{g}_{-j}]$ to be Gaussian and independent of $\bm{g}_{-j}$~\citep{klebanov2021linear}.
For the purposes of this paper, we require $\mathscr{B}_j$ to be in a more narrow class, namely the class of Hilbert-Schmidt operators. Therefore, we make the following assumption.
\begin{assump}
\label{assump:Hilbert-schmidt-assump}
For all $j \in [p]$, we assume that $\mathscr{B}_j \in \mathcal{B}_{\text{HS}}(\mathbb{H}^{p-1},\mathbb{H})$.
\end{assump}
The intuition of the requirement for Assumption~\ref{assump:Hilbert-schmidt-assump} is associated with the infinite-dimensional nature of functional data. To characterize $\mathscr{B}_j$ in general, one will need to estimate an infinite number of parameters, which is prohibitive with a finite sample size. For this reason, any practical solution must approximate $\mathscr{B}_j$ with a finite-dimensional truncation. Since any linear bounded operator between two finite-dimensional Hilbert spaces is congruent to a matrix that has a bounded Hilbert-Schmidt (Frobenius) norm, Assumption~\ref{assump:Hilbert-schmidt-assump} is necessary to ensure a bounded truncation error. This assumption is also made in~\citet{solea2021nonparametric}---see Assumption 4.6 therein.

To understand what kind of data generation process will satisfy Assumption~\ref{assump:Hilbert-schmidt-assump}, let us consider a special case. Suppose that
\begin{equation*}
\text{ran}\left(\mathscr{K}_{\mathscr{N}_j,j} \right) \subseteq  \text{ran}\left(\mathscr{K}_{\mathscr{N}_j} \right).
\end{equation*}
Then by Theorem 4.8 in~\citet{klebanov2021linear} and noting that $\mathbb{E}[g_j \mid \bm{g}_{-j}]=\mathbb{E}[g_j \mid g_k,k \in \mathcal{N}_j]$, we have
\begin{equation*}
\mathscr{B}_j = \left( \mathscr{K}^{\dagger}_{\mathscr{N}_j} \mathscr{K}_{\mathscr{N}_j,j} \right)^{*}.
\end{equation*}
Therefore, requiring that $\mathscr{B}_j$ is Hilbert-Schmidt is equivalent to requiring that $\mathscr{K}^{\dagger}_{\mathscr{N}_j} \mathscr{K}_{\mathscr{N}_j,j}$ is Hilbert-Schmidt. 
To illustrate when this condition holds, we assume that the left singular functions of $\mathscr{K}_{\mathscr{N}_j,j}$ ordered by singular values coincide with the eigenfunctions of $\mathscr{K}_{\mathscr{N}_j}$ ordered by eigenvalues. Let $\{s_{1k}\}^{\infty}_{k=1}$ be the non-increasing singular values of $\mathscr{K}_{\mathscr{N}_j,j}$ and $\{s_{2k}\}^{\infty}_{k=1}$ be the non-increasing eigenvalues of $\mathscr{K}_{\mathscr{N}_j}$. Then requiring $\mathscr{B}_j$ to be Hilbert-Schmidt will be equivalent to requiring $\sum^{\infty}_{k=1} (s_{1k}/s_{2k})^2<\infty$. Intuitively, $s_{1k}$ corresponds to the covariance between $g_j$ and its neighbors $(g_l)_{l \in \mathscr{N}_j}$ along a direction in $\mathbb{H}^{\vert \mathscr{N}_j \vert}$, while $s_{2k}$ represents the variance of $(g_l)_{l \in \mathscr{N}_j}$ along that direction.
The condition that $\mathscr{B}_j$ is Hilbert-Schmidt basically requires that the covariance between $g_j$ and its neighbors decreases sufficiently fast compared to the decreasing speed of the variance of its neighbors.
When Assumption~\ref{assump:Hilbert-schmidt-assump} is violated, then regardless of the dimension of the space used for truncation, there always exists a subspace orthogonal to it, such that the projection of $(g_l)_{l \in \mathscr{N}_j}$ onto it has small variance, but the covariance between the projection and $g_j$ is relatively large. As a result, the behavior of $\mathscr{B}_j$ on this subspace cannot be ignored, and thus we cannot get a good approximation of $\mathscr{B}_j$ by using any finite-dimensional truncation.

Based on Assumption~\ref{assump:Hilbert-schmidt-assump}, we have a representation similar to~\eqref{eq:linear-repre-scalar} for $\bm{g}$, which we state in the following theorem.
\begin{theorem}
\label{thm:neigh-represent}
Assume that Assumption~\ref{assump:Hilbert-schmidt-assump} holds for all $j \in [p]$. Then for all $j \in [p]$, there exists $\left\{\beta_{jk}(t, t^\prime)\right\}_{k \neq j}$ such that
\begin{equation}
\label{eq:function-linear}
g_{j}(t) = \sum_{k\neq j} \int_{\mathcal{T}} \beta_{jk} (t,t^{\prime}) g_{k}(t^{\prime}) d t^{\prime} + e_{j}(t),
\end{equation}
where $e_{j}(\cdot) \independent g_{k}(\cdot)$, $k \neq j$, and $\Vert \beta_{jk} (t,t^{\prime}) \Vert_{\text{HS}} < \infty$. In addition, for any sequence $\{\phi_m\}^{\infty}_{m=1}$ being a CONS of $\mathbb{H}$, we have
\begin{equation}
\label{eq:beta-l-expansion}
\beta_{jk} (t,t^{\prime}) = \sum^{\infty}_{m,m^{\prime}=1} b^{\ast}_{jk,m m^{\prime}} \phi_m (t) \phi_{m^{\prime}} (t^{\prime}) \quad \text{a.e.},
\end{equation}
where
\begin{equation}
\label{eq:def-b-ast}
b^{\ast}_{jk,m m^{\prime}} = \int_{\mathcal{T}\times \mathcal{T}} \beta_{jk} (t^{\prime},t) \phi_m (t) \phi_{m^{\prime}} (t^{\prime}) dt^{\prime} dt.
\end{equation}
\end{theorem}
\begin{proof}
See Appendix~\ref{sec:proof-neigh-represent}.
\end{proof}
Although it is straightforward to postulate that such a linear representation holds for multivariate Gaussian random functions, to the best of our knowledge, we are the first to strictly prove it. 
When the index is clear from the context, we will remove the subscript $j$ from $\beta_{jk}(t,t^{\prime})$.  
Given the representation in~\eqref{eq:function-linear}, it is clear that $\mathscr{N}_j$ defined in~\eqref{eq:nighborhood-definition} is equivalent to 
\begin{equation}
\label{eq:neighborhoodDef}
\mathscr{N}_j = \{ k \in [p]\backslash\{j\} \, : \, \Vert \beta_{jk} \Vert_{\text{HS}} > 0 \}.
\end{equation}
We can thus adapt the neighborhood selection approach to functional data and seek to construct an estimate of the graph by first estimating each neighborhood. 

We denote the size of the neighborhood as $s_j = \vert \mathscr{N}_j \vert$. To estimate the neighborhood for $j \in V$, we regress $g_{j}$ on $\{ g_{k} : k \in [p]\backslash\{j\} \}$ using a penalized functional regression approach. Despite the conceptual simplicity and high level similarity to \citet{Meinshausen2006High}, there are numerous technical challenges that need to be addressed in the functional data setting, which we discuss in Section~\ref{sec:thm-prop}. 

\begin{algorithm}[t]
\caption{Functional neighborhood selection}
\label{Alg:functionalNeighborhood}
\begin{algorithmic}
\REQUIRE{Observed random functions $\{\bm{g_i}(\cdot)\}_{i=1}^n$}
\FOR{$j \in V$}
\STATE{Estimate the projection basis $\bm{\phi_j}$ if it is not fixed in advance}
\STATE{Use~\eqref{eq:calcProjScores} to calculate projection scores for all observed functions on $\bm{\phi_j}$}
\STATE{Given projection scores, solve \eqref{eq:loss-lasso}}
\STATE{Estimate $ \mathscr{\hat N}_j$ using \eqref{eq:nbg-set-est}} 
\ENDFOR
\STATE{Combine all neighborhoods into the estimated edge set using AND/OR rule}
\ENSURE{Return $\hat E$}
\end{algorithmic}
\end{algorithm}

\subsection{Vector-on-Vector Regression}
\label{sec:f-on-f2v-on-v}
When the observed functions $\{\bm{g_i}(\cdot)\}_{i=1}^n$ are infinite dimensional objects, the regression problem suggested by \eqref{eq:function-linear} cannot be solved directly. As a practical estimation procedure, we first approximate the function-on-function regression problem with a tractable finite dimensional vector-on-vector regression problem.

Suppose we seek to estimate $\mathscr{N}_j$ for a fixed target node $j \in [p]$. As a first step, we represent potentially infinite dimensional functions using a finite $M$-dimensional basis. Let $\bm{\phi_j}=\{ \phi_{jm} \}^{\infty}_{m=1}$ be an orthonormal basis of $\mathbb{H}$; for now, we assume that it is given, and details on selecting an appropriate basis will be discussed in Section~\ref{sec:choice-fun-basis}. Using the first $M$ basis functions, we compute the projection scores for each $k \in [p]$ and $m \in [M]$:
\begin{equation}\label{eq:calcProjScores}
a_{ikm} = \langle g_{ik}, \phi_{jm} \rangle = \int_{\mathcal{T}} g_{ik}(t)\phi_{jm}(t) dt,    
\end{equation}
and form the projection score vectors $\bm{a_{i,k,M}}=(a_{ik1}, \dots, a_{ikM})^{\top}$. For each observed function, the scores encode the $L_2$ projection onto the first $M$ elements of $\bm{\phi_j}$ and $g_{ik}(\cdot) \approx \sum_{m=1}^M a_{ikm} \phi_{jm}(\cdot)$.

The target node $j$, will typically be fixed, so for ease of presentation, we assume $j = p$. Furthermore, we follow the commonly used regression notation and denote the random function of the target node, $g_{ij}(\cdot)$, as $g^Y_i(\cdot)$ and denote the other $p-1$ random functions as $( g^{X_1}_i(\cdot),\dots,g^{X_{p-1}}_i(\cdot) )^{\top}$. 
We let $a^Y_{i m}=\langle g^Y_i, \phi_{jm} \rangle$ and $a^{X_k}_{i m}=\langle g^{X_k}_i, \phi_{jm} \rangle$ denote the scores for observed functions and let $\bm{a^Y_{i,M}}$ and $\bm{a^{X_k}_{i,M}}$ denote the vectors of scores. At times, we will also use the notation 
\[\bm{a^X_{i,M}} = \left( (\bm{a^{X_1}_{i,M}})^{\top}, \ldots, (\bm{a^{X_{p-1}}_{i,M}})^{\top}\right)^{\top} \in \mathbb{R}^{(p-1)M} .\]
As shown in Appendix~\ref{sec:deri2formulas}, $\bm{a^Y_{i,M}}$ can be represented as
\begin{equation}
\label{eq:regression-target}
\bm{a^Y_{i,M}} = \sum^{p-1}_{k=1} \bm{B^{\ast}_{k, M}}  \bm{a^{X_k}_{i,M}} + \bm{w_{i,M}} + \bm{r_{i,M}},
\end{equation}
where
\begin{equation}
\label{eq:reg-coeff-mat}
\bm{B^{\ast}_{k, M}}=(b^{\ast}_{k,m m^{\prime}})_{1 \leq m, m^{\prime} \leq M} \in \mathbb{R}^{M \times M}
\end{equation}
is a regression matrix parameter corresponding to $\beta_{jk}(\cdot,*)$ defined in Theorem~\ref{thm:neigh-represent} and
\begin{equation*}
b^{\ast}_{k,m m^{\prime}} = \int_{\mathcal{T}\times \mathcal{T}} \beta_{jk} (t^{\prime},t) \phi_m (t) \phi_{m^{\prime}} (t^{\prime}) dt^{\prime} dt \quad \text{for all } m,m^{\prime} \geq 1.
\end{equation*}
{\newTextTwo
For better illustrating the proposed method, we also compare the regression matrix $\bm{B^{\ast}_{k, M}}$ with the conditional covariance operator $\mathscr{B}_j$ in Assumption~\ref{assump:Hilbert-schmidt-assump}. By Assumption~\ref{assump:Hilbert-schmidt-assump} and Appendix~\ref{sec:proof-neigh-represent}, we have
\begin{equation*}
\mathscr{B}_j = \sum_{k \neq j} \mathscr{B}_{jk}, \text{ and  } \mathscr{B}_{jk}=\sum^{\infty}_{m=1}\sum^{\infty}_{m^{\prime}=1} b^{\ast}_{k,m m^{\prime}} \phi_m \otimes \phi_{m^{\prime}}.
\end{equation*}
Compared to~\eqref{eq:reg-coeff-mat}, we can see that the regression matrix $\bm{B^{\ast}_{k, M}}$ can be regarded as a finite-dimensional approximation of $\mathscr{B}_{jk}$ with respect to orthonormal basis $\{ \phi_{m} \}^{\infty}_{m=1}$. 
}

Besides, we have
\begin{equation*}
r_{im} = \sum^{p-1}_{k=1} \sum^{\infty}_{m^{\prime}=M+1} b^{\ast}_{k,m m^{\prime}} a^{X_k}_{i m^{\prime}},
\end{equation*}
$\bm{r_{i,M}}=(r_{i1},\ldots,r_{iM})^{\top}$, $w_{im}=\int_{\mathcal{T}} e_{ij}(t) \phi_m (t) dt$ and $\bm{w_{i,M}}=(w_{i1},\ldots,w_{iM})^{\top}$. The term
$\bm{w_{i,M}}$ is the noise vector corresponding to $e_{ij}(\cdot)$ defined in Theorem~\ref{thm:neigh-represent}, and $\bm{r_{i,M}}$ is a bias term, which arises due to only using the first $M$ basis functions. More details are provided in Section~\ref{sec:thm-prop}.

Based on~\eqref{eq:regression-target}, we may define the truncated neighborhood of node $j$ as
\begin{equation}
\label{eq:trunc-neighbor}
\mathscr{N}^M_j \coloneqq \left\{ k \in [p]\backslash\{j\} \, : \, \Vert \bm{B^{\ast}_{k, M}} \Vert_{\text{F}} > 0 \right\}.
\end{equation}
Note that in contrast to $\mathscr{N}_j$, $\mathscr{N}^M_j$ depends on the finite-dimensional objects $\bm{B^{\ast}_{1, M}},\ldots,\bm{B^{\ast}_{p-1, M}}$, and thus it is estimable with a finite sample size.
Since for $j \notin \mathscr{N}_j$, we have $\beta_{jk}=0$ a.e., which implies that $\Vert \bm{B^{\ast}_{k, M}} \Vert_{\text{F}}=0$ for all $M \geq 1$, thus we have $j \notin \mathscr{N}^M_j$ for all $M \geq 1$. This way, it is clear that $\mathscr{N}^M_j \subseteq \mathscr{N}_j$ for all $M \geq 1$. On the other hand, when we choose $M$ large enough, such that $\Vert \bm{B^{\ast}_{k, M}} \Vert_{\text{F}} > 0$ for all $k \in \mathscr{N}_j$, we then have $\mathscr{N}^M_j = \mathscr{N}_j$.

Given $n$ i.i.d.~samples $\{\bm{g_i}(\cdot)\}_{i=1}^n$, we estimate $\bm{B^{\ast}_{k, M}}$---and subsequently $\mathscr{N}^M_j$ and $\mathscr{N}_j$---using a penalized least squares approach. Let $a^Y_{i m}$, $a^{X_k}_{i m}$, $\bm{a^Y_{i,M}}$, and $\bm{a^{X_k}_{i, M}}$ denote the quantities arising from the $i$th observed sample. We select $\bm{B^{\ast}_{k, M}}$ by minimizing the following objective:
\begin{equation}\label{eq:loss-lasso}
\bm{\hat{B}_{1, M}}, \cdots, \bm{\hat{B}_{p-1, M}} \in \arg \min_{\bm{B_1},\dots,\bm{B_{p-1}}} \left\{  \frac{1}{2n} \sum^n_{i=1} \left\Vert \bm{a^Y_{i, M}} - \sum^{p-1}_{k=1} \bm{B_{k}} \bm{a^{X_k}_{i, M}} \right\Vert^2_2 + \lambda_n \sum^{p-1}_{k=1} \Vert \bm{B_k} \Vert_{\text{F}} \right\},
\end{equation}
where $\lambda_n$ is a tuning parameter. In Section~\ref{sec:ADMM}, we propose an efficient optimization algorithm to solve~\eqref{eq:loss-lasso}. 
The challenge in giving statistical guarantees for the estimators obtained by~\eqref{eq:loss-lasso} lies in the fact that $\bm{r_{i,M}}$ and $\bm{a^X_{i,M}}$ in~\eqref{eq:regression-target} are correlated, so $\bm{B^{\ast}_{1, M}}, \ldots, \bm{B^{\ast}_{p-1, M}}$ are not the coefficients of the best linear unbiased estimators for predicting $\bm{a^Y_{i,M}}$ by $\bm{a^X_{i,M}}$, which is the general setting assumed in the group LASSO analysis.
However, when the covariance between $\bm{r_{i,M}}$ and $\bm{a^X_{i,M}}$ is small enough in the sense discussed in Section~\ref{sec:thm-prop}, $\bm{\hat{B}_{1,M}}, \dots, \bm{\hat{B}_{p-1, M}}$ may still be good estimators of $\bm{B^{\ast}_{1, M}}, \ldots, \bm{B^{\ast}_{p-1, M}}$.

Given $\bm{\hat{B}_{1,M}}, \dots, \bm{\hat{B}_{p-1, M}}$, the estimated neighborhood set is then
\begin{equation}
\label{eq:nbg-set-est}
\hat{\mathscr{N}}_j = \{k \in [p-1] \, :\, \|\bm{\hat{B}_k}\|_F>\epsilon_n\},
\end{equation}
where the threshold $\epsilon_n$ is a tuning parameter. Finally, the estimated edge set $\hat{E}$ is obtained by combining the estimated neighborhoods of each node. Following \citet{Meinshausen2006High}, the edge set $\hat{E}$ can be computed by one of the following schemes: 
\begin{itemize}
    \item \textbf{AND}: if both $j\in\hat{\mathscr{N}}_l$ and $l\in\hat{\mathscr{N}}_j$ hold, then $(j,l)\in\hat{E}$;
    \item \textbf{OR}: if either $j\in\hat{\mathscr{N}}_l$ or $l\in\hat{\mathscr{N}}_j$ holds, then $(j,l)\in\hat{E}$.
\end{itemize}

To operationalize the procedure, we discuss the choice of basis functions and the choice of tuning parameters in the following two sections.

\subsection{Choice of Basis}
\label{sec:choice-fun-basis}

A key element in the above procedure is the choice of the basis $\bm{\phi_j}$. Throughout the paper, we assume that the basis is orthonormal; if the user specifies a non-orthonormal basis, it can first be orthonormalized with a procedure such as the Gram-Schmidt algorithm (Theorem 2.4.10 of~\citet{Hsing2015Theoretical}).

At a high level, there are two different approaches that can be used: the basis can be fixed in advance, or the basis can depend on the data. In the first approach, one uses a known basis, which could be selected via prior knowledge, or simply a commonly used basis for which projection scores can be efficiently computed (e.g., the Fourier, B-spline, and wavelet bases). The second approach uses a basis that is determined by unobserved population quantities and needs to be estimated before computing the projection scores. For example, functional PCA (FPCA) can be used to estimate a basis \citep[Chapter 8]{Ramsay2005Functional}. In the previous section, we discussed vector-on-vector regression assuming that the basis $\bm{\phi_j}$ was known a priori, and here we discuss the case where the basis must be estimated.

For a chosen node $j \in [p]$ and any $i \in [n]$, suppose that we have an estimate $\{\hat{\phi}_{jm}\}_{m \geq 1}$ of the ``true'' basis $\{\phi_{jm}\}_{m \geq 1}$. Let $\hat{a}^Y_{i m}=\langle g^Y_i, \hat{\phi}_{jm} \rangle$,  $\hat{a}^{X_k}_{i m}=\langle g^{X_k}_{i}, \hat{\phi}_{jm} \rangle$, $\bm{\hat{a}^Y_{i,M}}=(\hat{a}^Y_{i1},\ldots,\hat{a}^Y_{iM})^{\top}$, and $\bm{\hat{a}^{X_k}_{i,M}}=(\hat{a}^{X_k}_{i1},\ldots,\hat{a}^{X_k}_{iM})^{\top}$. Similarly to \eqref{eq:regression-target}, we have
\begin{equation}
\label{eq:regression-target-hat}
\bm{\hat{a}^Y_{i,M}} = \sum^{p-1}_{k=1} \bm{B^{\ast}_{k,M}} \bm{\hat{a}^{X_k}_{i,M}} + \bm{w_{i,M}} + \bm{r_{i,M}} + \bm{v_{i,M}},
\end{equation}
where the additional term $\bm{v_{i,M}}$ is defined in~\eqref{eq:def-v} in the appendix, which arises from using $\bm{\hat \phi_j}$ instead of $\bm{\phi_j}$. When $\bm{\hat \phi_j}$ is close to $\bm{\phi_j}$, the error term $\bm{v_{i,M}}$ should be small.
See the derivation of \eqref{eq:regression-target-hat} in Appendix~\ref{sec:deri2formulas}.

Based on the relationship in \eqref{eq:regression-target-hat}, we estimate the graph structure as in the previous section, where $\hat{\bm{B}}_1,\cdots,\hat{\bm{B}}_{p-1}$ are estimated using~\eqref{eq:loss-lasso} with $\bm{a^Y_{i,M}}$ and $\bm{a^{X_k}_{i,M}}$ replaced by $\bm{{\hat a}^Y_{i,M}}$ and $\bm{{\hat a}^{X_k}_{i,M}}$. The subsequently estimated neighborhood sets are given by~\eqref{eq:nbg-set-est}.

The most popular data-driven basis is the FPCA basis. Recall the linear Hilbert-Schmidt covariance operator $\mathscr{K}_j$ defined in~\eqref{eq:cov-operator}, which is the integral operator with the kernel being the covariance function of $g_j$, that is,
\begin{equation*}
K_{jj}(t,t^{\prime}) = \text{Cov}\left( g_j(t),g_j(t^{\prime}) \right).
\end{equation*}
Then there exist eigenpairs $\{\sigma_{jm},\phi_{jm}(\cdot)\}_{m\in\mathbb{N}}$ of $\mathscr{K}_j$ (Theorem 7.2.6 of~\citet{Hsing2015Theoretical}), where $\{\sigma_{jm}\}_{m\in\mathbb{N}}$ are the eigenvalues and $\{\phi_{jm}(\cdot)\}_{m\in\mathbb{N}}$ are orthonormal eigenfunctions. Since the covariance operator, $\mathscr{K}_j$, is symmetric and positive semidefinite, we assume that $\sigma_{j1}\geq \sigma_{j2}\geq \dots\geq 0$ without loss of generality. According to the Karhunen-Loève theorem, $g_{ij}$ can be represented as $g_{ij}(\cdot) = \sum_{m=1}^\infty a_{ijm}\phi_{jm}(\cdot)$, where $a_{ijm}=\int_\mathcal{T}g_{ij}(t)\phi_{ijm}(t)dt\sim N(0,\sigma_{jm})$ are the FPCA scores and $a_{ijm}$ is independent of $a_{ijm'}$ for $m\neq m'$ \citep[Theorem~1.5]{Bosq2000Linear}. We will refer to $\{\phi_{jm}(\cdot)\}_{m=1}^{\infty}$ as the \emph{FPCA basis}. Since the basis is orthonormal, the function $g_{ij}^M(\cdot) = \sum_{k=1}^M a_{ijk}\phi_{jk}(\cdot)$ is the $\mathcal{L}^2$-projection of $g_{ij}(\cdot)$ onto the basis spanned by the first $M$ FPCA functions. The main advantage of this basis is that it provides the best approximation in the $L_2$ sense when projecting a function onto a fixed number of basis functions.

Unfortunately, the FPCA basis is typically unknown, as $K_{jj}(t^{\prime},t)$ is unknown. Therefore, we first estimate the functional covariance with the empirical version:
\begin{equation}
\label{eq:emp-func-cov}
\hat{K}_{jj}(t,t^{\prime}) = \frac{1}{n}\sum_{i=1}^n g_{ij}(t)g_{ij}(t^{\prime}).
\end{equation}
An eigen-decomposition of $\hat{K}_{jj}(t, t^{\prime})$ produces the estimated eigen-pairs $\{\hat{\sigma}_{jm},\hat{\phi}_{jm}(\cdot)\}_{m=1}^M$, which in turn can be used to estimate FPCA scores $\hat{a}_{ijm} = \int_{\mathcal{T}}g_{ij}(t)\hat{\phi}_{jm}(t)dt$.

\citet{Qiao2015Functional} and \citet{solea2021nonparametric} also use projection scores from a dimension reduction procedure. However, there are several key differences between our approach and theirs. First, although it is the most commonly used basis, we do not restrict ourselves to the FPCA basis, and instead consider a generic basis. This provides additional flexibility and allows us to explore the effect of the chosen basis on empirical performance. See Section~\ref{sec:th-ch-fun-basis} for more details. Our methodology also differs in a second, more substantial way. Both \citet{Qiao2015Functional} and \citet{solea2021nonparametric} project each random function on its own FPCA basis and consider the resulting projection scores for all subsequent tasks. In contrast, when estimating the neighborhood of a specific node---rather than projecting each random function onto its own subspace---we project all random functions onto the same subspace. Concisely put, the subspace to estimate $\mathscr{N}_j$, $\bm{\phi_j}$, may differ from $\bm{\phi_k}$, the subspace used to estimate $\mathscr{N}_k$. However, when estimating $\mathscr{N}_j$ we use projection scores for all functions projected on a single basis $\bm{\phi_j}$. 

Intuitively, the advantage of this approach is that we can tailor the finite dimensional representation to maximize the information relevant to selecting the neighborhood of a specific node, $\mathscr{N}_j$. The FPCA basis for each random function maximizes the ``retained information'' for that random function. Although there may be significant features of $g_{ik}(\cdot)$ that are captured by its FPCA basis, these features may not be relevant to estimate the neighborhood of a specific node $\mathscr{N}_j$. Ultimately, we should care more about how $g_{ik}(\cdot)$ behaves in the subspace spanned by $g_{ij}$'s FPCA basis, which captures $g_{i,Y}$'s variability, rather than the subspace spanned by its own FPCA basis. We examine a theoretical justification in Section~\ref{sec:th-ch-fun-basis} and also illustrate the advantages in simulations.

More concretely, using a single basis for selecting $\mathscr{N}_j$ also avoids issues of colinearity that may arise artificially. For example, suppose that $g_{ik}(\cdot) = \sum_{m=1}^\infty a_{ikm} \phi_{km}(\cdot)$ and $g_{il}(\cdot) = \sum_{m=1}^\infty a_{ilm} \phi_{lm}(\cdot)$ have eigenfunctions $\{\phi_{km}(\cdot)\}_{m \geq 1}$ and $\{\phi_{lm}(\cdot)\}_{m \geq 1}$ that differ drastically, but $a_{ikm}$ and $a_{ilm}$ are highly correlated. When estimating $\mathscr{N}_j$ using the projection scores from the FPCA basis of $k$ and $l$, this would result in a poorly conditioned problem that may violate the irrepresentability condition (e.g., Assumption 4.8 in \citet{solea2021nonparametric} or Condition 5 in \citet{Qiao2015Functional}), despite the fact that the actual random functions $g_{ik}(\cdot)$ and $g_{il}(\cdot)$ are not difficult to distinguish. Projecting $g_{ik}(\cdot)$ and $g_{il}(\cdot)$ onto the same basis---$\bm{\phi_j}$---would avoid this concern, and the resulting projection scores would only be colinear if the actual random functions are similar and the problem is intrinsically hard.

While our methodology and theory allow for any orthonormal basis, we show both theoretically and in simulations that a well-chosen basis can improve performance. When choosing a basis, there are at least two objectives to consider. First, we want to minimize the covariance between $\bm{r_{iM}}$ and $\{\bm{a^{X_k}_{i,M}} \}_{k \in [p-1]}$ in \eqref{eq:regression-target}. Second, we want to maximize the signal strength $\min_{k \in \mathscr{N}_j} \Vert \bm{B^{\ast}_{k,M}} \Vert_{\text{F}}$. In general, simultaneously achieving these two objectives is practically infeasible. Thus, in practice, we focus on achieving at least one of the two. Achieving the first objective is generally infeasible without further restrictive assumptions (see Section~\ref{sec:th-ch-fun-basis}). Thus, in practice, we generally focus on the second objective, which will lead us to use the FPCA basis of $g^Y_i$, which we recommend as a default choice. 
Finally, we acknowledge that our study on the choice of function basis is far from complete. One should treat our guidance as a heuristic design, and we leave more thorough studies on this topic for further research.

\subsection{Selection of Tuning Parameters}
\label{CH_Param_Selection}

{\newText

There are three tuning parameters that need to be chosen to implement our algorithm: the number of basis functions used for dimension reduction, $M$; the thresholding parameter from~(13), $\epsilon_n$; and the group lasso penalty parameter $\lambda_n$ in~(12). We now discuss how to choose them in practice.

We first discuss how to choose the number of basis functions $M$. We follow the same cross-validation (CV) tuning strategy as in~\cite{Qiao2015Functional}. In practice, we have access to observations $\{(t_{ike},h_{ike})\}^{E_{ik}}_{e=1}$, $i \in [n]$ and $k \in [p]$, where $h_{ike}$ is a noisy observation of $g_{ik}(\cdot)$ at a time point $t_{ike} \in \mathcal{T}$. We then divide the time interval $\mathcal{T}$ into $J$ equal-size folds $\mathcal{J}_1,\ldots,\mathcal{J}_J$ with $\mathcal{T}=\cup^J_{l=1}\mathcal{J}_l$. For $a \in [J]$, we treat fold $\mathcal{J}_a$ as the validation set, and the remaining $J-1$ folds as the training set. For a chosen node $j \in [p]$, if $\bm{\phi}_j$ is known, we then fit each function $g_{ik}(\cdot)$ with an $M$-dimensional $\bm{\phi}_j$ basis $\{\phi_{j1}(\cdot),\ldots,\phi_{jM}(\cdot)\}$ via least-square on the observations $\{(t_{ike},h_{ike})\}$ where $t_{ike} \notin \mathcal{J}_a$ to get $\hat{g}_{ik}(\cdot)$; we then calculate the squared error between $h_{ike}$ and $\hat{g}_{ik}(t_{ike})$ on the validation set. We repeat this procedure for $a=1,\ldots,J$ to compute the CV error and choose $M$ that minimizes the CV error. In the case when $\bm{\phi}_j$ is unknown, we first fit $g_{ij}(\cdot)$ on observations $\{(t_{ije},h_{ije})\}$ where $t_{ije} \notin \mathcal{J}_a$ via a $L$-dimensional $B$-spline basis~\citep[Chpater 5]{Ramsay2005Functional} to get $\hat{g}_{ij}(\cdot)$, and subsequently use $\hat{g}_{ij}(\cdot)$ to get $\hat{\bm{\phi}}_j(\cdot)$. Next, we fit all functions $g_{ik}(\cdot)$ by $\{\hat{\phi}_{j1}(\cdot),\ldots,\hat{\phi}_{jM}(\cdot)\}$ via least-square on the observations $\{(t_{ike},h_{ike})\}$ where $t_{ike} \notin \mathcal{J}_a$ to get $\hat{g}_{ik}(\cdot)$. After following the same procedure to compute CV error, we then choose $(M,L)$ simultaneously over a grid of $M \leq L$ values and choose the pair with the lowest error.

Next, we describe the selection process for $\epsilon_n$ and $\lambda_n$. When $\lambda_n$ is large enough, all estimated coefficients $\hat{\bm{B}}_k$ will be set to zero. Specifically, by Proposition~1, there exists a threshold $\lambda_{\text{max},n}>0$ that can be calculated from the data, such that for any $\lambda_n>\lambda_{\max,n}$, the result $\hat{\bm{B}}_k=\mathbf{0}$ for all $k\in[p-1]$. Thus, we only need to consider $\lambda_n \in (0, \lambda_{\max,n}]$. We found empirically that traditional $K$-fold cross-validation performs poorly in our setting. Therefore, for each $j \in [p]$, we select $\lambda_n,\epsilon_n$ pair using selective cross-validation (SCV) \citep{she2012iterative}.

For each value of $\lambda_n$, we use the entire data set to estimate $ \bm{\hat{B}}_{\lambda_n} = (\bm{\hat{B}}_{\lambda_n,1},\dots,\bm{\hat{B}}_{\lambda_n,p-1})$ by solving~\eqref{eq:loss-lasso}. Given any threshold parameter $\epsilon_n$, we can obtain an index set $\check{\mathscr{N}}_j(\lambda_n, \epsilon_n) \subseteq [p-1]$ that indicates the blocks in $\bm{\hat{B}}_{\lambda_n}$ that are large enough in terms of Frobenius norm, that is,
\begin{equation}\label{eq:N.check}
    l\in \hat{\mathscr{N}}_j(\lambda_n, \epsilon_n) \quad \text{if and only if }\quad \|\bm{\hat{B}}_{\lambda_n,l}\|_\text{F} > \epsilon_n.
\end{equation}

\begin{algorithm}[t]
\caption{The Selective Cross-Validation (SCV) algorithm to choose $(\lambda_n, \epsilon_n)$.}
\label{Alg:scv.grid}
\begin{algorithmic}
\REQUIRE{$\bm{A^X}$, $\bm{A^Y},j\in [p]$}\;
\FOR{all $\lambda_n$}
\STATE{Run Group Lasso ADMM on $(\bm{A^X}, \bm{A^Y})$ and obtain $\bm{\hat{B}}_{\lambda_n}$}\;
    \FOR{all $\epsilon_n$}
    \STATE{Obtain $\hat{\mathscr{N}}_j(\lambda_n, \epsilon_n)$ by \eqref{eq:N.check}}\;
        \FOR{$k \in [K]$}
        \STATE{Re-estimate $\bm{\tilde{B}_l}$ for $l \in [p-1]$ by solving \eqref{eq:scv.unpen} with the $k$-th-fold training set}\;
        \STATE{Evaluate the estimate on the $k$-th-fold test data using the SCV-RSS criterion}\;
        \ENDFOR
    \STATE{Calculate the mean of the criterion across all $K$ folds}\;
    \ENDFOR
\ENDFOR
\STATE{Pick the $(\lambda_n,\epsilon_n)$ pair that minimizes the mean criterion}\;
\end{algorithmic}
\end{algorithm}

For $l\in \hat{\mathscr{N}}_j(\lambda_n, \epsilon_n)$, we then re-estimate $\bm{\tilde{B}_k}$ by minimizing the unpenalized least squares objective using the $k$-th-fold training set, which we denote as $I_k$, and we set $\bm{\tilde{B}_l}=0$ for all $l \notin \hat{\mathscr{N}}_j(\lambda_n, \epsilon_n)$. Specifically, we obtain $\bm{\tilde{B}_1},\dots,\bm{\tilde{B}_{p-1}}$ by solving the optimization problem below:
\begin{equation}
\label{eq:scv.unpen}
\begin{aligned}
\bm{\tilde{B}_1}, \dots, \bm{\tilde{B}_{p-1}} &\in \arg \min_{\bm{B_1},\dots,\bm{B_{p-1}}} \left\{  \sum_{i \in I_k}\left\Vert \bm{a^Y_{i,M}} - \sum^{p-1}_{l=1} \bm{B_{l}} \bm{a^{X_l}_{i,M}} \right\Vert^2_2 \right\}, \\
& \text{s.t. }  \bm{\tilde{B}_l}=0 \, \text{ for all } l \notin \hat{\mathscr{N}}_j(\lambda_n, \epsilon_n).
\end{aligned}
\end{equation}

We propose an error criterion named SCV-RSS, where RSS stands for the residual sum of squares. The criterion performs well in practice and adds the BIC penalty term to the squared norm of the empirical estimation error. Let
\begin{equation*}
\hat{\mathbf{re}}_i \coloneqq  \bm{a^Y_{i,M}} - \sum^{p-1}_{l=1} \tilde{\bm{B}}_{l} \bm{a^{X_l}_{i,M}},
\end{equation*}
and SCV-RSS on the test set $I_{\text{test}}$ is defined as
\begin{equation}
\label{EQ_scv-rss}
\text{SCV-RSS}(\lambda_n,\epsilon_n) \coloneqq \sum_{i \in I_{\text{test}}  }\left\Vert \hat{\mathbf{re}}_i \right\Vert^2_2 + \log ( \vert I_{\text{test}} \vert ) \cdot \vert \hat{\mathscr{N}}_j(\lambda_n, \epsilon_n) \vert.
\end{equation}
We then finally choose the $(\lambda_n,\epsilon_n)$ pair that minimizes the mean of SCV-RSS over all $K$ folds.
The pseudo-code of the procedure is given in Algorithm \ref{Alg:scv.grid}.

}

\section{Optimization Algorithm}
\label{sec:ADMM}

We propose an optimization method to solve~\eqref{eq:loss-lasso} using the alternating direction method of multipliers (ADMM) \citep{gabay1976dual,Boyd2011Distributed}. Note that \eqref{eq:loss-lasso} has a composite objective structure where the objective is composed of a convex smooth loss and a convex non-smooth regularization term. This composite objective is well studied in the convex optimization literature \citep[see, for example, Section 5.1][]{Bubeck2015Convex}.
In this section, we provide an easy-to-use practical solution. Commonly used alternative methods to solve such a composite objective include ISTA (Iterative Shrinkage-Thresholding Algorithm) and FISTA (Fast ISTA). See Section 5.1 in \citet{Bubeck2015Convex} for more details.
One advantage of ADMM is that it is easy to parallelize. 
Therefore, it is preferable when there are several machines available and the sample size or number of vertices is large~\citep[Chapter 8 and Chapter 10]{Boyd2011Distributed}.

The pseudo-code of our method is given in Algorithm~\ref{Alg:admm-nb} and we provide additional details below.
Let
\begin{equation*}
\bm{A^Y} = 
\begin{bmatrix}
(\bm{a^Y_{1,M}})^{\top} \\
(\bm{a^Y_{2,M}})^{\top}\\
\vdots \\
(\bm{a^Y_{n,M}})^{\top}
\end{bmatrix}
\in \mathbb{R}^{n \times M},
\qquad
\bm{A^{X_k}} 
=
\begin{bmatrix}
(\bm{a^{X_k}_{1, M}})^{\top} \\
(\bm{a^{X_k}_{2, M}})^{\top} \\
\vdots \\
(\bm{a^{X_k}_{n, M}})^{\top}
\end{bmatrix}
\in \mathbb{R}^{n \times M}.
\end{equation*}
Consider the concatenated matrices 
$\bm{A^X}
=
\begin{bmatrix}
\bm{A^{X_1}} & \bm{A^{X_2}} & \dots & \bm{A^{X_{p-1}}} 
\end{bmatrix} \in \mathbb{R}^{n \times (p-1)M}$, and
\begin{equation*}
\bm{P} = 
\begin{bmatrix}
\bm{P_1} \\
\bm{P_2}\\
\vdots \\
\bm{P_{p-1}}
\end{bmatrix}
\in \mathbb{R}^{(p-1)M \times M},
\qquad
\bm{Q} 
=
\begin{bmatrix}
\bm{Q_1} \\
\bm{Q_2} \\
\vdots \\
\bm{Q_{p-1}}
\end{bmatrix}
\in \mathbb{R}^{(p-1)M \times M}.
\end{equation*}
Then~\eqref{eq:loss-lasso} can be reformulated as:
\begin{equation*}
    \min_{\bm{P},\bm{Q}}\ \frac{1}{2n} \left\Vert\bm{A^Y}-\bm{A^{X}}\bm{Q}\right\Vert_{\text{F}}^2 + \lambda_n\sum^{p-1}_{k=1}\|\bm{P_k}\|_{\text{F}}
    \quad \text{subject to } \bm{P} - \bm{Q} = 0,
\end{equation*}
which can be minimized by solving a series of optimization problems. At the $h$'th iteration, for all $k \in [p-1]$:
\begin{align}
    \bm{P_k}^{h+1} &= \arg\min_{\bm{P_k}} \left( \lambda_n \|\bm{P_k}\|_{\text{F}} + \frac{\rho}{2}\|\bm{P_k}-\bm{Q_k}^h + \bm{U_k}^h\|_{\text{F}}^2 \right), \quad k\in[p-1],
    \label{EQ_POpt}\\
    \bm{Q}^{h+1} &= \arg\min_{\bm{Q}}\left(\frac{1}{2n}\|\bm{A^Y}-\bm{A^X}\bm{Q}\|_{\text{F}}^2 + \frac{\rho}{2}\|\bm{Q} - \bm{P}^{h+1} - \bm{U}^h\|_{\text{F}}^2 \right),\label{EQ_QOpt}\\
    \bm{U}^{h+1} &= \bm{U}^h + \bm{P}^{h+1} - \bm{Q}^{h+1}.\label{EQ_U_Update}
\end{align}
Here, $\rho$ is the penalty parameter for the augmented Lagrangian.
The solution to \eqref{EQ_POpt} is a group soft-thresholding update of $\bm{P}$. For each $k\in[p-1]$,
\begin{equation}\label{EQ_P_Final}
    \bm{P_k}^{h+1} = \left[1-\frac{\lambda}{\rho\|\bm{Q_k}^h - \bm{U_k}^h\|_{\text{F}}}\right]_+ (\bm{Q_k}^h - \bm{U_k}^h).
\end{equation}
The solution to \eqref{EQ_QOpt}, i.e. the update of $\bm{Q}$, is
\begin{equation}\label{EQ_Q_Update}
    \bm{Q}^{h+1} = \left(\frac{1}{n}(\bm{A^X})^\top\bm{A^X} + \rho \bm{I}_M\right)^{-1}\left(\frac{1}{n}(\bm{A^X})^\top\bm{A^X} + \rho \bm{P}^{h+1} + \rho \bm{U}^h\right).
\end{equation}
Iteratively using updates \eqref{EQ_P_Final}, \eqref{EQ_Q_Update}, and \eqref{EQ_U_Update}, the matrix $\bm{P_k}^{h}$ will eventually converge to $\bm{P_k}^*$, $k\in[p-1]$, as $h\rightarrow\infty$ \citep{Boyd2011Distributed}. The solution of~\eqref{eq:loss-lasso} is given by $\bm{\hat{B}_k}= \bm{P_{k}}^*$, $k\in[p-1]$. The stopping criterion for the iteration process depends on the primal residual, which indicates how well the constraints are satisfied, and the dual residual, which indicates stability of updates between two consecutive iterations \citep{Boyd2011Distributed}. In our settings, $\bm{s_1}^h = \bm{P}^h - \bm{Q}^h$, and $\bm{s_2}^h = \bm{Q}^{h} - \bm{Q}^{h-1}$ are the primal and dual residuals respectively.
The algorithm terminates when both residuals are below their respective tolerances:
$$
\|\bm{s_1}^h\|_{\text{F}}\leq \bm{\epsilon}^{\text{pri},h}\text{ and }\|\bm{s_2}^h\|_{\text{F}}\leq \bm{\epsilon}^{\text{dual},h},
$$
where
\begin{align*}
    \bm{\epsilon}^{\text{pri},h} &= \sqrt{(p-1)M^2}\bm{\epsilon}^{\text{abs}} + \bm{\epsilon}^{\text{rel}}\left(\|\bm{P}^h\|_{\text{F}} \vee \|\bm{Q}^{h}\|_{\text{F}}\right), \\
    \bm{\epsilon}^{\text{dual},h} &= \sqrt{(p-1)M^2}\bm{\epsilon}^{\text{abs}} + \bm{\epsilon}^{\text{rel}}\|\bm{U}^{h}\|_{\text{F}}.
\end{align*}
The factor $\sqrt{(p-1)M^2}$ is because the Frobenius norms are computed on $\mathbb{R}^{(p-1)M^2}$ matrices. In the following experiments, we use $\bm{\epsilon}^{\text{abs}}=10^{-4}$ and $\bm{\epsilon}^{\text{rel}}=10^{-4}$ by default.

The penalty parameter $\rho$ of the augmented Lagrangian can be adjusted adaptively. We use Strategy S3 in Table 1 of \citet{he2000alternating} with $\varphi=10$, $\tau^{\text{incr}} = \tau^{\text{decr}} = 2$:
\[
\rho^{h+1} =
\begin{cases}
\tau^{\text{incr}} \rho^h, & \|\bm{s_1}^h \|_2 > \varphi\|\bm{s_2}^h\|_2\\
\rho^h / \tau^{\text{decr}}, & \|\bm{s_2}^h \|_2 > \varphi \|\bm{s_1}^h\|_2\\
\rho^h, & \text{otherwise},
\end{cases}
\]
with $\rho^0 = 1$. This guarantees that the primal and dual residuals do not vary significantly across iterations and ensures stability regardless of the initial $\bm{P}^0$ and $\bm{U}^0$.

\begin{algorithm}[t]
\caption{ADMM for functional neighborhood selection}
\label{Alg:admm-nb}
\begin{algorithmic}
\REQUIRE{$\bm{A^X}$, $\bm{A^Y}$, and $\lambda_n$}\;
\STATE{Set initial values of $\rho^0$, $\bm{P}^0$, $\bm{Q}^0$, and $\bm{U}^0$}\;
\FOR{$h=0,1,2,\dots$}
    \STATE{Update $\bm{P}^{h+1}$ by \eqref{EQ_P_Final}}\;
    \STATE{Update $\bm{Q}^{h+1}$ by \eqref{EQ_Q_Update}}\;
    \STATE{Update $\bm{U}^{h+1}$ by \eqref{EQ_U_Update}}\;
    \STATE{Break if primal and dual residuals meet stopping criteria}\;
    \STATE{Update $\rho^{h+1}$ for next round}\;
\ENDFOR
\ENSURE{$\bm{\hat{B}_k}$ for $k \in [p-1]$.}
\end{algorithmic}
\end{algorithm}

\section{Theoretical Properties}
\label{sec:thm-prop}

We now discuss the statistical properties of the estimator proposed in Section~\ref{sec:methodology}. In particular, we give conditions under which the neighborhood of a single variable can be consistently recovered. Using a union bound extends the guarantees to recovery of the entire graph. First, we discuss a procedure that uses a fixed function basis, and, subsequently, we discuss a procedure that uses an estimated function basis.

Since we first consider a single node $j$, we assume without loss of generality that $j = p$. To simplify the notation, we also drop the subscript $j$ from $\beta_{jk}(t^{\prime},t)$, $\bm{\phi_j}$, $\phi_{j,m}$, $\ldots$, in this section. 
By~\eqref{eq:beta-l-expansion}, we have $\Vert \beta_k (t^{\prime},t) \Vert_{\text{HS}} =\sqrt{ \sum^{\infty}_{m,m^{\prime}=1} (b^*_{k,m m^{\prime}})^2 }$ and $b^*_{k,m m^{\prime}}=0$ for all $m,m^{\prime}\geq 1$ when $\Vert \beta_k(t, t^\prime) \Vert_{\text{HS}}=0$.
Let $\bm{B^{\ast}_{k,M}}$ be a $M \times M$ matrix whose $m$-th row is $(b^{\ast}_{k,m1},b^{\ast}_{k,m2},\dots,b^{\ast}_{k,mM})$. The scores of the ``error'' projected onto the function basis are denoted as $w_{im}=\int_{\mathcal{T}} e_i(t) \phi_m (t) dt$, $m \geq 1$, and  $\bm{w_{i,M}}=(w_{i1},\cdots,w_{iM})^{\top}$. Let $\bm{r_{i,M}}=(r_{i1},\cdots,r_{iM})^{\top} \in \mathbb{R}^{M}$ denote the ``bias'' arising from using the first $M$ basis elements to represent $\beta_k(t^{\prime},t)$ where
$r_{im} = \sum^{p-1}_{k=1} \sum^{\infty}_{m^{\prime}=M+1} b^{\ast}_{k,m m^{\prime}} a^{X_k}_{im^{\prime}}$.
Let
\begin{equation}\label{eq:trun-beta}
\begin{aligned}
\beta_{k,M} (t^{\prime},t) &=  \sum^{M}_{m,m^{\prime}=1} b^{\ast}_{k,m m^{\prime}} \phi_m (t) \phi_{m^{\prime}} (t^{\prime}),\\
\beta_{k, >M} (t^{\prime},t) &=  \sum^{\infty}_{m>M \,\text{or}\, m^{\prime}>M} b^{\ast}_{k,m m^{\prime}} \phi_m (t) \phi_{m^{\prime}} (t^{\prime}).
\end{aligned}
\end{equation}
Then $\Vert \beta_{k, M} (t^{\prime},t) \Vert_{\text{HS}}=\sqrt{ \sum^{M}_{m,m^{\prime}=1} (b^*_{k,m m^{\prime}})^2 }$, $\Vert \beta_{k, >M} (t^{\prime},t) \Vert_{\text{HS}}=\sqrt{ \sum^{\infty}_{m>M \, \text{or} \, m^{\prime}>M} (b^*_{k,m m^{\prime}})^2 }$, and
\begin{align*}
\Vert \beta_k(t^{\prime},t) \Vert_{\text{HS}} -  \Vert \bm{B^{\ast}_{k, M}} \Vert_{\text{F}} & =  \Vert \beta_{k, M}(t^{\prime},t) + \beta_{k, >M}(t^{\prime},t) \Vert_{\text{HS}} -  \Vert \bm{B^{\ast}_{k, M}} \Vert_{\text{F}} \\
& \leq  \Vert \beta_{k, M}(t^{\prime},t) \Vert_{\text{HS}} + \Vert \beta_{k, >M}(s,t) \Vert_{\text{HS}} -  \Vert \bm{B^{\ast}_{k, M}} \Vert_{\text{F}} \\
& =  \Vert \beta_{k, >M}(t^{\prime},t) \Vert_{\text{HS}}.
\end{align*}
When $M$ is large enough, then the term $\Vert \beta_{k, >M}(t^{\prime},t) \Vert_{\text{HS}}$ is small; when $n$ is also large enough, $\bm{\hat{B}^M_k}$ is close to $\bm{B^{\ast}_{k, M}}$, and $\mathscr{\hat N}_j$ will be a good estimator of $\mathscr{N}_j$.

Both $\bm{w_{i,M}}$ and $\bm{r_{i,M}}$ are Gaussian vectors with mean zero, and we denote their covariance matrices as $\bm{\Sigma^{w}}$ and $\bm{\Sigma^{r}}$ respectively; in addition, we define $\bm{\Sigma^{r,w}}=\text{Cov}(\bm{r_{i,M}},\bm{w_{i,M}})$ and $\bm{\Sigma^{w,r}}=(\bm{\Sigma^{r,w}})^{\top}$.
To simplify the notation, we drop the explicit dependence on $M$. Let $\bm{\Sigma^{X_k,r}} = \Cov (\bm{a^{X_k}_{i,M}}, \bm{r_{i,M}}) \in \mathbb{R}^{M \times M}$, $\bm{\Sigma^{r,X_k}}=(\bm{\Sigma^{X_k,r}})^{\top}$, $\bm{\Sigma^{X_k, X_l}}=\Cov (\bm{a^{X_k}_{i,M}}, \bm{a^{X_l}_{i,M}}) \in \mathbb{R}^{M \times M}$, and $\bm{\Sigma^X}=(\bm{\Sigma^{X_k, X_l}})_{1\leq k,l \leq p-1}$ is a matrix composed of $M\times M$-blocks $\bm{\Sigma^{X_k, X_l}}$, $k,l \in [p-1]$. The following quantities will be used to state the results:
\begin{equation}
\begin{gathered}\label{eq:add-nota}
\Xi_1(M)  =  \max_{k \in [p-1]} \left\{ \rho_{\max}(\bm{\Sigma^w}+\bm{\Sigma^r}-\bm{\Sigma^{r,X_k}} \left( \bm{\Sigma^{X_k, X_k}} \right)^{-1} \bm{\Sigma^{{X_k},r}}) \right\}, \\
\Xi_2(M)  =  \max_{k \in[p-1]}\rho_{\max}(\bm{\Sigma^{X_k, X_k}}), \qquad 
\Xi_3(M) =  \max_{k \in[p-1]}\tr(\bm{\Sigma^{X_k, X_k}}), \\
\Xi_4(M)  =  \tr \left\{ \bm{\Sigma^r} + \bm{\Sigma^w} + \bm{\Sigma^{r,w}} + \bm{\Sigma^{w,r}} \right\}, \qquad 
\omega(M) =  \max_{k \in [p-1]} \left\Vert \bm{\Sigma^{r,X_k}}
\right\Vert_{\text{F}}.
\end{gathered}
\end{equation}
Note that $\bm{\Sigma^r}-\bm{\Sigma^{r,X_k}} \left( \bm{\Sigma^{X_k, X_k}} \right)^{-1} \bm{\Sigma^{{X_k},r}}$ is a conditional variance of $\bm{r_i}$ given $\bm{a^{X_k}_{i,M}}$, so the arguments in $\Xi_1(M)$ are always positive semidefinite. The functions $\Xi_1(M)$ -- $\Xi_4(M)$ are used to express an upper bound on the covariance between the projection scores $\bm{a^{X_k}_{i,M}}$ and the error terms $(\bm{r_{i,M}}+\bm{w_{i,M}})$. This upper bound then provides a lower bound for the regularization parameter $\lambda_n$. The function $\omega(M)$ measures the correlation of residuals $\bm{r_{i,M}}$ with $\bm{a^{X_k}_{i,M}}$. A large correlation implies that the problem is more difficult to solve. Finally, let 
\begin{equation}\label{eq:def-K0}
K_0 = \max_{k \in [p-1], m \in M} \mathbb{E} [(a^{X_k}_{i m})^2] = \max_{k \in [p-1], m \in M} \left( \bm{\Sigma^{X_k, X_k}}\right)_{m,m} <\infty.
\end{equation}
The quantity $K_0$ is used to provide an upper bound on the estimation error for the covariance matrix of $\bm{a^{X}_{i,M}}$. Subsequently, this is used to prove a lower bound on restricted eigenvalues, which is a crucial step in proving Theorem~\ref{thm:nb-edge-recov-fix} and Theorem~\ref{thm:nb-edge-recov-data-driven}.

\subsection{Prior Fixed Function Basis}
\label{sec:fix-fun-basis}

Let $\sigma_{j0} = \mathbb{E} [ \Vert g_{ij} \Vert^2]$, $\sigma_{\text{max},0}=\max_{j\in[p]}\sigma_{j0}$, and $\sigma_{jr}=\mathbb{E} [ \Vert e_{ij} \Vert^2]$, where $e_{ij}$ is defined in \eqref{eq:function-linear}. Note that $\sigma_{jr} \leq \sigma_{j0}$. We introduce several assumptions before stating the main results.

\begin{assump}\label{assump:uni-bd-mom2}
There exists a constant $C>0$ that does not depend on $p$ such that $\sigma_{\text{max},0} \leq C$.
\end{assump}
Assumption~\ref{assump:uni-bd-mom2} requires that the norm of the random functions have a finite second moment that does not grow with $p$ and is a basic requirement for functional graphical models to be well defined. Note that $\Xi_k(M) \leq \max_{j \in [p]} \mathbb{E} \Vert g_{ij} \Vert^2 $ for all $k=1,2,3,4$ and any $M$. Thus, $\Xi_k (M) \leq C$ for all $M \geq 1$ and $p \geq 1$.

\begin{assump}\label{assump:LowBoundKappa}
Let $\bm{\Sigma^X}_{\mathscr{N}_j, \mathscr{N}_j}=(\bm{\Sigma^{X_k, X_{k'}}})_{k,k' \in \mathscr{N}_j} \in \mathbb{R}^{ \vert \mathscr{N}_j \vert M \times \vert \mathscr{N}_j \vert M }$ be the submatrix with blocks indexed by the elements of the neighborhood set $\mathscr{N}_j$ and define
\begin{equation}\label{eq:kappa-def}
\kappa = \kappa(M) = \rho_{\min}\left( \bm{\Sigma^X}_{\mathscr{N}_j, \mathscr{N}_j}\right).
\end{equation}
For any $M$, we assume that $\kappa > 0$. When $\mathscr{N}_j$ is empty, we let $\kappa=\infty$ for all $M$.
\end{assump}
Assumption~\ref{assump:LowBoundKappa} requires that the projection scores of all functions in the neighborhood of node $j$ are linearly independent. As discussed in Section~\ref{sec:choice-fun-basis}, because we project all functions onto the same basis, projection scores would only be colinear if the functions are truly difficult to distinguish. 

Let $\tau(M)$ be the relevant signal strength:
\begin{equation}
\tau (M) = \min_{k \in \mathscr{N}_j} \left\Vert \bm{B^{\ast}_k} \right\Vert_{\text{F}} = \min_{k \in \mathscr{N}_j} \left\Vert \beta_{k, M} \right\Vert_{\text{HS}},
\end{equation}
where $\beta_{k, M} (t^{\prime},t)$ is defined in \eqref{eq:trun-beta}. For any orthonormal basis, the signal strength $\tau (M)$ is an increasing function of $M$. When $\mathscr{N}_j$ is empty, we define $\tau (M)=\infty$ for all $M$. Recall that we use $s = s_j$ to denote the size of the neighborhood, $s = \vert \mathscr{N}_j \vert$. As discussed in Section~\ref{sec:f-on-f2v-on-v}, when $M$ is large enough such that $\tau (M)>0$, we have $\mathscr{N}^M_j=\mathscr{N}_j$, where $\mathscr{N}^M_j$ is defined in~\eqref{eq:trunc-neighbor}.

\begin{assump}[Signal Strength]
\label{assump:sig-strength}
We assume that $$\frac{\omega(M)}{\sqrt{\kappa(M)}\tau (M)}$$ is a non-increasing function of $M$, and
\begin{equation*}
\lim_{M \rightarrow \infty} 24 \sqrt{s} \frac{\omega(M)}{\sqrt{\kappa(M)} \tau (M)} < 1.
\end{equation*}
\end{assump}
Assumption~\ref{assump:sig-strength} requires that the function $\omega(M)$---a measure of bias due to truncation---must decay quickly compared to $\sqrt{\kappa(M)}$, which roughly measures the conditioning of the design matrix, after dividing by $\tau(M)$, which measures the signal strength. 

{\newText
We also compare our Assumption~\ref{assump:sig-strength} with Assumption 4.6 of~\citet{solea2021nonparametric}. 
Assumption 4.6 of~\citet{solea2021nonparametric} assumes that there exists a universal constant $C_{\min}>0$ such that $\kappa(M) \geq C_{\min}$ for all $M \geq 0$, where $\kappa(M)$---defined in our Assumption~3---is the minimum eigenvalue of of the covariance matrix of the function scores when using an $M$-dimensional basis. However, by~\citet[Theorem 7.2.5]{Hsing2015Theoretical}, the covariance operator is a compact operator; furthermore, by~\citet[Theorem 4.2.3]{Hsing2015Theoretical}, we must have $\kappa(M) \to 0$ as $M \to \infty$ unless $\mathbb{H}$ has finite-rank. Thus, Assumption 4.6 of~\citet{solea2021nonparametric} is equivalent to assuming that the random functions lie in a finite-dimensional space.
In contrast, in our Assumption~4, instead of assuming that $\kappa(M)$ is uniformly bounded away from $0$, we study the interplay between $\kappa(M),\omega(M)$ and $\tau(M)$. 
When $\kappa(M)$ is bounded from $0$, our Assumption~4 holds; however, our assumptions also allow $\kappa(M) \to 0$ as $M \to \infty$, as long as it does not decrease too quickly when compared with $s\, \omega^2(M) / \tau^2(M)$. Thus, our theory can deal with infinite-dimensional random functions.
In this sense, our Assumption~4 should be considered strictly weaker than Assumption 4.6 of~\citet{solea2021nonparametric}.
}

Let
\begin{equation}
\label{eq:sig-strength}
\nu (M) = \tau (M) - 24 \sqrt{s} \frac{\omega(M)}{\sqrt{\kappa(M)}}.
\end{equation}
Under Assumption~\ref{assump:sig-strength}, $\lim_{M \rightarrow \infty} \nu (M) > 0$. We denote $M^*$ as the smallest integer such that $\nu(M')>0$ for all $M' \geq M^*$, that is,
\begin{equation}
\label{eq:Mstar}
\begin{aligned}
M^* 
& = \min \left\{ M \, :\, \nu(M') > 0 \text{ for all } M' \geq M \right\} \\
& = \min \left\{ M \, :\, 24 \sqrt{s} \frac{\omega(M')}{\sqrt{\kappa(M')} \tau (M')} < 1 \text{ for all } M' \geq M \right\}.
\end{aligned}
\end{equation}
Let
\begin{equation}\label{eq:chi-fun}
\chi (n,p,M,\delta) = \frac{6\sqrt{s}}{\sqrt{\kappa(M)}} \tilde{\lambda} (n, p, M, \delta),
\end{equation}
where
\begin{multline}\label{eq:lambda_tilde}
\tilde{\lambda} (n, p, M, \delta)  =  \tilde{O} \left( \mathcal{C}_{n,\delta} \left\{\frac{M  \sqrt{\Xi_1 (M)}}{\sqrt{n}}+ \sqrt{\Xi_1 (M)} \sqrt{ \frac{ \log (p/\delta) }{n} } \right\} \right. \\
+  \left. \omega (M) \left\{ 
\frac{M \sqrt{ \log (pM^2/\delta) }}{\sqrt{n}} +  \frac{  M  \log (pM^2/\delta)  }{n} \right\} \right),
\end{multline}
and
\begin{equation}\label{eq:Cndelta}
\mathcal{C}_{n,\delta} = \sqrt{\Xi_2 (M)}\cdot\left(1 + \sqrt{\frac{2}{n}\log\left(\frac{(p-1)}{\delta} \right)} \right)+\sqrt{\frac{\Xi_3 (M)}{n}}.
\end{equation}
The exact form of $\tilde{\lambda} (n, p, M, \delta)$ can be found in \eqref{eq:lambda_tilde_exact} in appendix. The function $\tilde{\lambda} (n, p, M, \delta)$ provides a theoretical guidance on how to select the regularization parameter $\lambda_n$ and the function $\chi (n,p,M,\delta)$ is used to provide theoretical guidance on the choice of thresholding parameter $\epsilon_n$. 

We are now ready to state our main result on the consistency of the neighborhood selection procedure.

{\newText
\begin{theorem}[Neighborhood Recovery with Prior Fixed Basis]
\label{thm:nb-edge-recov-fix}
Suppose that Assumptions~\ref{assump:Hilbert-schmidt-assump}-\ref{assump:sig-strength} hold. Furthermore, suppose $M \geq M^{\ast}$, $\lambda_n= \tilde{\lambda}(n,p,M,\delta)$, and $\epsilon_n = \chi (n,p,M,\delta) + 12 \sqrt{s / \kappa(M)} \omega (M)$. Fix $\delta \in (0, 1]$. If the sample size $n$ satisfies 
\begin{equation*}
n \geq \tilde{O} \left( \max \left\{ \frac{M^4 s^2 \log (p^2 M^2/\delta)}{\kappa^2(M)}, \frac{s \cdot \max \left\{ M^2, \log (p/\delta), M^2 \omega^2 (M) \log (M^2 p /\delta) \right\} }{\kappa (M) \nu^2 (M)} \right\} \right),
\end{equation*}
then with probability at least $1-\delta$, we have
\begin{equation*}
\sqrt{\sum^{p-1}_{k=1} \left\Vert \bm{\hat{B}_k} - \bm{B^*_k} \right\Vert^2_{\text{F}}} \leq \chi (n,p,M,\delta) + 12 \sqrt{s} \frac{\omega (M)}{ \sqrt{\kappa(M)}},
\end{equation*}
so that $\mathscr{\hat{N}}_j = \mathscr{N}_j$.
\end{theorem}
\begin{proof}
See Appendix~\ref{sec:proof-thm1}. 
\end{proof}
}

Note that the quantities $\kappa(M)$, $\omega(M)$, $\nu(M)$, $s$, and $M$ all implicitly depend on $j$. One key difference between Theorem~\ref{thm:nb-edge-recov-fix} and a typical group lasso result
\citep[Chapter 9.6 of][]{Wainwright2019High}
is that the error term is correlated with the covariates---recall that the projection scores $\bm{a^X_{i,M}}$ are correlated with the residual $\bm{r_{i,M}}$ due to the finite-dimensional approximation. The effect of the correlation is captured by the function $\omega(M)$. When there is no truncation bias---i.e., the random functions are finite-dimensional and $M$ is large enough---then $\omega(M)$ will be zero.

Using a union bound, the following corollary directly follows from Theorem~\ref{thm:nb-edge-recov-fix} and provides guarantees for recovery of the entire graph.
\begin{corollary}[Graph Recovery with Prior Fixed Basis]
\label{corollary:graph-recov-fix-basis}
Suppose the conditions of Theorem~\ref{thm:nb-edge-recov-fix} hold for all nodes $j \in [p]$. 
We use $\kappa_j(M_j)$, $\omega_j(M_j)$, $\nu_j(M_j)$, $s_j$, and $M_j$ to take the place of $\kappa(M)$, $\omega(M)$, $\nu(M)$, $s$, and $M$ in Theorem~\ref{thm:nb-edge-recov-fix} to show their dependency on $j$ explicitly.
Let $\hat{E}$ be the estimated edge set obtained by applying either the AND or OR rule to the estimated neighborhood of each node. If the sample size $n$ satisfies 
\begin{equation*}
\begin{aligned}
n & \geq \max_{j \in [p]} \tilde{O} \left( \max \left\{ \frac{M_j^4 s_j^2 \log (p^3 M_j^2/\delta)}{\kappa_j^2(M_j)}, \right.\right. \\
& \qquad \qquad \qquad \qquad \left.\left. \frac{s_j \cdot \max \left\{ M_j^2, \log (p^2/\delta), M_j^2 \omega_j^2 (M_j) \log (M_j^2 p^2 /\delta) \right\} }{\kappa_j (M_j) \nu_j^2 (M_j)} \right\} \right),
\end{aligned}
\end{equation*}
then $\mathbb{P} \left\{ \hat{E} = E \right\} \geq 1 - \delta$.
\end{corollary}

Before moving to the next section, we compare our theorems with some existing literature.

Compared with~\citet{Qiao2015Functional}, we do not assume that the functional data are finite dimensional. Instead, we study the truly infinite dimensional functional data and discuss the trade-off between bias, signal strength, and conditioning of the design matrix explicitly. When Condition 2 of~\citet{Qiao2015Functional} (which is required for the correct graph recovery therein) holds, that is, when $g_{ij}$ is $M(n)$-dimensional for all $j \in [p]$ and some positive integer $M(n)$, then $\omega(M(n))=0$ and $\nu(M(n))>0$. However, when Assumption~\ref{assump:sig-strength} holds, we do not necessarily need $g_{ij}$ to be finite dimensional. Thus, Condition 2 of~\citet{Qiao2015Functional} is strictly stronger than Assumption~\ref{assump:sig-strength}.

{\newText
We also compare our results to those in~\citet{solea2021nonparametric}.
In addition to the difference between our Assumption~\ref{assump:sig-strength} and Assumption 4.6 of~~\citet{solea2021nonparametric} as discussed previously, our theoretical analysis offers an explicit characterization of a pivotal threshold in $M$. Specifically, we necessitate $M$ to exceed $M^{\star}$---as defined in~\eqref{eq:Mstar}---which is contingent solely on the characteristics of the functional data, rather than $n$. This allows our theoretical analysis to account for finite $M$. In contrast, \citet{solea2021nonparametric} necessitate a sieve-type estimator where $M$ scales with $n$ (Assumption 4.1 (ii)). We argue that our finding is more intuitive because increasing $M$ seeks to decrease the approximation error. If $M$ is not large enough, the approximation error remains too large to ensure consistent graph recovery, irrespective of how large $n$ may be. Therefore, accurate graph recovery becomes infeasible. Conversely, once $M$ is large enough to render the approximation error small, consistent graph recovery is achievable irrespective of the specific $M$ chosen, as long as $n$ is sufficiently large. In this regard, $M^{\star}$ encapsulates this threshold, a concept that is absent in~\citet{solea2021nonparametric}.
Furthermore, when we prescribe $M$ to scale with $n$ and treat $\kappa(M)$ as a constant, our result still delivers a superior rate. To highlight this, we initially presume all other parameters remain constant and only contemplate how the sample size $n$ relates to $M$. In Theorem~\ref{thm:nb-edge-recov-fix}, once $M \geq M^{\star}$, the dominant term becomes the first one. When $\kappa(M)$ is constant, we have $n=\Omega(M^4)$. By contrast, Assumption 4.1 (ii) of \citet{solea2021nonparametric} stipulates that $n=\Omega(M^{2+3\beta})$. Noting that Assumption 4.1 (i) of \citet{solea2021nonparametric} demands $\beta > 1$, it requires $n=\Omega(M^{2+3\beta})=\Omega(M^5)$, rendering it inferior to our rate.
}


{\newText
Finally, to obtain consistency of neighborhood recovery and graph recovery, we take a thresholding idea by introducing a tuning threshold $\epsilon_n$ in~\eqref{eq:nbg-set-est}, while both \citet{Qiao2015Functional} and \citet{solea2021nonparametric} rely on the irrepresentability condition~\citep{zhao06model}. Although our approach requires an additional tuning parameter, the irrepresentability condition is known to be a strong assumption. Both ideas are widely used in the literature.
The hard thresholding after initial group LASSO estimation has been broadly applied in high-dimensional linear regression~\citep{meinshausen08lasso} and graphical modeling~\citep{Cai2011Constrained}.
The theoretically appropriate choice of $\epsilon_n$ depends on problem parameters, typically unknown in practice. For this reason, the hard thresholding step is primarily employed for theoretical, rather than practical, purposes. 
Despite recognizing the gap between practice and theory, it's crucial to note that bridging this gap is a non-trivial task and remains a long-standing challenge in high-dimensional statistics.
In the simulations of Section~\ref{sec:experimemts}, we set $\epsilon_n=0$.
In Appendix~\ref{sec:effec-epsilon-n}, we empirically demonstrate how a non-zero $\epsilon_n$ impacts practical performance. The result shows that the benefit of a nonzero $\epsilon_n$ is not substantial, which justifies our choice in simulations.
Another way to choose $\epsilon_n$ is by cross-validation (CV) as we described in Section~\ref{CH_Param_Selection}. See Section~\ref{sec:perf-cv} for the empirical results of CV.

}

\subsection{Data-Dependent Function Basis}
We now consider the setting where the function basis used for dimension reduction is not known in advance, and instead the basis used is an estimate of some population basis. We will assume we have access to estimates satisfying the following property.
\begin{assump}
\label{assump:basis-est-accu}
There exist constants $c_1,c_2>0$ such that for all $0<\delta \leq 1$, we have
\begin{equation}\label{eq:basis-est-accu}
\mathbb{P} \left\{ \Vert \hat{\phi}_{jm} - \phi_{jm} \Vert \leq d_{jm} \sqrt{\frac{ (1/c_1) \log \left( c_2 / \delta \right)}{n}} \, \text{for all } m \geq 1 \right\} \geq 1-\delta,
\end{equation}
where $\{d_{jm}\}$, ${j \in [p], m \geq 1}$, are constants that depend on $j,m$ and satisfy $d^2_{j0} = \sum^{\infty}_{m=1} d^2_{jm} < \infty$, $j \in [p]$.
\end{assump}

Assumption~\ref{assump:basis-est-accu} holds when $\phi_{jm}(t)$'s are the
FPCA eigenfunctions of $K_{jj}(t^{\prime},t)$ and $\hat \phi_{jm}(t)$'s are the estimated FPCA eigenfunctions of $\hat{K}_{jj}(t^{\prime},t)$---see Lemma 6 and Lemma 8 in the Supplementary Material of \citet{Qiao2015Functional}. When Assumption~\ref{assump:basis-est-accu} holds, let $d_{j,\text{max}}=\max_{m \geq 1} d_{jm}$, and
\begin{equation}\label{eq:add-nota-2}
d_{js} (M) =  \sqrt{ \sum^M_{m=1} d^2_{jm} } \quad \text{for all } M \geq 1,
\end{equation}
so that $d_{js} (M) \leq d_{j0}$ for all $M \geq 1$. In addition, let
\begin{equation}\label{eq:add-nota-3}
\Phi (M) = \sqrt{ \sum^{p-1}_{k=1} \sum^M_{m=1} \sum^{\infty}_{m^{\prime}=M+1} \left( b^{\ast}_{k,m m^{\prime}} \right)^2 }
\end{equation}
and
\begin{equation}\label{eq:chi-fun-check}
\check \chi (n,p,M,\delta) = \frac{6\sqrt{s}}{\sqrt{\kappa(M)}} \check{\lambda} (n, p, M, \delta),
\end{equation}
where $\check{\lambda} (n,p,M,\delta)$ is given in~\eqref{eq:lambda_check} in appendix.

{\newText
\begin{theorem}[Neighborhood Recovery with Data-Dependent Basis]
\label{thm:nb-edge-recov-data-driven}
Suppose Assumptions~\ref{assump:Hilbert-schmidt-assump}-\ref{assump:basis-est-accu}, hold. Furthermore, suppose $M \geq M^{\ast}$, $ \lambda_n= \check{\lambda}(n,p,M,\delta)$, and
$\epsilon_n = \check{\chi} (n,p,M,\delta) + 12 \sqrt{s/\kappa(M)} \omega (M)$. Fix $\delta \in (0, 1]$. If the sample size $n$ satisfies
\begin{multline*}
n \geq \tilde{O} \left( \max \left\{ \frac{M^4 s^2 \log (p^2 M^2/\delta)}{\kappa^2(M)}, \frac{ \max \left\{ sM^2, s\log (p/\delta), sM^2 \omega^2 (M) \log (M^2 p /\delta) \right\}  }{\kappa (M) \nu^2 (M)}  , \right.\right.\\
\left.\left. \frac{ \max \left\{ s^3 M^2 (\log(1/\delta)^2), s (d^2_{j0}-d^2_{js}(M) \Phi^2(M)) \right\}  }{\kappa (M) \nu^2 (M)} \right\}\right),
\end{multline*}
then with probability at least $1-\delta$, we have
\begin{equation*}
\sqrt{\sum^{p-1}_{k=1} \left\Vert \hat{\bm{B_k}} - \bm{B^*_k} \right\Vert^2_{\text{F}}} \leq \check{\chi} (n,p,M,\delta) + 12 \sqrt{s} \frac{\omega (M)}{ \sqrt{\kappa(M)}},
\end{equation*}
so that $\hat{\mathscr{N}}_j = \mathscr{N}_j$.
\end{theorem}
\begin{proof}
See Appendix~\ref{sec:proof-thm2}.
\end{proof}
}

Comparing Theorem~\ref{thm:nb-edge-recov-data-driven} with Theorem~\ref{thm:nb-edge-recov-fix}, one key difference is that the regularization parameter is increased by a term that corresponds to the estimation error of the basis functions. As a result, the sample complexity also increases due to this additional error source.

Similar to before, the following corollary provides guarantees for recovery of the whole graph and directly follows from Theorem~\ref{thm:nb-edge-recov-data-driven} when applying a union bound.
\begin{corollary}[Graph Recovery with Data-Dependent Basis]
\label{corollary:graph-recov-est-basis}
Suppose the conditions of Theorem~\ref{thm:nb-edge-recov-data-driven} hold for all nodes $j \in [p]$.
We use $\kappa_j(M_j)$, $\omega_j(M_j)$, $\nu_j(M_j)$, $s_j$, and $M_j$ to take the place of $\kappa(M)$, $\omega(M)$, $\nu(M)$, $s$, and $M$ in Theorem~\ref{thm:nb-edge-recov-fix} to show their dependency on $j$ explicitly.
Let $\hat{E}$ be the estimated edge set obtained by applying either the AND or OR rule to the estimated neighborhood of each node. If the sample size $n$ satisfies 
\begin{multline*}
n \geq \max_{j \in [p]} \tilde{O}  \left( \max \left\{ \frac{M_j^4 s_j^2 \log (p^3 M_j^2/\delta)}{\kappa_j^2(M_j)}, 
 \frac{ \max \left\{  s_j M_j^2, s_j\log (p^2/\delta), s_j M_j^2 \omega_j^2 (M_j) \log (M_j^2 p^2 /\delta) \right\}   }{\kappa_j (M_j) \nu_j^2 (M_j)}  , \right.\right.\\
 \left.\left. \frac{ \max \left\{ s_j^3 M_j^2 (\log(1/\delta)^2), s_j (d^2_{j0}-d^2_{js}(M_j) \Phi^2_j(M_j)) \right\} }{\kappa_j (M_j) \nu_j^2 (M_j)} \right\}\right),
\end{multline*}
then $\mathbb{P} \left\{ \hat{E} = E \right\} \geq 1 - \delta$.
\end{corollary}

Compared to~\citet{solea2021nonparametric}, our theoretical analysis is more general, since we allow, but do not restrict, $\phi_{jm}(t)$'s to be the FPCA eigenfunctions of $K_{jj}(t^{\prime},t)$ and $\hat \phi_{jm}(t)$'s to be the estimated FPCA eigenfunctions of $\hat{K}_{jj}(t^{\prime},t)$.

\subsection{Theoretical Guidance on the Choice of Function Basis}
\label{sec:th-ch-fun-basis}

We give a theoretical guide for choosing the function basis. Note that we treat the guidance in this section as a heuristic design, and leave more thorough study on this topic for further research. Our theory can successfully explain why the PSKL basis is a good choice for functional data satisfying the partial separability condition---see Section~\ref{sec:minimize-omega-M}.

Suppose that we use $\bm{\phi}=\{ \phi_m \}^{\infty}_{m=1}$ as a function basis to represent the data. Let 
\begin{equation}\label{eq:functionBasisObjective}
\Lambda (M, \bm{\phi}) = \frac{\omega(M)}{\sqrt{\kappa(M)}\tau(M)},
\end{equation}
where $\omega(M)$ measures the covariance between the scores of the basis elements we include and the basis elements we truncate, $\kappa(M)$ is the minimum eigenvalue of the covariance of the scores, and $\tau(M)$ measures the signal strength of the basis elements we include. The function $\Lambda (M, \bm{\phi})$ appears in Assumption~\ref{assump:sig-strength}, and according to the previous section, a good choice of $\bm{\phi}$ should minimize $\Lambda (M, \bm{\phi})$ for all $M \geq 1$. Unfortunately, minimizing $\Lambda (M, \bm{\phi})$ is typically infeasible, as it involves unknown quantities. We motivate two approaches for selecting the function basis. First, we show that when additional assumptions hold, a basis that minimizes $\omega(M)$ can be used. Second, we consider a more general case and suggest approximately minimizing an upper bound on $\Lambda (M, \bm{\phi})$.

\subsubsection{Minimize $\omega(M)$}
\label{sec:minimize-omega-M}

Our first approach to choosing the function basis is to minimize $\omega(M)$. To achieve that, the function basis $\bm{\phi}$ should minimize the covariance between $\{ \{ g^{X_k}_i, \phi_m  \}^{p-1}_{k=1} \}^{M}_{m=1}$ and $\{ \{ g^{X_k}_i, \phi_m  \}^{p-1}_{k=1} \}_{m > M}$. Although minimizing this covariance is intractable in general, under the assumption of partial separability \citep{zapata2021partial}, we can solve the minimization problem exactly.

\begin{definition}[Partial Separability]
\label{def:par-sep}
An orthonormal function basis $\{ \phi_m \}^{\infty}_{m=1}$ is called the partial separability Karhunen-Lo\`{e}ve expansion (PSKL) basis if the random vectors 
\[
(\langle g^Y_i, \phi_m \rangle, \langle g^{X_1}_i, \phi_m \rangle, \dots, \langle g^{X_{p-1}}_i, \phi_m \rangle), \qquad m \in \mathbb{N}
\]
are mutually uncorrelated. 
\end{definition}

When the PSKL basis exists, then $\{ \{ \langle g^{X_k}_i, \phi_m \rangle \}^{p-1}_{k=1} \}^{M}_{m=1}$ and $\{ \{ \langle g^{X_k}_i, \phi_m \rangle  \}^{p-1}_{k=1} \}_{m > M}$ are uncorrelated for any $M \geq 1$, and $\omega(M)=0$ for all $M \geq 1$. Note that $\Lambda (M, \bm{\phi})$ is nonnegative, thus, the PSKL basis minimizes~\eqref{eq:functionBasisObjective} when the data generating process is partially separable. \citet{lynch2018test} proposed a test to verify the partial separability assumption. 
{\newTextTwo
When the partial separability condition holds, by Theorem~2 of~\cite{zapata2021partial}, one can use the eigenfunctions of $\hat{K}(t,t^{\prime}) = (1/p)\sum^p_{j=1}\hat{K}_{jj}(t,t^{\prime})$ as estimates of the PSKL basis, where $\hat{K}_{jj}(t,t^{\prime})$ is defined in~\eqref{eq:emp-func-cov}.}
However, the partial separability assumption is strong and may not hold in general settings.

\subsubsection{Minimize an Approximate Upper Bound}
When the PSKL basis does not exist, we suggest choosing a function basis by approximately minimizing the following upper bound on~\eqref{eq:functionBasisObjective}.
By the calculation in Appendix~\ref{sec:der-minimize-upp-bd}, we have
\begin{multline}
- \log \Lambda(M,\phi) \gtrsim - \log \left\{ \max_{k \in [p-1]} \Vert \bm{\Sigma^{r,X_k}} \Vert_{\text{F}} \right\} + \frac{1}{2} \log\left\{ \tr \left( \bm{\Sigma^{Y}} \right)  \right\} + \frac{1}{2} \log \left\{   \frac{ \rho_{\min} \left( \bm{\Sigma^X_{\mathscr{N}_j \mathscr{N}_j}} \right) }{ \rho_{max} \left( \bm{\Sigma^{X}_{ \mathscr{N}_j \mathscr{N}_j }} \right) }   \right\} \\
+ \frac{1}{2} \log \left\{   \rho_{\min}  \left( \left[ \bm{R^{Y,X_k}}  \right]_{k \in \mathscr{N}_j}  \left( \bm{R^X_{\mathscr{N}_j \mathscr{N}_j}}  \right)^{-2} \left[ \bm{R^{X_k,Y}}  \right]_{k \in \mathscr{N}_j} \right)   \right\}. \label{eq:upp-bd-log-Lambda}
\end{multline}
Therefore, by maximizing the right hand side of \eqref{eq:upp-bd-log-Lambda}, we are approximately minimizing $\Lambda(M,\phi)$. Unfortunately, most of the terms depend on $\mathscr{N}_j$, which is unknown. As a consequence, we choose to maximize $ \log\left\{ \tr \left( \bm{\Sigma^{Y}} \right)  \right\}$, which does not depend on $\mathscr{N}_j$. This term is maximized when the function basis $\bm{\phi}$ is the FPCA basis of $g^Y_i$. More intuitively, the FPCA basis of $g^Y_i$ maximizes the signal strength of the response variable. In Section~\ref{sec:experimemts}, we confirm by extensive simulations that the FPCA basis of $g^Y_i$ indeed performs well by a slight margin.

\section{Simulations}
\label{sec:experimemts}

We illustrate the finite sample properties of our neighborhood selection procedure through a simulation study. 
{\newText
We defer the wall-clock runtime analysis of different methods to Appendix~\ref{sec:wall-clock}.
}
We generate the simulated data with the following procedure. Let 
\[
g_{ij}(t) = (\bm{a_{ij}})^{\top}\bm{f}(t), \qquad i \in [n], j\in [p],
\]
where $\bm{a_{ij}}\in\mathbb{R}^{M^{\ast}}$ and $((\bm{a_{i1}})^\top,\ldots,(\bm{a_{ip}})^\top)^\top\in\mathbb{R}^{pM^{\ast}}$ follows a mean zero Gaussian distribution with covariance matrix $\bm{\Sigma}=\bm{\Theta}^{-1}$, and $\bm{f}(t)$ is a vector that contains the first $M^{\ast}$ Fourier basis functions. We consider the following four settings for the precision matrix $\bm{\Theta}$. 
 
\begin{itemize}
\item\textbf{Model A}. (Block Banded -- Full) We generate a block-banded precision matrix $\bm{\Theta} \in \mathbb{R}^{pM^{\ast} \times pM^{\ast}}$ with $M^{\ast} = 15$. Define a Toeplitz matrix $\bm{T}$ such that $\bm{T_{jj}} = 1$, and $\bm{T_{jl}} = 1/2|j-l|$ for all $j\neq l$. Let $\bm{A} \in \mathbb{R}^{M^{\ast} \times M^{\ast}}$ be a tridiagonal matrix with $\bm{A}_{kk}=1$ and $\bm{A}_{k,k+1}=\bm{A}_{k+1,k}=0.5$. All other entries of $\bm{A}$ are set to 0. The blocks of the precision matrix $\bm{\Theta}$ are then given as $\bm{\Theta_{jj}}=\bm{T}$, $\bm{\Theta_{j,j+1}}= \bm{\Theta_{j+1,j}}= 0.4\bm{A}$, and $\bm{\Theta_{j,j+2}}=\bm{\Theta_{j+2,j}}=0.2\bm{I}_{M^*}$. All remaining blocks of $\bm{\Theta}$ are set to $0$.

\item\textbf{Model B}. (Block Banded -- Partial) We generate a partially block-banded precision matrix $\bm{\Theta} \in \mathbb{R}^{pM^{\ast} \times pM^{\ast}}$ with $M^{\ast} = 15$. In this setting, every alternating block of 10 nodes have similar connection pattern as in Model A, and the remaining nodes are fully isolated. Precisely, $\bm{\Theta}$ is a block diagonal matrix, with each of its $10M^{\ast}\times 10M^{\ast}$ blocks denoted by $\bm{\Theta^{(k)}}$, $k=1,\cdots,\lceil \frac{p}{10}\rceil$. For even $k$, we set $\bm{\Theta^{(k)}} = \bm{I}_{10M^{\ast}}$. For odd $k$, we set $\bm{\Theta^{(k)}}\in \mathbb{R}^{10M^{\ast} \times 10M^{\ast}}$ to be the block-banded precision matrix such that $\bm{(\Theta^{(k)})_{jj}}=\bm{A}$, $\bm{(\Theta^{(k)})_{j,j+1}}= \bm{(\Theta^{(k)})_{j+1,j}}= 0.4\bm{A}$, and $\bm{(\Theta^{(k)})_{j,j+2}}=\bm{(\Theta^{(k)})_{j+2,j}}=0.2\bm{A}$. All remaining blocks of $\bm{\Theta^{(k)}}$ are set to $0$.

\item\textbf{Model C}. (Hub Model) We generate a hub-connected precision matrix $\bm{\Theta} \in \mathbb{R}^{pM^{\ast} \times pM^{\ast}}$ with $M^{\ast} = 5$. We generate the edge set $E$ from a power law distribution as follows. For each node, the number of neighbors $m$ follows a power law distribution $p_m=m^{-\alpha}$ with $\alpha=2$, and the exact neighbors are sampled uniformly. A disjoint sequence of edge sets $E_1,\dots,E_5$ is generated from $E$, yielding 5 adjacency matrices $\bm{G_l}, l=1,\dots,5$. The detailed algorithm of this step is given in Section 8 of \citet{zapata2021partial}. Then we generate $p\times p$ precision submatrices $\bm{\Omega_l}, l=1,\dots,5$, whose supports are exactly $\bm{G_l}$. First, $p \times p$ matrices $\bm{\tilde{\Omega}_l}$ are generated:
    $$(\bm{\tilde{\Omega}_l})_{ij} =
    \begin{cases}
    1 & \text{if } i=j\\
    0 & \text{if } (\bm{G_l})_{ij}=0 \text{ or }i<j\\
    \sim\text{Unif}(\mathcal{D}) &\text{if } (\bm{G_l})_{ij}=1
    \end{cases}$$
    where $\mathcal{D} = [-\frac{2}{3},-\frac{1}{3}]\cup[\frac{1}{3},\frac{2}{3}]$. Next, we rescale the rows of $\bm{\tilde{\Omega}_l}$ so that the $\ell_2$-norm of each row is 1. We then obtain $\bm{\Omega_l}$ by symmetrizing  $\bm{\tilde{\Omega}_l}$; we average $\bm{\tilde{\Omega}_l}$ with its transpose, and set the diagonals to one. Let $\bm{\Sigma}_{\text{ps}} = \text{diag}(\bm{\Sigma_1},\dots,\bm{\Sigma_5})$ with $\bm{\Sigma_l} = \varpi_l\bm{\Omega_l}^{-1}$, where $\varpi_l = 3l^{-1.8}$. To break the partial separability condition, we define a block precision matrix $\bm{\bar{\Omega}}\in\mathbb{R}^{pM^* \times pM^*}$, whose diagonal blocks of $p\times p$ are $\bm{\bar{\Omega}_{l,l}}=\bm{\Omega_l}$ and off-diagonal blocks are $\bm{\bar{\Omega}_{l,l+1}}=\bm{\bar{\Omega}_{l+1,l}} = (\bm{\Omega^*_l} + \bm{\Omega^*_{l+1}})/2$ where $\bm{\Omega_l^*} = \bm{\Omega_l} - \text{diag}(\bm{\Omega_l})$. We then calculate
    \[\bm{\Sigma}_{\text{nps}} = \text{diag}(\bm{\Sigma}_{\text{ps}})^\frac{1}{2}\left(\text{diag}(\bm{\bar{\Omega}})^{-\frac{1}{2}}\bm{\bar{\Omega}}\text{diag}(\bm{\bar{\Omega}})^{-\frac{1}{2}}\right)^{-1}\text{diag}(\bm{\Sigma}_{\text{ps}})^\frac{1}{2}.\]
    Finally, we obtain the covariance matrix as  $(\bm{\Sigma}_{jl})_{st}=(\bm{\Sigma}_{\text{nps},st})_{jl}$, $1\leq j,l\leq p$ and $1\leq s,t\leq 5$. Finally, the precision matrix $\bm{\Theta} = (\bm{\Sigma})^{-1}$.

\item\textbf{Model D}. (Randomly Connected) This model is similar to the setting introduced in \citet{rothman2008sparse}, but modified to fit functional data. We generate a random block sparse precision matrix $\bm{\Theta} \in \mathbb{R}^{pM^{\ast} \times pM^{\ast}}$ with $M^{\ast} = 15$. Each off-diagonal block $\bm{\Theta_{jl}}, 1\leq j\neq l\leq p$ is set to $0.5\bm{I}_{M^{\ast}}$ with probability 0.1, and $\bm{0}_{M^{\ast}}$ otherwise. The diagonal blocks are set as $\bm{\Theta_{jj}} = \delta' \bm{I}_{M^{\ast}}$, where $\delta'$ is chosen to guarantee the positive definiteness of the precision matrix, i.e., $\bm{\Theta}\succ 0$. It is sufficient to choose $\delta'$ such that it exceeds the maximum row sum of the absolute values of the off-diagonal entries of $\bm{\Theta}$, thus the diagonal dominance ensures positive definiteness. Notice that the partial separability condition is satisfied under this model.
\end{itemize}

Models A and B are similar to Models 1 and 2 in Section 5.1 of \citet{Qiao2015Functional}, but modified so that partial separability is violated. 
Model C---where partial separability is also violated---is used as a counter-example in \citet{zapata2021partial}. In these three models, the partial separability condition is violated, that is, it is impossible to separate the multivariate and functional aspects of the data. However, Model D satisfies the partial separability by construction. 

For each setting, we fix $n=100$ and let $p =50,100,150$. Each random function is observed on a grid with $T=100$ equally spaced points on $[0,1]$. For $T$ observed time points, $(t_1,\dots,t_T)$, uniformly spread on $[0,1]$, the observed values are generated by
\[g^{\text{obs}}_{ij}(t_k) = \sum_{l=1}^{M^{\ast}} a_{ijl}\bm{f}_l(t_k) + \epsilon_{ijk}\]
where $\epsilon_{ijk}\sim N(0,\sigma^2)$. In models A, B, and D, we set $\sigma = 0.5$, while in model C, $\sigma^2=0.05\times {\sum_{l=1}^{M^{\ast}} \text{tr}(\bm{\Sigma_l})}/{p}$. We report the results averaged over 50 independent runs.

In all experiments, we set $M$, the number of principal components used to model each function, to be $5$ for all nodes. 
This is a typical value selected by the cross-validation process described in Section \ref{CH_Param_Selection}.
The first simulation experiment compares the performance of our proposed method to current baseline methods, including FGLasso \citep{Qiao2015Functional} and PSKL \citep{zapata2021partial}.
Since the theoretically appropriate choice of the threshold $\epsilon_n$ depends on problem parameters that are generally unknown in practice, we use $\epsilon_n = 0$ for this part of experiment. We plot the ROC curve as the penalty parameter $\lambda_n$ changes. More specifically, let $\lambda_n=\lambda_{j,n}$ explicitly denote the parameter choice for the node $j$. According to Proposition~\ref{prop:lambda_max}, there exists $\lambda_{\max}$ such that $\mathscr{\hat N}_j$ is empty. Let $\lambda_{j,n,t}=t_\lambda \cdot \lambda_{j,\max}$, where $t_\lambda \in [0,1]$ is the same for all nodes. We plot the ROC curve as $t_\lambda$ changes from $1$ to $0$.
The second simulation experiment illustrates the performance of our method under various choices of $\epsilon_n$. From the comparison of ROC's under each model setting, our empirical results confirm that optimal performance is typically achieved with a non-zero $\epsilon_n$.
The third simulation experiment is dedicated to assess the accuracy of a single selected graph. We use the SCV-RSS criterion introduced in Section~\ref{CH_Param_Selection} to select $\lambda_{j,n}$ for all nodes, and then evaluate the performance by reporting the precision and recall on the specific graph that is selected.

\begin{figure}
    \centering
    
    \includegraphics[width=15cm]{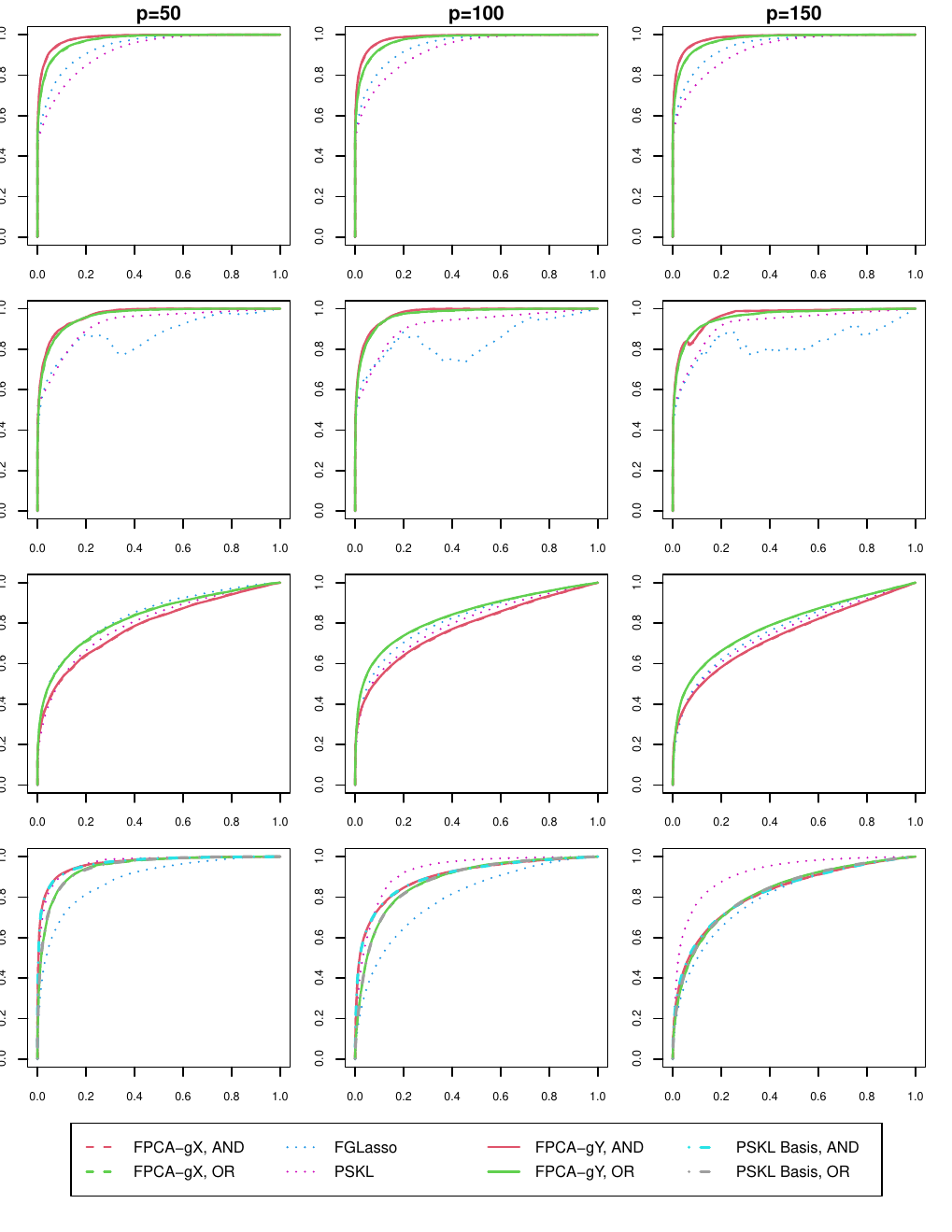}
    
    \caption{ROC curves for different models and values of $p$. From top to bottom: Models A, B, C, D. From left to right: $p=50,100,150$. Horizontal axis: FPR; vertical axis: TPR. For FPCA-$g_X$, each function is projected onto its own FPCA basis, while FPCA-$g_Y$ projects all other functions onto the FPCA basis of $g_{ij}$.}
    \label{fig:roc}
\end{figure}

\subsection{Comparison with Baseline Methods}

We compare our method with the functional Graphical lasso (FGLasso) procedure \citep{Qiao2015Functional} and the PSKL procedure \citep{zapata2021partial}. For FGLasso, we select the parameters as proposed in \citep{Qiao2015Functional}. For PSKL, we use the package ``fgm'' with the default setting \citep{zapata2021partial}. As Model D is partially separable, we also implemented our method using the PSKL function basis that we assume is known a priori, which we call PSKL Basis---in this case, it is Fourier basis. To demonstrate the advantages of using a single function basis when estimating $\mathscr{N}_j$---as explained in Section \ref{sec:choice-fun-basis}---we implemented the following two methods to estimate the FPCA scores and compared their performances. The first method, which we call FPCA-$g_X$, projects each function onto its own FPCA basis and uses those projection scores for all subsequent tasks. The second method, which we call FPCA-$g_Y$, projects all other functions onto the FPCA basis of $g_j$ when selecting $\mathscr{N}_j$.

To compare the methods, we plot their respective ROC curves for each model and different values of $p$. For each value of the tuning parameters $\lambda_n$ and $t_\epsilon$, we compare the estimated edge set to the true edge set. Specifically, we calculate the true positive rate $\text{TPR}(\lambda_n, t_\epsilon) = \text{TP}(\lambda_n, t_\epsilon) / (\text{TP}(\lambda_n, t_\epsilon) + \text{FN}(\lambda_n, t_\epsilon))$ and the false positive rate $\text{FPR}(\lambda_n, t_\epsilon) = \text{FP}(\lambda_n, t_\epsilon) / (\text{TP}(\lambda_n, t_\epsilon) + \text{TN}(\lambda_n, t_\epsilon))$, where $\text{TP}(\lambda_n, t_\epsilon), \text{FP}(\lambda_n, t_\epsilon), \text{TN}(\lambda_n, t_\epsilon), \text{FN}(\lambda_n, t_\epsilon)$ stand for the number of true positive, false positive, true negative, and false negative number of edges, respectively. Recall that we use $t_\epsilon = 0$ for the comparison of different methods. The ROC curves are plotted by varying the penalty parameter $\lambda_n$, with $\text{TPR}(\lambda_n,0)$ on the vertical axis and $\text{FPR}(\lambda_n,0)$ on the horizontal axis. We also calculate the area under the ROC curve (AUC). The ROC curves are shown in Figure \ref{fig:roc} and the average AUC is given in Table~\ref{tab:auc}.

\begin{table}
\centering
    \begin{tabular}{|p{1cm}|p{0.5cm}||p{1.1cm}|p{1.1cm}|p{1.1cm}|p{1.1cm}|p{1.2cm}|p{1.1cm}|p{1.1cm}|p{1.1cm}|}
    \hline
    Model & $p$ & FPCA-$g_Y$, AND & FPCA-$g_Y$, OR & FPCA-$g_X$, AND & FPCA-$g_X$, OR & FGLasso & PSKL & FPCA-PSKL, AND & FPCA-PSKL, OR\\
    \hhline{|=|=||=|=|=|=|=|=|=|=|}
    \multicolumn{1}{|c|}{\multirow{6}{*}{A}} & 
    \multirow{2}{4em}{50} & {\bf 0.984} (0.004) & 0.974 (0.007) & {\bf 0.984} (0.005) & 0.973 (0.007) & 0.942 (0.010) & 0.920 (0.010) && \\
    \cline{2-8}
    &\multirow{2}{4em}{100} & {\bf0.985} (0.003) & 0.976 (0.004) & 0.984 (0.003) & 0.975 (0.004) & 0.947 (0.006) & 0.925 (0.007) &
    \multicolumn{1}{|c|}{\multirow{2}{*}{N/A}} & 
    \multicolumn{1}{|c|}{\multirow{2}{*}{N/A}}\\
    \cline{2-8}
    &\multirow{2}{4em}{150} & {\bf 0.985} (0.003) & 0.976 (0.003) & 0.984 (0.003) & 0.975 (0.003) & 0.948 (0.005) & 0.927 (0.007) && \\
    \hhline{|=|=||=|=|=|=|=|=|=|=|}
    \multicolumn{1}{|c|}{\multirow{6}{*}{B}} & 
    \multirow{2}{4em}{50} & {\bf 0.969} (0.008) & 0.964 (0.009) & {\bf 0.969} (0.008) & 0.964 (0.009) & 0.806 (0.100) & 0.924 (0.013) && \\
    \cline{2-8}
    &\multirow{2}{4em}{100} & {\bf 0.976} (0.005) & 0.971 (0.006) & {\bf 0.976} (0.005) & 0.970 (0.006) & 0.703 (0.077) & 0.918 (0.021) &
    \multicolumn{1}{|c|}{\multirow{2}{*}{N/A}}&
    \multicolumn{1}{|c|}{\multirow{2}{*}{N/A}} \\
    \cline{2-8}
    &\multirow{2}{4em}{150} & {\bf 0.965} (0.006) & 0.961 (0.007) & 0.964 (0.006) & 0.960 (0.008) & 0.620 (0.067) & 0.924 (0.012) && \\
    \hhline{|=|=||=|=|=|=|=|=|=|=|}
    \multicolumn{1}{|c|}{\multirow{6}{*}{C}} &
    \multirow{2}{4em}{50} & 0.785 (0.035) & 0.828 (0.037) & 0.785 (0.035) & 0.828 (0.038) & {\bf 0.838} (0.037) & 0.799 (0.042) && \\
    \cline{2-8}
    &\multirow{2}{4em}{100} & 0.780 (0.040) & {\bf 0.839} (0.036) & 0.777 (0.039) & 0.837 (0.036) & 0.822 (0.101) & 0.797 (0.061) &
    \multicolumn{1}{|c|}{\multirow{2}{*}{N/A}}&
    \multicolumn{1}{|c|}{\multirow{2}{*}{N/A}} \\
    \cline{2-8}
    &\multirow{2}{4em}{150} & 0.740 (0.061) & {\bf 0.792} (0.053) & 0.738 (0.060) & 0.790 (0.053) & 0.768 (0.115) & 0.755 (0.077) &&\\
    \hhline{|=|=||=|=|=|=|=|=|=|=|}
    \multicolumn{1}{|c|}{\multirow{6}{*}{D}} &
    \multirow{2}{4em}{50} & {\bf 0.967} (0.012) & 0.948 (0.017) & 0.966 (0.013) & 0.947 (0.017) & 0.888 (0.081) & 0.966 (0.044) & 0.966 (0.013) & 0.948 (0.017)\\
    \cline{2-10}
    &\multirow{2}{4em}{100} & 0.902 (0.029) & 0.882 (0.022) & 0.900 (0.029) & 0.881 (0.022) & 0.798 (0.092) & {\bf 0.929} (0.037) & 0.902 (0.030) & 0.881 (0.022)\\
    \cline{2-10}
    &\multirow{2}{4em}{150} & 0.823 (0.013) & 0.824 (0.010) & 0.821 (0.013) & 0.822 (0.011) & 0.802 (0.040) & {\bf 0.917} (0.009) & 0.823 (0.013) & 0.824 (0.010)\\
    \hline
    \end{tabular}
    
    \caption{\label{tab:auc}The average AUC ROC for each method across 50 runs; the standard errors of AUC are given in the parentheses. For FPCA-$g_X$, each function is projected onto its own FPCA basis, while FPCA-$g_Y$ projects all other functions onto the FPCA basis of $g_{ij}$. The maximum of each row is marked in bold.}

\end{table}

Although FPCA-$g_Y$ and FPCA-$g_X$ perform similarly, FPCA-$g_Y$ slightly outperforms FPCA-$g_X$ across all settings. Moreover, both FPCA-$g_X$ and FPCA-$g_Y$ drastically outperform FGLasso in Models A, B, and D, and slightly outperforms FGLasso in most settings under Model C.  In Models A, B, and C, where partial separability does not hold, the PSKL procedure generally underperforms the other procedures. Even in Model D, where partial separability holds, PSKL has a performance that is comparable to ours when the dimension is low. We also note that in Model D, the neighborhood selection procedure that uses the ``optimal'' PSKL basis, which we assume is known a priori, has a very similar performance to the procedure that uses the FPCA basis and does not require prior knowledge. This suggests that using the data-selected FPCA basis is a good idea across a variety of settings.

\subsection{The Effect of $\epsilon_n$}
\label{sec:effec-epsilon-n}

{\newText

In this section, we empirically demonstrate how $\epsilon_n$ impacts practical performance. The experimental setting remains identical to that in Section 5. We apply our proposed method by setting $\epsilon_n = t_\epsilon\lambda_n$, and compute an ROC for each fixed $t_\epsilon$ value. For each model, we select five distinct $t_\epsilon$ values, inclusive of 0. Figure~\ref{fig:roc.t} provides a visually intuitive comparison and Table~\ref{tab:thres} illustrates the area under the curve (AUC) for each of the five $t_\epsilon$ values under each setting. In most scenarios, the maximal AUC is achieved by a non-zero $t_\epsilon$. However, the marginal benefit of using the optimal $t_\epsilon$ over simply setting $t_\epsilon=0$ is relatively insignificant in most cases. This empirical investigation justifies the practical choice of $\epsilon_n=0$.

\begin{figure}[p]
\centering
    
\includegraphics[width=15cm]{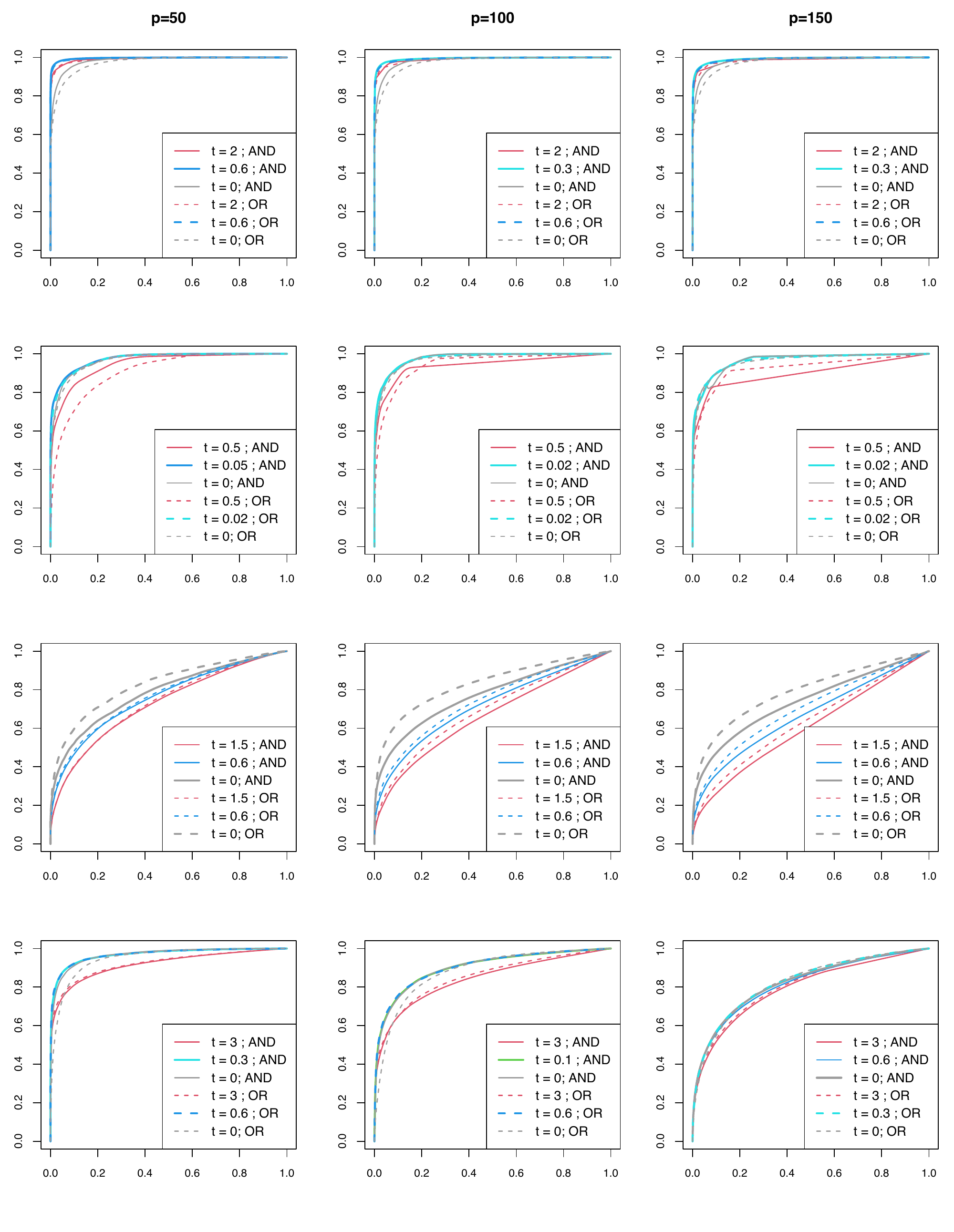}
    
\caption{ROC for different models and values of $p$ under different threshold coefficient $t_\epsilon$ using FPCA-$g_X$ method. From top to bottom: Models A, B, C, D. From left to right: $p=50,100,150$. Horizontal axis: FPR; vertical axis: TPR.}
\label{fig:roc.t}
\end{figure}

\begin{table}[p]
\centering
    \begin{tabular}{|p{1cm}|p{0.8cm}||p{2.0cm}|p{2.1cm}|p{2.1cm}|p{2.1cm}|p{2.1cm}|}
    \hline
    \multicolumn{2}{|c||}{Model A}& $t_\epsilon = 2.0$ & $t_\epsilon = 0.6$ & $t_\epsilon = 0.3$ & $t_\epsilon = 0.1$ & $t_\epsilon = 0$\\
    \hhline{|=|=||=|=|=|=|=|}
    \multirow{2}{4em}{p=50} &
    \multirow{1}{4em}{AND} & 0.992 (0.004) & {\bf 0.997} (0.003) & 0.996 (0.003) & 0.993 (0.004) & 0.984 (0.005) \\
    
    &\multirow{1}{4em}{OR} & 0.991 (0.004) &  {\bf 0.996} (0.003) & 0.995 (0.003) & 0.989 (0.004) & 0.973 (0.007) \\
    \hline
    
    \multirow{2}{4em}{p=100} &
    \multirow{1}{4em}{AND} & 0.970 (0.009) & {\bf 0.994} (0.003) & {\bf 0.994} (0.003) & 0.992 (0.003) & 0.984 (0.003) \\
    &\multirow{1}{4em}{OR} & 0.973 (0.005) & {\bf 0.993} (0.003) & {\bf 0.993} (0.003) & 0.989 (0.003) & 0.975 (0.004) \\
    \hline
    
    \multirow{2}{4em}{p=150} &
    \multirow{1}{4em}{AND} & 0.985 (0.006) & 0.991 (0.003) & {\bf 0.992} (0.003) & 0.990 (0.002) & 0.984 (0.003) \\
    &\multirow{1}{4em}{OR} & 0.987 (0.004) & {\bf 0.991} (0.002) & {\bf 0.991} (0.002) & 0.987 (0.002) & 0.975 (0.003) \\
    \hline
    \end{tabular}

    \vspace{1cm}
    
    \begin{tabular}{|p{1cm}|p{0.8cm}||p{2.0cm}|p{2.1cm}|p{2.1cm}|p{2.1cm}|p{2.1cm}|}
    \hline
    \multicolumn{2}{|c||}{Model B}& $t_\epsilon = 0.5$ & $t_\epsilon = 0.2$ & $t_\epsilon = 0.05$ & $t_\epsilon = 0.02$ & $t_\epsilon = 0$\\
    \hhline{|=|=||=|=|=|=|=|}
    \multirow{2}{4em}{p=50} &
    \multirow{1}{4em}{AND} & 0.901 (0.026) & 0.964 (0.010) & {\bf 0.974} (0.009) & 0.973 (0.009) & 0.969 (0.008) \\
    
    &\multirow{1}{4em}{OR} & 0.856 (0.020) & 0.949 (0.011) & 0.970 (0.009) & {\bf 0.971} (0.009) & 0.964 (0.009) \\
    \hline
    
    \multirow{2}{4em}{p=100} &
    \multirow{1}{4em}{AND} & 0.940 (0.013) & 0.966 (0.009) & 0.977 (0.005) & {\bf 0.978} (0.005) & 0.976 (0.005) \\
    &\multirow{1}{4em}{OR} & 0.943 (0.009) & 0.963 (0.007) & 0.975 (0.006) & {\bf 0.976} (0.006) & 0.970 (0.006) \\
    \hline
    
    \multirow{2}{4em}{p=150} &
    \multirow{1}{4em}{AND} & 0.896 (0.015) & 0.946 (0.009) &  {\bf 0.965} (0.007) & {\bf 0.965} (0.007) & 0.964 (0.006) \\
    &\multirow{1}{4em}{OR} & 0.924 (0.012) & 0.952 (0.008) & 0.962 (0.008) & {\bf 0.964} (0.008) & 0.960 (0.008) \\
    \hline
    \end{tabular}

    \vspace{1cm}
    
    \begin{tabular}{|p{1cm}|p{0.8cm}||p{2.0cm}|p{2.1cm}|p{2.1cm}|p{2.1cm}|p{2.1cm}|}
    \hline
    \multicolumn{2}{|c||}{Model C}& $t_\epsilon = 1.5$ & $t_\epsilon = 0.6$ & $t_\epsilon = 0.1$ & $t_\epsilon = 0.05$ & $t_\epsilon = 0$\\
    \hhline{|=|=||=|=|=|=|=|}
    \multirow{2}{4em}{p=50} &
    \multirow{1}{4em}{AND} & 0.725 (0.043) & 0.758 (0.041) & 0.783 (0.037) & {\bf 0.785} (0.037) & {\bf 0.785} (0.035) \\
    
    &\multirow{1}{4em}{OR} & 0.730 (0.035) & 0.766 (0.036) & 0.808 (0.032) & 0.817 (0.033) & {\bf 0.828} (0.038) \\
    \hline
    
    \multirow{2}{4em}{p=100} &
    \multirow{1}{4em}{AND} & 0.659 (0.079) & 0.715 (0.073) & 0.761 (0.060) & 0.766 (0.058) & {\bf 0.771} (0.054) \\
    &\multirow{1}{4em}{OR} & 0.688 (0.076) & 0.738 (0.072) & 0.805 (0.062) & 0.818 (0.059) & {\bf 0.830} (0.061) \\
    \hline
    
    \multirow{2}{4em}{p=150} &
    \multirow{1}{4em}{AND} & 0.602 (0.076) & 0.666 (0.070) &  0.723 (0.063) & 0.730 (0.062) & {\bf 0.738} (0.059) \\
    &\multirow{1}{4em}{OR} & 0.632 (0.081) & 0.699 (0.070) & 0.765 (0.067) & 0.778 (0.062) & {\bf 0.791} (0.052) \\
    \hline
    \end{tabular}

    \vspace{1cm}

    \begin{tabular}{|p{1cm}|p{0.8cm}||p{2.0cm}|p{2.1cm}|p{2.1cm}|p{2.1cm}|p{2.1cm}|}
    \hline
    \multicolumn{2}{|c||}{Model D} & $t_\epsilon = 3.0$ & $t_\epsilon = 0.6$ & $t_\epsilon = 0.3$ & $t_\epsilon = 0.1$ & $t_\epsilon = 0$\\
    \hhline{|=|=||=|=|=|=|=|}
    \multirow{2}{4em}{p=50} &
    \multirow{1}{4em}{AND} & 0.918 (0.115) & 0.966 (0.024) & {\bf 0.969} (0.015) & 0.968 (0.013) & 0.966 (0.013) \\
    
    &\multirow{1}{4em}{OR} & 0.922 (0.113) & {\bf 0.969} (0.018) & 0.966 (0.014) & 0.956 (0.014) & 0.948 (0.017) \\
    \hline
    
    \multirow{2}{4em}{p=100} &
    \multirow{1}{4em}{AND} & 0.840 (0.104) & 0.896 (0.036) & 0.900 (0.031) & {\bf 0.901} (0.030) & 0.900 (0.029) \\
    &\multirow{1}{4em}{OR} & 0.851 (0.100) & {\bf 0.903} (0.031) & 0.899 (0.027) & 0.889 (0.024) & 0.881 (0.022) \\
    \hline
    
    \multirow{2}{4em}{p=150} &
    \multirow{1}{4em}{AND} & 0.794 (0.042) & 0.814 (0.011) & 0.818 (0.012) & 0.820 (0.013) & {\bf 0.821} (0.013) \\
    &\multirow{1}{4em}{OR} & 0.805 (0.034) & 0.824 (0.013) & {\bf 0.825} (0.013) & 0.823 (0.011) & 0.822 (0.011) \\
    \hline
    \end{tabular}
    
    \caption{The average AUC for each $t_\epsilon$ across 50 runs; the standard errors of AUC are given in the parentheses. The method is FPCA-$g_X$. The maximum of each row is marked in bold.}
    \label{tab:thres}

\end{table}

}

\subsection{Performance of Cross-Validation}
\label{sec:perf-cv}

{\newText

Practitioners may want a single graph rather than a series of graphs corresponding to different penalty and threshold parameters. Thus, we also evaluate the precision and recall of the final graph selected using the parameters obtained through selective cross-validation algorithm stated in Algorithm~\ref{Alg:scv.grid}. When choosing $\lambda_n,\epsilon_n$, we let candidate $\epsilon_n$'s to be $\epsilon_n=t_\epsilon\cdot\lambda_n$, where $t_\epsilon \in \{0, 0.2, 0.4, 0.8, 1.2, 1.6, 2\}$. 
We denote the chosen tuning parameters as $(\lambda_n^*, t_\epsilon^*)$.
The precision and recall of $(\lambda_n^*, t_\epsilon^*)$ are defined as
\begin{align*}
\text{Precision}(\lambda_n^*, t_\epsilon^*) &= \text{TP}(\lambda_n^*, t_\epsilon^*) / (\text{TP}(\lambda_n^*, t_\epsilon^*) + \text{FP}(\lambda_n^*, t_\epsilon^*)), \\
\text{Recall}(\lambda_n^*, t_\epsilon^*) &= \text{TP}(\lambda_n^*, t_\epsilon^*) / (\text{TP}(\lambda_n^*, t_\epsilon^*) + \text{FN}(\lambda_n^*, t_\epsilon^*)).
\end{align*}
A larger value of precision and recall indicate better performance. The results under all models using FPCA-$g_X$ basis are shown in Table~\ref{tab:precAndRecall}. From Table~\ref{tab:precAndRecall} we see that our method obtains satisfactory performance under most models, even in the  high-dimensional setting.
In applications where a type-I error is more costly, the AND scheme may be preferable because it enjoys a higher precision; when we want to minimize type-II errors, the OR scheme is preferred.

\begin{table}
    \begin{tabular}{|p{1.2cm}|p{1cm}||p{2.55cm}|p{2.55cm}|p{2.55cm}|p{2.55cm}|}
    \hline
    Model & $p$ & AND, Precision & AND, Recall & OR, Precision & OR, Recall\\
    \hhline{|=|=||=|=|=|=|}
    \multirow{3}{4em}{A} & 
    50 & 1.000 (0.000) & 0.644 (0.050) & 0.985 (0.012) & 0.843 (0.037) \\
    &100 & 1.000 (0.002) & 0.630 (0.032) & 0.970 (0.013) & 0.826 (0.031) \\
    &150 & 1.000 (0.001) & 0.626 (0.029) & 0.964 (0.010) & 0.815 (0.022) \\
    \hline
    \multirow{3}{4em}{B} & 
    50 & 0.934 (0.071) & 0.463 (0.066) & 0.442 (0.120) & 0.659 (0.058) \\
    &100 & 0.601 (0.105) & 0.528 (0.049) & 0.155 (0.024) & 0.746 (0.052) \\
    &150 & 0.338 (0.061) & 0.556 (0.038) & 0.102 (0.009) & 0.782 (0.032) \\
    \hline
    \multirow{3}{4em}{C} & 
    50 & 0.853 (0.212) & 0.050 (0.033) & 0.549 (0.124) & 0.145 (0.048))\\
    &100 & 0.902 (0.080) & 0.076 (0.031) & 0.646 (0.093) & 0.211 (0.064)\\
    &150 & 0.849 (0.085) & 0.062 (0.023) & 0.590 (0.070) & 0.172 (0.057)\\
    \hline
    \multirow{3}{4em}{D} & 
    50 & 0.998 (0.014) & 0.122 (0.127) & 0.989 (0.030) & 0.263 (0.240)\\
    &100 & 0.966 (0.150) & 0.034 (0.024) & 0.957 (0.055) & 0.114 (0.086)\\
    &150 & 0.988 (0.058) & 0.004 (0.009) & 0.979 (0.044) & 0.012 (0.031)\\
    \hline
    \end{tabular}
    \caption{\label{tab:precAndRecall}The average precision and recall of the graph using FPCA-$g_X$ method. The optimal $\lambda_n$ and $t_\epsilon$ is selected by the SCV-RSS criterion across 50 runs; the standard deviation is given in paranthesis.}
\end{table}

}

\section{Data Analysis}
\label{sec:real_data}

{\newText
In this section, we illustrate the practical application of our method on two functional magnetic resonance imaging (fMRI) datasets. Raw brain magnetic resonance images are segmented into temporal signals for 116 regions of interest (ROIs) using the automatic anatomic labeling (AAL) parcellation approach~\citep{Tzourio-Mazoyer2002Automated}. Table~\ref{tab:aal} in Appendix~\ref{sec:roi-labels} lists the names and corresponding labels of all 116 ROIs. By applying this approach, we average the signal within all ROIs to obtain 116 distinct time series, which we interpret as observations of 116 corresponding random functions. Using the neighborhood selection procedure, we can recover the conditional independence (CI) graphs associated with different ROIs.

Recent research uncovers a hierarchical structure in brain connectivity. For instance, heteromodal areas, such as the prefrontal cortex, inferior parietal lobe, and superior temporal sulcus, project to paralimbic areas like the insula, orbitofrontal, cingulate, parahippocampal, and temporopolar regions. These, in turn, project to limbic areas, namely the amygdala and hippocampus. The latter two are the only parts of the cortex with substantial connections to the hypothalamus, a key node for homeostatic, autonomic, and endocrine aspects of the internal milieu \citep{mesulam2012evolving}. By learning the conditional independence graph of ROIs, we gain insight into these brain connectivity patterns. Moreover, comparing conditional independence graphs from populations with and without specific neurodevelopmental conditions could yield clues about the origins of certain symptoms.

Our functional graphical models approach offers significant advantages over traditional non-functional analyses of fMRI signals. Specifically, it can detect spatio-temporal interactions among ROIs. For instance, our method can identify the influence of one node at time $t$ on another node at a different time $t'$. We subsequently apply our method to two fMRI datasets: one pertaining to Autism Spectrum Disorder (ASD) and the other to Attention Deficit Hyperactivity Disorder (ADHD).

}

\subsection*{ASD Dataset}

{\newText

\begin{figure}[t]
    \centering
    
    \begin{subfigure}[b]{0.48\textwidth}
        \centering
        \includegraphics[width=\textwidth]{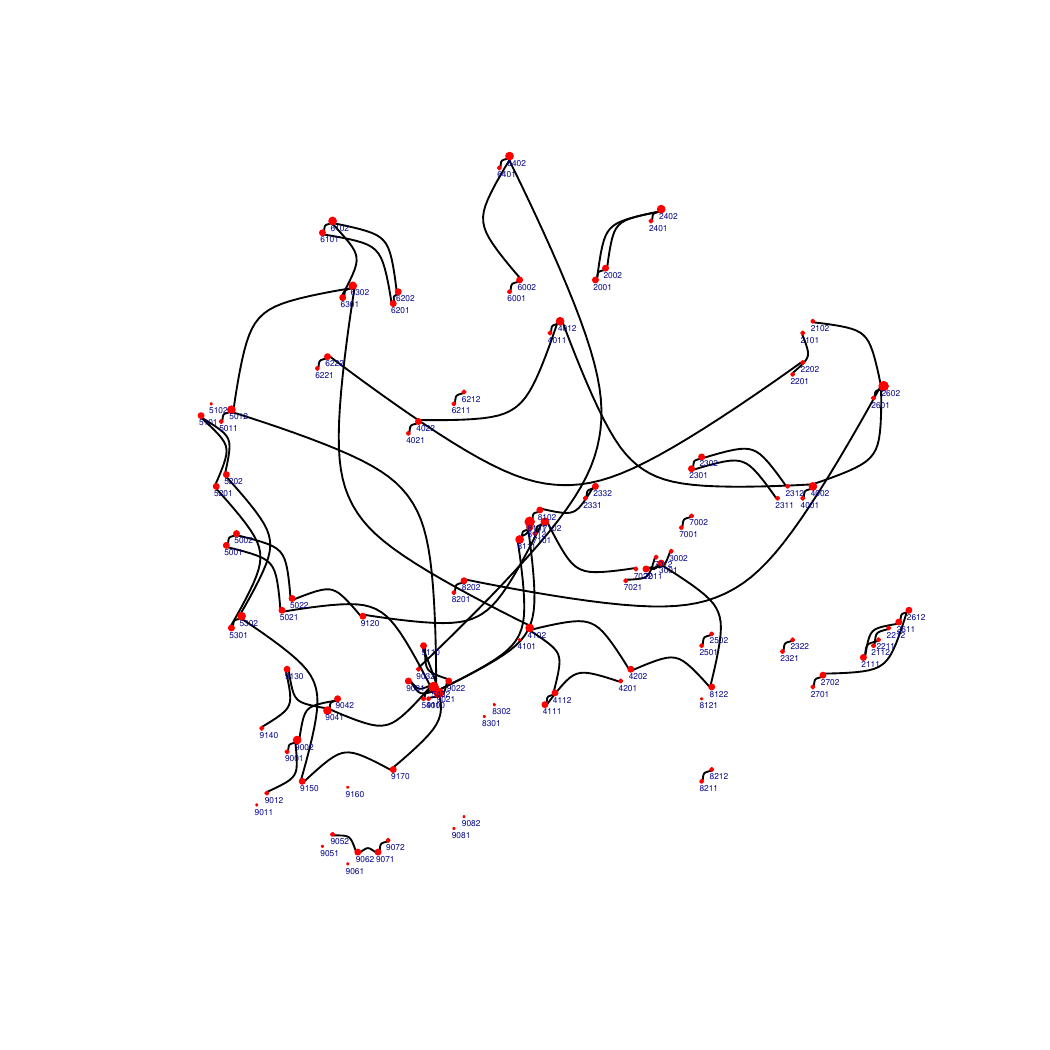}
        \caption{Autism group}
        \label{fig:ABIDE.SCV.autism}
    \end{subfigure}
    \hfill
    \begin{subfigure}[b]{0.48\textwidth}
        \centering
        \includegraphics[width=\textwidth]{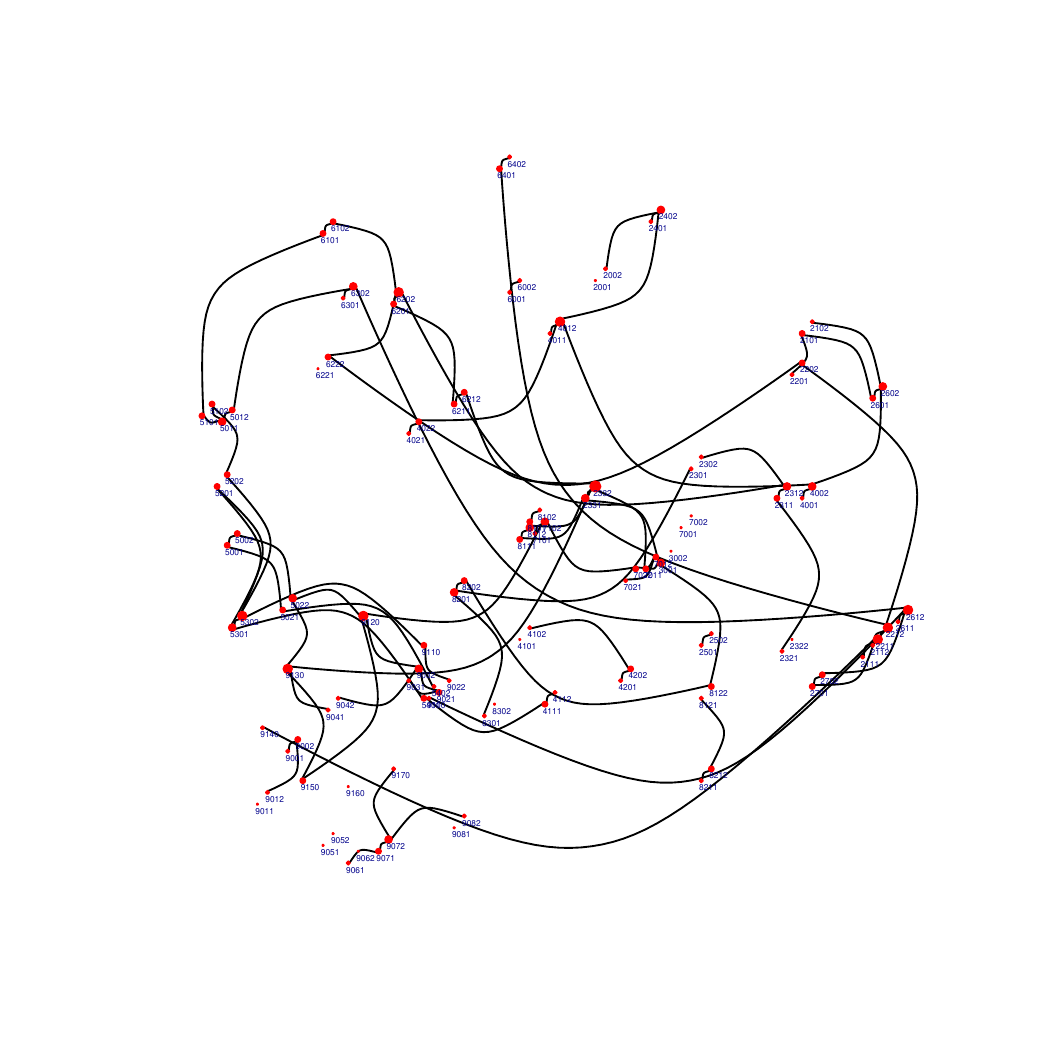}
        \caption{Control group}
        \label{fig:ABIDE.SCV.control}
    \end{subfigure}
    
    \caption{Brain connectome graph of autism and control groups using FPCA-$g_X$ method, obtained by SCV process}
    \label{fig:ABIDE.SCV}
\end{figure}

\begin{figure}[t]
    \centering
    
    \begin{subfigure}[b]{0.48\textwidth}
        \centering
        \includegraphics[width=\textwidth]{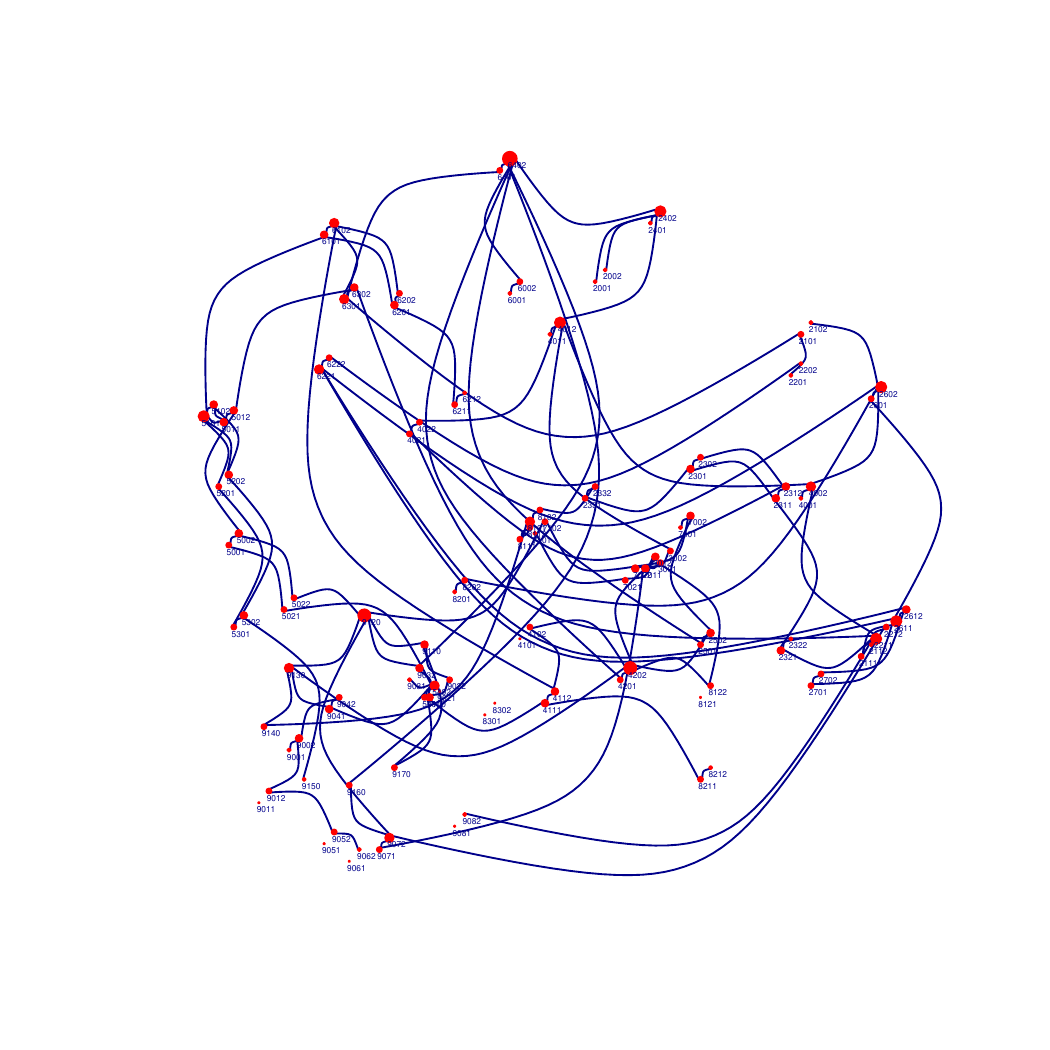}
        \caption{Autism group}
        \label{fig:ABIDE.two.pct.autism}
    \end{subfigure}
    \hfill
    \begin{subfigure}[b]{0.48\textwidth}
        \centering
        \includegraphics[width=\textwidth]{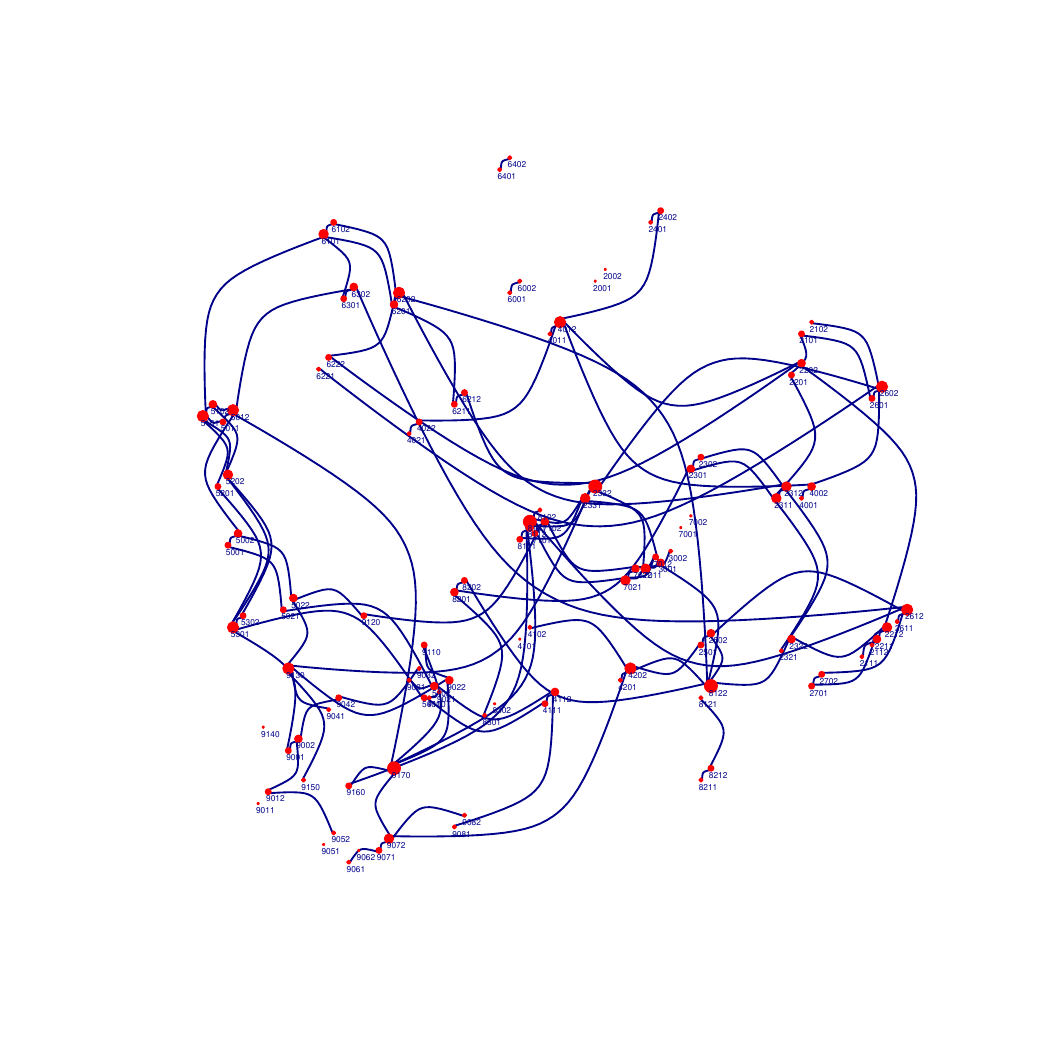}
        \caption{Control group}
        \label{fig:ABIDE.two.pct.control}
    \end{subfigure}
    
    \caption{Brain connectome graph of autism and control groups using FPCA-$g_X$ method, with sparsity fixed to 2\%}
    \label{fig:ABIDE.two.pct}
\end{figure}

Autism Spectrum Disorder (ASD) is a chronic neurodevelopmental disorder associated with both sensory processing and high-level functional deficits \citep{christensen2018prevalence}. Functional magnetic resonance imaging (fMRI) analysis provides a method for characterizing connectome anomalies in individuals with ASD.

ASD is characterized by a dissociation of a transmodal core, which combines long-distance connections from peripheral networks with primarily short-range connectivity \citep{hong2019atypical}. In contrast to a neurotypical brain, which exhibits distributed functional activation patterns, an autistic brain features more regionally localized connections where selective core activation is less prominent \citep{belmonte2004autism}.

We apply our procedure to data from the Autism Brain Imaging Data Exchange (ABIDE), a consortium that provides previously gathered fMRI data from both autism and control groups \citep{di2014autism}. The selected samples encompass whole-brain fMRI scans from 73 ASD-diagnosed patients ($n_{\text{autism}}=73$) and 98 controls ($n_{\text{control}}=98$).\footnote{The dataset includes fMRI measurements from eight different sites. For consistency, we only used data from New York University.} Given that $p=116$, this dataset is high-dimensional. We use the time series, preprocessed by \citet{craddock2013neuro} using AAL parcellation, derived from the raw data.

{\newTextTwo The interpretation of the results naturally depend on the sparsity level, and the network sparsity level should be treated as a tuning parameter, which may be chosen by either domain knowledge or a data-driven approach.}
We initially estimate the CI graphs of both autism and control groups using the SCV procedure separately, as depicted in Figure \ref{fig:ABIDE.SCV}. Comparison of the connectivity graphs of autism and control groups reveals an overall reduction in connectivity across different brain centers in the autism group, aligning with prior findings of cortical underconnection in ASD \citep{maximo2014implications}. Notably, the orbitofrontal regions (nodes 2111, 2112, 2211, 2212, 2321, 2322, 2611, 2612) appear almost isolated from the rest of the brain. This finding suggests that the orbito-frontal region, a typical paralimbic area according to \citet{mesulam2012evolving}, is less connected to limbic areas like the amygdala and hippocampus. This result is consistent with previous findings of diminished activity in the hypothalamus, leading to decreased oxytocin and vasopressin synthesis and release, which may contribute to impaired social cognition and behavior in ASD \citep{caria2020morphofunctional}.

Additionally, we estimate the CI graphs of both groups under a fixed 2\% sparsity {\newTextTwo following the same approach as previous analyses~\citep{Qiao2015Functional,Li2018nonparametric} where the authors also set the network density to a small fixed level.}
{\newTextTwo We choose 2\% because we observe that further increasing the sparsity level will induce substantially more suspicious connections in the estimated networks for both Autism and Control groups.}
The results are provided in Figure \ref{fig:ABIDE.two.pct}. One notable observation is the reduced rich-club connection \footnote{In neuroscience literature, the brain connectome structure where connections are centered around certain hub nodes is called rich-club.} in the autism group. Figure \ref{fig:ABIDE.two.pct} shows a less hierarchical brain connectome in the autism group compared to the control group. The control group exhibits more centralized connections and fewer regions without connections, while the autism group displays a more evenly distributed connection pattern across all nodes. For the control group, 22 nodes have at least 4 connections, 13 nodes have at least 5 connections, and 4 nodes have at least 6 connections. In comparison, for the autism group, 17 nodes have at least 4 connections, 9 nodes have at least 5 connections, and 3 nodes have at least 6 connections. Given that the total number of edges in both groups is identical, the standard deviation of the degree of all nodes is 1.52 for the control group and 1.37 for the autism group. These observations corroborate the results in \citet{hong2019atypical}, suggesting that ASD is associated with selective disruption in long-range connectivity, coupled with a deficit in fully activating the ``rich-club.'' Our findings also align with previous fMRI studies showing that individuals with ASD exhibit more spatially diffuse activations in the cerebellum's motor-related regions \citep{allen2004cerebellar}.

Another notable observation from both Figures \ref{fig:ABIDE.SCV} and \ref{fig:ABIDE.two.pct} is that the autism group displays increased connectivity in the precentral (Nodes 2001, 2002), postcentral (Nodes 6001, 6002), and paracentral (Nodes 6401, 6402) regions. This observation aligns with reports by \citet{patriquin2016neuroanatomical}. These regions are critical components of the motor control network, and abnormal activities within these areas could potentially be associated with ASD \citep{nebel2014precentral}.
}

\subsection*{ADHD Dataset}

{\newText

\begin{figure}[t]
    \centering
    
    \begin{subfigure}[b]{0.48\textwidth}
        \centering
        \includegraphics[width=\textwidth]{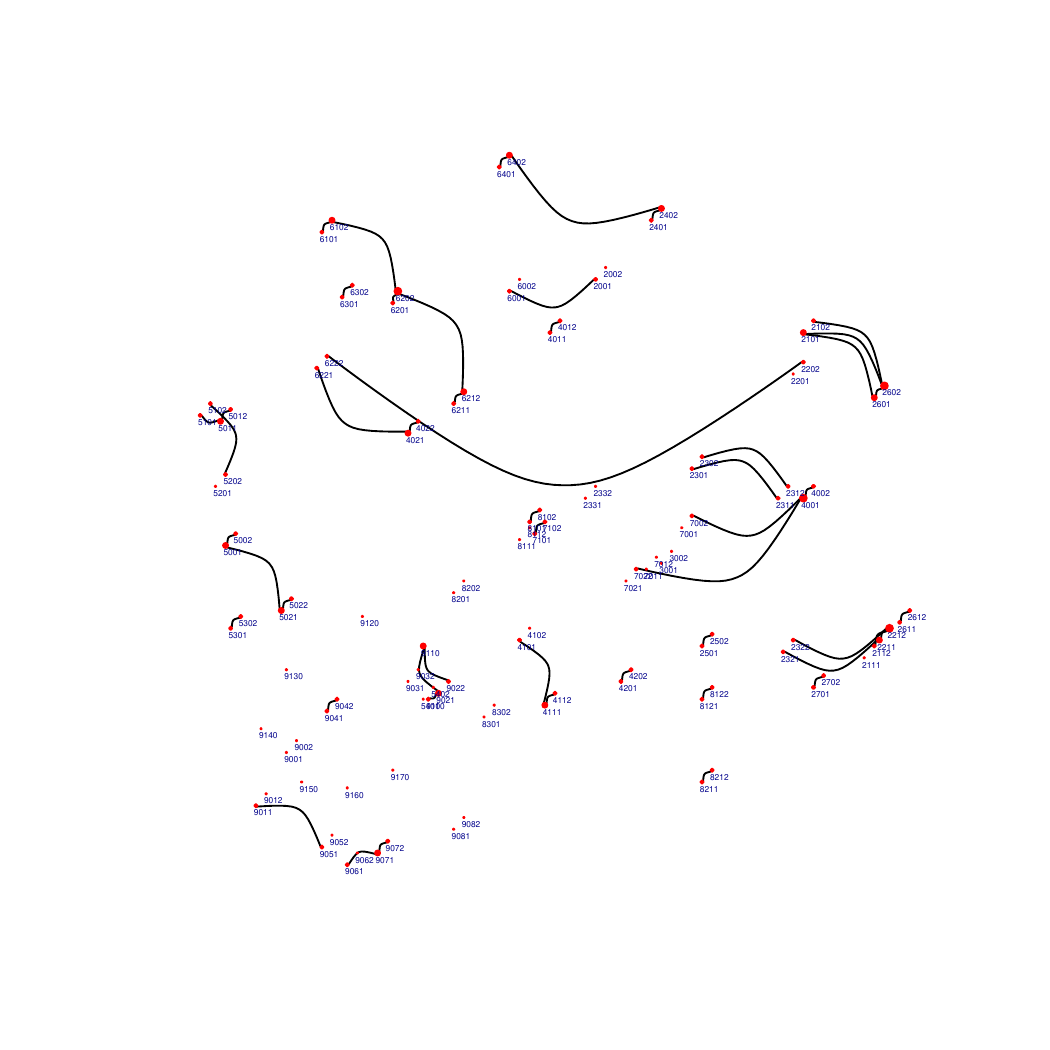}
        \caption{ADHD group}
        \label{fig:ADHD.SCV.ADHD}
    \end{subfigure}
    \hfill
    \begin{subfigure}[b]{0.48\textwidth}
        \centering
        \includegraphics[width=\textwidth]{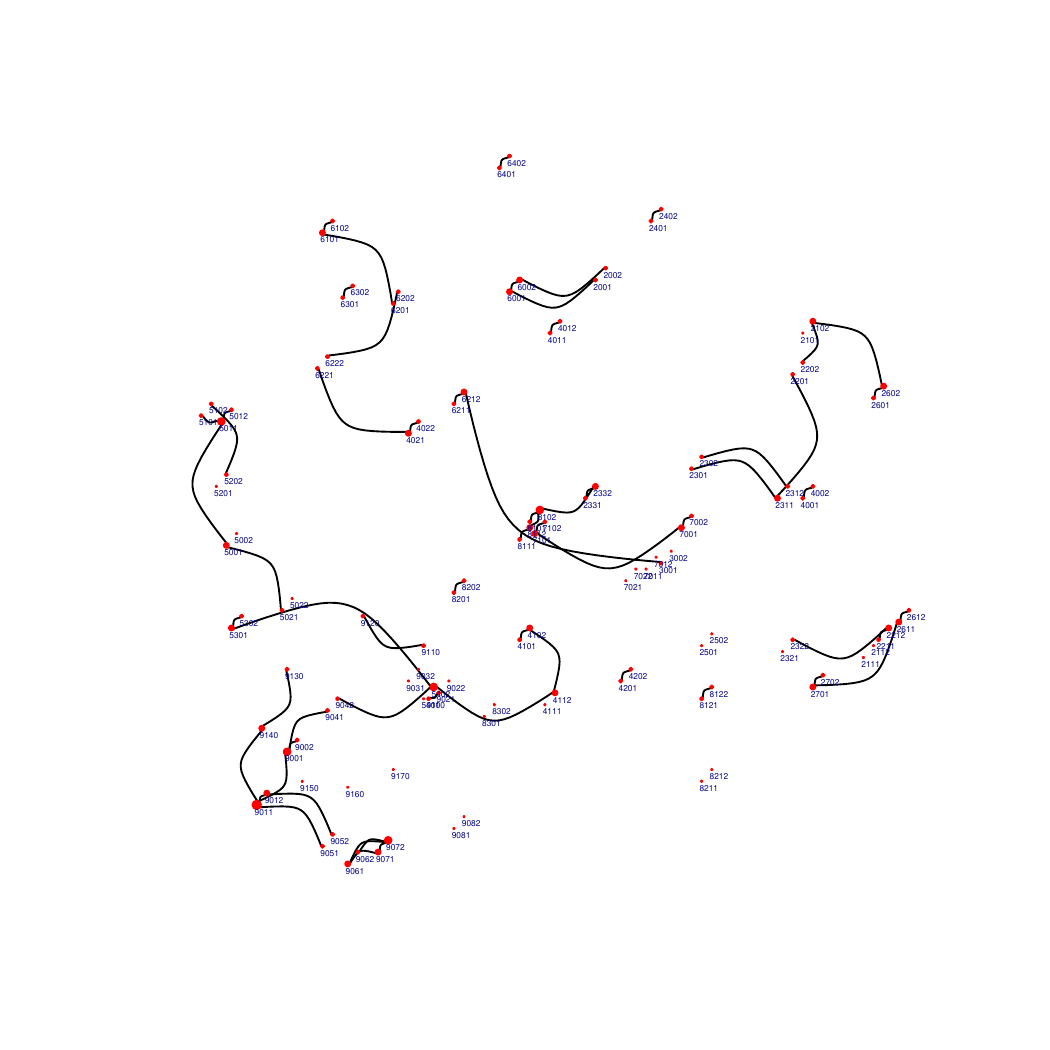}
        \caption{Control group}
        \label{fig:ADHD.SCV.control}
    \end{subfigure}
    
    \caption{Brain connectome graph of ADHD and control groups using FPCA-$g_X$ method, obtained by SCV process}
    \label{fig:ADHD.SCV}
\end{figure}

\begin{figure}[t]
    \centering
    
    \begin{subfigure}[b]{0.48\textwidth}
        \centering
        \includegraphics[width=\textwidth]{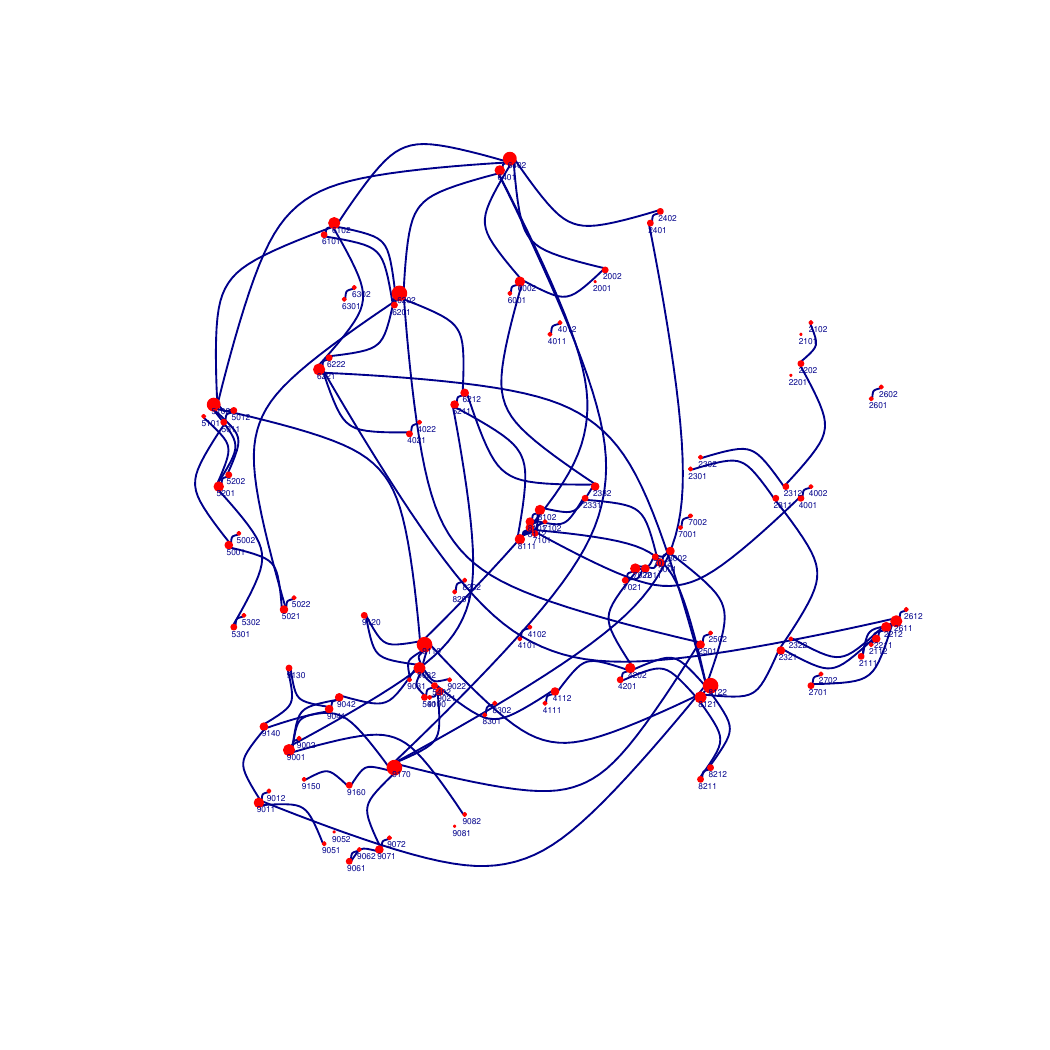}
        \caption{ADHD group}
        \label{fig:ADHD.two.pct.ADHD}
    \end{subfigure}
    \hfill
    \begin{subfigure}[b]{0.48\textwidth}
        \centering
        \includegraphics[width=\textwidth]{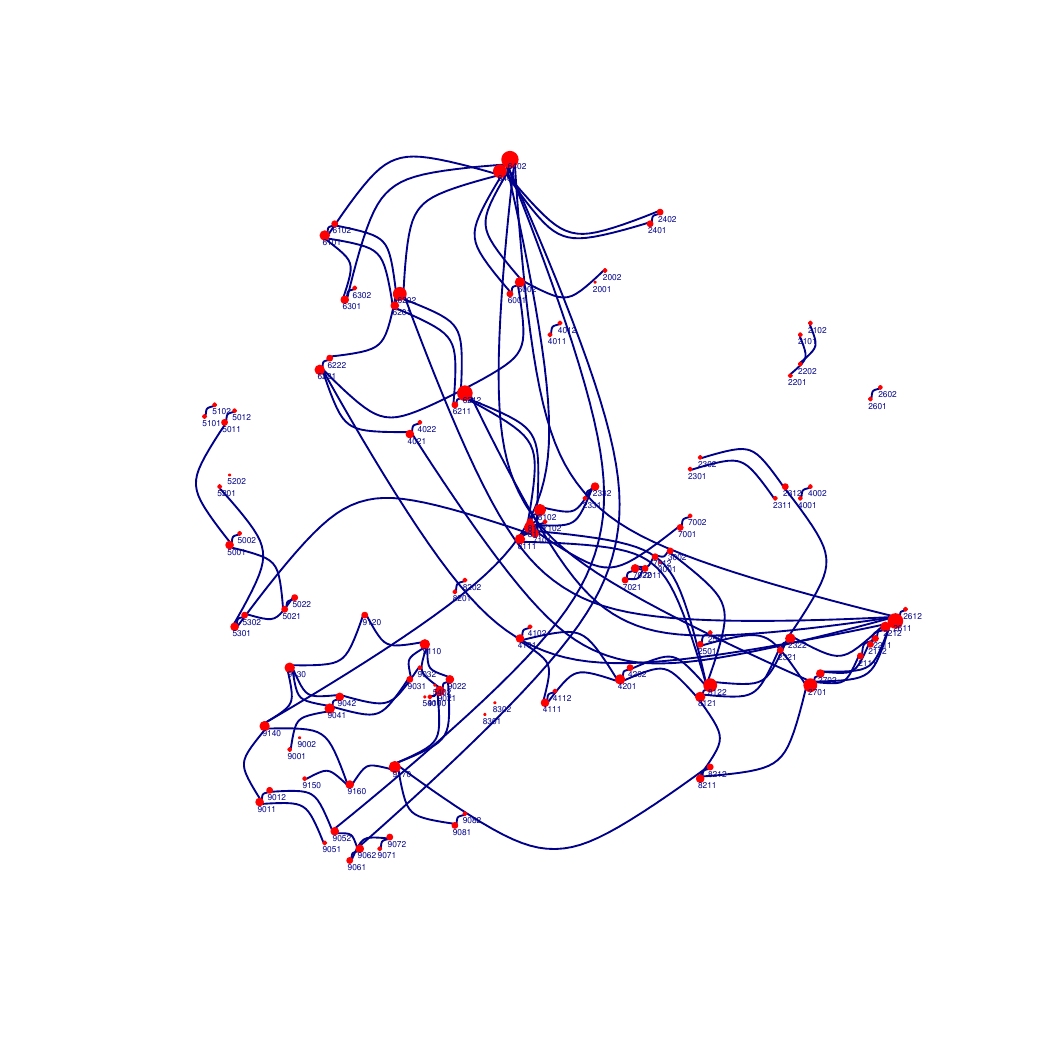}
        \caption{Control group}
        \label{fig:ADHD.two.pct.control}
    \end{subfigure}
    
    \caption{Brain connectome graph of ADHD and control groups using FPCA-$g_X$ method, with sparsity fixed to 2\%}
    \label{fig:ADHD.two.pct}
\end{figure}

\begin{figure}[t]
    \centering
    
    \begin{subfigure}[b]{0.48\textwidth}
        \centering
        \includegraphics[width=\textwidth]{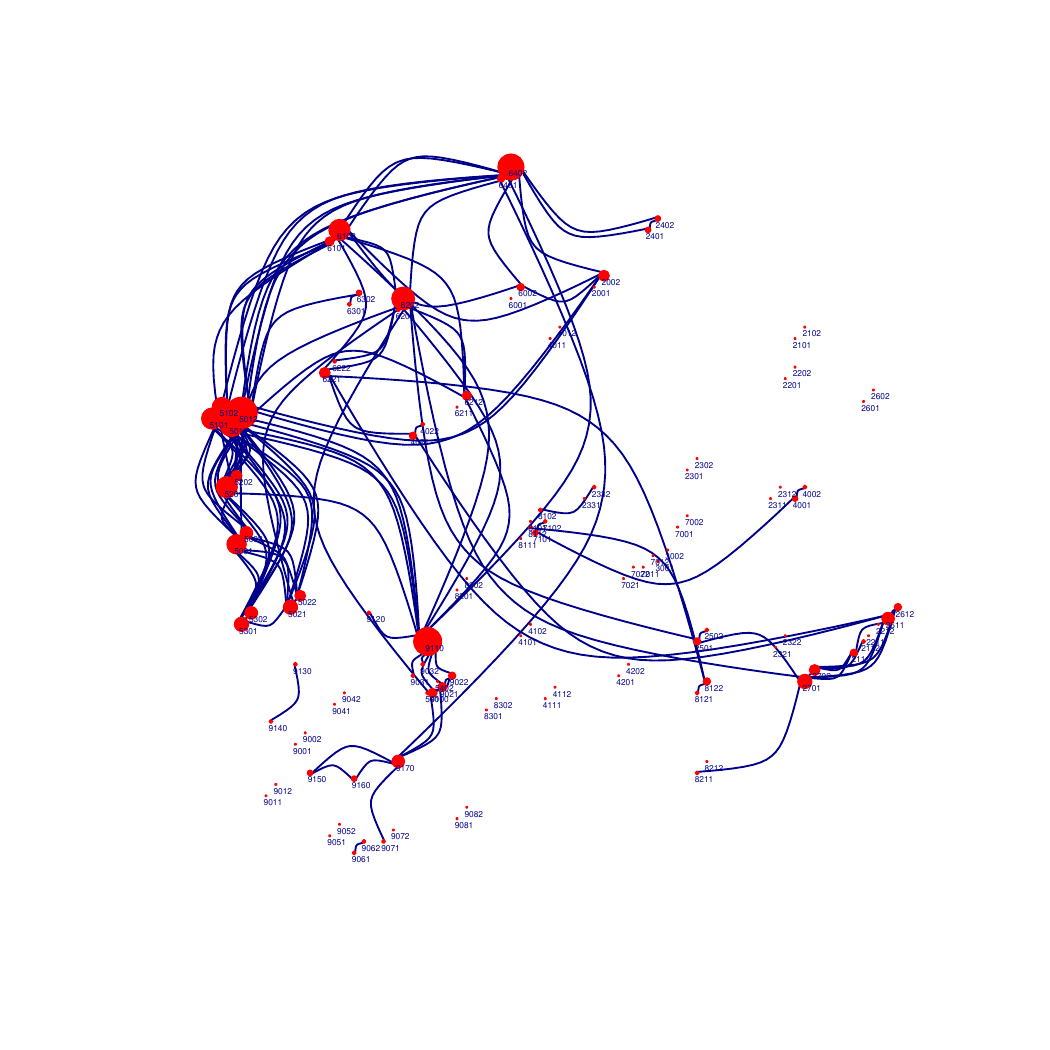}
        \caption{ADHD group}
        \label{fig:ADHD.FGLasso.ADHD}
    \end{subfigure}
    \hfill
    \begin{subfigure}[b]{0.48\textwidth}
        \centering
        \includegraphics[width=\textwidth]{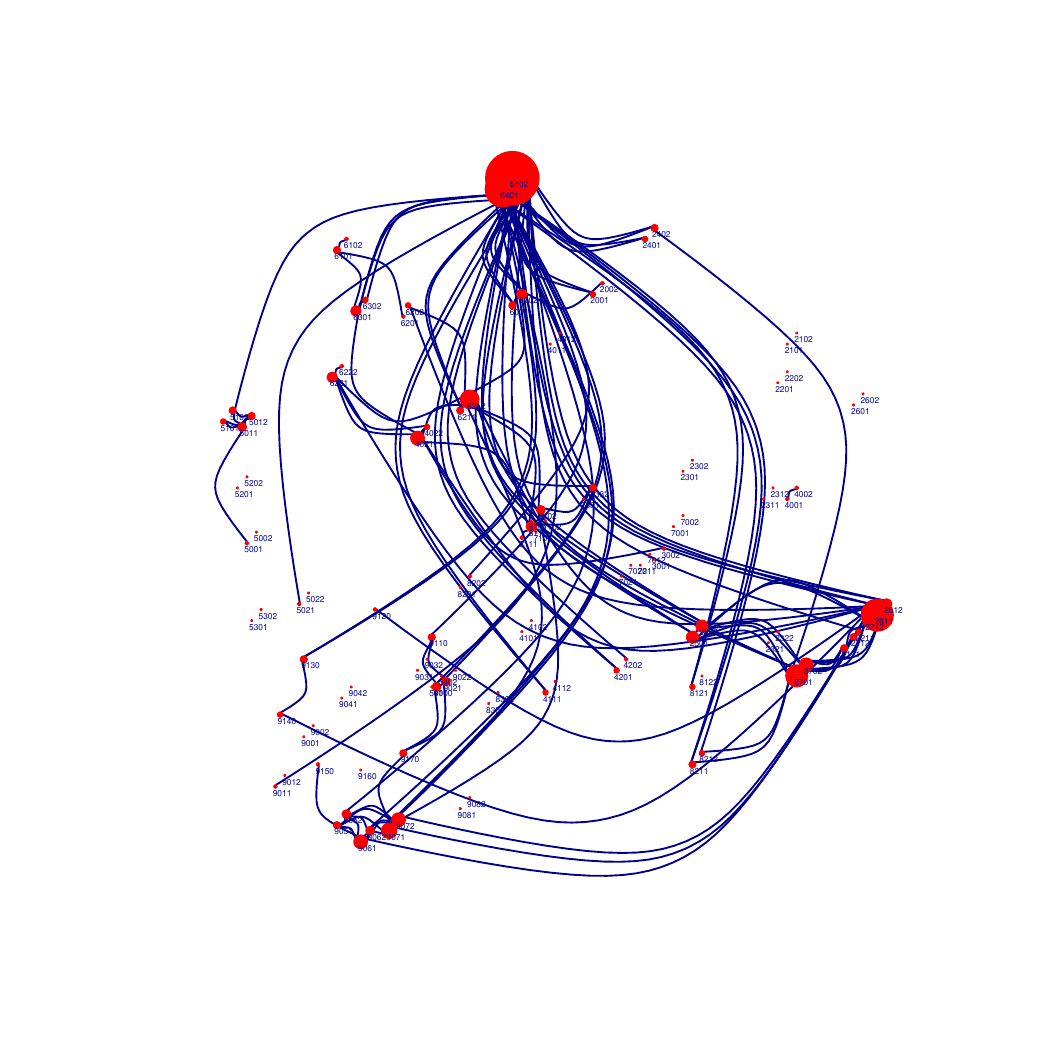}
        \caption{Control group}
        \label{fig:ADHD.FGLasso.control}
    \end{subfigure}
    
    \caption{Brain connectome graph of ADHD and control groups using FGLasso method~\citep{Qiao2015Functional}, with sparsity fixed to 2\%}
    \label{fig:ADHD.FGLasso}
\end{figure}

Attention Deficit Hyperactivity Disorder (ADHD) is a mental health disorder characterized by persistent issues such as difficulty maintaining attention, hyperactivity, and impulsive behavior. Functional graphical modeling may be instrumental in identifying abnormal brain connectivity associated with this condition.

We apply our procedure to data from the ADHD-200 Consortium \citep{milham2012adhd}. The samples used in our analysis include whole-brain fMRI scans from 74 ADHD-diagnosed patients ($n_{\text{ADHD}}=74$) and 109 controls ($n_{\text{control}}=109$)\footnote{The dataset includes fMRI measurements from eight different sites. For consistency, we solely utilized data from Peking University that passed quality tests.}. This dataset is high-dimensional, as neither sample size exceeds $p=116$. The time series preprocessed by \citet{bellec2017neuro} using AAL parcellation from the raw data is used in our study.

We initially estimate the CI graphs of both the ADHD and control groups using the SCV procedure separately, as demonstrated in Figure \ref{fig:ADHD.SCV}. We observe significantly reduced brain connectivity in the ADHD group across the entire brain network. The connectivity graph of the ADHD group in Figure \ref{fig:ADHD.SCV.ADHD} contains 51 edges, while the control group in Figure \ref{fig:ADHD.SCV.control} has 62 edges. This observation aligns with the findings in \citet{wang2020disrupted} suggesting decreased homotopic, intrahemispheric, and heterotopic functional connectivity (i.e., disconnection) within the ADHD group. Specifically, a weaker connection is apparent within the cerebellum regions (nodes on the bottom left of Figures \ref{fig:ADHD.SCV.ADHD} and \ref{fig:ADHD.SCV.control} with labels beginning with "90") in the ADHD group. This observation is consistent with the conclusion in \citet{cao2013probabilistic} stating that individuals with ADHD exhibit altered connectivity in cerebellum circuits, which are linked to timing disorders.

Furthermore, we estimate the CI graphs of both groups under a fixed 2\% sparsity.
{\newTextTwo Similar to the analysis of ASD dataset, we choose 2\% because we find that further increasing the sparsity level will induce substantially more suspicious connections in the estimated network for both ADHD and Control groups.}
The results are shown in Figure \ref{fig:ADHD.two.pct}. The connectivity graph of the control group in Figure \ref{fig:ADHD.two.pct.control} features several highly centralized areas, for instance, the paracentral lobule (Nodes 6401 and 6402 at the top) and prefrontal regions (e.g., Nodes 2111, 2112, 2211, 2212, 2321, 2322, 2611, 2612 on the right). The connectivity to these rich-club nodes is markedly reduced in the ADHD group. These rich-club connections are theorized to play a central role in integrating information among different brain subsystems. ADHD may be characterized by diminished structural integrity of the rich-club backbone, potentially leading to a decrease in globally efficient communication capacity and altered functional brain dynamics \citep{wang2021rich}. Specifically, the diminished prefrontal activities in ADHD have been pinpointed by neuroscientific studies \citep{konrad2006dysfunctional, cao2013probabilistic}. Deficits in these regions have been associated with impairments in cognitive functions and the capacity to adapt behavior to changing circumstances flexibly \citep{bu2021structural}. Differences between the ADHD and control group are also identified in non-rich-club regions. The connections stemming from the precuneus regions (Nodes 6301 and 6302 atthe top left of the graph) are markedly reduced in the ADHD group---their connections to the inferior parietal (Nodes 6101, 6102) and paracentral regions are no longer detected. The precuneus is linked to functional disturbances in regulatory control, attention, and aspects of executive function. Our observation aligns with the findings of \citet{noordermeer2017structural}, which underscore connectivity abnormalities in the precuneus among ADHD patients.
}

{\newText
We have also applied the FGLasso method by \citet{Qiao2015Functional} to the ADHD dataset, adjusting the connection sparsity to 2\% by tuning the penalty parameter. The resulting connectome graph can be seen in Figure~\ref{fig:ADHD.FGLasso}. When compared with Figure~\ref{fig:ADHD.two.pct.control}, it's noticeable that FGLasso tends to generate more rich-club results. This observation aligns with simulation results wherein FGLasso exhibits relatively good performance when the underlying model features a rich-club connection structure (Model C). However, even when the underlying model has minimal rich-club structure (e.g., Models A, B, and D), FGLasso still tends to impose a rich-club structure, leading to subpar performance. As a consequence, while the FGLasso method is effective in identifying the most active regions in the connectome, it may result in a biased conclusion if such a rich-club structure is not present in reality.

For instance, according to Figure~\ref{fig:ADHD.FGLasso}, the visual cortex region of the ADHD patients (Nodes 5001, 5002, 5011, 5012, 5021, 5022, 5101, 5102, 5201, 5202, 5301, 5302) appears to be densely connected to other brain regions. However, considering that the fMRI dataset we use is gathered during a resting state and that the visual cortex is primarily dedicated to visual functions~\citep{grill2004human}, such a connection pattern within the ADHD group is unexpected. Furthermore, within the ADHD group, Node 9110, part of the cerebellum, appears densely connected to many regions of the cerebrum. Modern neuroscience, however, posits that the cerebellum and cerebrum serve relatively independent functions \citep{glickstein2009cerebellum}, which suggests that such extensive connections are unlikely to occur. In contrast, our method tends to yield a graph in which the node degrees are more evenly distributed, thereby offering a more balanced and potentially accurate representation.

}

\section{Conclusion}
\label{sec:conclusion}

We propose a neighborhood selection method for estimating the structure of a functional graphical model and show that it can consistently recover the conditional independence graph in the high-dimensional setting. Specifically, we pose the problem of graph selection as a series of function-on-function regressions, and we approximate the function-on-function regressions with a vector-on-vector regression approach that is achieved by functional dimension reduction. Through extensive simulations, we demonstrate that the proposed method outperforms existing approaches in a variety of settings. Finally, we apply our method on fMRI data sets that include patients with ASD and patients with ADHD, as well as corresponding control groups.
We estimate the connectivity pattern between brain regions and find results that agree with previous neuroscience research.

A key step in our method is the choice of the basis for dimension reduction. Although we suggest using the FPCA basis for most settings, our methodology allows an arbitrary orthonormal basis. We also provide a theoretically motivated procedure for choosing a particular basis. However, developing a more rigorous data-driven approach is still an open problem that we hope to study in the future. Another fruitful avenue for future work is the development of methods that allow for inference and hypothesis testing in functional graphs. For example, \citet{Kim2019Two} has developed inferential tools for high-dimensional Markov networks, and future work may extend their results to the functional graph setting.

\section*{Acknowledgements}

We would like to thank Zhaohan Wu from Florida State University for his suggestions on fMRI data analysis.
This work was completed in part with resources provided by the University of Chicago Booth Mercury Computing Cluster.
The research of MK is supported in part by NSF Grant ECCS-2216912.

\FloatBarrier

\clearpage

\newpage

\bibliography{boxinz-papers}

\begin{thebibliography}{}

\bibitem[Adler and Taylor, 2007]{MR2319516}
Adler, R.~J. and Taylor, J.~E. (2007).
\newblock {\em Random Fields and Geometry}.
\newblock Springer, New York.

\bibitem[Allen et~al., 2004]{allen2004cerebellar}
Allen, G., M{\"u}ller, R.-A., and Courchesne, E. (2004).
\newblock Cerebellar function in autism: functional magnetic resonance image activation during a simple motor task.
\newblock {\em Biological psychiatry}, 56(4):269--278.

\bibitem[Bellec et~al., 2017]{bellec2017neuro}
Bellec, P., Chu, C., Chouinard-Decorte, F., Benhajali, Y., Margulies, D.~S., and Craddock, R.~C. (2017).
\newblock The neuro bureau adhd-200 preprocessed repository.
\newblock {\em Neuroimage}, 144:275--286.

\bibitem[Belmonte et~al., 2004]{belmonte2004autism}
Belmonte, M.~K., Allen, G., Beckel-Mitchener, A., Boulanger, L.~M., Carper, R.~A., and Webb, S.~J. (2004).
\newblock Autism and abnormal development of brain connectivity.
\newblock {\em Journal of Neuroscience}, 24(42):9228--9231.

\bibitem[Besag, 1975]{Besag1975Statistical}
Besag, J. (1975).
\newblock Statistical analysis of non-lattice data.
\newblock {\em Journal of the Royal Statistical Society. Series D}, 24(3):179--195.

\bibitem[Bosq, 2000]{Bosq2000Linear}
Bosq, D. (2000).
\newblock {\em Linear Processes in Function Spaces}.
\newblock Springer-Verlag, New York.

\bibitem[Boyd et~al., 2011]{Boyd2011Distributed}
Boyd, S.~P., Parikh, N., Chu, E., Peleato, B., and Eckstein, J. (2011).
\newblock Distributed optimization and statistical learning via the alternating direction method of multipliers.
\newblock {\em Foundations and Trends in Machine Learning}, 3(1):1--122.

\bibitem[Bu et~al., 2021]{bu2021structural}
Bu, X., Cao, M., Huang, X., and He, Y. (2021).
\newblock The structural connectome in adhd.
\newblock {\em Psychoradiology}, 1(4):257--271.

\bibitem[Bubeck, 2015]{Bubeck2015Convex}
Bubeck, S. (2015).
\newblock Convex optimization: Algorithms and complexity.
\newblock {\em Foundations and Trends in Machine Learning}, 8(3-4):231--357.

\bibitem[Cai et~al., 2011]{Cai2011Constrained}
Cai, T., Liu, W., and Luo, X. (2011).
\newblock A constrained l1 minimization approach to sparse precision matrix estimation.
\newblock {\em Journal of the American Statistical Association}, 106(494):594--607.

\bibitem[Cao et~al., 2013]{cao2013probabilistic}
Cao, Q., Shu, N., An, L., Wang, P., Sun, L., Xia, M.-R., Wang, J.-H., Gong, G.-L., Zang, Y.-F., Wang, Y.-F., et~al. (2013).
\newblock Probabilistic diffusion tractography and graph theory analysis reveal abnormal white matter structural connectivity networks in drug-naive boys with attention deficit/hyperactivity disorder.
\newblock {\em Journal of Neuroscience}, 33(26):10676--10687.

\bibitem[Caria et~al., 2020]{caria2020morphofunctional}
Caria, A., Ciringione, L., and de~Falco, S. (2020).
\newblock Morphofunctional alterations of the hypothalamus and social behavior in autism spectrum disorders.
\newblock {\em Brain Sciences}, 10(7):435.

\bibitem[Chiou and M{\"u}ller, 2016]{chiou2016pairwise}
Chiou, J.-M. and M{\"u}ller, H.-G. (2016).
\newblock A pairwise interaction model for multivariate functional and longitudinal data.
\newblock {\em Biometrika}, 103(2):377--396.

\bibitem[Christensen et~al., 2018]{christensen2018prevalence}
Christensen, D.~L., Braun, K. V.~N., Baio, J., Bilder, D., Charles, J., Constantino, J.~N., Daniels, J., Durkin, M.~S., Fitzgerald, R.~T., Kurzius-Spencer, M., et~al. (2018).
\newblock Prevalence and characteristics of autism spectrum disorder among children aged 8 years—autism and developmental disabilities monitoring network, 11 sites, united states, 2012.
\newblock {\em MMWR Surveillance Summaries}, 65(13):1.

\bibitem[Craddock et~al., 2013]{craddock2013neuro}
Craddock, C., Benhajali, Y., Chu, C., Chouinard, F., Evans, A., Jakab, A., Khundrakpam, B.~S., Lewis, J.~D., Li, Q., Milham, M., et~al. (2013).
\newblock The neuro bureau preprocessing initiative: open sharing of preprocessed neuroimaging data and derivatives.
\newblock {\em Frontiers in Neuroinformatics}, 7:27.

\bibitem[Di~Martino et~al., 2014]{di2014autism}
Di~Martino, A., Yan, C.-G., Li, Q., Denio, E., Castellanos, F.~X., Alaerts, K., Anderson, J.~S., Assaf, M., Bookheimer, S.~Y., Dapretto, M., et~al. (2014).
\newblock The autism brain imaging data exchange: towards a large-scale evaluation of the intrinsic brain architecture in autism.
\newblock {\em Molecular psychiatry}, 19(6):659--667.

\bibitem[Gabay and Mercier, 1976]{gabay1976dual}
Gabay, D. and Mercier, B. (1976).
\newblock A dual algorithm for the solution of nonlinear variational problems via finite element approximation.
\newblock {\em Computers \& mathematics with applications}, 2(1):17--40.

\bibitem[Glickstein et~al., 2009]{glickstein2009cerebellum}
Glickstein, M., Strata, P., and Voogd, J. (2009).
\newblock Cerebellum: history.
\newblock {\em Neuroscience}, 162(3):549--559.

\bibitem[Grill-Spector and Malach, 2004]{grill2004human}
Grill-Spector, K. and Malach, R. (2004).
\newblock The human visual cortex.
\newblock {\em Annual Review of Neuroscience}, 27:649--677.

\bibitem[He et~al., 2000]{he2000alternating}
He, B., Yang, H., and Wang, S. (2000).
\newblock Alternating direction method with self-adaptive penalty parameters for monotone variational inequalities.
\newblock {\em Journal of Optimization Theory and applications}, 106(2):337--356.

\bibitem[Hong et~al., 2019]{hong2019atypical}
Hong, S.-J., Vos~de Wael, R., Bethlehem, R.~A., Lariviere, S., Paquola, C., Valk, S.~L., Milham, M.~P., Di~Martino, A., Margulies, D.~S., Smallwood, J., et~al. (2019).
\newblock Atypical functional connectome hierarchy in autism.
\newblock {\em Nature communications}, 10(1):1022.

\bibitem[Hsing and Eubank, 2015]{Hsing2015Theoretical}
Hsing, T. and Eubank, R. (2015).
\newblock {\em Theoretical foundations of functional data analysis, with an introduction to linear operators}.
\newblock John Wiley \& Sons, Ltd., Chichester.

\bibitem[Ivanescu et~al., 2015]{Ivanescu2015Penalized}
Ivanescu, A.~E., Staicu, A.-M., Scheipl, F., and Greven, S. (2015).
\newblock Penalized function-on-function regression.
\newblock {\em Computational Statistics}, 30(2):539--568.

\bibitem[Janson, 1997]{janson_1997}
Janson, S. (1997).
\newblock {\em Gaussian Hilbert Spaces}.
\newblock Cambridge University Press.

\bibitem[Kallenberg, 1997]{kallenberg1997foundations}
Kallenberg, O. (1997).
\newblock {\em Foundations of modern probability}.
\newblock Probability and its Applications. Springer-Verlag, New York.

\bibitem[Kim et~al., 2021]{Kim2019Two}
Kim, B., Liu, S., and Kolar, M. (2021).
\newblock Two-sample inference for high-dimensional {M}arkov networks.
\newblock {\em Journal of the Royal Statistical Society. Series B. Statistical Methodology}, 83(5):939--962.

\bibitem[Klebanov et~al., 2021]{klebanov2021linear}
Klebanov, I., Sprungk, B., and Sullivan, T.~J. (2021).
\newblock The linear conditional expectation in {H}ilbert space.
\newblock {\em Bernoulli}, 27(4):2267--2299.

\bibitem[Kolar et~al., 2013]{kolar13multiatticml}
Kolar, M., Liu, H., and Xing, E.~P. (2013).
\newblock Markov network estimation from multi-attribute data.
\newblock In {\em International Conference on Machine Learning, {ICML}}.

\bibitem[Kolar et~al., 2014]{Kolar2014Graph}
Kolar, M., Liu, H., and Xing, E.~P. (2014).
\newblock Graph estimation from multi-attribute data.
\newblock {\em Journal of Machine Learning Research}, 15(1):1713--1750.

\bibitem[Konrad et~al., 2006]{konrad2006dysfunctional}
Konrad, K., Neufang, S., Hanisch, C., Fink, G.~R., and Herpertz-Dahlmann, B. (2006).
\newblock Dysfunctional attentional networks in children with attention deficit/hyperactivity disorder: evidence from an event-related functional magnetic resonance imaging study.
\newblock {\em Biological psychiatry}, 59(7):643--651.

\bibitem[Kuchibhotla and Chakrabortty, 2022]{kuchibhotla2018moving}
Kuchibhotla, A.~K. and Chakrabortty, A. (2022).
\newblock Moving beyond sub-{G}aussianity in high-dimensional statistics: applications in covariance estimation and linear regression.
\newblock {\em Information and Inference}, 11(4):1389--1456.

\bibitem[Lauritzen, 1996]{Lauritzen1996Graphical}
Lauritzen, S. (1996).
\newblock {\em Graphical Models}.
\newblock Oxford Statistical Science Series. Clarendon Press.

\bibitem[Lee et~al., 2016a]{Lee2016additive}
Lee, K.-Y., Li, B., and Zhao, H. (2016a).
\newblock On an additive partial correlation operator and nonparametric estimation of graphical models.
\newblock {\em Biometrika}, 103(3):513--530.

\bibitem[Lee et~al., 2016b]{lee2016variable}
Lee, K.-Y., Li, B., and Zhao, H. (2016b).
\newblock Variable selection via additive conditional independence.
\newblock {\em Journal of the Royal Statistical Society: Series B}, 78(5):1037--1055.

\bibitem[Lee et~al., 2022]{lee2022nonparametric}
Lee, K.-Y., Li, L., Li, B., and Zhao, H. (2022).
\newblock Nonparametric functional graphical modeling through functional additive regression operator.
\newblock {\em Journal of the American Statistical Association}, 0(0):1--15.

\bibitem[Li et~al., 2014]{Li2014additive}
Li, B., Chun, H., and Zhao, H. (2014).
\newblock On an additive semigraphoid model for statistical networks with application to pathway analysis.
\newblock {\em Journal of the American Statistical Association}, 109(507):1188--1204.

\bibitem[Li and Solea, 2018]{Li2018nonparametric}
Li, B. and Solea, E. (2018).
\newblock A nonparametric graphical model for functional data with application to brain networks based on f{MRI}.
\newblock {\em Journal of the American Statistical Association}, 113(524):1637--1655.

\bibitem[Luo and Qi, 2017]{Luo2017Function}
Luo, R. and Qi, X. (2017).
\newblock Function-on-function linear regression by signal compression.
\newblock {\em Journal of the American Statistical Association}, 112(518):690--705.

\bibitem[Luo et~al., 2016]{Luo2016Functional}
Luo, R., Qi, X., and Wang, Y. (2016).
\newblock Functional wavelet regression for linear function-on-function models.
\newblock {\em Electronic Journal of Statistics}, 10(2):3179--3216.

\bibitem[Lynch and Chen, 2018]{lynch2018test}
Lynch, B. and Chen, K. (2018).
\newblock A test of weak separability for multi-way functional data, with application to brain connectivity studies.
\newblock {\em Biometrika}, 105(4):815--831.

\bibitem[Maximo et~al., 2014]{maximo2014implications}
Maximo, J.~O., Cadena, E.~J., and Kana, R.~K. (2014).
\newblock The implications of brain connectivity in the neuropsychology of autism.
\newblock {\em Neuropsychology review}, 24:16--31.

\bibitem[Meinshausen and B\"{u}hlmann, 2006]{Meinshausen2006High}
Meinshausen, N. and B\"{u}hlmann, P. (2006).
\newblock High dimensional graphs and variable selection with the lasso.
\newblock {\em Annals of Statistics}, 34(3):1436--1462.

\bibitem[Meinshausen and Yu, 2009]{meinshausen08lasso}
Meinshausen, N. and Yu, B. (2009).
\newblock Lasso-type recovery of sparse representations for high-dimensional data.
\newblock {\em Annals of Statistics}, 37(1):246--270.

\bibitem[Mesulam, 2012]{mesulam2012evolving}
Mesulam, M. (2012).
\newblock The evolving landscape of human cortical connectivity: facts and inferences.
\newblock {\em Neuroimage}, 62(4):2182--2189.

\bibitem[Milham et~al., 2012]{milham2012adhd}
Milham, M.~P., Fair, D., Mennes, M., Mostofsky, S.~H., et~al. (2012).
\newblock The adhd-200 consortium: a model to advance the translational potential of neuroimaging in clinical neuroscience.
\newblock {\em Frontiers in systems neuroscience}, 6:62.

\bibitem[Nebel et~al., 2014]{nebel2014precentral}
Nebel, M.~B., Eloyan, A., Barber, A.~D., and Mostofsky, S.~H. (2014).
\newblock Precentral gyrus functional connectivity signatures of autism.
\newblock {\em Frontiers in systems neuroscience}, 8:80.

\bibitem[Negahban et~al., 2012]{negahban2010unified}
Negahban, S.~N., Ravikumar, P., Wainwright, M.~J., and Yu, B. (2012).
\newblock A unified framework for high-dimensional analysis of $ m $-estimators with decomposable regularizers.
\newblock {\em Statistical Science}, 27(4):538--557.

\bibitem[Noordermeer et~al., 2017]{noordermeer2017structural}
Noordermeer, S.~D., Luman, M., Greven, C.~U., Veroude, K., Faraone, S.~V., Hartman, C.~A., Hoekstra, P.~J., Franke, B., Buitelaar, J.~K., Heslenfeld, D.~J., et~al. (2017).
\newblock Structural brain abnormalities of attention-deficit/hyperactivity disorder with oppositional defiant disorder.
\newblock {\em Biological Psychiatry}, 82(9):642--650.

\bibitem[Patriquin et~al., 2016]{patriquin2016neuroanatomical}
Patriquin, M.~A., DeRamus, T., Libero, L.~E., Laird, A., and Kana, R.~K. (2016).
\newblock Neuroanatomical and neurofunctional markers of social cognition in autism spectrum disorder.
\newblock {\em Human brain mapping}, 37(11):3957--3978.

\bibitem[Qi and Luo, 2018]{Qi2018Function}
Qi, X. and Luo, R. (2018).
\newblock Function-on-function regression with thousands of predictive curves.
\newblock {\em Journal of Multivariate Analysis}, 163:51--66.

\bibitem[Qi and Luo, 2019]{Xin2019Nonlinear}
Qi, X. and Luo, R. (2019).
\newblock Nonlinear function-on-function additive model with multiple predictor curves.
\newblock {\em Statistica Sinica}, 29(2):719--739.

\bibitem[Qiao et~al., 2019]{Qiao2015Functional}
Qiao, X., Guo, S., and James, G.~M. (2019).
\newblock Functional {G}raphical {M}odels.
\newblock {\em Journal of the American Statistical Association}, 114(525):211--222.

\bibitem[Qiao et~al., 2020]{Qiao2017Doubly}
Qiao, X., Qian, C., James, G.~M., and Guo, S. (2020).
\newblock Doubly functional graphical models in high dimensions.
\newblock {\em Biometrika}, 107(2):415--431.

\bibitem[Ramsay and Silverman, 2005]{Ramsay2005Functional}
Ramsay, J.~O. and Silverman, B.~W. (2005).
\newblock {\em Functional data analysis}.
\newblock Springer, New York, second edition.

\bibitem[Rothman et~al., 2008]{rothman2008sparse}
Rothman, A.~J., Bickel, P.~J., Levina, E., Zhu, J., et~al. (2008).
\newblock Sparse permutation invariant covariance estimation.
\newblock {\em Electronic Journal of Statistics}, 2:494--515.

\bibitem[Scheipl et~al., 2015]{Scheipl2015Functional}
Scheipl, F., Staicu, A.-M., and Greven, S. (2015).
\newblock Functional additive mixed models.
\newblock {\em Journal of Computational and Graphical Statistics}, 24(2):477--501.

\bibitem[She, 2012]{she2012iterative}
She, Y. (2012).
\newblock An iterative algorithm for fitting nonconvex penalized generalized linear models with grouped predictors.
\newblock {\em Computational Statistics \& Data Analysis}, 56(10):2976--2990.

\bibitem[Solea and Dette, 2022]{solea2021nonparametric}
Solea, E. and Dette, H. (2022).
\newblock Nonparametric and high-dimensional functional graphical models.
\newblock {\em Electronic Journal of Statistics}, 16(2):6175--6231.

\bibitem[Solea and Li, 2022]{solea2020copula}
Solea, E. and Li, B. (2022).
\newblock Copula {G}aussian graphical models for functional data.
\newblock {\em Journal of the American Statistical Association}, 117(538):781--793.

\bibitem[Tsai et~al., 2023]{tsai2023latent}
Tsai, K., Zhao, B., Koyejo, S., and Kolar, M. (2023).
\newblock Latent multimodal functional graphical model estimation.
\newblock {\em Journal of the American Statistical Association}, 0(ja):1--25.

\bibitem[Tsay and Pourahmadi, 2017]{tsay2017modelling}
Tsay, R.~S. and Pourahmadi, M. (2017).
\newblock Modelling structured correlation matrices.
\newblock {\em Biometrika}, 104(1):237--242.

\bibitem[Tzourio-Mazoyer et~al., 2002]{Tzourio-Mazoyer2002Automated}
Tzourio-Mazoyer, N., Landeau, B., Papathanassiou, D., Crivello, F., Etard, O., Delcroix, N., Mazoyer, B., and Joliot, M. (2002).
\newblock Automated anatomical labeling of activations in {SPM} using a macroscopic anatomical parcellation of the {MNI} {MRI} single-subject brain.
\newblock {\em {NeuroImage}}, 15(1):273--289.

\bibitem[Wainwright, 2019]{Wainwright2019High}
Wainwright, M.~J. (2019).
\newblock {\em High-Dimensional Statistics}.
\newblock Cambridge University Press.

\bibitem[Wang et~al., 2021]{wang2021rich}
Wang, B., Wang, G., Wang, X., Cao, R., Xiang, J., Yan, T., Li, H., Yoshimura, S., Toichi, M., and Zhao, S. (2021).
\newblock Rich-club analysis in adults with adhd connectomes reveals an abnormal structural core network.
\newblock {\em Journal of Attention Disorders}, 25(8):1068--1079.

\bibitem[Wang et~al., 2020]{wang2020disrupted}
Wang, M., Hu, Z., Liu, L., Li, H., Qian, Q., and Niu, H. (2020).
\newblock Disrupted functional brain connectivity networks in children with attention-deficit/hyperactivity disorder: evidence from resting-state functional near-infrared spectroscopy.
\newblock {\em Neurophotonics}, 7(1):015012--015012.

\bibitem[Wei and Li, 2008]{wei2008hidden}
Wei, Z. and Li, H. (2008).
\newblock A hidden spatial-temporal {M}arkov random field model for network-based analysis of time course gene expression data.
\newblock {\em The Annals of Applied Statistics}, 2(1):408--429.

\bibitem[Zapata et~al., 2022]{zapata2021partial}
Zapata, J., Oh, S.~Y., and Petersen, A. (2022).
\newblock Partial separability and functional graphical models for multivariate {G}aussian processes.
\newblock {\em Biometrika}, 109(3):665--681.

\bibitem[Zhao et~al., 2019]{Zhao2019Direct}
Zhao, B., Wang, Y.~S., and Kolar, M. (2019).
\newblock Direct estimation of differential functional graphical models.
\newblock In {\em Advances in Neural Information Processing Systems, NeurIPS}.

\bibitem[Zhao et~al., 2022]{zhao2022fudge}
Zhao, B., Wang, Y.~S., and Kolar, M. (2022).
\newblock Fudge: A method to estimate a functional differential graph in a high-dimensional setting.
\newblock {\em Journal of Machine Learning Research}, 23(82):1--82.

\bibitem[Zhao and Yu, 2006]{zhao06model}
Zhao, P. and Yu, B. (2006).
\newblock On model selection consistency of lasso.
\newblock {\em Journal of Machine Learning Research}, 7:2541--2563.

\bibitem[Zhu et~al., 2016]{Zhu2016Bayesian}
Zhu, H., Strawn, N., and Dunson, D.~B. (2016).
\newblock Bayesian graphical models for multivariate functional data.
\newblock {\em Journal of Machine Learning Research}, 17(204):1--27.

\end{thebibliography}
\bibliographystyle{apalike}

\newpage

\appendix

\section{Technical Proofs}

We give proofs of the technical results that appear in the main text.

\subsection{Proof of Theorem~\ref{thm:neigh-represent}}
\label{sec:proof-neigh-represent}

For all $k \in [p]$ and $k \neq j$, we define $\mathscr{B}_{jk}: \mathbb{H} \mapsto \mathbb{H}$ as 
\begin{equation*}
\mathscr{B}_{jk}(h) \coloneqq \mathscr{B}_j((0,\ldots,0,\underset{\text{k-th}}{\underline{h}},0,\ldots,0)) \quad \text{ for all } h \in \mathbb{H}.
\end{equation*}
Since $\mathscr{B}_j$ is Hilbert-Schmidt, we claim that $\mathscr{B}_{jk} \in \mathcal{B}_{\text{HS}}(\mathbb{H})$. To prove this claim, note that for any CONS of $\mathbb{H}$ denoted by $\{e_n\}^{\infty}_{n=1}$, we have
\begin{equation*}
\left\{ \left\{ (e_n,0,\ldots,0)  \right\}^{\infty}_{n=1}, \left\{ (0,e_n,0,\ldots,0)\right\}^{\infty}_{n=1},\ldots,\left\{ (0,\ldots,0,e_n)  \right\}^{\infty}_{n=1}  \right\} 
\end{equation*}
to be a CONS of $\mathbb{H}^{p-1}$. Given the assumption that $\mathscr{B}_j$ is Hilbert-Schmidt, we have
\begin{align*}
\sum^{\infty}_{n=1} \left\Vert \mathscr{B}_{jk} \left( e_n \right) \right\Vert^2
& =
\sum^{\infty}_{n=1} \left\Vert \mathscr{B}_{j} \left( ( 0,\ldots,0,\underset{\text{k-th}}{\underline{e_n}},\ldots,0 ) \right) \right\Vert^2 \\
& \leq 
\sum^{\infty}_{n=1} \left\Vert \mathscr{B}_{j} \left( ( e_n,0,\ldots,0 ) \right) \right\Vert^2 + \ldots + \left\Vert \mathscr{B}_{j} \left( ( 0,\ldots,0,e_n ) \right) \right\Vert^2 \\
& < \infty,
\end{align*}
which implies that $\mathscr{B}_{jk} \in \mathcal{B}_{\text{HS}}(\mathbb{H})$.
By the linearity of $\mathscr{B}_j$, then for all 
$$
h=(h_1,\ldots,h_{j-1},h_{j+1}\ldots,h_{p}) \in \mathbb{H}^{p-1},
$$
we have
\begin{equation*}
\mathscr{B}_j(h) = \mathscr{B}_j \left( (h_1,0,\ldots,0) \right) + \ldots + \mathscr{B}_j \left( (0,\ldots,0,h_{p}) \right) = \sum^{p}_{k=1,k\neq j} \mathscr{B}_{jk}(h_k).
\end{equation*}
Thus, we have
\begin{equation}
\label{eq:proof-neigh-represent-eq1}
\mathbb{E} \left[ g_j \mid \bm{g}_{-j} \right] = \mathscr{B}_j \left( \bm{g}_{-j} \right) = \sum_{k \neq j} \mathscr{B}_{jk} \left( g_k \right).
\end{equation}

The rest of the proof is composed of two steps. 
We first construct functions $\{\beta_{jk}(t,t^{\prime})\}_{k \neq j}$ such that~\eqref{eq:function-linear} holds and then show~\eqref{eq:beta-l-expansion} and~\eqref{eq:def-b-ast}.
For any choice of CONS $\{\phi_m\}^{\infty}_{m=1}$ for $\mathbb{H}$, by Theorem 4.4.5 of~\citet{Hsing2015Theoretical} and the fact that $\mathscr{B}_{jk} \in \mathcal{B}_{\text{HS}}(\mathbb{H})$, we have $$\mathscr{B}_{jk}=\sum^{\infty}_{m=1}\sum^{\infty}_{m^{\prime}=1} b^{\ast}_{jk,m m^{\prime}} \phi_m \otimes \phi_{m^{\prime}},$$ 
where $b^{\ast}_{jk,m m^{\prime}} \coloneqq \langle \mathscr{B}_{jk}(\phi_{m^{\prime}}),\phi_m \rangle$ and $\Vert \mathscr{B}_{jk} \Vert^2_{\text{HS}} = \sum^{\infty}_{m=1}\sum^{\infty}_{m^{\prime}} (b^{\ast}_{jk,m m^{\prime}})^2 < \infty$.
Let
\begin{equation*}
\beta_{jk} (t,t^{\prime}) = \sum^{\infty}_{m,m^{\prime}=1} b^{\ast}_{jk,m m^{\prime}} \phi_m (t) \phi_{m^{\prime}} (t^{\prime})
\end{equation*}
for all $(t,t^{\prime}) \in \mathcal{T} \times \mathcal{T}$.
Then, for any $h \in \mathbb{H}$, we have
\begin{align*}
\mathscr{B}_{jk}(h)(t) &= \sum^{\infty}_{m=1}\sum^{\infty}_{m^{\prime}=1} b^{\ast}_{jk,m m^{\prime}} \langle h, \phi_{m^{\prime}} \rangle \phi_m(t) \\
&= \sum^{\infty}_{m=1}\sum^{\infty}_{m^{\prime}=1} b^{\ast}_{jk,m m^{\prime}} \left( \int_{\mathcal{T}} h(t^{\prime}) \phi_{m^{\prime}}(t^{\prime}) d t^{\prime} \right) \phi_m(t) \\
&= \int_{\mathcal{T}} \sum^{\infty}_{m=1}\sum^{\infty}_{m^{\prime}=1} b^{\ast}_{jk,m m^{\prime}} \phi_m(t) \phi_{m^{\prime}}(t^{\prime}) h(t^{\prime}) d t^{\prime} \\
& = \int_{\mathcal{T}} \beta_{jk} (t,t^{\prime}) h(t^{\prime}) d t^{\prime} \numberthis \label{eq:proof-neigh-represent-eq2}
\end{align*}
for all $t \in \mathcal{T}$, where the third equality is by Fubini's Theorem. 
In this way, $\mathscr{B}_{jk}$ is the integral operator with the kernel $\beta_{jk}(t,t^{\prime})$. By Theorem 4.6.7 of~\citet{Hsing2015Theoretical}, we have $\Vert \beta_{jk} \Vert_{\text{HS}}=\Vert \mathscr{B}_{jk} \Vert_{\text{HS}}<\infty$.
The relation~\eqref{eq:function-linear} follows by combining~\eqref{eq:proof-neigh-represent-eq2} and~\eqref{eq:proof-neigh-represent-eq1}.

We then show~\eqref{eq:beta-l-expansion} and~\eqref{eq:def-b-ast}. 
Let $\{\tilde{\phi}_m\}^{\infty}_{m=1}$ be another CONS of $\mathbb{H}$. Let $\tilde{b}^{\ast}_{jk,m m^{\prime}} \coloneqq \langle \mathscr{B}_{jk}(\tilde{\phi}_{m^{\prime}}),\tilde{\phi}_m \rangle$, 
and we can similarly define $\tilde{\beta}_{jk}(t,t^{\prime})$ by
\begin{equation*}
\tilde{\beta}_{jk} (t,t^{\prime}) = \sum^{\infty}_{m,m^{\prime}=1} \tilde{b}^{\ast}_{jk,m m^{\prime}} \tilde{\phi}_m (t) \tilde{\phi}_{m^{\prime}} (t^{\prime}).
\end{equation*}
Similar to~\eqref{eq:proof-neigh-represent-eq2}, we can show $\mathscr{B}_{jk}(h)(t)=\int_{\mathcal{T}} \tilde{\beta}_{jk} (t,t^{\prime}) h(t^{\prime}) d t^{\prime}$ for all $h \in \mathbb{H}$ and $t \in \mathcal{T}$. Thus, we have
\begin{equation*}
\tilde{b}^{\ast}_{jk,m m^{\prime}} = \langle \mathscr{B}_{jk}(\tilde{\phi}_{m^{\prime}}),\tilde{\phi}_m \rangle = \int_{\mathcal{T}} \tilde{\beta}_{jk}(t,t^{\prime}) \tilde{\phi}_m(t) \tilde{\phi}_{m^{\prime}}(t^{\prime}) dt^{\prime} dt.
\end{equation*}
In this way, to finish the proof, we only need to show that $\beta_{jk}(t,t^{\prime})=\tilde{\beta}_{jk}(t,t^{\prime})$ a.e., or equivalently $\Vert \beta_{jk} - \tilde{\beta}_{jk} \Vert_{\text{HS}}=0$. This is obvious since, for any CONS $\{\bar{\phi}_m\}_{m \geq 1}$ for $\mathbb{H}$, we have
\begin{equation*}
\mathscr{B}_{jk}(\bar{\phi}_m)(t)=\int_{\mathcal{T}} \beta_{jk} (t,t^{\prime}) \bar{\phi}_m(t^{\prime}) d t^{\prime}=\int_{\mathcal{T}} \tilde{\beta}_{jk} (t,t^{\prime}) \bar{\phi}_m(t^{\prime}) d t^{\prime}
\end{equation*}
for all $t \in \mathcal{T}$, which implies that
\begin{equation*}
\Vert \beta_{jk} - \tilde{\beta}_{jk} \Vert^2_{\text{HS}} = \sum^{\infty}_{m=1} \Vert \mathscr{B}_{jk}(\bar{\phi}_m) - \mathscr{B}_{jk}(\bar{\phi}_m) \Vert^2 = \sum^{\infty}_{m=1} 0 = 0.
\end{equation*}

\subsection{Derivation of \eqref{eq:regression-target} and \eqref{eq:regression-target-hat}}
\label{sec:deri2formulas}

Recall that $\beta_{jk}(t,t^{\prime})$ is defined in~\eqref{eq:function-linear} and $\hat{\phi}_{jm}$ is an estimate of the true basis function $\phi_{jm}$.
Let $b^{\ast}_{jk,m m^{\prime}} = \int_{\mathcal{T}\times \mathcal{T}} \beta_{jk} (t^{\prime},t) \phi_m (t) \phi_{m^{\prime}} (t^{\prime}) dt^{\prime} dt$.
We focus on a given node $j \in [p]$, and we drop the index $j$ from the notation to simplify the discussion.

We first prove~\eqref{eq:regression-target}.
By \eqref{eq:function-linear} and \eqref{eq:beta-l-expansion}, we have
\begingroup
\allowdisplaybreaks
\begin{align*}
a^Y_{i m} & = \sum^{p-1}_{k=1} \int_{\mathcal{T} \times \mathcal{T}} \beta_{k} (t,t^{\prime}) g^{X_k}_i(t^{\prime}) \phi_m (t) dt^{\prime} dt + \int_{\mathcal{T}} e_i(t) \phi_m (t) dt \\
& = \sum^{p-1}_{k=1} \int_{\mathcal{T} \times \mathcal{T}} \left( \sum^{\infty}_{m^{\prime\prime},m^{\prime}=1} b^{\ast}_{k,m^{\prime\prime} m^{\prime}} \phi_{m^{\prime\prime}} (t) \phi_{m^{\prime}} (t^{\prime}) \right) g^{X_k}_i(t^{\prime}) \phi_m (t) dt^{\prime} dt + \int_{\mathcal{T}} e_i(t) \phi_m (t) dt \\
& = \sum^{p-1}_{k=1} \sum^{\infty}_{m^{\prime}, m^{\prime\prime}=1} b^{\ast}_{k,m^{\prime\prime} m^{\prime}} \int_{\mathcal{T} \times \mathcal{T}} \phi_{m^{\prime\prime}} (t) \phi_{m^{\prime}} (t^{\prime})  g^{X_k}_i(t^{\prime}) \phi_m (t) dt^{\prime} dt + \int_{\mathcal{T}} e_i(t) \phi_m (t) dt \\
& = \sum^{p-1}_{k=1} \sum^{\infty}_{m^{\prime}, m^{\prime\prime}=1} b^{\ast}_{k,m^{\prime\prime} m^{\prime}} \left( \int_{\mathcal{T}} \phi_{m^{\prime\prime}} (t) \phi_m (t) dt \right) \left( \int_{\mathcal{T}} \phi_{m^{\prime}} (t^{\prime}) g^{X_k}_i(t^{\prime}) dt^{\prime} \right) + \int_{\mathcal{T}} e_i(t) \phi_m (t) dt \\
& = \sum^{p-1}_{k=1} \sum^{\infty}_{m^{\prime}=1} b^{\ast}_{k,m m^{\prime}} \left( \int_{\mathcal{T}} g^{X_k}_i(t^{\prime}) \phi_{m^{\prime}} (t^{\prime}) dt^{\prime} \right) + \int_{\mathcal{T}} e_i(t) \phi_m (t) dt \\
& = \sum^{p-1}_{k=1} \sum^{\infty}_{m^{\prime}=1} b^{\ast}_{k,m m^{\prime}} a^{X_k}_{i m^{\prime}} + \int_{\mathcal{T}} e_i(t) \phi_m (t) dt \\
& = \sum^{p-1}_{k=1} \sum^{\infty}_{m^{\prime}=M+1} b^{\ast}_{k,m m^{\prime}} a^{X_k}_{i m^{\prime}} + \sum^{p-1}_{k=1} \sum^{M}_{m^{\prime}=1} b^{\ast}_{k,m m^{\prime}} a^{X_k}_{i m^{\prime}} + \int_{\mathcal{T}} e_i(t) \phi_m (t) dt \numberthis \label{eq:aYm}.
\end{align*}
\endgroup
Then \eqref{eq:regression-target} follows directly from \eqref{eq:aYm} by setting $\bm{B^{\ast}_{k, M}}=(b^{\star}_{k,m m^{\prime}})_{1 \leq m,m^{\prime} \leq M}$,
\begin{equation*}
r_{im} = \sum^{p-1}_{k=1} \sum^{\infty}_{m^{\prime}=M+1} b^{\ast}_{k,m m^{\prime}} a^{X_k}_{i m^{\prime}},
\end{equation*}
$\bm{r_{i,M}}=(r_{i1},\ldots,r_{iM})^{\top}$, $w_{im}=\int_{\mathcal{T}} e_i(t) \phi_m (t) dt$ and $\bm{w_{i,M}}=(w_{i1},\ldots,w_{iM})^{\top}$.

To show~\eqref{eq:regression-target-hat}, we only need to redefine relevant concepts. We define
\begin{equation}\label{eq:def-b-tilde}
\tilde{b}_{k,m m^{\prime}} = \int_{\mathcal{T}\times \mathcal{T}} \beta_k (t^{\prime},t) \hat{\phi}_m (t^{\prime}) \hat{\phi}_{m^{\prime}} (t) dt^{\prime} dt.
\end{equation}
Thus $\tilde{b}_{k,m m^{\prime}}=0$ for all $m,m^{\prime}\geq 1$ when $k \notin \mathscr{N}_j$. Similarly, let
\[\tilde{w}_{im}=\int_{\mathcal{T}} e_i(t) \hat{\phi}_m (t) dt
\qquad\text{and}\qquad
\tilde{r}_{im} = \sum^{p-1}_{k=1} \sum^{\infty}_{m^{\prime}=M+1} \tilde{b}_{k,m m^{\prime}} \hat{a}^{X_k}_{i m^{\prime}}.
\]
Furthermore, let $\bm{\tilde{w}_{i,M}}=(\tilde{w}_{i1},\cdots,\tilde{w}_{iM})^{\top}$, $\bm{\tilde{r}_{i,M}}=(\tilde{r}_{i1},\cdots,\tilde{r}_{iM})^{\top} \in \mathbb{R}^{M}$, $\bm{\tilde{B}^M_l}$ is a $M \times M$ matrix with the $m$-th row equal to $(\tilde{b}_{l,m1},\dots,\tilde{b}_{l,mM})$,
and
\begin{equation}\label{eq:def-v}
\bm{v_{iM}} = \sum^{p-1}_{k=1} (\bm{\tilde{B}_{k,M}}-\bm{B^{\ast}_{k,M}}) \bm{\hat{a}^{X_k}_{i,M}} + (\bm{\tilde{r}_{i,M}}-\bm{r_{i,M}}) + (\bm{\tilde{w}_{i,M}} - \bm{w_{i,M}}).
\end{equation}
By~\eqref{eq:beta-l-expansion}, we have
\begin{equation*}
\beta_k (t^{\prime},t) = \sum^{\infty}_{m,m^{\prime}=1} \tilde{b}^{\ast}_{k,m m^{\prime}} \hat{\phi}_m (t) \hat{\phi}_{m^{\prime}} (t^{\prime}) \quad \text{almost everywhere}.
\end{equation*}
Then by a similar argument to~\eqref{eq:aYm}, we have
\begin{equation}
\label{eq:hataYm}
\hat{a}^Y_{im} = \sum^{p-1}_{k=1} \sum^{\infty}_{m^{\prime}>M} \tilde{b}_{k,m m^{\prime}} \hat{a}^{X_l}_{i m^{\prime}} + \sum^{p-1}_{k=1} \sum^{M}_{m^{\prime}=1} \tilde{b}_{k,m m^{\prime}} \hat{a}^{X_l}_{i m^{\prime}} + \int_{\mathcal{T}} e_i(t) \hat{\phi}_m (t) dt,
\end{equation}
which implies \eqref{eq:regression-target-hat} combined with~\eqref{eq:def-v}.

\subsection{Simplification of ADMM Optimization Problems}\label{CH_ADMM}

We explain how to obtain the problem in \eqref{EQ_QOpt}. Let $g(\sum^{p-1}_{k=1} \bm{Q_k}) = \frac{1}{2n}\|\bm{A^Y}-\sum^{p-1}_{k=1} \bm{Q_k}\|^2_{\text{F}}$ and $\bm{W_k}^{h+1} = \bm{A^{X_k}}\bm{P_k^{h+1}} + \bm{U_k^h}$. The update to matrices $\{\bm{Q_k}\}_{k\in[p-1]}$ can be rewritten as
\begin{align*}
    \min_{\{\bm{Q_k}\}_{k\in[p-1]}} g((p-1)\bar{\bm{Q}}) + \frac{\rho}{2}\sum^{p-1}_{k=1}\|\bm{Q_k} - \bm{W_k}^{h+1}\|_F^2
    \quad\text{subject to}\quad
    \bar{\bm{Q}} = \frac{1}{p-1}\sum^{p-1}_{k=1} \bm{Q_k}.
\end{align*}
Let $\phi:\mathbb{R}^{n\times M}\rightarrow\mathbb{R}$ be a function that satisfies $\phi(\mathbf{0})=0$ and $\nabla\phi(\mathbf{0})\neq 0$.
The Lagrangian function is then
\[
(\bm{Q_1},\dots,\bm{Q_{p-1}};\mu) = g((p-1)\bar{\bm{Q}})+\frac{\rho}{2}\sum^{p-1}_{k=1}\|\bm{Q_k}-\bm{W_k}^{h+1}\|_F^2 + \mu \phi\left(\bar{\bm{Q}}-\frac{1}{p-1}\sum^{p-1}_{k=1}\bm{Q_k}\right).
\]
The matrix $\bm{Q_k}$ that minimizes the Lagrangian satisfies
\begin{align*}
    \frac{\partial L}{\partial \bm{Q_k}} &= \rho(\bm{Q_k}-\bm{W_k}^{h+1}) - \frac{\mu}{p-1} \nabla\phi\left(\bar{\bm{Q}}-\frac{1}{p-1}\sum^{p-1}_{k=1} \bm{Q_k}\right) = 0,\\
    \frac{\partial L}{\partial \mu} &= \bar{\bm{Q}}-\frac{1}{p-1}\sum^{p-1}_{k=1} \bm{Q_k} = 0.
\end{align*}
This is equivalent to $\rho(\bm{Q_k}-\bm{W_k}^{h+1})=\frac{\mu}{p-1}\nabla\phi(\mathbf{0})$. As a result, we see that each entry in $\bm{Q_k}-\bm{W_k}^{h+1}$ does not vary with $k$. Let $\bar{\bm{R}}^{h+1}=\frac{1}{p-1}\sum^{p-1}_{k=1} \bm{W_k}^{h+1}$. Then $\bm{Q_k}^{h+1}$ can be replaced by $\bar{\bm{Q}}^{h+1}+\bm{W_k}^{h+1}-\bar{\bm{R}}^{h+1}$ and 
\[
\bm{Q_k}^{h+1} = \bar{\bm{Q}}^{h+1} + \bm{A^{X_k}}\bm{P_k}^{h+1} + \bm{W_k}^h - \overline{\bm{A^X P}}^{h+1} - \bar{\bm{U}}^h,
\]
where $\overline{\bm{A^X P}}^h = \frac{1}{p-1}\sum^{p-1}_{k=1}\bm{A^{X_k}}\bm{P_k}^h$. Therefore, we have obtained the problem in \eqref{EQ_QOpt}.

\subsection{Derivation of \eqref{eq:upp-bd-log-Lambda}}
\label{sec:der-minimize-upp-bd}

We drop the subscript $M$. By \eqref{eq:regression-target} and the definition of $\mathscr{N}_j$, we have
\begin{equation*}
\bm{a^{Y}_i} = \sum_{k \in \mathscr{N}_j} \bm{B^{\ast}_k} \bm{a^{X_k}_i} + \bm{w_i} + \bm{r_i},
\end{equation*}
where $\bm{r_i}=(r_{i1},r_{i2},\dots,r_{iM})^{\top}$ with
\begin{equation*}
r_{im} = \sum_{k \in \mathscr{N}_j} \sum_{m^{\prime}>M} b^{\ast}_{k,m m^{\prime}} a^{X_k}_{im^{\prime}} \, .
\end{equation*}
Let $\bm{\Sigma^Y}=\mathbb{E} \left[ \bm{a^Y_i} (\bm{a^Y_i})^{\top}  \right]$, $\bm{\Sigma^{Y,X_k}}=\mathbb{E} \left[ \bm{a^Y_i} (\bm{a^{X_k}_i})^{\top} \right]$, $k \in [p-1]$. Note that
\[
((\bm{a^{Y}}_i)^{\top}, (\bm{a^{X_1}}_i)^{\top}, (\bm{a^{X_2}}_i)^{\top}, \dots, (\bm{a^{X_{p-1}}_i})^{\top} )^{\top}
\]
is a multivariate Gaussian vector. Then 
\begin{equation}\label{eq:max-strategy-1}
\left[ \bm{B^*_k} \right]_{k \in \mathscr{N}_j} = \left[ \bm{\Sigma^{Y,X_k}} - \bm{\Sigma^{r,X_k}}  \right]_{k \in \mathscr{N}_j} \left( \bm{\Sigma^X_{\mathscr{N}_j \mathscr{N}_j}}  \right)^{-1}
\end{equation}
and 
\begin{align*}
\left\Vert \left[ \bm{B^*_k} \right]_{k \in \mathscr{N}_j} \right\Vert_{\text{F}} & =  \left\Vert \left[ \bm{\Sigma^{Y,X_k}} - \bm{\Sigma^{r,X_k}}  \right]_{k \in \mathscr{N}_j} \left( \bm{\Sigma^X_{\mathscr{N}_j \mathscr{N}_j}}  \right)^{-1} \right\Vert_{\text{F}} \\ & \geq  \left\Vert  \left[ \bm{\Sigma^{Y,X_k}} \right]_{k \in \mathscr{N}_j} \left( \bm{\Sigma^X_{\mathscr{N}_j \mathscr{N}_j}}  \right)^{-1}  \right\Vert_{\text{F}} - \left\Vert  \left[ \bm{\Sigma^{r,X_k}} \right]_{k \in \mathscr{N}_j} \left( \bm{\Sigma^X_{\mathscr{N}_j \mathscr{N}_j}}  \right)^{-1}  \right\Vert_{\text{F}}.
\end{align*}
Since the correlation between $\bm{a^Y_i}$ and $\bm{a^{X_k}_i}$ is larger than correlation between $\bm{r_i}$ and $\bm{a^{X_k}_i}$, $k \in \mathscr{N}_j$, when $M$ is large enough, we will have
\begin{equation*}
\left\Vert  \left[ \bm{\Sigma^{Y,X_k}} \right]_{k \in \mathscr{N}_j} \left( \bm{\Sigma^X_{\mathscr{N}_j \mathscr{N}_j}}  \right)^{-1}  \right\Vert_{\text{F}} \gg \left\Vert  \left[ \bm{\Sigma^{r,X_k}} \right]_{k \in \mathscr{N}_j} \left( \bm{\Sigma^X_{\mathscr{N}_j \mathscr{N}_j}}  \right)^{-1}  \right\Vert_{\text{F}}.
\end{equation*}
Combining the last two displays, we have
\begin{equation}\label{eq:low-bd-Bk-norm}
\left\Vert \left[ \bm{B^*_k} \right]_{k \in \mathscr{N}_j} \right\Vert_{\text{F}} \gtrsim \left\Vert  \left[ \bm{\Sigma^{Y,X_k}} \right]_{k \in \mathscr{N}_j} \left( \bm{\Sigma^X_{\mathscr{N}_j \mathscr{N}_j}}  \right)^{-1}  \right\Vert_{\text{F}}.
\end{equation}
Furthermore, let $\bm{R}^{Y,X_k}=(\bm{\Sigma^Y})^{-1/2} \bm{\Sigma^{Y,X_k}} (\bm{\Sigma^{X_k}})^{-1/2}$, $k \in  [p-1]$, and
\begin{equation*}
\bm{R^X_{\mathscr{N}_j \mathscr{N}_j}} = 
\text{diag}\left( \left[  \left( \bm{\Sigma^X_{kk}} \right)^{-1/2}  \right]_{k \in \mathscr{N}_j} \right)
\bm{\Sigma^X_{\mathscr{N}_j \mathscr{N}_j}} \,
\text{diag}\left( \left[  \left( \bm{\Sigma^X_{kk}} \right)^{-1/2}  \right]_{k \in \mathscr{N}_j} \right).
\end{equation*}
Then
\begin{align*}
\left[ \bm{\Sigma^{Y,X_k}}  \right]_{k \in \mathscr{N}_j} \left( \bm{\Sigma^X_{\mathscr{N}_j \mathscr{N}_j}}  \right)^{-1} & = 
\left( \bm{\Sigma^{Y}} \right)^{\frac{1}{2}} \left[ \bm{R^{Y,X_k}}  \right]_{k \in \mathscr{N}_j} \text{diag} \left( \left[ \left( \bm{\Sigma^{X}_{kk}} \right)^{1/2} \right]_{k \in \mathscr{N}_j} \right) \left( \bm{\Sigma^X_{\mathscr{N}_j \mathscr{N}_j}}  \right)^{-1} \\
& =  \left( \bm{\Sigma^{Y}} \right)^{\frac{1}{2}} \left[ \bm{R^{Y,X_k}}  \right]_{k \in \mathscr{N}_j}  \left( \bm{R^X_{\mathscr{N}_j \mathscr{N}_j}}  \right)^{-1} \text{diag} \left( \left[ \left( \bm{\Sigma^{X}_{kk}} \right)^{-1/2} \right]_{k \in \mathscr{N}_j} \right)
\end{align*}
and 
\begin{equation*}
\left[ \bm{\Sigma^{r,X_k}}  \right]_{k \in \mathscr{N}_j} \left( \bm{\Sigma^X_{\mathscr{N}_j \mathscr{N}_j}}  \right)^{-1} = \left( \bm{\Sigma^{r}} \right)^{\frac{1}{2}} \left[ \bm{R^{r,X_k}}  \right]_{k \in \mathscr{N}_j}  \left( \bm{R^X_{\mathscr{N}_j \mathscr{N}_j}}  \right)^{-1} \text{diag} \left( \left[ \left( \bm{\Sigma^{X}_{kk}} \right)^{-1/2} \right]_{k \in \mathscr{N}_j} \right).
\end{equation*}
By Lemma~\ref{lemma:mat-prop}, we have 
\begin{align*}
& \left\Vert  \left[ \bm{\Sigma^{Y,X_k}}  \right]_{k \in \mathscr{N}_j} \left( \bm{\Sigma^X_{\mathscr{N}_j \mathscr{N}_j}}  \right)^{-1}  \right\Vert_{\text{F}} \\
& \geq  \left\Vert \left( \bm{\Sigma^{Y}} \right)^{\frac{1}{2}} \right\Vert_{\text{F}} \left\{ \rho_{\min}  \left( \left[ \bm{R^{Y,X_k}}  \right]_{k \in \mathscr{N}_j}  \left( \bm{R^X_{\mathscr{N}_j \mathscr{N}_j}}  \right)^{-1} \cdot \right.\right. \\
& \qquad \left.\left. \text{diag} \left( \left[ \left( \bm{\Sigma^{X}_{kk}} \right)^{-1} \right]_{k \in \mathscr{N}_j} \right)  \left( \bm{R^X_{\mathscr{N}_j \mathscr{N}_j}}  \right)^{-1} \left[ \bm{R^{X_k,Y}}  \right]_{k \in \mathscr{N}_j} \right) \right\}^{1/2} \\
& =  \sqrt{ \tr \left( \bm{\Sigma^{Y}} \right) } \left\{ \rho_{\min}  \left( \left[ \bm{R^{Y,X_k}}  \right]_{k \in \mathscr{N}_j}  \left( \bm{R^X_{\mathscr{N}_j \mathscr{N}_j}}  \right)^{-1} \cdot \right.\right. \\
& \qquad \left.\left. \text{diag} \left( \left[ \left( \bm{\Sigma^{X}_{kk}} \right)^{-1} \right]_{k \in \mathscr{N}_j} \right)  \left( \bm{R^X_{\mathscr{N}_j \mathscr{N}_j}}  \right)^{-1} \left[ \bm{R^{X_k,Y}}  \right]_{k \in \mathscr{N}_j} \right) \right\}^{1/2} \\
& \geq  \sqrt{ \tr \left( \bm{\Sigma^{Y}} \right) } \sqrt{ \rho_{\min} \left(  \text{diag} \left( \left[ \left( \bm{\Sigma^{X}_{kk}} \right)^{-1} \right]_{k \in \mathscr{N}_j} \right) \right) } \cdot \\
& \qquad \sqrt{ \rho_{\min}  \left( \left[ \bm{R^{Y,X_k}}  \right]_{k \in \mathscr{N}_j}  \left( \bm{R^X_{\mathscr{N}_j \mathscr{N}_j}}  \right)^{-2} \left[ \bm{R^{X_k,Y}}  \right]_{k \in \mathscr{N}_j} \right) } \\
& \geq  \sqrt{ \tr \left( \bm{\Sigma^{Y}} \right) } \sqrt{ \frac{1}{\max_{k \in \mathscr{N}_j} \rho_{max} \left( \bm{\Sigma^{X}_{kk}} \right) } } \sqrt{ \rho_{\min}  \left( \left[ \bm{R^{Y,X_k}}  \right]_{k \in \mathscr{N}_j}  \left( \bm{R^X_{\mathscr{N}_j \mathscr{N}_j}}  \right)^{-2} \left[ \bm{R^{X_k,Y}}  \right]_{k \in \mathscr{N}_j} \right) },
\end{align*}
and, therefore,
\begin{align*}
& \sqrt{\kappa(M)} \left\Vert  \left[ \bm{\Sigma^{Y,X_k}} \right]_{k \in \mathscr{N}_j} \left( \bm{\Sigma^X_{\mathscr{N}_j \mathscr{N}_j}}  \right)^{-1}  \right\Vert_{\text{F}} \\
& \geq  \sqrt{ \tr \left( \bm{\Sigma^{Y}} \right) } \sqrt{ \frac{ \rho_{\min} \left( \bm{\Sigma^X_{\mathscr{N}_j \mathscr{N}_j}} \right) }{\max_{k \in \mathscr{N}_j} \rho_{max} \left( \bm{\Sigma^{X}_{kk}} \right) } } \sqrt{ \rho_{\min}  \left( \left[ \bm{R^{Y,X_k}}  \right]_{k \in \mathscr{N}_j}  \left( \bm{R^X_{\mathscr{N}_j \mathscr{N}_j}}  \right)^{-2} \left[ \bm{R^{X_k,Y}}  \right]_{k \in \mathscr{N}_j} \right) } \\
& \geq  \sqrt{ \tr \left( \bm{\Sigma^{Y}} \right) } \sqrt{ \frac{ \rho_{\min} \left( \bm{\Sigma^X_{\mathscr{N}_j \mathscr{N}_j}} \right) }{ \rho_{max} \left( \bm{\Sigma^{X}_{ \mathscr{N}_j \mathscr{N}_j }} \right) } } \sqrt{ \rho_{\min}  \left( \left[ \bm{R^{Y,X_k}}  \right]_{k \in \mathscr{N}_j}  \left( \bm{R^X_{\mathscr{N}_j \mathscr{N}_j}}  \right)^{-2} \left[ \bm{R^{X_k,Y}}  \right]_{k \in \mathscr{N}_j} \right) }. \numberthis \label{eq:low-bd-term1}
\end{align*}
Combining \eqref{eq:low-bd-Bk-norm}, \eqref{eq:low-bd-term1}, and the definition of $\Lambda(M,\phi)$, we arrive at \eqref{eq:upp-bd-log-Lambda}.

\subsection{Proposition~\ref{prop:lambda_max} and its Proof}
\label{sec:app4opt}

\begin{proposition}
\label{prop:lambda_max}
If the tuning parameter $\lambda_n$ satisfies
\begin{equation}\label{eq:lambda-cond-opt}
\lambda_n > \max_{k\in[p-1]} \frac{1}{n}\|(\bm{A^{X_k}})^{\top}\bm{A^Y}\|_F,
\end{equation}
where $\bm{A^{X_k}}$ and $\bm{A^Y}$ are defined in Section~\ref{sec:ADMM},
then the estimated support set $\hat{\mathscr{N}}_j$ is empty.
\end{proposition}

\begin{proof}
This threshold of $\lambda_{n}$ can be derived using the KKT condition. 
We use the notation introduced in Section~\ref{sec:ADMM}.
The subgradient of the objective in \eqref{eq:loss-lasso} with respect to $\bm{B_k}$ can be written as
\begin{equation}\label{eq:subgradient}
-\frac{1}{n}(\bm{A^{X_k}})^\top \left(\bm{A^Y}-\sum^{p-1}_{l=1}\bm{A^{X_l}}\bm{B_l}\right) + \lambda_n \bm{\Upsilon_k},
\end{equation}
where $\bm{\Upsilon_k} = \|\bm{B_k}\|_F^{-1}\bm{B_k}$ if $\bm{B_k}\ne \mathbf{0}$, and $\bm{\Upsilon_k} \in \mathbb{R}^{M\times M}$, $\|\bm{\Upsilon_k}\|_F\leq 1$ otherwise. 
We assume that $\hat{\mathscr{N}}_j$ is non-empty. That is, there exists some $k$ such that $\bm{B_k}\neq \bm{0}$. By~\eqref{eq:subgradient}, we have
\begin{equation}\label{eq:KKTsubgr}
\lambda_n \frac{\bm{B_k}}{\Vert \bm{B_k} \Vert_F} = \frac{1}{n}(\bm{A^{X_k}})^\top \bm{A^Y} - \frac{1}{n}(\bm{A^{X_k}})^\top \sum_{l \in \hat{\mathscr{N}}_j} \bm{A^{X_l}} \bm{B_l} \quad \text{for all } k \in \hat{\mathscr{N}}_j.
\end{equation}
Let $\bm{A^X_{\hat{\mathscr{N}}_j}} = \left[ \bm{A^{X_k}} \right]_{k \in \hat{\mathscr{N}}_j} \in \mathbb{R}^{n \times \vert \hat{\mathscr{N}}_j\vert M}$ and $\bm{B_{\hat{\mathscr{N}}_j}}=[\bm{B_l}]_{l \in \hat{\mathscr{N}}_j} \in \mathbb{R}^{M \times \vert \hat{\mathscr{N}}_j\vert M }$
be the submatrix of $\bm{A^X}$ and $[\bm{B_1}^{\top},\ldots,\bm{B_p}^{\top}]^{\top}$ that correspond to $k \in \hat{\mathscr{N}}_j$.
By~\eqref{eq:KKTsubgr}, we have
\begin{equation}\label{eq:KKTsubgr-2}
\left( \frac{1}{n} \left( \bm{A^X_{\hat{\mathscr{N}}_j}} \right)^{\top}  \left( \bm{A^X_{\hat{\mathscr{N}}_j}} \right) + \lambda_n \text{diag} \left( \left\{ \frac{1}{ \Vert \bm{B_l} \Vert_F } \bm{I_M} \right\}_{l \in \hat{\mathscr{N}}_j} \right) \right)\bm{B_{\hat{\mathscr{N}}_j}} = \frac{1}{n} \left( \bm{A^X_{\hat{\mathscr{N}}_j}} \right)^{\top} \bm{A^Y}.
\end{equation}
Since $\frac{1}{n} \left( \bm{A^X_{\hat{\mathscr{N}}_j}} \right)^{\top}  \left( \bm{A^X_{\hat{\mathscr{N}}_j}} \right)$ is positive semi-definite and we have assumed that~\eqref{eq:lambda-cond-opt} holds, the left hand side of~\eqref{eq:KKTsubgr-2} then satisfies that
\begin{align*}
& \quad \left\Vert \left( \frac{1}{n} \left( \bm{A^X_{\hat{\mathscr{N}}_j}} \right)^{\top}  \left( \bm{A^X_{\hat{\mathscr{N}}_j}} \right) + \lambda_n \text{diag} \left( \left\{ \frac{1}{ \Vert \bm{B_l} \Vert_F }\bm{I_M} \right\}_{l \in \hat{\mathscr{N}}_j} \right) \right)\bm{B_{\hat{\mathscr{N}}_j}} \right\Vert_{\text{F}} \\
& \geq \left\Vert \lambda_n \text{diag} \left( \left\{ \frac{1}{ \Vert \bm{B_l} \Vert_F } \bm{I_M} \right\}_{l \in \hat{\mathscr{N}}_j} \right)\bm{B_{\hat{\mathscr{N}}_j}} \right\Vert_{\text{F}} \\
& \geq \lambda_n \vert \hat{\mathscr{N}}_j \vert \\
& > \vert \hat{\mathscr{N}}_j \vert \cdot \max_{k\in[p-1]} \frac{1}{n}\|(\bm{A^{X_k}})^{\top}\bm{A^Y}\|_F,
\end{align*}
where the first inequality follows from Lemma~\ref{lemma:mat-prop2}.
On the other hand, the right hand side of~\eqref{eq:KKTsubgr-2} satisfies that
\begin{equation*}
\left\Vert \frac{1}{n} \left( \bm{A^X_{\hat{\mathscr{N}}_j}} \right)^{\top} \bm{A^Y} \right\Vert_F \leq \vert \hat{\mathscr{N}}_j \vert \cdot \max_{k\in[p-1]} \frac{1}{n}\|(\bm{A^{X_k}})^{\top}\bm{A^Y}\|_F.
\end{equation*}
Combine the above two equations with~\eqref{eq:KKTsubgr-2}, we have
a contradiction. Thus, we conclude that $\hat{\mathscr{N}}_j$ must be empty.


\end{proof}

\subsection{Proof of Theorem~\ref{thm:nb-edge-recov-fix}}
\label{sec:proof-thm1}

In this section, we prove Theorem~\ref{thm:nb-edge-recov-fix}. We first introduce some useful notation.

Let $\tilde{\lambda}(n,p,M,\delta)$ be defined as 
\begin{multline}\label{eq:lambda_tilde_exact}
\tilde{\lambda} (n, p, M, \delta) =  2 \mathcal{C}_{n,\delta} \left(\frac{M  \sqrt{\Xi_1 (M)}}{\sqrt{n}}+ 2 \sqrt{\Xi_1 (M)} \sqrt{ \frac{ \log (4(p-1)/\delta) }{n} } \right) \\
+ 2 \omega (M) \left\{ 
7\sqrt{3} \frac{M \sqrt{ \log (6(p-1)/\delta) + 2 \log M}}{\sqrt{n}} \right. \\
\left. + 
 \frac{ 8 M c (\log (2n)) (\log (6(p-1)/\delta) + 2\log M) }{3n} \right\},
\end{multline}
where $c$ is some universal constant that does not depend on $n$, $p$ or $M$.

To simplify the notation, we omit the basis dimension, $M$, and let $\bm{a^Y_i}=\bm{a^Y_{i,M}}$, $\bm{a^{X_k}_i} = \bm{a^{X_k}_{i,M}}$, and $\bm{B^*_k}=\bm{B^{\ast}_{k, M}}$ for all $k \in [p-1]$.
Then by \eqref{eq:regression-target}, for all $i \in [n]$, we have
\begin{equation}\label{eq:score-equation}
\bm{a^{Y}_i} = \sum^{p-1}_{k=1} \bm{B^{\ast}_{k}} \bm{a^{X_k}_i} + \bm{u}_i,
\end{equation}
where $\bm{u_i}=\bm{w_i}+\bm{r_i}$, and $\bm{w_i}$, $\bm{r_i}$ are defined in Appendix~\ref{sec:deri2formulas}.
With this notation, we give a proof of Theorem~\ref{thm:nb-edge-recov-fix}.

\begin{proof}
The equation \eqref{eq:score-equation} can be rewritten as
\begin{equation}\label{eq:loss-lasso-vec}
\bm{a^Y_i} = \sum^{p-1}_{k=1} \left( (\bm{a^{X_k}_i})^{\top} \otimes \bm{I_M} \right) \text{vec}(\bm{B^{\ast}_k}) + \bm{u_i}.
\end{equation}
Let 
$\bm{Z^{X_k}_i}=\bm{a^{X_k}_i} \otimes \bm{I_M} \in \mathbb{R}^{M^2 \times M}$, $i \in [n]$, and let $\bm{\beta^{\ast}_k}=\text{vec}(\bm{B^{\ast}_k}) \in \mathbb{R}^{M^2}$, $k \in [p-1]$. 
Furthermore, let
\begin{align*}
\bm{a^{Y}}&=((\bm{a^Y_1})^{\top}, (\bm{a^Y_2})^{\top}, \dots, (\bm{a^Y_n})^{\top})^{\top} \in \mathbb{R}^{nM}, \\
\bm{\beta^{\ast}}&=((\bm{\beta^{\ast}_1})^{\top},(\bm{\beta^{\ast}_2})^{\top},\dots,(\bm{\beta^{\ast}_{p-1}})^{\top})^{\top} \in \mathbb{R}^{(p-1)M^2},
\end{align*}
and
\begin{equation}
\label{eq:u-def}
\bm{u}=(\bm{u_1}^{\top},\bm{u_2}^{\top},\dots,\bm{u_n}^{\top})^{\top} \in \mathbb{R}^{nM}.
\end{equation}
Let $\bm{Z_i}=((\bm{Z^{X_1}_i})^{\top}, (\bm{Z^{X_2}_i})^{\top}, \dots, (\bm{Z^{X_{p-1}}_i})^{\top})^{\top} \in \mathbb{R}^{(p-1) M^2 \times M}$ for $i \in [n]$, and
\begin{equation*}
\bm{Z} = 
\begin{bmatrix}
\bm{Z_1}^{\top}\\
\bm{Z_2}^{\top}\\
\vdots\\
\bm{Z_n}^{\top}
\end{bmatrix}
=
\begin{bmatrix}
(\bm{Z^{X_1}_1})^{\top} & (\bm{Z^{X_2}_1})^{\top} & \dots & (\bm{Z^{X_{p-1}}_1})^{\top}\\
(\bm{Z^{X_1}_2})^{\top} & (\bm{Z^{X_2}_2})^{\top} & \dots & (\bm{Z^{X_{p-1}}_2})^{\top}\\
\vdots\\
(\bm{Z^{X_1}_n})^{\top} & (\bm{Z^{X_2}_n})^{\top} & \dots & (\bm{Z^{X_{p-1}}_n})^{\top}
\end{bmatrix}
\in\mathbb{R}^{nM\times (p-1)M^2}.
\end{equation*}
Then we can further formulate \eqref{eq:loss-lasso-vec} as
\begin{equation}\label{eq:loss-lasso-comp}
\bm{a^Y} = \bm{Z} \bm{\beta^*} + \bm{u}.
\end{equation}

Recall that
\begin{equation*}
\bm{A^X} = 
\begin{bmatrix}
(\bm{a^X_1})^{\top} \\
(\bm{a^X_2})^{\top} \\
\vdots \\
(\bm{a^X_n})^{\top}
\end{bmatrix}
=
\begin{bmatrix}
(\bm{a^{X_1}_1})^{\top} & (\bm{a^{X_2}_1})^{\top} & \dots & (\bm{a^{X_{p-1}}_1})^{\top} \\
(\bm{a^{X_1}_2})^{\top} & (\bm{a^{X_2}_2})^{\top} & \dots & (\bm{a^{X_{p-1}}_2})^{\top} \\
\vdots \\
(\bm{a^{X_1}_n})^{\top} & (\bm{a^{X_2}_n})^{\top} & \dots & (\bm{a^{X_{p-1}}_n})^{\top}
\end{bmatrix}
\in \mathbb{R}^{n \times (p-1)M},
\end{equation*}
then we have $\bm{Z}=\bm{A^X} \otimes \bm{I_M}$.
We divide columns of $\bm{A^X}$ into $p-1$ groups with equal group size $M$, that is, $\bm{A^X}=(\bm{A^{X_1}},\bm{A^{X_2}},\dots,\bm{A^{X_{p-1}}})$, where $\bm{A^{X_k}} \in \mathbb{R}^{n \times M}$ for all $k \in [p-1]$. 
Similarly, we divide the columns of $\bm{Z}$ into $p-1$ groups with equal group size $M^2$, that is, $\bm{Z}=(\bm{Z_1},\bm{Z_2},\dots, \bm{Z_{p-1}})$, where $\bm{Z_k}\in\mathbb{R}^{nM\times M^2}$ for all $k\in [p-1]$. Then, we have $(\bm{Z}^{\top}\bm{u})_{k}=\bm{Z_k}^{\top}\bm{u}$. Besides, by definition of $\bm{Z}$, it is easy to see that $\bm{Z_k}=\bm{A^{X_k}}\otimes \bm{I_M}$.

Besides, we can rewrite \eqref{eq:loss-lasso} as
\begin{equation}
\bm{\hat{\beta}} = \argmin_{\bm{\beta}} \left\{ \mathcal{L}(\bm{\beta}) + \lambda_n \mathcal{R}(\bm{\beta}) \right\},
\end{equation}
where
\begin{align*}
\mathcal{L}(\bm{\beta}) & =  \frac{1}{2n} \Vert \bm{a^Y} - \bm{Z} \bm{\beta} \Vert^2_2, \numberthis \label{eq:Loss} \\
\mathcal{R}(\bm{\beta}) & =  \sum^{p-1}_{k=1} \Vert \bm{\beta_k} \Vert_2, \numberthis \label{eq:Penalty}
\end{align*}
and $\bm{\beta}=((\bm{\beta_1})^{\top},\dots,(\bm{\beta_{p-1}})^{\top})^{\top} \in \mathbb{R}^{(p-1)M^2}$, with $\bm{\beta_k} \in \mathbb{R}^{M^2}$ for $k \in [p-1]$.

Thus, the support set defined in \eqref{eq:neighborhoodDef} and its estimator defined in \eqref{eq:nbg-set-est} can be expressed as
\begin{equation*}
\mathscr{N}_j = \left\{ k \in [p-1]: \Vert \bm{\beta^*_k} \Vert_2 > 0 \right\},
\end{equation*}
and
\begin{equation*}
\hat{\mathscr{N}}_j = \left\{ k \in [p-1]: \Vert \bm{\hat{\beta}_k} \Vert_2 > \epsilon_n \right\}.
\end{equation*}

We define the model space $\mathcal{M}(\mathscr{N}_j)$ with which the penalty term $\mathcal{R}(\cdot)$ is decomposable. Let
\begin{equation*}
\mathcal{M}=\mathcal{M}(\mathscr{N}_j) = \{\bm{\beta}=((\bm{\beta_1})^{\top},\dots,(\bm{\beta_{p-1}})^{\top})^{\top} \in \mathbb{R}^{(p-1)M^2}: \bm{\beta_{k}}=0 \; \text{for all} \; k \notin \mathscr{N}_j \},
\end{equation*}
we then have its orthogonal complement as
\begin{equation*}
\mathcal{M}^{\bot}=\mathcal{M}(\mathscr{N}_j)^{\bot} = \{\bm{\beta}=((\bm{\beta_1})^{\top},\dots,(\bm{\beta_{p-1}})^{\top})^{\top} \in \mathbb{R}^{(p-1)M^2}: \bm{\beta_{k}}=0 \; \text{for all} \; k \in \mathscr{N}_j \}.
\end{equation*}
It is then easy to verify that $\mathcal{R}(\cdot)$ defined in \eqref{eq:Penalty} is decomposable with respect to\\ $(\mathcal{M}(\mathscr{N}_j),\mathcal{M}(\mathscr{N}_j)^{\bot})$ (Example 2, Section 2.2 of \citet{negahban2010unified}), that is
\begin{equation*}
\mathcal{R}(\bm{\theta}+\bm{\gamma})=\mathcal{R}(\bm{\theta})+\mathcal{R}(\bm{\gamma}) \quad \text{for all} \; \bm{\theta}\in\mathcal{M}(\mathscr{N}_j) \; \text{and} \; \bm{\gamma}\in\mathcal{M}(\mathscr{N}_j)^{\bot}.
\end{equation*}

When $\lambda_n = \tilde{\lambda}(n,p,M,\delta)$, where $\tilde{\lambda}(n,p,M,\delta)$ is defined in \eqref{eq:lambda_tilde_exact}, then by Lemma~\ref{lemma:LowBoundLambda}, we have
\begin{equation}\label{eq:CondForLambda}
\lambda_n \geq \frac{2}{n} \max_{k \in [p-1]} \Vert (\bm{Z}^{\top} \bm{u})_{k} \Vert_2
\end{equation}
hold with probability at least $1-2\delta$. This way, by Lemma 1 of \citet{negahban2010unified}, we have that $\hat{\bm{\beta}}-\bm{\beta^{\ast}}$ lies in a constrained space $\mathbb{C}(\mathscr{N}_j)$ defined by
\begin{equation}\label{eq:ConstrainedSpace}
\mathbb{C}(\mathscr{N}_j) = \{\bm{\theta} \in\mathbb{R}^{(p-1)M^2}: \Vert \bm{\theta_{\mathcal{M}^{\bot}}} \Vert_{1,2} \leq 3 \Vert \bm{\theta_{\mathcal{M}}} \Vert_{1,2} \}.
\end{equation}
with probability at least $1-2\delta$.
Note that $\mathbb{C}(\mathscr{N}_j)$ depends on support $\mathscr{N}_j$ through $\mathcal{M}$.

The error term of a first-order Taylor series expansion is
\begin{equation*}
\mathcal{L}(\bm{\beta^{\ast}}+\bm{\Delta\hat{\beta}})-\mathcal{L}(\bm{\beta^{\ast}})-\langle \nabla\mathcal{L}(\bm{\beta^{\ast}}),\bm{\Delta\hat{\beta}}\rangle = \frac{1}{2n}\Vert \bm{Z} \bm{\Delta\hat{\beta}} \Vert^2_2
\end{equation*}
where $\bm{\Delta\hat{\beta}}=\bm{\hat{\beta}}-\bm{\beta^{\ast}}$.
Then by Lemma~\ref{lemma:restricted-eigen-value} and $\bm{\Delta\hat{\beta}} \in \mathbb{C}(\mathscr{N}_j)$ with probability at least $1-2\delta$, we have
\begin{equation}\label{eq:RSCPre}
\mathbb{P} \left\{ \frac{1}{2n}\Vert \bm{Z} \bm{\Delta\hat{\beta}} \Vert^2_2 \geq \frac{\kappa}{4} \Vert \bm{\Delta\hat{\beta}} \Vert^2_2 \right\} \geq 1-3\delta.
\end{equation}
Thus, by Lemma~\ref{lemma:bd-B-err}, we then have
\begin{equation}\label{eq:BetaL2ErrBound}
\mathbb{P} \left\{ \sqrt{\sum^{p-1}_{k=1} \left\Vert \hat{\bm{B_k}} - \bm{B^*_k} \right\Vert^2_{\text{F}}} \leq \chi (n,p,M,\delta) + 12 \sqrt{s} \frac{\omega (M)}{ \sqrt{\kappa(M)}} \right\} \geq 1-3\delta,
\end{equation}
where $\kappa(M)$ is defined in \eqref{eq:kappa-def} and $\chi (n,p,M,\delta)$ is defined in \eqref{eq:chi-fun}.
Given the inequality in the left hand side parenthesis of \eqref{eq:BetaL2ErrBound} holds, we then have
\begin{equation*}
\left\Vert \bm{\hat{B}_k} - \bm{B^*_k} \right\Vert_{\text{F}} \leq \chi (n,p,M,\delta) + 12 \sqrt{s} \frac{\omega (M)}{\sqrt{\kappa(M)}} \quad \text{for all } k \in [p-1].
\end{equation*}

Next, we prove that $\mathbb{P}\{\hat{\mathscr{N}}_j=\mathscr{N}_j\}\geq 1-3\delta$.
To show that, we only need to prove that the above inequality implies that $\hat{\mathscr{N}}_j=\mathscr{N}_j$. Under the assumption that the above inequality holds, note that for any $k \notin \mathscr{N}_j$, we have $\bm{B^*_k}=\bm{0}$ for all $M$, thus, we have
\begin{equation*}
\left\Vert \bm{\hat{B}_k} \right\Vert_{\text{F}} = \left\Vert \bm{\hat{B}_k} - \bm{B^*_k} \right\Vert_{\text{F}} \leq \chi (n,p,M,\delta) + 12 \sqrt{s} \frac{\omega (M)}{\sqrt{\kappa(M)}}.
\end{equation*}
On the other hand, for $k \in \mathscr{N}_j$, we have
\begin{align*}
\left\Vert \bm{\hat{B}_k} \right\Vert_{\text{F}} &= \left\Vert \bm{B^*_k} + \bm{\hat{B}_k} - \bm{B^*_k} \right\Vert_{\text{F}} \geq \left\Vert \bm{B^*_k} \right\Vert_{\text{F}} - \left\Vert \bm{\hat{B}_k} - \bm{B^*_k} \right\Vert_{\text{F}} \\
& \geq \tau (M) - \chi (n,p,M,\delta) - 12 \sqrt{s} \frac{\omega (M)}{\sqrt{\kappa(M)}}.
\end{align*}
Because $M \geq M^*$, and by the definition of $M^*$ in \eqref{eq:Mstar} and the definition of $\nu (M)$ in \eqref{eq:sig-strength}, when
\begin{equation}\label{eq:req-2}
\chi (n,p,M,\delta) \leq \frac{\nu (M)}{3},
\end{equation}
we have
\begin{align*}
\tau (M) - \chi (n,p,M,\delta) - 12 \sqrt{s} \frac{\omega (M)}{\sqrt{\kappa(M)}} > \chi (n,p,M,\delta) + 12 \sqrt{s} \frac{\omega (M)}{\sqrt{\kappa(M)}}.
\end{align*}
Recall that $\epsilon_n = \chi (n,p,M,\delta) + 12 \sqrt{s / \kappa(M)} \omega (M)$, we then have
\begin{equation*}
\max_{k \notin \mathscr{N}_j} \left\Vert \bm{\hat{B}_k} \right\Vert_{\text{F}} \leq \epsilon_n < \min_{k \in \mathscr{N}_j} \left\Vert \bm{\hat{B}_k} \right\Vert_{\text{F}}.
\end{equation*}
which implies that $\hat{\mathscr{N}}_j=\mathscr{N}_j$ by \eqref{eq:nbg-set-est}. 

Finally, we only need to show an asymptotic lower bound for $n$ such that $\Gamma (n,p,M,\delta) \leq \kappa(M)/(32 M^2 s)$ and $\chi (n,p,M,\delta) \leq \nu(M)/3$ are both satisfied. 

First, to satisfy $\Gamma (n,p,M,\delta) \leq \kappa(M)/(32 M^2 s)$, where $\Gamma (n,p,M,\delta)$ is defined in \eqref{eq:Gamma_n}, by dropping the $\log n$ term, we have
\begin{equation*}
n \geq \tilde{O} \left( \frac{M^4 s^2 \log (p^2 M^2/\delta)}{\kappa^2(M)} \right).
\end{equation*}
Note that $\Xi_i(M),i=1,2,3$ are all uniformly bounded for all $M$. 
Next, to satisfy \[\chi (n,p,M,\delta) \leq \nu(M)/3,\] where $\chi (n,p,M,\delta)$ is defined in \eqref{eq:chi-fun}, we need
\begin{equation*}
n \geq \tilde{O} \left( \frac{s \cdot \max \left\{ M^2, \log (p/\delta), M^2 \omega^2 (M) \log (M^2 p /\delta) \right\} }{\kappa (M) \nu^2 (M)} \right).
\end{equation*}
Combine the above results and note that decreasing $3\delta$ to $\delta$ doesn't affect the asymptotic order of $n$, we the have the final result.

\end{proof}

\subsection{Proof of Theorem~\ref{thm:nb-edge-recov-data-driven}}
\label{sec:proof-thm2}

In this section, we prove Theorem~\ref{thm:nb-edge-recov-data-driven}.
In addition to the notations introduced in Appendix~\ref{sec:proof-thm1}, we define more notations that will be used in this section. Let
\begin{equation}
\label{eq:useful-quantities}
\begin{aligned}
\mathcal{Q}^1_{n,\delta} &\coloneqq  1 + 8 \left(  \frac{\log (2/\delta)}{n} + \sqrt{ \frac{\log (2/\delta)}{n} } \right), \\
\mathcal{Q}^2_{n,p,\delta} &\coloneqq 1 + 8 \left(  \frac{\log (2 p/\delta)}{n} + \sqrt{ \frac{\log (2p/\delta)}{n} } \right), \\
\mathcal{Q}^3_{n,p,M,\delta} &\coloneqq 1 + 8 \left(  \frac{\log (2 (p-1) M/\delta)}{n} + \sqrt{ \frac{\log (2 (p-1) M/\delta)}{n} } \right).
\end{aligned}
\end{equation}
Also, let
\begin{multline}
\label{eq:def-Gamma-v}
\Gamma_v =   \frac{ 12 \sigma_{\text{max},0} \vert \mathscr{N}_j \vert d^2_{j0} \log (2p/\delta)}{n} \left( \sum_{k \in \mathscr{N}_j} \Vert \beta_{k} (t,t^{\prime}) \Vert^2_{\text{HS}} \right)   \mathcal{Q}^1_{n,\delta} \\
+ 3 \Phi(M) \sigma_{\text{max},0} \vert \mathscr{N}_j \vert \mathcal{Q}^1_{n,\delta} \left( d^2_{j0} - d^2_{js}(M) \right) \\
+ 3 \sigma_{jr} \mathcal{Q}^1_{n,\delta} \sqrt{d_{js}(M)}.
\end{multline}
We then state the exact form of $\check{\lambda}(n,p,M,\delta)$ as below:
\begin{equation}\label{eq:lambda_check}
\begin{aligned}
\check{\lambda} (n,p,M,\delta)& =  \tilde{\lambda} (n,p,M,\delta) + 2 \Gamma^{1/2}_v \sqrt{ \max_{k \in [p-1]} \tr \left(\bm{\Sigma^X_{kk}}\right)} \sqrt{\mathcal{Q}^2_{n,p-1,\delta}} \\
& \quad + \sqrt{\sigma_{\text{max},0}} \sqrt{ \mathcal{Q}^2_{n,p,\delta} } \sqrt{\frac{(1/c_1)\log(c_2/\delta)}{n}} d_{js}(M) \cdot \left( \sqrt{\Xi_4 (M)} \sqrt{ \mathcal{Q}^2_{n,p,\delta} } + \Gamma^{1/2}_v \right).
\end{aligned}
\end{equation}
We also define that
\begin{equation}\label{eq:Gamma_check}
\begin{aligned}
\check{\Gamma} (n,p,M,\delta)
& =  \Gamma (n,p,M,\delta) +  d^2_{j,\text{max}} \sigma_{\text{max},0} \mathcal{Q}^2_{n,p,\delta} \frac{ (1/c_1) \log \left( c_2 / \delta \right)}{n} \\
&\quad +  2 \sqrt{\Xi_2 (M) } d_{j,\text{max}} \sqrt{\sigma_{\text{max},0}} \sqrt{\mathcal{Q}^2_{n,p,\delta}} \sqrt{\frac{ (1/c_1) \log \left( c_2 / \delta \right)}{n}} \cdot \sqrt{\mathcal{Q}^3_{n,p,M,\delta}}.
\end{aligned}
\end{equation}

As in the proof of Theorem~\ref{thm:nb-edge-recov-fix}, we omit the basis dimension, $M$, and let $\bm{\hat{a}^Y_i}=\bm{\hat{a}^Y_{i,M}}$, $\bm{\hat{a}^{X_k}_i} = \bm{\hat{a}^{X_k}_{i,M}}$, and $\bm{B^*_k}=\bm{B^{\ast}_{k, M}}$ for all $k \in [p-1]$.
Then by \eqref{eq:regression-target-hat}, for all $i \in [n]$, we have
\begin{equation}\label{eq:score-equation-approx}
\bm{\hat{a}^{Y}_i} = \sum^{p-1}_{k=1} \bm{B^{\ast}_{k}} \bm{\hat{a}^{X_k}_i} + \bm{\tilde{u}}_i,
\end{equation}
where $\bm{\tilde{u}_i}=\bm{w_i}+\bm{r_i}+\bm{v_i}$, and $\bm{w_i}$, $\bm{r_i}$, $\bm{v_i}$ are defined in Appendix~\ref{sec:deri2formulas}.
We now give the proof of Theorem~\ref{thm:nb-edge-recov-data-driven}.

\begin{proof}

Let
\begin{equation*}
\bm{\hat{A}^X} = 
\begin{bmatrix}
(\bm{\hat{a}^{X_1}_1})^{\top} &  (\bm{\hat{a}^{X_2}_1})^{\top} & \cdots & (\bm{\hat{a}^{X_{p-1}}_1})^{\top} \\
(\bm{\hat{a}^{X_1}_2})^{\top} &  (\bm{\hat{a}^{X_2}_2})^{\top} & \cdots & (\bm{\hat{a}^{X_{p-1}}_2})^{\top} \\
\vdots & & \vdots \\
(\bm{\hat{a}^{X_1}_n})^{\top} &  (\bm{\hat{a}^{X_2}_n})^{\top} & \cdots & (\bm{\hat{a}^{X_{p-1}}_n})^{\top} \\
\end{bmatrix}
\in \mathbb{R}^{n \times (p-1) M}.
\end{equation*}
We divide columns of $\bm{\hat{A}^X}$ into $p-1$ groups with equal group size $M$, that is, 
$$\bm{\hat{A}^X}=[\bm{\hat{A}^{X_1}},\bm{\hat{A}^{X_2}},\dots,\bm{\hat{A}^{X_{p-1}}}],$$
where $\bm{\hat{A}^{X_k}} \in \mathbb{R}^{n \times M}$ for all $k \in [p-1]$. 
Let $\bm{\hat{Z}}=\bm{\hat{A}^X} \otimes \bm{I_M}$.
Similarly, we divide the columns of $\bm{\hat{Z}}$ into $p-1$ groups with equal group size $M^2$, that is, $\bm{\hat{Z}^X}=[\bm{\hat{Z}^{X_1}},\bm{\hat{Z}^{X_2}},\dots,\bm{\hat{Z}^{X_{p-1}}}]$, where $\bm{\hat{Z}^{X_k}}\in\mathbb{R}^{nM\times M^2}$ for all $k\in [p-1]$. Then, we have $(\bm{\hat{Z}}^{\top}\bm{u})_{k}=\bm{\hat{Z}_k}^{\top}\bm{u}$.
By definition of $\bm{\hat{Z}}$, it is easy to see that $\bm{\hat{Z}_k}=\bm{\hat{A}^{X_k}}\otimes \bm{I_M}$.
In addition, let
$\bm{\hat{a}^{Y}}=((\bm{\hat{a}^Y_1})^{\top}, \cdots, (\bm{\hat{a}^Y_n})^{\top})^{\top} \in \mathbb{R}^{nM}$.
Furthermore, let
$\bm{\tilde{\Sigma}^X_n}=\frac{1}{n}(\bm{\hat{A}^X})^{\top} \bm{\hat{A}^X}$.
Thus, we can rewrite \eqref{eq:loss-lasso} as
\begin{equation*}
\bm{\hat{\beta}} = \argmin_{\bm{\beta}} \left\{ \hat{\mathcal{L}}(\bm{\beta}) + \lambda_n \mathcal{R}(\bm{\beta}) \right\},
\end{equation*}
where
\begin{align*}
\hat{\mathcal{L}}(\bm{\beta}) & =  \frac{1}{2n} \Vert \bm{\hat{a}^Y} - \bm{\hat{Z}} \bm{\beta} \Vert^2_2, \\
\mathcal{R}(\bm{\beta}) & =  \sum^{p-1}_{k=1} \Vert \bm{\beta_k} \Vert_2,
\end{align*}
and $\bm{\beta}=((\bm{\beta_1})^{\top},\dots,(\bm{\beta_{p-1}})^{\top})^{\top} \in \mathbb{R}^{(p-1)M^2}$, with $\bm{\beta_k} \in \mathbb{R}^{M^2}$ for $k \in [p-1]$.

We can then follow the similar proof of Theorem~\ref{thm:nb-edge-recov-fix} to prove Theorem~\ref{thm:nb-edge-recov-data-driven}. The only two modifications needed are new upper bounds for 
\[
\frac{2}{n} \max_{k \in [p-1]} \Vert (\bm{\hat{Z}}^{\top} \bm{\tilde{u}})_{k} \Vert_2 \, \text{ and } \, \vertiii{\bm{\tilde{\Sigma}^X_n}-\bm{\Sigma^{X}}}_{\infty}.
\]
In fact, when we have
\begin{equation}
\label{eq:thm2-proof-helper}
\lambda_n = \check{\lambda} (n,p,M,\delta) \geq \frac{2}{n} \max_{k \in [p-1]}  \Vert (\bm{\hat{Z}}^{\top} \bm{\tilde{u}} )_{k} \Vert_2 \, \text{ and } \, \vertiii{\bm{\tilde{\Sigma}^X_n}-\bm{\Sigma^{X}}}_{\infty} \leq  \check{\Gamma} (n,p,M,\delta)
\end{equation}
hold, where $\check{\chi} (n,p,M,\delta)$ is defined in~\eqref{eq:chi-fun-check}, then by the similar argument in the proof of Theorem~\ref{thm:nb-edge-recov-fix}, we can show that
\begin{equation*}
\sqrt{\sum^{p-1}_{k=1} \left\Vert \hat{\bm{B_k}} - \bm{B^*_k} \right\Vert^2_{\text{F}}} \leq \check{\chi} (n,p,M,\delta) + 12 \sqrt{s} \frac{\omega (M)}{ \sqrt{\kappa(M)}},
\end{equation*}
which further implies that $\hat{\mathscr{N}}_j=\mathscr{N}_j$. Thus, to show that $\hat{\mathscr{N}}_j=\mathscr{N}_j$ holds with high probability, we only need to prove that \eqref{eq:thm2-proof-helper} holds with high probability.


We define the following events.
\begin{align*}
\mathcal{A}_1 & =  \left\{ \frac{1}{n} \sum^n_{i=1} \Vert g^{(i)}_j \Vert^2 \leq  \sigma_{\text{max},0} \mathcal{Q}^2_{n,p,\delta} \, \text{ for all } j \in [p]  \right\}, \\
\mathcal{A}_2 & =  \left\{ \frac{1}{n} \sum^n_{i=1} \Vert e^{(i)} \Vert^2 \leq  \sigma_{jr} \mathcal{Q}^1_{n,p,\delta} \right\}, \\
\mathcal{A}_3 & =  \left\{ \frac{1}{n} \sum^n_{i=1} \Vert \bm{r_i} + \bm{w_i} \Vert^2_2 \leq  \Xi_4 (M) \mathcal{Q}^1_{n,p,\delta} \right\}, \\
\mathcal{A}_4 & =  \left\{ \frac{1}{n} \sum^n_{i=1} \Vert \bm{a^{X_k}}_i \Vert^2_2 \leq \tr \left(\bm{\Sigma^X_{kk}}\right) \mathcal{Q}^2_{n,p-1,\delta} \, \text{ for all } k \in [p-1] \right\}, \\
\mathcal{A}_5 & =  \left\{ \frac{2}{n} \max_{k \in [p-1]} \Vert (\bm{Z}^{\top} (\bm{w}+\bm{r}) )_{k} \Vert_2 \leq \tilde{\lambda} (n,p,M,\delta) \right\}, \\
\mathcal{A}_6 & =  \left\{ \vertiii{\bm{\hat{\Sigma}^X_n}-\bm{\Sigma^{X}}}_{\infty} \leq \Gamma (n,p,M,\delta)  \right\}, \\
\mathcal{A}_7 & =  \left\{ \frac{1}{n} \sum^n_{i=1} \Vert a^{X_k}_{i m} \Vert^2_2 \leq \Xi_2 (M) \mathcal{Q}^3_{n,p,\delta} \, \text{ for all } m \in [M] \text{ and } k \in [p-1] \right\},
\end{align*}
where $\tilde{\lambda} (n,p,M,\delta)$ is defined in \eqref{eq:lambda_tilde}. Then we claim that under Assumption~\ref{assump:basis-est-accu}, we have
\begin{equation}\label{eq:claim2proof}
\cap^{7}_{i=1} \mathcal{A}_i \Longrightarrow \frac{2}{n} \max_{k \in [p-1]} \Vert (\bm{\hat{Z}}^{\top} \bm{\tilde{u}} )_{k} \Vert_2 \leq \check{\lambda} (n,p,M,\delta) \, \text{and} \, \vertiii{\bm{\tilde{\Sigma}^X_n}-\bm{\Sigma^{X}}}_{\infty} \leq  \check{\Gamma} (n,p,M,\delta).
\end{equation}

We now prove the above claim. We first prove that $\cap^{7}_{i=1} \mathcal{A}_i$ implies that 
\[
\frac{2}{n} \max_{k \in [p-1]} \Vert (\bm{\hat{Z}}^{\top} \bm{\tilde{u}} )_{k} \Vert_2 \leq \check{\lambda} (n,p,M,\delta).
\] 
Note that for all $k \in [p-1]$, we have
\begin{align*}
\Vert (\bm{\hat{Z}}^{\top} \bm{\tilde{u}} )_{k} \Vert_2 & = \Vert \bm{\hat{Z}_k}^{\top} \bm{\tilde{u}} \Vert_2 \leq \Vert (\bm{\hat{Z}_k} - \bm{Z}_k)^{\top}\bm{\tilde{u}} \Vert_2 + \Vert \bm{Z}_k^{\top}\bm{\tilde{u}} \Vert_2 \\
& = \Vert (\bm{\hat{Z}_k} - \bm{Z}_k)^{\top}\bm{\tilde{u}} \Vert_2 + \Vert \bm{Z_k}^{\top} (\bm{w}+\bm{r}+\bm{v}) \Vert_2 \\
& \leq \Vert (\bm{\hat{Z}_k} - \bm{Z}_k)^{\top}\bm{\tilde{u}} \Vert_2 + \Vert \bm{Z_k}^{\top} (\bm{w}+\bm{r}) \Vert_2 + \Vert \bm{Z_k}^{\top} \bm{v} \Vert_2,
\end{align*}
where $\bm{v} = (v_{im})_{\{m \in [M]\}}$ with $v_{im}$ defined in~\eqref{eq:def-v}.
The above inequality implies that
\begin{multline}\label{eq:Zu-decompose}
\frac{2}{n} \max_{k \in [p-1]} \Vert (\bm{Z}^{\top} \bm{u} )_{k} \Vert_2 \leq \frac{2}{n} \max_{k \in [p-1]} \Vert (\bm{\hat{Z}_k} - \bm{Z}_k)^{\top}\bm{\tilde{u}} \Vert_2 \\
+  \frac{2}{n} \max_{k \in [p-1]} \Vert \bm{Z_k}^{\top} (\bm{w}+\bm{r}) \Vert_2 + \frac{2}{n} \max_{k \in [p-1]} \Vert \bm{Z_k}^{\top} \bm{v} \Vert_2
\end{multline}
By $\mathcal{A}_5$, we have the second term to be bounded by $\tilde{\lambda}(n,p,M,\delta)$. To bound the third term, note that
\begin{equation}
\label{eq:proof-thm2-helper1}
\left\Vert \frac{1}{n} \bm{Z_k}^{\top} \bm{v} \right\Vert_2 = \left\Vert \frac{1}{n} (\bm{A^{X_k}} \otimes \bm{I_M})^{\top} \bm{v} \right\Vert_2 = \left\Vert \frac{1}{n} \sum^n_{i=1} (\bm{a^{X_k}_i} \otimes \bm{I_M}) \bm{v_i} \right\Vert_{2} = \left\Vert \frac{1}{n} \sum^n_{i=1} \bm{a^{X_k}_i} \bm{v_i}^{\top} \right\Vert_{\text{F}},
\end{equation}
and by Lemma~\ref{lemma:Frob-bd}, we further have
\begin{equation}
\label{eq:proof-thm2-helper2}
\left\Vert \frac{1}{n} \bm{Z_k}^{\top} \bm{v} \right\Vert_2 \leq \sqrt{\frac{1}{n} \sum^n_{i=1} \left\Vert \bm{a^{X_k}_i} \right\Vert^2_2 } \sqrt{\frac{1}{n} \sum^n_{i=1} \left\Vert \bm{v_i} \right\Vert^2_2 }.
\end{equation}
By Lemma~\ref{lemma:v-bd} with
\begin{align*}
\Gamma_g(j) & = \sigma_{\text{max},0} \mathcal{Q}^2_{n,p,\delta} \, \text{ for all } j \in [p] \\
\Gamma_e & = \sigma_{jr} \mathcal{Q}^1_{n,\delta} \\
\Gamma_{\phi}(m) & = d_{jm} \sqrt{\frac{ (1/c_1) \log \left( c_2 / \delta \right)}{n}}
\end{align*}
because of $\mathcal{A}_1$, $\mathcal{A}_2$ and Assumption~\ref{assump:basis-est-accu}, we have
\begin{equation}
\label{eq:def-Gamma_v}
\frac{1}{n} \sum^n_{i=1} \Vert \bm{v_i} \Vert^2 \leq \Gamma_v,
\end{equation}
where $\Gamma_v$ is defined by~\eqref{eq:def-Gamma-v}.
And combine with $\mathcal{A}_4$, we have
\begin{align*}
& \left\Vert \frac{2}{n} \bm{Z_k}^{\top} \bm{v} \right\Vert_2 \leq 2 \sqrt{\frac{1}{n} \sum^n_{i=1} \left\Vert \bm{a^{X_k}_i} \right\Vert^2_2 } \sqrt{\frac{1}{n} \sum^n_{i=1} \left\Vert \bm{v_i} \right\Vert^2_2 } \leq  2 \sqrt{ \tr \left(\bm{\Sigma^X_{kk}}\right)} \sqrt{\mathcal{Q}^2_{n,p-1,\delta}}\cdot \Gamma^{1/2}_v.
\end{align*}
Finally, to bound the first term in~\eqref{eq:Zu-decompose}, by similar arguments as~\eqref{eq:proof-thm2-helper1} and~\eqref{eq:proof-thm2-helper2}, we have
\begin{equation}
\label{eq:proof-thm2-helper3}
\left\Vert \frac{1}{n} (\bm{\hat{Z}_k} -\bm{Z_k})^{\top} \bm{\tilde{u}} \right\Vert_2 \leq \sqrt{\frac{1}{n} \sum^n_{i=1} \left\Vert \bm{\hat{a}^{X_k}_i} - \bm{a^{X_k}_i} \right\Vert^2_2 } \sqrt{\frac{1}{n} \sum^n_{i=1} \left\Vert \bm{\tilde{u}_i} \right\Vert^2_2 }.
\end{equation}
Since we have
\begin{align*}
\frac{1}{n} \sum^n_{i=1} \left\Vert \bm{\hat{a}^{X_k}_i} - \bm{a^{X_k}_i} \right\Vert^2_2 & \leq \left( \frac{1}{n} \sum^n_{i=1} \Vert g_{ik} \Vert^2 \right) \sum^M_{m=1} \Vert \hat{\phi}_m - \phi_m \Vert^2 \\
& \leq \sigma_{\text{max},0} \mathcal{Q}^2_{n,p,\delta} \frac{(1/c_1)\log(c_2/\delta)}{n} d^2_{js}(M)
\end{align*}
by $\mathcal{A}_1$ and Assumption~\ref{assump:basis-est-accu},
and
\begin{align*}
\sqrt{\frac{1}{n} \sum^n_{i=1} \left\Vert \bm{\tilde{u}_i} \right\Vert^2_2} & \leq \sqrt{\frac{1}{n} \sum^n_{i=1} \left\Vert \bm{r_i} + \bm{w_i} \right\Vert^2_2} + \sqrt{\frac{1}{n} \sum^n_{i=1} \left\Vert \bm{v_i} \right\Vert^2_2} \\
& \leq \sqrt{\Xi_4 (M)} \sqrt{ \mathcal{Q}^1_{n,\delta} } + \Gamma^{1/2}_v
\end{align*}
by $\mathcal{A}_3$ and~\eqref{eq:def-Gamma_v}, then by~\eqref{eq:proof-thm2-helper3}, we have
\begin{multline*}
\left\Vert \frac{1}{n} (\bm{\hat{Z}_k} -\bm{Z_k})^{\top} \bm{\tilde{u}} \right\Vert_2 \leq \sqrt{\sigma_{\text{max},0}} \sqrt{ \mathcal{Q}^2_{n,p,\delta} } \sqrt{\frac{(1/c_1)\log(c_2/\delta)}{n}} \cdot \\ d_{js}(M) \left( \sqrt{\Xi_4 (M)} \sqrt{ \mathcal{Q}^1_{n,\delta} } + \Gamma^{1/2}_v \right).
\end{multline*}

Thus, combine the bounds of three terms in~\eqref{eq:Zu-decompose} and by the definition of $\check{\lambda}(n,p,M,\delta)$ in~\eqref{eq:lambda_check}, we have
\begin{equation*}
\frac{2}{n} \max_{k \in [p-1]} \left\Vert \bm{Z_k}^{\top} \bm{v} \right\Vert_2 \leq \check{\lambda}(n,p,M,\delta).
\end{equation*}

We then prove that $\cap^{7}_{i=1} \mathcal{A}_i$ implies that $\vertiii{\bm{\tilde{\Sigma}^X_n}-\bm{\Sigma^{X}}}_{\infty} \leq \check{\Gamma} (n,p,M,\delta)$. Note that
$$\vertiii{\bm{\tilde{\Sigma}^X_n}-\bm{\Sigma^{X}}}_{\infty} \leq \vertiii{\bm{ \tilde{\Sigma}^{X}} - \bm{\hat{\Sigma}^X_n} }_{\infty} + \vertiii{ \bm{\hat{\Sigma}^X_n} -\bm{\Sigma^{X}} }_{\infty},
$$
and by $\mathcal{A}_6$, we thus only need to prove that $\vertiii{\bm{ \tilde{\Sigma}^{X}} - \bm{\hat{\Sigma}^X_n} }_{\infty} \leq \check{\Gamma}(n,p,M,\delta)-\Gamma(n,p,M,\delta)$. Note that
\begin{align*}
\vertiii{\bm{ \tilde{\Sigma}^{X}} - \bm{\hat{\Sigma}^X_n} }_{\infty} & =  \max_{1 \leq k, k^{\prime} \leq p-1} \vertiii{ \frac{1}{n} (\bm{ \hat{A}^{X_k} })^{\top} \bm{ \hat{A}^{X_{k^{\prime}}} } - \frac{1}{n} \bm{ (A^{X_k} })^{\top} \bm{ A^{X_{k^{\prime}}} } }_{\infty} \\
& =  \max_{1 \leq k, k^{\prime} \leq p-1}  \max_{1 \leq m, m^{\prime} \leq M} \left\vert \frac{1}{n} \sum^n_{i=1} \hat{a}^{X_k}_{im} \hat{a}^{X_{k^{\prime}}}_{im^{\prime}} - a^{X_{k}}_{im} a^{X_{k^{\prime}}}_{im^{\prime}} \right\vert, \numberthis \label{eq:intermi-step-1}
\end{align*}
and
\begin{align*}
& \left\vert \frac{1}{n} \sum^n_{i=1} \hat{a}^{X_k}_{im} \hat{a}^{X_{k^{\prime}}}_{im^{\prime}} - a^{X_{k}}_{im} a^{X_{k^{\prime}}}_{im^{\prime}} \right\vert \\
& =  \left\vert \frac{1}{n} \sum^n_{i=1} \left( \hat{a}^{X_k}_{im} -  a^{X_{k}}_{im} \right) \left( \hat{a}^{X_{k^{\prime}}}_{im^{\prime}} - a^{X_{k^{\prime}}}_{im^{\prime}}\right) + \frac{1}{n} \sum^n_{i=1} a^{X_{k}}_{im} \left( \hat{a}^{X_{k^{\prime}}}_{im^{\prime}} - a^{X_{k^{\prime}}}_{im^{\prime}}\right) \right. \\
& \quad \left. + \frac{1}{n} \sum^n_{i=1} a^{X_{k^{\prime}}}_{im^{\prime}} \left( \hat{a}^{X_k}_{im} -  a^{X_{k}}_{im} \right) \right\vert \\
& \leq  \left\vert \frac{1}{n} \sum^n_{i=1} \left( \hat{a}^{X_k}_{im} -  a^{X_{k}}_{im} \right) \left( \hat{a}^{X_{k^{\prime}}}_{im^{\prime}} - a^{X_{k^{\prime}}}_{im^{\prime}} \right) \right\vert + \left\vert \frac{1}{n} \sum^n_{i=1} a^{X_{k}}_{im} \left( \hat{a}^{X_{k^{\prime}}}_{im^{\prime}} - a^{X_{k^{\prime}}}_{im^{\prime}}\right) \right\vert \\
& \quad + \left\vert \frac{1}{n} \sum^n_{i=1} a^{X_{k^{\prime}}}_{im^{\prime}} \left( \hat{a}^{X_k}_{im} -  a^{X_{k}}_{im} \right) \right\vert \\
& \leq   \sqrt{\frac{1}{n} \sum^n_{i=1} \left( \hat{a}^{X_k}_{im} -  a^{X_{k}}_{im} \right)^2} \sqrt{\frac{1}{n} \sum^n_{i=1} \left( \hat{a}^{X_{k^{\prime}}}_{im^{\prime}} - a^{X_{k^{\prime}}}_{im^{\prime}} \right)^2} \\
& \quad + \sqrt{\frac{1}{n} \sum^n_{i=1} \left( \hat{a}^{X_k}_{im} \right)^2} \sqrt{\frac{1}{n} \sum^n_{i=1} \left( \hat{a}^{X_{k^{\prime}}}_{im^{\prime}} - a^{X_{k^{\prime}}}_{im^{\prime}} \right)^2} + \sqrt{\frac{1}{n} \sum^n_{i=1} \left( a^{X_{k^{\prime}}}_{im^{\prime}} \right)^2} \sqrt{\frac{1}{n} \sum^n_{i=1} \left( \hat{a}^{X_k}_{im} -  a^{X_{k}}_{im} \right)^2}. \numberthis \label{eq:intermi-step-2}
\end{align*}
By Lemma~\ref{lemma:a-diff-bd} with
\begin{align*}
\Gamma_g(j) & = \sigma_{\text{max},0} \mathcal{Q}^2_{n,p,\delta} \, \text{ for all } j \in [p], \\
\Gamma_{\phi}(m) & = d_{jm} \sqrt{\frac{ (1/c_1) \log \left( c_2 / \delta \right)}{n}},
\end{align*}
we have
\begin{equation*}
\frac{1}{n} \sum^n_{i=1} \left( \hat{a}^{X_k}_{im} - a^{X_k}_{im} \right)^2 \leq  d^2_{j,\text{max}} \sigma_{\text{max},0} \mathcal{Q}^2_{n,p,\delta} \frac{ (1/c_1) \log \left( c_2 / \delta \right)}{n}
\end{equation*}
for all $m \in [M]$ and $k \in [p-1]$. And by $\mathcal{A}_7$, we have
\begin{equation*}
\frac{1}{n} \sum^n_{i=1} \Vert a^{X_k}_{i m} \Vert^2_2 \leq \Xi_2 (M) \mathcal{Q}^3_{n,p,\delta}
\end{equation*}
for all $m \in [M]$ and $k \in [p-1]$. Combine with \eqref{eq:intermi-step-1} and \eqref{eq:intermi-step-2}, we have
\begin{align*}
& \vertiii{\bm{ \tilde{\Sigma}^{X}} - \bm{\hat{\Sigma}^X_n} }_{\infty} \\
\leq \, &  d^2_{j,\text{max}} \sigma_{\text{max},0} \mathcal{Q}^2_{n,p,\delta} \frac{ (1/c_1) \log \left( c_2 / \delta \right)}{n} \\
+ & 2 \sqrt{\Xi_2 (M) } d_{j,\text{max}} \sqrt{\sigma_{\text{max},0}} \sqrt{\mathcal{Q}^2_{n,p,\delta}} \sqrt{\frac{ (1/c_1) \log \left( c_2 / \delta \right)}{n}} \cdot \sqrt{\mathcal{Q}^3_{n,p,M\delta}} \\
& =  \check{\Gamma}(n,p,M,\delta)-\Gamma(n,p,M,\delta) \leq \check{\Gamma}(n,p,M,\delta).
\end{align*}

Therefore, by \eqref{eq:claim2proof} and Lemma~\ref{lemma:LowBoundLambda}, Lemma~\ref{lemma:EleAbsCovMaxBound}, Lemma~\ref{lemma:g2moment} and Lemma~\ref{lemma:rw2moment}, we have
\begin{align*}
& \mathbb{P} \left\{ \frac{2}{n} \max_{k \in [p-1]}  \Vert (\bm{Z}^{\top} \bm{u} )_{k} \Vert_2 \leq \check{\lambda} (n,p,M,\delta) \, \text{and} \, \vertiii{\bm{\tilde{\Sigma}^X_n}-\bm{\Sigma^{X}}}_{\infty} \leq  \check{\Gamma} (n,p,M,\delta) \right\} \\
& \geq  1 - \mathbb{P} \left\{ \cup^7_{i=1} \bar{\mathcal{A}}_i \right\} \\
& \geq  1 - 8 \delta.
\end{align*}

Finally, we only need $n$ large enough such that $\check{\Gamma} (n,p,M,\delta) \leq \kappa(M) / (32 M^2 s)$ and $\check{\chi} (n,p,M,\delta) \leq \nu (M)/3$ where $\check{\chi} (n,p,M,\delta)$ is defined in \eqref{eq:chi-fun-check}. After dropping $\log(n)$ term, to satisfy the first condition, we need
\begin{equation*}
n \geq O \left(  \frac{ M^4 s^2 \log (p^2 M^2 /\delta) }{\kappa^2 (M)} \right),
\end{equation*}
and to satisfy the second condition, we need
\begin{equation*}
\begin{aligned}
n & \geq O \left( \frac{1 }{\kappa (M) \nu^2 (M)} \max \left\{ sM^2, s\log (p/\delta), sM^2 \omega^2 (M) \log (M^2 p /\delta), \right.\right. \\
& \qquad \qquad\qquad\qquad\qquad\left.\left. s^3 M^2 (\log(1/\delta)^2), s (d^2_{j0}-d^2_{js}(M) \Phi^2(M)) \right\}  \right).
\end{aligned}
\end{equation*}
Then combine the above results and note that decreasing $8\delta$ to $\delta$ doesn't affect the asymptotic order of $n$, we the have the final result.

\end{proof}

\section{Useful Lemmas}
\label{sec:useful-lemmas}

Recall that $\mathcal{C}_{n,\delta}$ is defined in \eqref{eq:Cndelta} and $\mathcal{Q}^1_{n,\delta}$, $\mathcal{Q}^2_{n,p,\delta}$, and $\mathcal{Q}^1_{n,p, M,\delta}$ are defined in~\eqref{eq:useful-quantities}.

\begin{lemma}\label{lemma:UppBoundL2X}
Let $\delta \in (0, 1]$. Then $\max_{k \in[p-1]}\Vert \bm{A^{X_k}} \Vert_2/\sqrt{n}\leq \mathcal{C}_{n,\delta}$ with probability at least $1-\delta$.
\end{lemma}
\begin{proof}
The result follows directly from Theorem 6.1 of \citet{Wainwright2019High} and a union bound.
\end{proof}

\begin{lemma}\label{lemma:UppBoundDualR}
Suppose $\max_{k \in [p-1]}\Vert \bm{A^{X_k}} \Vert_2/\sqrt{n}\leq \mathcal{C}_{n,\delta_1}$ for some $0<\delta_1 \leq 1$. Then there exists a constant $c>0$ such that 
\begin{multline*}
\max_{k \in [p-1]} \left\Vert \frac{1}{n} \bm{Z_k}^\top \bm{u} \right\Vert_2
\leq  \frac{M \mathcal{C}_{n,\delta_1} \sqrt{ \Xi_1(M) }}{\sqrt{n}}+ 2 \mathcal{C}_{n,\delta_1} \sqrt{\Xi_1 (M)} \sqrt{ \frac{ \log (4(p-1)/\delta_2) }{n} } \\
+   \omega(M) \left\{ 1 + 7\sqrt{3} \frac{M \sqrt{ \log (6(p-1)/\delta_2) + 2 \log M}}{\sqrt{n}} + \right. \\
 \left. \frac{8 c}{3} \frac{ M (\log (2n)) (\log (6(p-1)/\delta_2) + 2\log M) }{n} \right\}
\end{multline*}
holds with probability at least $1-\delta_2$, $\delta_2 \in (0, 1]$.
\end{lemma}
\begin{proof}
Note that $\Vert \bm{Z_k} \Vert_2=\Vert \bm{A^{X_k}} \otimes \bm{I_M} \Vert_2 = \Vert \bm{A^{X_k}} \Vert_2$. Thus $\Vert \bm{Z_k} / \sqrt{n} \Vert_2 \leq \mathcal{C}_{n,\delta_1}$, $k \in [p-1]$, since $\max_{k \in [p-1]}\Vert \bm{A^{X_k}} \Vert_2/\sqrt{n}\leq \mathcal{C}_{n,\delta_1}$. Recall that $\bm{u}=\bm{w} + \bm{r}$, where $\bm{u}$ is defined in~\eqref{eq:u-def}, $\bm{w}=(\bm{w_1}^{\top},\bm{w_2}^{\top},\dots,\bm{w_n}^{\top})^{\top} \in \mathbb{R}^{nM}$ and $\bm{r}=(\bm{r_1}^{\top},\bm{r_2}^{\top},\dots,\bm{r_n}^{\top})^{\top} \in \mathbb{R}^{nM}$. Both $\bm{w_i}$ and $\bm{r_i}$ are Gaussian vectors with mean zero, and covariance matrices $\bm{\Sigma^{w}}$ and $\bm{\Sigma^{r}}$ respectively for all $i \in [n]$ , where we dropped the superscript $M$ to simplify the notation, and $\bm{w_i} \independent \bm{r_i}$.

Fix $k \in [p-1]$. Let $\bm{r_i} = \bm{r^1_i} + \bm{r^2_i}$, where 
\[
\bm{r^1_i} = \bm{r_i} - \bm{\Sigma^{r,X_k}} \left( \bm{\Sigma^{X}_{kk}} \right)^{-1} \bm{a^{X_k}_i}
\qquad \text{and} \qquad \bm{r^2_i}= \bm{\Sigma^{r,X_k}} \left( \bm{\Sigma^{X}_{kk}} \right)^{-1} \bm{a^{X_k}_i}.
\]
Since both $\bm{r_i}$ and $\bm{a^{X_k}_i}$ are Gaussian vectors and $\Cov (\bm{a^{X_k}_i}, \bm{r^1_i})=\bm{0}$, we have $\bm{a^{X_k}_i} \independent \bm{r^1_i}$. Then
\begin{equation}\label{eq:lemma2-1}
\left\Vert \frac{1}{n} \bm{Z_k}^\top \bm{u} \right\Vert_2 \leq
\left\Vert \frac{1}{n} \bm{Z_k}^\top (\bm{w}+\bm{r^1}) \right\Vert_2 + \left\Vert \frac{1}{n} \bm{Z_k}^\top \bm{r^2} \right\Vert_2.
\end{equation}

We upper bound the first term on the right hand side of \eqref{eq:lemma2-1}. Let $\bm{\xi_i}=\bm{w_i} + \bm{r^1_i}$. Then $\bm{\xi_i}$ is a Gaussian vector with mean zero and covariance matrix $\bm{\Sigma^{\xi}_k }$, with $\bm{\Sigma^{\xi}_k } = \bm{\Sigma^w}+\bm{\Sigma^r}-\bm{\Sigma^{r,X_k}} \left( \bm{\Sigma^{X}_{kk}} \right)^{-1} \bm{\Sigma^{{X_k},r}}$. Note that $\bm{\xi} \independent \bm{Z_k}$ and we establish a bound conditional on $\bm{Z_k}$ first.

For any pair $\bm{\theta},\bm{\theta^{\prime}} \in \mathbb{R}^{nM}$, we have
\begin{equation*}
\left\vert \left\Vert \frac{\bm{Z_k^{\top}} \bm{\theta}}{n} \right\Vert_2 - \left\Vert \frac{\bm{Z_k^{\top}} \bm{\theta}^{\prime}}{n} \right\Vert_2 \right\vert \leq \frac{1}{n} \left\Vert \bm{Z_k}^{\top} (\bm{\theta}-\bm{\theta^{\prime}}) \right\Vert_2\leq \frac{1}{n} \Vert \bm{Z_k} \Vert_2 \Vert \bm{\theta}-\bm{\theta^{\prime}} \Vert_2 \leq
\frac{\mathcal{C}_{n,\delta_1}}{\sqrt{n}} \Vert \bm{\theta}-\bm{\theta^{\prime}} \Vert_2.
\end{equation*}
Therefore, $\bm{\theta} \mapsto\Vert \frac{\bm{Z_k}^{\top} \bm{\theta}}{n}\Vert_2$ is a Lipschitz function with a Lipschitz constant $\mathcal{C}_{n,\delta_1}/\sqrt{n}$. Note that $\bm{\xi}$ has a strongly log-concave distribution with the parameter $(\rho_{max}(\bm{I_M} \otimes \bm{\Sigma^{\xi}_k} ))^{-1}=(\rho_{max}( \bm{\Sigma^{\xi}_k} ))^{-1}$. 
Then, by Theorem 3.16 of \citet{Wainwright2019High}, we have
\begin{equation}\label{eq:LipSchitzBound}
\mathbb{P}\left\{ \left\Vert \frac{ \bm{Z_k}^{\top} \bm{\xi} }{n} \right\Vert_2 \geq \mathbb{E}\left[\left\Vert \frac{\bm{Z_k}^{\top} \bm{\xi}}{n} \right\Vert_2 \right]+\Delta \right\}\leq 2\exp{\left(-\frac{n\Delta^2}{4C^{2}_{n,\delta_1}\rho_{\max}(\bm{\Sigma^{\xi}_k})}\right)}\quad \text{for all }\Delta>0.
\end{equation}
Next, we bound $\mathbb{E}\left[ n^{-1}\left\Vert {\bm{Z_k}^{\top} \bm{\xi}}\right\Vert_2 \right]$ using the Sudakov-Fernique inequality \citep[Theorem 2.2.3 of][]{MR2319516}. For any vector $\bm{\theta} \in \mathbb{R}^{M^2}$, let $\zeta_{\bm{\theta}} = {n}^{-1} \langle \bm{\theta}, \bm{Z_k}^{\top} \bm{\xi} \rangle$. Then $n^{-1}\Vert \bm{Z_k}^{\top} \bm{\xi}  \Vert_2 = \max_{\Vert \bm{\theta} \Vert_2=1} \zeta_{\bm{\theta}}$. For any two vectors $\bm{\theta}, \bm{\theta^{\prime}} \in \mathbb{R}^{M^2}$,
we have
\begin{multline*}
\mathbb{E}\left[(\zeta_{\bm{\theta}}-\zeta_{\bm{\theta^{\prime}}})^2\right] =\mathbb{E}\left[\left(\frac{1}{n}\left\langle \bm{\theta}-\bm{\theta^{\prime}}, \bm{Z_k}^{\top} \bm{\xi} \right\rangle\right)^2 \right]
=\frac{1}{n^2}\mathbb{E}\left[(\bm{\theta}-\bm{\theta^{\prime}})^{\top} \bm{Z_k}^{\top} \bm{\xi} \bm{\xi}^{\top} \bm{Z_k} (\bm{\theta}-\bm{\theta^{\prime}}) \right] \\
\leq \frac{\rho_{\max}(\bm{\Sigma^{\xi}_k})}{n} \frac{\Vert \bm{Z_k} \Vert^2_2}{n} \Vert \bm{\theta}-\bm{\theta^{\prime}} \Vert^2_2 
\leq \frac{C^2_{n,\delta_1} \rho_{\max}(\bm{\Sigma^{\xi}_k})}{n}\Vert \bm{\theta}-\bm{\theta^{\prime}} \Vert^2_2.
\end{multline*}
We define another Gaussian process, 
\[
\tilde{\zeta}_{\bm{\theta}}=\left(\mathcal{C}_{n,\delta_1} \sqrt{\rho_{\max}(\bm{\Sigma^{\xi}_k})/n}\right) \cdot \langle \bm{\theta}, \bm{\epsilon} \rangle,
\] 
where $\bm{\epsilon} \sim N(\bm{0},\bm{I_{M^2}})$. For any pair $\bm{\theta}, \bm{\theta^{\prime}} \in \mathbb{R}^{M^2}$, we have
\begin{equation*}
\mathbb{E}\left[(\tilde{\zeta}_{\bm{\theta}}-\tilde{\zeta}_{\bm{\theta^{\prime}}})^2\right]=\frac{C^2_{n,\delta_1} \rho_{\max}(\bm{\Sigma^{\xi}_k})}{n}\Vert \bm{\theta}-\bm{\theta^{\prime}} \Vert^2_2 \geq \mathbb{E}\left[(\zeta_{\bm{\theta}}-\zeta_{\bm{\theta^{\prime}}})^2\right].
\end{equation*}
Then, by the Sudakov-Fernique inequality, we have
\begin{equation*}
\mathbb{E}\left[\left\Vert \frac{ \bm{Z_k}^{\top} \bm{\xi} }{n} \right\Vert_2 \right]=\mathbb{E}\left[\max_{\Vert \bm{\theta} \Vert_2=1} \zeta_{\bm{\theta}}\right] \leq \mathbb{E} \left[\max_{\Vert \bm{\theta} \Vert_2=1}\tilde{\zeta}_{\bm{\theta}} \right].
\end{equation*}
On the other hand,
\begin{equation*}
\begin{aligned}
\mathbb{E}\left[\max_{\Vert \bm{\theta} \Vert_2=1} \tilde{\zeta}_{\bm{\theta}}\right]  &=  \frac{\mathcal{C}_{n,\delta_1} \sqrt{\rho_{\max}(\bm{\Sigma^{\xi}_k})}}{\sqrt{n}}\mathbb{E}\left[\Vert \bm{\epsilon} \Vert_2\right]
\leq\frac{\mathcal{C}_{n,\delta_1}\sqrt{\rho_{\max}(\bm{\Sigma^{\xi}_k})}}{\sqrt{n}}\sqrt{\mathbb{E}\left[\Vert \bm{\epsilon} \Vert^2_2\right]}
\\
& =\frac{M \mathcal{C}_{n,\delta_1} \sqrt{\rho_{\max} (\bm{\Sigma^{\xi}_k})}}{\sqrt{n}}.
\end{aligned}
\end{equation*}
Combining \eqref{eq:LipSchitzBound} with the above upper bound, we have
\begin{equation*}
\mathbb{P}\left\{\left\Vert \frac{\bm{Z_k}^{\top} \bm{\xi}}{n} \right\Vert_2 \geq \frac{M \mathcal{C}_{n,\delta_1} \sqrt{\rho_{\max}(\bm{\Sigma^{\xi}_k})}}{\sqrt{n}}+\Delta \right\}\leq 2\exp{\left(-\frac{n\Delta^2}{4C^{2}_{n,\delta_1}\rho_{\max}(\bm{\Sigma^{\xi}_k})}\right)}, \quad \Delta>0.
\end{equation*}
The above inequality holds for any $\bm{Z_k}$ such that $\Vert \bm{Z_k} / \sqrt{n} \Vert_2 \leq \mathcal{C}_{n,\delta_1}$ and, therefore, is valid unconditionally as well. Since $\Xi_1(M) = \max_{k \in [p-1]} \rho_{\max}(\bm{\Sigma^{\xi}_k})$, if follows from the union bound that \begin{equation}\label{eq:lemma2-5}
\mathbb{P}\left\{\max_{k \in [p-1]}\;\left\Vert \frac{ \bm{Z_k}^{\top} \bm{\xi}}{n} \right\Vert_2 \geq \frac{M \mathcal{C}_{n,\delta_1} \sqrt{\Xi_1(M)}}{\sqrt{n}}+\Delta \right\}\leq 2 (p-1) \exp{\left(-\frac{n\Delta^2}{4C^{2}_{n,\delta_1} \Xi_1(M)}\right)},
\quad \Delta>0.
\end{equation}

Next, we upper bound the second term in~\eqref{eq:lemma2-1}. For any $k \in [p-1]$, we have
\begin{align*}
\left\Vert \frac{1}{n} \bm{Z_k}^\top \bm{r^2} \right\Vert_2 
& =  \left\Vert \frac{1}{n} \sum^n_{i=1}  (\bm{a^{X_k}_i} \otimes \bm{I_M}) \bm{r^2_i} \right\Vert_2 
=  \left\Vert \frac{1}{n} \sum^n_{i=1}  \bm{r^2_i} (\bm{a^{X_k}_i})^{\top} \right\Vert_{\text{F}} \\
& =  \left\Vert \frac{1}{n} \sum^n_{i=1}  \bm{\Sigma^{r,X_k}} \left( \bm{\Sigma^{X}_{kk}} \right)^{-1} \bm{a^{X_k}_i} (\bm{a^{X_k}_i})^{\top} \right\Vert_{\text{F}} \\
& \leq  \left\Vert \frac{1}{n} \sum^n_{i=1}  \bm{\Sigma^{r,X_k}} \left( \bm{\Sigma^{X}_{kk}} \right)^{-1} \bm{a^{X_k}_i} (\bm{a^{X_k}_i})^{\top} - \bm{\Sigma^{r,X_k}} \right\Vert_{\text{F}} + \left\Vert \bm{\Sigma^{r,X_k}} \right\Vert_{\text{F}}. \\
& =  \left\Vert \bm{\Sigma^{r,X_k}} \left\{ \frac{1}{n} \sum^n_{i=1}  \left( \bm{\Sigma^{X}_{kk}} \right)^{-1} \bm{a^{X_k}_i} (\bm{a^{X_k}_i})^{\top}  - \bm{I_M} \right\} \right\Vert_{\text{F}} + \left\Vert \bm{\Sigma^{r,X_k}} \right\Vert_{\text{F}}\\
& \leq  \left\Vert \bm{\Sigma^{r,X_k}} \right\Vert_{\text{F}} \left\Vert  \frac{1}{n} \sum^n_{i=1}  \left( \bm{\Sigma^{X}_{kk}} \right)^{-1} \bm{a^{X_k}_i} (\bm{a^{X_k}_i})^{\top}  - \bm{I_M} \right\Vert_{\text{F}} + \left\Vert \bm{\Sigma^{r,X_k}} \right\Vert_{\text{F}}\\
& =  \left\Vert \bm{\Sigma^{r,X_k}} \right\Vert_{\text{F}} \left\Vert  \frac{1}{n} \sum^n_{i=1} \left( \left( \bm{\Sigma^{X}_{kk}} \right)^{-1/2} \bm{a^{X_k}_i} \right) \left( \left( \left( \bm{\Sigma^{X}_{kk}} \right)^{-1/2} \bm{a^{X_k}_i} \right) \right)^{\top} - \bm{I_M}  \right\Vert_{\text{F}} + \left\Vert \bm{\Sigma^{r,X_k}} \right\Vert_{\text{F}} \\
& \overset{\Delta}{=} \left\Vert \bm{\Sigma^{r,X_k}} \right\Vert_{\text{F}} \left\{ \left\Vert \frac{1}{n} \sum^n_{i=1} \bm{b^{X_k}_i} \left( \bm{b^{X_k}_i} \right)^{\top} - \bm{I_M} \right\Vert_{\text{F}} + 1 \right\}, \numberthis \label{eq:lemma2-2}
\end{align*}
where $\bm{b^{X_k}_i} = ( \bm{\Sigma^{X}_{kk}} )^{-1/2} \bm{a^{X_k}_i} \sim N(0,\bm{I_M})$. Since
\begin{equation}\label{eq:lemma2-3}
\left\Vert \frac{1}{n} \sum^n_{i=1} \bm{b^{X_k}_i} \left( \bm{b^{X_k}_i} \right)^{\top} - \bm{I_M} \right\Vert_{\text{F}} \leq M \vertiii{\frac{1}{n} \sum^n_{i=1} \bm{b^{X_k}_i} \left( \bm{b^{X_k}_i} \right)^{\top} - \bm{I_M} }_{\infty},
\end{equation}
we only need to bound $ \vertiii{\frac{1}{n} \sum^n_{i=1} \bm{b^{X_k}_i} \left( \bm{b^{X_k}_i} \right)^{\top} - \bm{I_M} }_{\infty}$. 
We use Theorem 4.1 in \citet{kuchibhotla2018moving} to bound this term. We first check the conditions therein. Note that for $b^{X_k}_{im}$ is a standard normal random variable for all $i \in [n]$ and $m \in [M]$, we have $(b^{X_k}_{im})^2$ to be chi-squared distributed with degree of freedom 1. By the moment generating function of chi-square distribution, we have
\begin{equation*}
\mathbb{E}\left[\exp\{\eta (b^{X_k}_{im})^2 \} \right]=\frac{1}{\sqrt{1-2\eta}} \quad \text{for all } 0<\eta<1/2.
\end{equation*}
Thus, we have
\begin{align*}
\Vert b^{X_k}_{im} \Vert_{\psi_2} & =  \inf\{\eta>0:\mathbb{E}\left[{\psi_2 (\vert b^{X_k}_{im} \vert / \eta)}\right]\leq 1\}\\
& =  \inf\{\eta>0:\mathbb{E}\left[{\exp ( (b^{X_k}_{im})^2 / \eta^2)}\right]\leq 2\}\\
& =  \inf\{\eta>0:1/\sqrt{1-2/\eta^2}\leq 2\}\\
& =  \frac{2\sqrt{2}}{\sqrt{3}}.
\end{align*}
Furthermore, we have
\begin{equation*}
\Vert b^{X_k}_i \Vert_{M,\psi_2} = \max_{m \in [M]} \Vert b^{X_k}_{im} \Vert_{\psi_2} = \frac{2\sqrt{2}}{\sqrt{3}}.
\end{equation*}
for all $i \in [n]$. In addition, note that
\begin{equation*}
\Var{\left( b^{X_k}_{im} b^{X_k}_{i m^{\prime}} \right)} \leq \mathbb{E}\left[(b^{X_k}_{i m})^2 (b^{X_k}_{i m^{\prime}})^2\right] \leq \left( \mathbb{E}\left[(b^{X_k}_{i m})^4 \right] \right)^{1/2} \left( \mathbb{E}\left[(b^{X_k}_{i m^{\prime}})^4\right] \right)^{1/2} = 3
\end{equation*}
for all $m, m^{\prime} \in [M]$, thus we have
\begin{equation*}
\max_{m,m^{\prime} \in [M]} \left\{ \frac{1}{n} \sum^n_{i=1} \Var{\left( b^{X_k}_{im} b^{X_k}_{i m^{\prime}} \right)} \right\} \leq 3.
\end{equation*}
Therefore, by Theorem 4.1 in \citet{kuchibhotla2018moving}, we have
\begin{align*}
& \mathbb{P} \left\{ \vertiii{\frac{1}{n} \sum^n_{i=1} \bm{b^{X_k}_i} \left( \bm{b^{X_k}_i} \right)^{\top} - \bm{I_M} }_{\infty} > 7\sqrt{3} \sqrt{\frac{\Delta + 2 \log M}{n}} + \frac{8 c}{3} \frac{ (\log (2n)) (\Delta + 2\log M) }{n} \right\} \\
& \leq 3 e^{-\Delta},
\end{align*}
for all $\Delta>0$, where $c$ is a constant.
Thus, combining the above equation with \eqref{eq:lemma2-2} and \eqref{eq:lemma2-3}, we have
\begin{align*}
& \mathbb{P} \left\{ \left\Vert \frac{1}{n} \bm{Z_k}^\top \bm{r^2} \right\Vert_2 > \left\Vert \bm{\Sigma^{r,X_k}} \right\Vert_{\text{F}} \left\{ 1 + 7\sqrt{3} \frac{M \sqrt{\Delta + 2 \log M}}{\sqrt{n}} + \frac{8 c}{3} \frac{ M (\log (2n)) (\Delta + 2\log M) }{n} \right\} \right\}\\
&\leq 3 e^{-\Delta}.
\end{align*}
Let $\omega(M)=\max_{k \in [p-1]} \left\Vert \bm{\Sigma^{r,X_k}} \right\Vert_{\text{F}}$, then
\begin{equation}\label{eq:lemma2-4}
\begin{aligned}
& \mathbb{P} \left\{ \max_{k \in [p-1]} \left\Vert \frac{1}{n} \bm{Z_k}^\top \bm{r^2} \right\Vert_2 >  \omega(M) \left( 1 + 7\sqrt{3} \frac{M \sqrt{\Delta + 2 \log M}}{\sqrt{n}} \right.\right. \\
& \qquad \qquad \qquad \qquad \qquad \left.\left. + \frac{8 c}{3} \frac{ M (\log (2n)) (\Delta + 2\log M) }{n} \right) \right\} \\ 
& \leq  3 (p-1) e^{-\Delta}.
\end{aligned}
\end{equation}

The result follows by combining \eqref{eq:lemma2-1}, \eqref{eq:lemma2-5}, and \eqref{eq:lemma2-4}.
\end{proof}

\begin{lemma}
\label{lemma:LowBoundLambda}
Let $\lambda_n = \tilde{\lambda}(n,p,M,\delta)$, where
\begin{multline*}
\tilde{\lambda} (n, p, M, \delta) =  2 \mathcal{C}_{n,\delta} \left(\frac{M  \sqrt{\Xi_1 (M)}}{\sqrt{n}}+ 2 \sqrt{\Xi_1 (M)} \sqrt{ \frac{ \log (4(p-1)/\delta) }{n} } \right) \\
+ 2 \omega (M) \left\{ 
7\sqrt{3} \frac{M \sqrt{ \log (6(p-1)/\delta) + 2 \log M}}{\sqrt{n}} \right. \\
\left. + 
 \frac{ 8 M c (\log (2n)) (\log (6(p-1)/\delta) + 2\log M) }{3n} \right\},
\end{multline*}
is defined in \eqref{eq:lambda_tilde_exact}, then \eqref{eq:CondForLambda} holds with probability at least $1-2\delta$. That is, we have
\begin{equation*}
\lambda_n \geq \frac{2}{n} \max_{k \in [p-1]} \Vert (\bm{Z}^{\top} \bm{u})_{k} \Vert_2
\end{equation*}
hold with probability at least $1-2\delta$.
\end{lemma}
\begin{proof}
The result follows directly from Lemma~\ref{lemma:UppBoundL2X}, Lemma~\ref{lemma:UppBoundDualR} and Bonferroni inequality.
\end{proof}

\begin{lemma}
\label{lemma:bd-B-err}
\begin{equation*}
\mathbb{P} \left\{ \sqrt{\sum^{p-1}_{k=1} \left\Vert \bm{\hat{B}_k} - \bm{B^*_k} \right\Vert^2_{\text{F}}} \leq \chi (n,p,M,\delta) + 12 \sqrt{s} \frac{\omega (M)}{ \sqrt{\kappa(M)}} \right\} \geq 1-3\delta,
\end{equation*}
where $\kappa(M)$ is defined in \eqref{eq:kappa-def} and $\chi (n,p,M,\delta)$ is defined in \eqref{eq:chi-fun}.
\end{lemma}
\begin{proof}
By Lemma~\ref{lemma:LowBoundLambda} and \eqref{eq:RSCPre}, we can get an error bound for $\Vert \hat{\beta}-\beta^{\ast} \Vert_2$ by applying Theorem 1 in \citet{negahban2010unified}.
Note that by Lemma C.4 of \citet{Zhao2019Direct}, we have the subspace compatibility constant defined in Definition 3 of \citet{negahban2010unified} to be $\sqrt{s}$. 
Then, for all $0<\delta \leq 1$, when $n$ is large enough such that $\Gamma (n,p,M,\delta) \leq \kappa(M)/(32 M^2 s)$, where $\Gamma (n,p,M,\delta)$ is defined in \eqref{eq:Gamma_n}, and let $\lambda_n=\tilde{\lambda} (n,p,M,\delta)$, where $\tilde{\lambda} (n,p,M,\delta)$ is defined in \eqref{eq:lambda_tilde}, then by Corollary 1 in \citet{negahban2010unified}, Lemma~\ref{lemma:LowBoundLambda}, \eqref{eq:RSCPre}, and a union bound, we have the desired result.
\end{proof}

\begin{lemma}\label{lemma:EleAbsCovMaxBound}
There exists a constant $c$ such that
\begin{multline}\label{eq:EleAbsCovMaxBound}
\vertiii{\bm{\hat{\Sigma}^X_n}-\bm{\Sigma^X}}_{\infty}\leq 7\sqrt{3}K_0\sqrt{\frac{\log{(3/\delta)}+2\log{((p-1)M)}}{n}}\\
+\frac{8c K_0 \log{(2n)}(\log{(3/\delta)}+2\log{((p-1)M)})}{3n}
\end{multline}
holds with probability at least $1-\delta$ where $K_0$ is defined by \eqref{eq:def-K0}.
\end{lemma}
\begin{proof}
The result follows Theorem 4.1 in \citet{kuchibhotla2018moving}. In order to apply the theorem, we need to check the conditions therein. First, we bound $\max_{k \in [p-1], m \in [M]}\Vert a^{X_k}_{im} \Vert_{\psi_2}$. Note that for $\zeta \sim N(0,\sigma^2)$, we have $(\zeta/\sigma)^2$ is chi-square distributed with degree of freedom 1. By the moment generating function of chi-square distribution, we have
\begin{equation*}
\mathbb{E}\left[\exp\{\eta(\zeta/\sigma)^2\} \right]=\frac{1}{\sqrt{1-2\eta}} \quad \text{for all } 0<\eta<1/2.
\end{equation*}
Thus, we have
\begin{equation*}
\mathbb{E}\left[\exp\left(\zeta^2/t^2\right) \right]=\mathbb{E}\left[\exp\left\{\left(\sigma^2/t^2\right)\cdot \left(\zeta^2/\sigma^2\right)  \right\} \right]=\frac{1}{\sqrt{1-2\sigma^2/t^2}}
\end{equation*}
for all $t>\sqrt{2}\sigma$. Let $1/\sqrt{1-2\sigma^2/t^2}\leq 2$, we have $t\geq (2\sqrt{2})/(\sqrt{3}\sigma)$. Thus, we have
\begin{equation*}
\begin{aligned}
\Vert \zeta \Vert_{\psi_2}&=\inf\{t>0:\mathbb{E}\left[{\psi_2 (\vert \zeta \vert / t)}\right]\leq 1\}\\
&=\inf\{t>0:\mathbb{E}\left[{\exp (\zeta^2 / t^2)}\right]\leq 2\}\\
&=\inf\{t>0:1/\sqrt{1-2\sigma^2/t^2}\leq 2\}\\
&=\frac{2\sqrt{2}}{\sqrt{3}}\sigma.
\end{aligned}
\end{equation*}
Based on the above result, we have for any $i\in[n]$,
\begin{equation*}
\begin{aligned}
\Vert \bm{a^{X_k}_i} \Vert_{M,\psi_2}&=\max_{k \in [p-1], m \in [M]} \Vert a^{X_k}_{im} \Vert_{\psi_2}\\
&=\frac{2\sqrt{2}}{\sqrt{3}}\max_{ k \in [p-1], m \in [M] }\sqrt{\Sigma^{X}_{kk,mm}}\\
&= \frac{2\sqrt{2}}{\sqrt{3}}\sqrt{K_0},
\end{aligned}
\end{equation*}
where $K_0$ is defined by \eqref{eq:def-K0}.
We then bound $\max_{k,k^{\prime} \in [p], m, m^{\prime} \in [M]}\Var{\left(a^{X_k}_{im} a^{X_{k^{\prime}}}_{i m^{\prime}} \right)}$. This followed by
\begin{equation*}
\begin{aligned}
\Var{\left(a^{X_k}_{im} a^{X_{k^{\prime}}}_{i m^{\prime}} \right)} & \leq  \mathbb{E}\left[(a^{X_k}_{im})^2  (a^{X_{k^{\prime}}}_{i m^{\prime}})^2 \right]\leq \left(  \mathbb{E}\left[(a^{X_k}_{im})^4 \right]  \mathbb{E}\left[ (a^{X_{k^{\prime}}}_{i m^{\prime}})^2 \right] \right)^{1/2}\\
&=\left(9 (\Sigma^X_{kk,mm})^2 (\Sigma^X_{k^{\prime} k^{\prime},m^{\prime} m^{\prime}})^2 \right)^{1/2}\\
&\leq 3K_0^2.
\end{aligned}
\end{equation*}
Thus, the final result is derived by applying Theorem 4.1 in \citet{kuchibhotla2018moving} with $K_{n,p}=(2\sqrt{2}/\sqrt{3})\sqrt{K_0}$ and $A^2_{n,p}=3K_0^2$.
\end{proof}

\begin{lemma}\label{lemma:Frob-bd}
For $\bm{x_i},\bm{y_i}\in \mathbb{R}^M$, $i \in [n]$, we have
\begin{equation*}
\left\Vert \frac{1}{n} \sum^n_{i=1} \bm{x_i} \bm{y_i}^{\top} \right\Vert_{\text{F}} \leq \sqrt{\frac{1}{n} \sum^n_{i=1} \Vert \bm{x_i} \Vert^2_2} \sqrt{\frac{1}{n} \sum^n_{i=1} \Vert \bm{y_i} \Vert^2_2}\, .
\end{equation*}
\end{lemma}
\begin{proof}
For any $m,m'\in[M]$, we have
\begin{equation*}
\left\vert \frac{1}{n} \sum^n_{i=1} x_{im} y_{i m^{\prime}} \right\vert^2 \leq \left( \frac{1}{n} \sum^n_{i=1} x^2_{im} \right) \left( \frac{1}{n} \sum^n_{i=1} y^2_{im^{\prime}} \right).
\end{equation*}
Therefore,
\begin{multline*}
\left\Vert \frac{1}{n} \sum^n_{i=1} \bm{x_i} \bm{y_i}^{\top} \right\Vert^2_{\text{F}} 
=  \sum^M_{m,m^{\prime}=1} \left\vert \frac{1}{n} \sum^n_{i=1} x_{im} y_{i m^{\prime}} \right\vert^2 \leq \sum^M_{m,m^{\prime}=1} \left( \frac{1}{n} \sum^n_{i=1} x^2_{im} \right) \left( \frac{1}{n} \sum^n_{i=1} y^2_{im^{\prime}} \right) \\
 =  \left( \frac{1}{n} \sum^n_{i=1} \sum^{M}_{m=1} x^2_{im} \right) \left( \frac{1}{n} \sum^n_{i=1} \sum^M_{m^{\prime}=1} y^2_{im^{\prime}} \right) 
 =  \left( \frac{1}{n} \sum^n_{i=1} \Vert \bm{x_i} \Vert^2_2 \right) \left( \frac{1}{n} \sum^n_{i=1} \Vert \bm{y_i} \Vert^2_2 \right),
\end{multline*}
and the result immediately follows.
\end{proof}

\begin{lemma}
\label{lemma:restricted-eigen-value}
Let
\begin{multline}\label{eq:Gamma_n}
\Gamma (n,p,M,\delta) = 7\sqrt{3}K_0\sqrt{\frac{\log{(3/\delta)}+2\log{((p-1)M)}}{n}}\\
+\frac{8c K_0 \log{(2n)}(\log{(3/\delta)}+2\log{((p-1)M)})}{3n},
\end{multline}
then when $\Gamma (n,p,M,\delta) \leq \kappa(M)/(32 M^2 s)$, we have 
\begin{equation*}
\mathbb{P} \left\{ \frac{1}{2n}\Vert \bm{Z} \bm{\Delta\beta} \Vert^2_2 \geq \frac{\kappa}{4} \Vert \bm{\Delta\beta} \Vert^2_2 \text{ for all } \bm{\Delta\beta} \in \mathbb{C}(\mathscr{N}_j) \right\} \geq 1-\delta,
\end{equation*}
where $\mathbb{C}(\mathscr{N}_j)$ is defined in \eqref{eq:ConstrainedSpace}. 

\end{lemma}

\begin{proof}
We want to first prove
\begin{equation*}
\frac{1}{2n}\Vert \bm{Z} \bm{\Delta\beta} \Vert^2_2 \geq \kappa_{\mathcal{L}}\Vert \bm{\Delta\beta} \Vert^2_2-\tau_{\mathcal{L}}(\bm{\beta^{\ast}}) \quad \text{for all } \bm{\Delta\beta} \in\mathbb{C}(\mathscr{N}_j),
\end{equation*}
where $\kappa_{\mathcal{L}}>0$ is a positive constant, $\tau_{\mathcal{L}}(\beta^{\ast})>0$, and $\mathbb{C}(\mathscr{N}_j)$ is defined in \eqref{eq:ConstrainedSpace}. 

Since $\bm{Z} = \bm{A^X} \otimes \bm{I_M}$, thus for any $ \bm{\theta} \in\mathbb{R}^{(p-1)M^2}$, we have
\begin{equation*}
\begin{aligned}
\frac{1}{n}\Vert \bm{Z} \bm{\theta} \Vert^2_2&=\frac{1}{n}\Vert (\bm{A^X} \otimes I_M) \bm{\theta} \Vert^2_2\\
&=\frac{1}{n}\bm{\theta}^{\top} ( \bm{A^X} \otimes \bm{I_M})^{\top}( \bm{A^X} \otimes \bm{I_M})\bm{\theta}\\
&=\frac{1}{n}\bm{\theta}^{\top}((\bm{A^X})^{\top} \otimes \bm{I_M})(\bm{A^X}\otimes \bm{I_M}) \bm{\theta}\\
&=\frac{1}{n}\bm{\theta}^{\top}((\bm{A^X})^{\top} \bm{A^X} \otimes \bm{I_M}) \bm{\theta}\\
&=\bm{\theta}^{\top}(\bm{\hat{\Sigma}^X_n} \otimes \bm{I_M})\bm{\theta},
\end{aligned}
\end{equation*}
where $\bm{\hat{\Sigma}^X_n} = \frac{1}{n} (\bm{A^X})^{\top} \bm{A^X}$. We then further have
\begin{equation*}
\begin{aligned}
\frac{1}{n}\Vert \bm{Z} \bm{\theta} \Vert^2_2&=\bm{\theta}^{\top}(\bm{\Sigma^X} \otimes \bm{I_M}) \bm{\theta}+\bm{\theta}^{\top}\left((\bm{\hat{\Sigma}^X_n}-\bm{\Sigma^X})\otimes \bm{I_M}\right)\bm{\theta}\\
&\geq \left\vert \bm{\theta}^{\top}(\bm{\Sigma^X} \otimes \bm{I_M}) \bm{\theta} \right\vert -\left\vert \bm{\theta}^{\top}\left((\bm{\hat{\Sigma}^X_n}-\bm{\Sigma^X})\otimes \bm{I_M}\right)\bm{\theta} \right\vert.
\end{aligned}
\end{equation*}

Note that
\begin{align*}
\min_{\bm{\theta}\in\mathbb{C}(\mathscr{N}_j)\setminus\{0\}}\frac{\bm{\theta}^{\top}(\bm{\Sigma^{X}}\otimes \bm{I_M})\bm{\theta}}{\Vert \bm{\theta} \Vert^2_2}&\geq \min_{\bm{\theta}\in\mathcal{M}\setminus\{0\}}\frac{\bm{\theta}^{\top}(\bm{\Sigma^{X}}\otimes \bm{I_M})\bm{\theta}}{\Vert \bm{\theta} \Vert^2_2}\\
&\geq\rho_{\min}\left(\bm{\Sigma^X_{\mathscr{N}_j \mathscr{N}_j}} \otimes \bm{I_M}\right)\\
&=\rho_{\min}(\bm{\Sigma^X_{\mathscr{N}_j \mathscr{N}_j}})\\
&=\kappa,
\end{align*}
and we have $\kappa(M)>0$ for all $M$ by Assumption~\ref{assump:LowBoundKappa}.
Thus, for any $\bm{\theta}\in\mathbb{C}(\mathscr{N}_j)$, we have
\begin{equation}\label{eq:RECDecompose}
\frac{1}{n}\Vert \bm{Z} \bm{\theta} \Vert^2_2\geq \kappa \Vert \bm{\theta}\Vert^2_2-\left\vert\bm{\theta}^{\top}\left((\bm{\hat{\Sigma}^X_n}-\bm{\Sigma^X})\otimes \bm{I_M}\right) \bm{\theta} \right\vert.
\end{equation}
To prove the RSC condition, it then suffices to give an upper bound for 
\[
\vert \bm{\theta}^{\top}\left((\bm{\hat{\Sigma}^X_n}-\bm{\Sigma^X})\otimes \bm{I_M}\right)\bm{\theta}\vert,
\]
where $\bm{\theta}\in\mathbb{C}(\mathscr{N}_j)$. By Lemma 5 in section D of the appendix of \citet{Zhao2019Direct} and the definition of $\mathbb{C}(\mathscr{N}_j)$, for any $\bm{\theta}\in\mathbb{C}(\mathscr{N}_j)$, we have
\begin{equation*}
\begin{aligned}
\left\vert \bm{\theta}^{\top}\left((\bm{\hat{\Sigma}^X_n}-\bm{\Sigma^{X}})\otimes \bm{I_M} \right) \bm{\theta} \right\vert &\leq M^2 \vertiii{(\bm{\hat{\Sigma}^X_n}-\bm{\Sigma^{X}})\otimes \bm{I_M}}_{\infty}\Vert \bm{\theta} \Vert^{2}_{1,2}\\
&=M^2 \vertiii{(\bm{\hat{\Sigma}^X_n}-\bm{\Sigma^{X}})\otimes \bm{I_M}}_{\infty}\left(\Vert \bm{\theta}_{\mathcal{M}}\Vert_{1,2}+\Vert \bm{\theta}_{\mathcal{M}^{\bot}}\Vert_{1,2} \right)^2\\
&\leq 16 M^2 \vertiii{(\bm{\hat{\Sigma}^X_n}-\bm{\Sigma^{X}})\otimes I_M}_{\infty} \Vert \bm{\theta}_{\mathcal{M}}\Vert^2_{1,2}\\
&\leq 16 s M^2 \vertiii{(\bm{\hat{\Sigma}^X_n}-\bm{\Sigma^{X}})\otimes I_M}_{\infty}\Vert \bm{\theta}_{\mathcal{M}}\Vert^2_{1,2},
\end{aligned}
\end{equation*}
where the penultimate line is by Lemma 6 in section D of \citet{Zhao2019Direct}. Note that 
\[
\vertiii{(\bm{\hat{\Sigma}^X_n}-\bm{\Sigma^{X}})\otimes \bm{I_M}}_{\infty}=\vertiii{\bm{\hat{\Sigma}^X_n}-\bm{\Sigma^{X}}}_{\infty},
\]
then combine \eqref{eq:RECDecompose} and Lemma~\ref{lemma:EleAbsCovMaxBound}, we have constant $c$ such that for any $\delta>0$, with probability at least $1-\delta$, we have
\begin{equation}
\begin{aligned}
\frac{1}{2n}\Vert \bm{Z} \bm{\Delta\beta} \Vert^2_2 & \geq  \left( \frac{\kappa}{2}-8M^2 s\cdot \left\{ 7\sqrt{3}K_0\sqrt{\frac{\log{(3/\delta)}+2\log{((p-1)M)}}{n}} + \right.\right.\\
& \left.\left. \frac{8c K_0 \log{(2n)}(\log{(3/\delta)}+2\log{((p-1)M)})}{3n} \right\} \right)\Vert \bm{\Delta\beta} \Vert^2_2
\end{aligned}
\end{equation}
for all $\bm{\Delta\beta}\in\mathbb{C}(\mathscr{N}_j)$.
Let $\Gamma (n,p,M,\delta)$ be defined by~\eqref{eq:Gamma_n},
then when $\Gamma (n,p,M,\delta) \leq \kappa(M)/(32 M^2 s)$, we have 
\begin{equation*}
\frac{1}{2n}\Vert \bm{Z} \bm{\Delta\beta} \Vert^2_2 \geq \frac{\kappa}{4} \Vert \bm{\Delta\beta} \Vert^2_2 \text{ for all } \bm{\Delta\beta} \in \mathbb{C}(\mathscr{N}_j)
\end{equation*}
with probability at least $1-\delta$.

\end{proof}

In the next few results, recall that
\[
\lfb =  1 + 8 \left(  \frac{\log (2/\delta)}{n} + \sqrt{ \frac{\log (2/\delta)}{n} } \right) .
\]

\begin{lemma}
\label{lemma:second-mem-bd}
Let $\xi_1,\cdots,\xi_n$ be i.i.d. mean zero random elements in some Hilbert space, and $\mathbb{E}[\Vert \xi_1 \Vert^2 ]=\sigma_{\xi}$. Besides, we assume that
\begin{equation*}
\mathbb{E}\left[ \Vert \xi_1\Vert^{2k} \right] \leq (2 \sigma_{\xi})^k \cdot k !  \quad \text{ for all } k = 1,2,\cdots .
\end{equation*}
Then for any given $\delta \in (0, 1]$, we have
\begin{equation*}
\mathbb{P} \left\{  \frac{1}{n} \sum^n_{i=1} \Vert \xi_i \Vert^2 \leq \sigma_{\xi} \lfb \right\} \geq 1 - \delta.
\end{equation*}
\end{lemma}
\begin{proof}
Note that $x \mapsto \vert x \vert^k$ is a convex function when $k\geq 1$ and $x \in \mathbb{R}$. By Jensen's inequality, we have
$\vert x/2+y/2 \vert^k \leq (\vert x \vert^k + \vert y \vert^k) / 2$, which implies that $\vert x + y \vert^k \leq 2^{k-1} (\vert x \vert^k + \vert y \vert^k)$.
Then by Lemma~\ref{lemma:2kmomentGuassFunc}, for $k \geq 2$, we have
\begin{equation*}
\mathbb{E} \left[  \left\vert \Vert \xi_1 \Vert^2 - \sigma_{\xi}  \right\vert^k \right] \leq 2^{k-1} \left(  \mathbb{E} \left[ \Vert \xi_1 \Vert^{2k} \right] + \sigma^k_{\xi} \right) \leq 2^{k-1} \left( (2 \sigma_{\xi})^k k! + \sigma^k_{\xi} \right) \leq (4 \sigma_{\xi})^k k!.
\end{equation*}
Thus, we have
\begin{equation*}
\frac{1}{n} \sum^n_{i=1} \mathbb{E} \left[  \left\vert \Vert \xi_i \Vert^2 - \sigma_{\xi}  \right\vert^k \right] \leq \frac{k!}{2} (32 \sigma^2_{\xi}) (4 \sigma_{\xi})^{k-2}
\end{equation*}
for all $k=2,3,\dots$. Then, by Theorem 2.5 \citet{Bosq2000Linear}, we have
\begin{equation*}
\mathbb{P} \left\{ \left\vert \frac{1}{n} \sum^n_{i=1} \Vert \xi_i \Vert^2 - \sigma_{\xi} \right\vert > \Delta \right\} \leq 2 \exp \left( - \frac{n \Delta^2 }{ 64 \sigma^2_{\xi} + 8 \sigma_{\xi} \Delta } \right), \quad \Delta > 0.
\end{equation*}
The result follows by rearranging the terms.
\end{proof}

\begin{lemma}\label{lemma:g2moment}
Let $\delta \in (0, 1]$. For any $j \in [p]$, we have
\begin{align*}
\mathbb{P} &\left\{ \frac{1}{n} \sum^n_{i=1} \Vert g_{ij} \Vert^2 \leq  \sigma_{\max,0} \lfb \right\} \geq 1 -\delta \\
\intertext{and}
\mathbb{P} &\left\{ \frac{1}{n} \sum^n_{i=1} \Vert e_{ij} \Vert^2 \leq  \sigma_{jr} \lfb \right\} \geq 1 -\delta,
\end{align*}
where $e_{ij}$ is defined in \eqref{eq:function-linear} and $\sigma_{jr}=\mathbb{E} [ \Vert e_{ij} \Vert^2]$.
\end{lemma}
\begin{proof}
The result follows directly from Lemma~\ref{lemma:second-mem-bd}.
\end{proof}

\begin{lemma}\label{lemma:rw2moment}
For all $0<\delta \leq 1$, we have
\begin{equation*}
\mathbb{P} \left\{ \frac{1}{n} \sum^n_{i=1} \Vert \bm{r_i} + \bm{w_i} \Vert^2_2 \leq  \Xi_4 (M) \lfb \right\} \geq 1 -\delta,
\end{equation*}
\begin{equation*}
\mathbb{P} \left\{ \frac{1}{n} \sum^n_{i=1} \Vert \bm{a^{X_k}}_i \Vert^2_2 \leq \tr \left(\bm{\Sigma^X_{kk}}\right) \lfb \right\} \geq 1 -\delta,
\end{equation*}
and
\begin{equation*}
\mathbb{P} \left\{ \frac{1}{n} \sum^n_{i=1} \Vert a^{X_k}_{i m} \Vert^2_2 \leq \Xi_2 (M) \lfb \right\} \geq 1 -\delta \, \text{for all } m \in [M], k \in [p-1] .
\end{equation*}
\end{lemma}
\begin{proof}
Note that $\mathbb{E} [ \Vert \bm{r_i} + \bm{w_i} \Vert^2_2 ] = \Xi_4 (M)$, $\mathbb{E} [ \Vert \bm{a^{X_k}_i} \Vert^2_2 ] = \tr (\bm{\Sigma^X_{kk}}) $ and $\mathbb{E} [ \Vert a^{X_k}_{im} \Vert^2_2 ] = \Sigma^{X_k X_k}_{mm} \leq \rho_{\max} (\bm{\Sigma^X_{kk}}) \leq \Xi_2 (M)$, where $\Xi_2 (M),\Xi_4 (M)$ is defined in \eqref{eq:add-nota}, and then the result follows directly from Lemma~\ref{lemma:second-mem-bd}.
\end{proof}

\begin{lemma}\label{lemma:2kmomentGuassFunc}
Let $g$ be a mean zero random Gaussian function in the Hilbert space $\mathbb{H}$. Let $\sigma_0=\mathbb{E}\left[ \Vert g \Vert^{2} \right]$.
Then 
\[
\mathbb{E}\left[ \Vert g \Vert^{2k} \right]\leq (2\sigma_0)^{k}\cdot k!  \quad \text{ for all } k = 1,2,\cdots .
\]
\end{lemma}
\begin{proof}
Let $\{\phi_{m} \}_{m\geq 1}$ be othornormal eigenfunctions of covariance function of $g$.
Let $a_m=\langle g,\phi_{m} \rangle$. We have $a_m$, $m \geq 1$, are independent mean zero Gaussian random variables with variance $\sigma_m$ and $\sigma_0=\sum_{m\geq 1}\sigma_{m}$. We further have
\begin{equation*}
g=\sum^{\infty}_{m=1}\sigma_{m}^{1/2} \xi_{m} \phi_{m},
\end{equation*}
where $\xi_{m}=\sigma^{-1/2}_{m}a_m$ are independent standard Gaussian, and
$\Vert g \Vert=\left(\sum_{m\geq 1}\sigma_{m}\xi^{2}_{m}\right)^{1/2}$.

Using the Jensen's inequality for $t \mapsto t^{2k}$, we have
\begin{multline*}
\Vert g \Vert^{2k}=\left(\sum_{m\geq 1}\sigma_{m}\right)^{k} \cdot  \left(\frac{\sum_{m\geq 1}\sigma_{m}\xi^{2}_{m}}{\sum_{m\geq 1}\sigma_{m}}\right)^{k}
\\
\leq \left(\sum_{m\geq 1}\sigma_{m}\right)^{k} \cdot \frac{\sum_{m\geq 1}\sigma_{m}\xi^{2k}_{m}}{\sum_{m\geq 1}\sigma_{m}}
=\left(\sum_{m\geq 1}\sigma_{m}\right)^{k-1} \cdot \left( \sum_{m\geq 1}\sigma_{m}\xi^{2k}_{m} \right).
\end{multline*}
Thus,
\begin{equation*}
\begin{aligned}
\mathbb{E}\left[ \Vert g \Vert^{2k} \right]
\leq \left(\sum_{m\geq 1}\sigma_{m}\right)^{k-1} \cdot \left( \sum_{m\geq 1}\sigma_{m}\mathbb{E}\left[\xi^{2k}_{m}\right] \right)
=\sigma_{0}^{k} \mathbb{E}\left[\xi^{2k}_{1}\right] 
\leq \sigma_{0}^{k}\cdot 2^k \cdot k!=(2\sigma_{0})^{k}k!\, ,
\end{aligned}
\end{equation*}
which completes the proof.
\end{proof}

\begin{lemma}\label{lemma:a-diff-bd}
Recall that $\bm{g_i}(\cdot)=(g_{i1}(\cdot), g_{i2}(\cdot), \dots, g_{ip}(\cdot) )^{\top}$ is our $i$-th observation defined in Section~\ref{sec:fgm}. Besides, recall that in Section~\ref{sec:choice-fun-basis}, we have $\phi_{m}=\phi_{jm}$ and $\hat{\phi}_{m}=\hat{\phi}_{jm}$ be the $m$-th basis function and and its corresponding estimate respectively used to do projection for $j$-th node, and
$\bm{\hat{a}^{X_k}_{i,M}}=(\hat{a}^{X_k}_{i1},\ldots,\hat{a}^{X_k}_{iM})^{\top}$ be the projection score vector of $\bm{g_i}$ by using $\{\hat{\phi}_{m}\}^M_{m=1}$.
Under the assumption that
\begin{equation*}
\frac{1}{n} \sum^n_{i=1} \Vert g_{ij} \Vert^2 \leq  \Gamma_g (j) \, \text{ for all } j \in [p],
\end{equation*}
and
\begin{equation*}
\Vert \hat{\phi}_m - \phi_m \Vert \leq \Gamma_{\phi}(m) \, \text{ for all } m \geq 1,
\end{equation*}
for some $0 < \delta \leq 1$, we then have
\begin{equation*}
\frac{1}{n} \sum^n_{i=1} \left( \hat{a}^{X_k}_{im} - a^{X_k}_{im} \right)^2 \leq \Gamma_g (k) \Gamma^2_{\phi}(m),
\end{equation*}
for all $k \in [p-1]$ and $m \geq 1$. Furthermore, we have
\begin{equation*}
\frac{1}{n} \sum^{n}_{i=1} \Vert \bm{\hat{a}^{X_k}_i} - \bm{a^{X_k}_i} \Vert^2 \leq \Gamma_g (k)  \sum^M_{m=1} \Gamma^2_{\phi}(m).
\end{equation*}
\end{lemma}
\begin{proof}
Note that
\begin{equation*}
\left( \hat{a}^{X_k}_{im} - a^{X_k}_{im} \right)^2 = \left( \langle g^{X_k}_i, \hat{\phi}_m - \phi_m \rangle \right)^2 \leq \Vert g^{X_k}_i \Vert^2 \Vert \hat{\phi}_m - \phi_m \Vert^2.
\end{equation*}
Thus we have
\begin{equation*}
\frac{1}{n} \sum^n_{i=1} \left( \hat{a}^{X_k}_{im} - a^{X_k}_{im} \right)^2 \leq \left( \frac{1}{n} \sum^n_{i=1} \Vert g^{X_k}_i \Vert^2  \right) \Vert \hat{\phi}_m - \phi_m \Vert^2.
\end{equation*}
The rest of the proof follows directly from the assumptions we made.
\end{proof}

\begin{lemma}
\label{lemma:b-diff-bd}
For a given node $j \in [p]$, recall that $\phi_{m}=\phi_{jm}$ and $\hat{\phi}_{m}=\hat{\phi}_{jm}$ are the $m$-th basis function and its corresponding estimate respectively used to do projection for $j$-th node, defined in Section~\ref{sec:choice-fun-basis}. Besides, let $\beta_k (t^{\prime},t)=\beta_{jk} (t^{\prime},t)$ and recall that $b^{\ast}_{k,m m^{\prime}} = \int_{\mathcal{T}\times \mathcal{T}} \beta_k (t^{\prime},t) \phi_m (t) \phi_{m^{\prime}} (t^{\prime}) dt^{\prime} dt$ is defined in~\eqref{eq:def-b-ast} and $\tilde{b}_{k,m m^{\prime}} = \int_{\mathcal{T}} \beta_k (t^{\prime},t) \hat{\phi}_m (t^{\prime}) \hat{\phi}_{m^{\prime}} (t) dt^{\prime} dt$ is defined in~\eqref{eq:def-b-tilde}.
Under the assumption that
\begin{equation*}
\Vert \hat{\phi}_m - \phi_m \Vert \leq \Gamma_{\phi}(m) \, \text{for all } m \geq 1,
\end{equation*}
and $\sum^{\infty}_{m=1}\Gamma^2_{\phi}(m)<\infty$, then we have
\begin{multline*}
\sum^{p-1}_{k=1} \sum^{\infty}_{m^{\prime}=1} \left( \tilde{b}_{k,m m^{\prime}} - b^{\ast}_{k,m m^{\prime}} \right)^2 \leq 2 \left( \sum_{k \in \mathscr{N}_j} \Vert \beta_{k} (t,t^{\prime}) \Vert^2_{\text{HS}} \right) \Gamma^2_{\phi}(m) \\
+ 2 \sum_{k \in \mathscr{N}_j} \sum^{\infty}_{m^{\prime}=1} \left\vert \int_{\mathcal{T}\times\mathcal{T}} \beta_{k} (t,t^{\prime}) \phi_m (t) \left( \hat{\phi}_{m^{\prime}} (t^{\prime}) - \phi_{m^{\prime}} (t^{\prime})  \right)  d t^{\prime} dt \right\vert^2,
\end{multline*}
where $\beta_{k} (t,t^{\prime})$ is defined in~\eqref{eq:function-linear}.
\end{lemma}

\begin{proof}
Note that
\begin{align*}
& \quad \left\vert \tilde{b}_{k,m m^{\prime}} - b^{\ast}_{k,m m^{\prime}} \right\vert \\
& =  \left\vert \int_{\mathcal{T}\times\mathcal{T} } \beta_{k} (t,t^{\prime}) \left( \hat{\phi}_m (t) \hat{\phi}_{m^{\prime}} (t^{\prime}) - \phi_m (t) \phi_{m^{\prime}} (t^{\prime}) \right) d t^{\prime} dt \right\vert \\
& = \left\vert \int_{\mathcal{T}\times\mathcal{T}} \beta_{k} (t,t^{\prime}) \left\{ \left(\hat{\phi}_m (t) - \phi_m (t) \right) \hat{\phi}_{m^{\prime}} (t^{\prime}) + \phi_m (t) \left( \hat{\phi}_{m^{\prime}} (t^{\prime}) - \phi_{m^{\prime}} (t^{\prime})  \right)  \right\} d t^{\prime} dt \right\vert \\
& \leq \left\vert \int_{\mathcal{T}\times\mathcal{T}} \beta_{k} (t,t^{\prime})  \left(\hat{\phi}_m (t) - \phi_m (t) \right) \hat{\phi}_{m^{\prime}} (t^{\prime}) d t^{\prime} dt \right\vert \\
& \quad + \left\vert \int_{\mathcal{T}\times\mathcal{T}} \beta_{k} (t,t^{\prime}) \phi_m (t) \left( \hat{\phi}_{m^{\prime}} (t^{\prime}) - \phi_{m^{\prime}} (t^{\prime})  \right)  d t^{\prime} dt \right\vert. \numberthis \label{eq:lemma12-helper-1}
\end{align*}
Since $\{ \phi_{m^{\prime}} \}^{\infty}_{m^{\prime}=1}$ is an orthonormal function basis, thus when we treat 
\[
\int_{\mathcal{T}} \beta_{k} (t,t^{\prime})  \left(\hat{\phi}_m (t) - \phi_m (t) \right) dt
\]
as a function of $t^{\prime}$, we have
\begin{align*}
& \quad \sum^{\infty}_{m^{\prime}=1} \left\vert \int_{\mathcal{T}\times\mathcal{T}} \beta_{k} (t,t^{\prime})  \left(\hat{\phi}_m (t) - \phi_m (t) \right) \hat{\phi}_{m^{\prime}} (t^{\prime}) d t^{\prime} dt \right\vert^2 \\
& = \int_{\mathcal{T}} \left( \int_{\mathcal{T}} \beta_{k} (t,t^{\prime}) \left(\hat{\phi}_m (t) - \phi_m (t) \right) d t  \right)^2 dt^{\prime} \\
& \leq \int_{\mathcal{T}} \left( \int_{\mathcal{T}} \beta^2_{k} (t,t^{\prime}) dt \right) \left( \int_{\mathcal{T}} \left(\hat{\phi}_m (t) - \phi_m (t) \right)^2 dt \right) dt^{\prime} \\
& = \Vert \hat{\phi}_m - \phi_m  \Vert^2 \Vert \beta_{k} (t,t^{\prime}) \Vert^2_{\text{HS}} \\
& \leq \Gamma^2_{\phi}(m) \Vert \beta_{k} (t,t^{\prime}) \Vert^2_{\text{HS}}. \numberthis \label{eq:lemma12-helper-2} 
\end{align*}
Combine \eqref{eq:lemma12-helper-1}-\eqref{eq:lemma12-helper-2}, and note that $b^*_{k,m m^{\prime}}=\tilde{b}_{k,m m^{\prime}}=0$ for all $m,m^{\prime}\geq 1$ when $k \notin \mathscr{N}_j$,
thus we have
\begin{multline*}
\sum^{p-1}_{k=1} \sum^{\infty}_{m^{\prime}=1} \left( \tilde{b}_{k,m m^{\prime}} - b^{\ast}_{k,m m^{\prime}} \right)^2 \leq 2 \left( \sum_{k \in \mathscr{N}_j} \Vert \beta_{k} (t,t^{\prime}) \Vert^2_{\text{HS}} \right) \Gamma^2_{\phi}(m) \\
+ 2 \sum_{k \in \mathscr{N}_j} \sum^{\infty}_{m^{\prime}=1} \left\vert \int_{\mathcal{T}\times\mathcal{T}} \beta_{k} (t,t^{\prime}) \phi_m (t) \left( \hat{\phi}_{m^{\prime}} (t^{\prime}) - \phi_{m^{\prime}} (t^{\prime})  \right)  d t^{\prime} dt \right\vert^2.
\end{multline*}
\end{proof}

\begin{lemma}
\label{lemma:v-bd}
Recall that $\bm{g_i}(\cdot)=(g_{i1}(\cdot), g_{i2}(\cdot), \dots, g_{ip}(\cdot) )^{\top}$ is our $i$-th observation defined in Section~\ref{sec:fgm} and $e_{ij}(\cdot)$ is the error term defined in~\eqref{eq:function-linear}. Besides, recall that $\phi_{m}=\phi_{jm}$ and $\hat{\phi}_{m}=\hat{\phi}_{jm}$ are the $m$-th basis function and its corresponding estimate respectively used to do projection for the $j$-th node, defined in Section~\ref{sec:choice-fun-basis}. Recall from \eqref{eq:def-v}
\begin{equation*}
\bm{v_{iM}} = \sum^{p-1}_{k=1} (\bm{\tilde{B}_{k,M}}-\bm{B^{\ast}_{k,M}}) \bm{\hat{a}^{X_k}_{i,M}} + (\bm{\tilde{r}_{i,M}}-\bm{r_{i,M}}) + (\bm{\tilde{w}_{i,M}} - \bm{w_{i,M}}),
\end{equation*}
and suppose that
\begin{equation*}
\frac{1}{n} \sum^n_{i=1} \Vert g_{ij} \Vert^2 \leq \Gamma_g (j) \, \text{ for all } j \in [p],
\end{equation*}
\begin{equation*}
\frac{1}{n} \sum^n_{i=1} \Vert e_{ij} \Vert^2 \leq  \Gamma_e \text{ for all } j \in [p],
\end{equation*}
and
\begin{equation*}
\Vert \hat{\phi}_m - \phi_m \Vert \leq \Gamma_{\phi}(m) \, \text{ for all } m \geq 1,
\end{equation*}
where $\sum^{\infty}_{m=1}\Gamma^2_{\phi}(m)<\infty$. Then,
\begin{equation}
\label{eq:lemma14-res1}
\frac{1}{n} \sum^n_{i=1} v^2_{i m} \leq \text{I}_m + \text{II}_m + \text{III}_m,
\end{equation}
where
\begin{align*}
\text{I}_m & = 6  \left( \sum_{k \in \mathscr{N}_j} \Gamma_g(k) \right) \left\{ \left( \sum_{k \in \mathscr{N}_j} \Vert \beta_{k} (t,t^{\prime}) \Vert^2_{\text{HS}} \right) \Gamma^2_{\phi}(m) \right. \\
& \quad \quad \left. + \sum_{k \in \mathscr{N}_j} \sum^{\infty}_{m^{\prime}=1} \left\vert \int_{\mathcal{T}\times\mathcal{T}} \beta_{k} (t,t^{\prime}) \phi_m (t) \left( \hat{\phi}_{m^{\prime}} (t^{\prime}) - \phi_{m^{\prime}} (t^{\prime})  \right)  d t^{\prime} dt \right\vert^2 \right\}.\\
\text{II}_m & = 3 \left( \sum_{k \in \mathscr{N}_j} \sum^{\infty}_{m^{\prime}=M+1} \left( b^{\ast}_{k,m m^{\prime}} \right)^2  \right) \left( \sum_{k \in \mathscr{N}_j} \Gamma_g(k) \right) \left( \sum^{\infty}_{m^{\prime}=M+1}  \Gamma^2_{\phi}(m^{\prime}) \right), \\
\text{III}_m & =  3 \Gamma_e \Gamma^2_{\phi}(m).
\end{align*}
Recall that $\beta_{jk} (t,t^{\prime})$ is defined in~\eqref{eq:function-linear}. Drop the subscript $j$ and let $\beta_{k} (t,t^{\prime})=\beta_{jk} (t,t^{\prime})$. Recall that $b^{\ast}_{k,m m^{\prime}} = \int_{\mathcal{T}\times \mathcal{T}} \beta_k (t^{\prime},t) \phi_m (t) \phi_{m^{\prime}} (t^{\prime}) dt^{\prime} dt$ is defined in~\eqref{eq:def-b-ast}.
Furthermore, we have
\begin{equation}
\label{eq:lemma14-res2}
\begin{aligned}
\frac{1}{n} \sum^n_{i=1} \Vert \bm{v_i} \Vert^2 & \leq 12 \left( \sum^{\infty}_{m=1}\Gamma^2_{\phi}(m) \right)\left( \sum_{k \in \mathscr{N}_j} \Vert \beta_{k} (t,t^{\prime}) \Vert^2_{\text{HS}} \right)   \left( \sum_{k \in \mathscr{N}_j} \Gamma_g(k) \right) \\
& + 3 \Phi^2(M) \left( \sum_{k \in \mathscr{N}_j} \Gamma_g(k) \right) \left( \sum^{\infty}_{m^{\prime}=M+1}  \Gamma^2_{\phi}(m^{\prime}) \right) + 3 \Gamma_e \sum^M_{m=1} \Gamma^2_{\phi}(m),
\end{aligned}
\end{equation}
where, as previously defined in~\eqref{eq:add-nota-3},
\begin{equation*}
\Phi (M) = \sqrt{ \sum^{p-1}_{k=1} \sum^M_{m=1} \sum^{\infty}_{m^{\prime}=M+1} \left( b^{\ast}_{k,m m^{\prime}} \right)^2 }.
\end{equation*}
\end{lemma}

\begin{proof}
Note that
\begin{align*}
v_{i m} & =  \sum^{p-1}_{k=1} \sum^{M}_{m^{\prime}=1} (\tilde{b}_{k,m m^{\prime}} - b^{\ast}_{k,m m^{\prime}}) \hat{a}^{X_k}_{i m^{\prime}} + \sum^{p-1}_{k=1} \sum^{\infty}_{m^{\prime}=M+1} \left(\tilde{b}_{k,m m^{\prime}} \hat{a}^{X_k}_{i m^{\prime}} - b^{\ast}_{k,m m^{\prime}} a^{X_k}_{i m^{\prime}} \right) \\
& \quad + \langle e_{ij}(t), \hat{\phi}_m(t) - \phi_m(t)  \rangle \\
& =  \sum_{k \in \mathscr{N}_j} \sum^{M}_{m^{\prime}=1} (\tilde{b}_{k,m m^{\prime}} - b^{\ast}_{k,m m^{\prime}}) \hat{a}^{X_k}_{i m^{\prime}} + \sum_{k \in \mathscr{N}_j} \sum^{\infty}_{m^{\prime}=M+1} \left(\tilde{b}_{k,m m^{\prime}} \hat{a}^{X_k}_{i m^{\prime}} - b^{\ast}_{k,m m^{\prime}} a^{X_k}_{i m^{\prime}} \right) \\
& \quad + \langle e_{ij}(t), \hat{\phi}_m(t) - \phi_m(t)  \rangle \\
& =  \sum_{k \in \mathscr{N}_j} \sum^{\infty}_{m^{\prime}=1} (\tilde{b}_{k,m m^{\prime}} - b^{\ast}_{k,m m^{\prime}}) \hat{a}^{X_k}_{i m^{\prime}} + \sum_{k \in \mathscr{N}_j} \sum^{\infty}_{m^{\prime}=M+1} b^{\ast}_{k,m m^{\prime}} (\hat{a}^{X_k}_{i m^{\prime}} - a^{X_k}_{i m^{\prime}}) \\
& \quad + \langle e_{ij}(t), \hat{\phi}_m(t) - \phi_m(t)  \rangle.
\end{align*}
By Jensen's inequality, we have
\begin{align*}
v^2_{i m} & \leq 3 \left( \sum_{k \in \mathscr{N}_j} \sum^{\infty}_{m^{\prime}=1} (\tilde{b}_{k,m m^{\prime}} - b^{\ast}_{k,m m^{\prime}}) \hat{a}^{X_k}_{i m^{\prime}}  \right)^2 + 3 \left( \sum_{k \in \mathscr{N}_j} \sum^{\infty}_{m^{\prime}=M+1} b^{\ast}_{k,m m^{\prime}} (\hat{a}^{X_k}_{i m^{\prime}} - a^{X_k}_{i m^{\prime}})  \right)^2 \\
& \quad + 3 \left( \langle e_{ij}(t), \hat{\phi}_m(t) - \phi_m(t)  \rangle \right)^2.
\end{align*}
By Cauchy-Schwartz inequality, we further have
\begin{align*}
v^2_{i m} & \leq  3 \left( \sum_{k \in \mathscr{N}_j} \sum^{\infty}_{m^{\prime}=1} (\tilde{b}_{k,m m^{\prime}} - b^{\ast}_{k,m m^{\prime}})^2  \right) \left( \sum_{k \in \mathscr{N}_j} \sum^{\infty}_{m^{\prime}=1} \left( \hat{a}^{X_k}_{i m^{\prime}} \right)^2 \right) \\ 
& \quad + 3 \left( \sum_{k \in \mathscr{N}_j} \sum^{\infty}_{m^{\prime}=M+1} \left( b^{\ast}_{k,m m^{\prime}} \right)^2  \right) \left( \sum_{k \in \mathscr{N}_j} \sum^{\infty}_{m^{\prime}=M+1} (\hat{a}^{X_k}_{i m^{\prime}} - a^{X_k}_{i m^{\prime}})^2  \right) \\
& \quad + 3 \Vert e_{ij}(t) \Vert^2 \Vert \hat{\phi}_m(t) - \phi_m(t) \Vert^2.
\end{align*}
Note that $\sum^{\infty}_{m=1} (\hat{a}^{X_k}_{i m} )^2=\Vert g^{X_k}_i \Vert^2$, thus we have
\begin{align*}
\frac{1}{n} \sum^n_{i=1} v^2_{i m} & \leq  3 \left( \sum_{k \in \mathscr{N}_j} \sum^{\infty}_{m^{\prime}=1} (\tilde{b}_{k,m m^{\prime}} - b^{\ast}_{k,m m^{\prime}})^2  \right) \left( \sum_{k \in \mathscr{N}_j} \frac{1}{n} \sum^n_{i=1} \Vert g^{X_k}_i \Vert^2 \right) \\ 
+ & 3 \left( \sum_{k \in \mathscr{N}_j} \sum^{\infty}_{m^{\prime}=M+1} \left( b^{\ast}_{k,m m^{\prime}} \right)^2  \right) \left( \sum_{k \in \mathscr{N}_j} \sum^{\infty}_{m^{\prime}=M+1} \frac{1}{n} \sum^n_{i=1} (\hat{a}^{X_k}_{i m^{\prime}} - a^{X_k}_{i m^{\prime}})^2  \right) \\
+ & 3 \left( \frac{1}{n} \sum^n_{i=1} \Vert e_{ij}(t) \Vert^2 \right) \Vert \hat{\phi}_m(t) - \phi_m(t) \Vert^2 \\
\overset{\Delta}{=} & \text{I}^{\prime}_m + \text{II}^{\prime}_m + \text{III}^{\prime}_m.
\end{align*}

By Lemma~\ref{lemma:b-diff-bd} and our assumption, we have
\begin{multline*}
\text{I}^{\prime}_m \leq 6  \left( \sum_{k \in \mathscr{N}_j} \Gamma_g(k) \right) \left\{ \left( \sum_{k \in \mathscr{N}_j} \Vert \beta_{k} (t,t^{\prime}) \Vert^2_{\text{HS}} \right) \Gamma^2_{\phi}(m) \right. \\
\left. + \sum_{k \in \mathscr{N}_j} \sum^{\infty}_{m^{\prime}=1} \left\vert \int_{\mathcal{T}\times\mathcal{T}} \beta_{k} (t,t^{\prime}) \phi_m (t) \left( \hat{\phi}_{m^{\prime}} (t^{\prime}) - \phi_{m^{\prime}} (t^{\prime})  \right)  d t^{\prime} dt \right\vert^2 \right\}.
\end{multline*}

By Lemma~\ref{lemma:a-diff-bd}, we have 
\begin{equation*}
\frac{1}{n} \sum^n_{i=1} \left( \hat{a}^{X_k}_{i m^{\prime}} - a^{X_k}_{i m^{\prime}} \right)^2 \leq \Gamma_g (k) \Gamma^2_{\phi}(m^{\prime}).
\end{equation*}
Thus,
\begin{align*}
\text{II}^{\prime}_m & \leq  3 \left( \sum_{k \in \mathscr{N}_j} \sum^{\infty}_{m^{\prime}=M+1} \left( b^{\ast}_{k,m m^{\prime}} \right)^2  \right) \left( \sum_{k \in \mathscr{N}_j} \Gamma_g(k) \right) \left( \sum^{\infty}_{m^{\prime}=M+1}  \Gamma^2_{\phi}(m^{\prime}) \right).
\end{align*}
In addition, we have
\begin{align*}
\text{III}^{\prime}_m \leq 3 \Gamma_e \Gamma^2_{\phi}(m).
\end{align*}
Combining the above results, we complete the proof of~\eqref{eq:lemma14-res1}. 

To show~\eqref{eq:lemma14-res1}, note that $\{ \phi_{m} \}^{\infty}_{m=1}$ is an orthonormal function basis, and we treat $$\int_{\mathcal{T}} \beta_{k} (t,t^{\prime}) \left( \hat{\phi}_{m^{\prime}} (t^{\prime}) - \phi_{m^{\prime}} (t^{\prime})  \right)  d t^{\prime}$$ as a function of $t$, then we have
\begin{align*}
& \quad \sum_{k \in \mathscr{N}_j} \sum^{\infty}_{m^{\prime}=1} \sum^{\infty}_{m=1} \left\vert \int_{\mathcal{T}\times\mathcal{T}} \beta_{k} (t,t^{\prime}) \phi_m (t) \left( \hat{\phi}_{m^{\prime}} (t^{\prime}) - \phi_{m^{\prime}} (t^{\prime})  \right)  d t^{\prime} dt \right\vert^2 \\
& = \sum_{k \in \mathscr{N}_j} \sum^{\infty}_{m^{\prime}=1} \int_{\mathcal{T}} \left\{ \int_{\mathcal{T}} \beta_{k} (t,t^{\prime}) \left( \hat{\phi}_{m^{\prime}} (t^{\prime}) - \phi_{m^{\prime}} (t^{\prime})  \right)  d t^{\prime}  \right\}^2 dt \\
& \leq \sum_{k \in \mathscr{N}_j} \sum^{\infty}_{m^{\prime}=1} \int_{\mathcal{T}} \left\{ \int_{\mathcal{T}} \beta^2_{k} (t,t^{\prime}) dt  \right\} \left\{ \int_{\mathcal{T}} \left( \hat{\phi}_{m^{\prime}} (t^{\prime}) - \phi_{m^{\prime}} (t^{\prime})  \right)^2  d t^{\prime} \right\} dt \\
& = \left( \sum_{k \in \mathscr{N}_j} \Vert \beta_{k} (t,t^{\prime}) \Vert^2_{\text{HS}} \right) \left( \sum^{\infty}_{m^{\prime}=1} \Gamma^2_{\phi}(m^{\prime}) \right).
\end{align*}

Thus, we have
\begin{equation*}
\sum^M_{m=1} I_m \leq \sum^{\infty}_{m=1} I_m \leq 12 \left( \sum^{\infty}_{m=1}\Gamma^2_{\phi}(m) \right)\left( \sum_{k \in \mathscr{N}_j} \Vert \beta_{k} (t,t^{\prime}) \Vert^2_{\text{HS}} \right)   \left( \sum_{k \in \mathscr{N}_j} \Gamma_g(k) \right).
\end{equation*}
\eqref{eq:lemma14-res2} then follows the combination of the above inequality and~\eqref{eq:lemma14-res1}.
\end{proof}

\begin{lemma}\label{lemma:mat-prop}
Let $\bm{A} \in \mathbb{R}^{m \times n}$, $\bm{B} \in \mathbb{R}^{n \times m}$. We have 
\begin{align*}
& \sqrt{\rho_{\min}\left( \bm{B}^{\top} \bm{B} \right)} \Vert \bm{A} \Vert_{\text{F}} \leq \Vert \bm{BA} \Vert_{\text{F}} \leq \Vert \bm{B} \Vert_2 \Vert \bm{A} \Vert_{\text{F}} \\
\intertext{and}
&\sqrt{\rho_{\min}\left( \bm{B} \bm{B}^{\top}  \right)} \Vert \bm{A} \Vert_{\text{F}} \leq \Vert \bm{AB} \Vert_{\text{F}} \leq \Vert \bm{B} \Vert_2 \Vert \bm{A} \Vert_{\text{F}}.
\end{align*}
\end{lemma}
\begin{proof}
Let $\bm{A}=\left[ \bm{a_1}, \bm{a_2}, \dots, \bm{a_n} \right]$.
Since
\begin{equation*}
\Vert \bm{BA} \Vert^2_{\text{F}} = \tr \left( \bm{A}^{\top} \bm{B}^{\top} \bm{B} \bm{A} \right) = \sum^{n}_{i=1} \bm{a_i}^{\top} \bm{B}^{\top} \bm{B} \bm{a_i},
\end{equation*}
we have
\begin{align*}
\rho_{\min}\left( \bm{B}^{\top} \bm{B} \right) \Vert \bm{A} \Vert^2_{\text{F}} &= \rho_{\min} \left( \bm{B}^{\top} \bm{B} \right) \sum^{n}_{i=1} \Vert \bm{a_i} \Vert^2_2 \\
& \leq \sum^{n}_{i=1} \bm{a_i}^{\top} \bm{B}^{\top} \bm{B} \bm{a_i} \\
& \leq \Vert \bm{B} \Vert^2_2 \sum^{n}_{i=1}   \Vert \bm{a_i} \Vert^2_2 \\
& = \Vert \bm{B} \Vert^2_2 \Vert \bm{A} \Vert^2_{\text{F}}.
\end{align*}
The final result follows from taking the square root on both sides.
\end{proof}

\begin{lemma}\label{lemma:mat-prop2}
Let $\bm{A} \in \mathbb{R}^{n \times n}$ be symmetric and positive semi-definite, $\bm{B} \in \mathbb{R}^{n \times m}$, and $\bm{C} \in \mathbb{R}^{n \times n}$ be a diagonal matrix with positive diagonal elements. Then we have
\begin{equation*}
\Vert (\bm{A} + \bm{C}) \bm{B} \Vert_F \geq \Vert \bm{C} \bm{B} \Vert_F.
\end{equation*}
\end{lemma}
\begin{proof}
Let $j$-th column of $\bm{B}$ be $\bm{b_j}$ for $j=1,\ldots,m$. Then we have
\begin{align*}
\Vert (\bm{A} + \bm{C}) \bm{B} \Vert^2_F & = \tr \left( \bm{B}^{\top} (\bm{A} + \bm{C})^{\top} (\bm{A} + \bm{C})  \bm{B} \right) \\
& = \sum^m_{j=1} \bm{b_j}^{\top} (\bm{A} + \bm{C})^{\top} (\bm{A} + \bm{C}) \bm{b_j} \\
& = \sum^m_{j=1} \bm{b_j}^{\top} \bm{A}^{\top} \bm{A} \bm{b_j} + \sum^m_{j=1} \bm{b_j}^{\top} \bm{C}^{\top} \bm{A} \bm{b_j} + \sum^m_{j=1} \bm{b_j}^{\top} \bm{A}^{\top} \bm{C} \bm{b_j} + \sum^m_{j=1} \bm{b_j}^{\top} \bm{C}^{\top} \bm{C} \bm{b_j}.
\end{align*}
Since $\bm{A}$ is symmetric and positive semi-definite and $\bm{C}$ is a diagonal matrix with positive diagonal elements, we have $\bm{C}^{\top} \bm{A}=\bm{A}^{\top} \bm{C}=\bm{C}\bm{A}$ to be symmetric and positive semi-definite. And note that $\bm{A}^{\top} \bm{A}$ is symmetric and positive semi-definite, thus we have
\begin{equation*}
\Vert (\bm{A} + \bm{C}) \bm{B} \Vert^2_F \geq \sum^m_{j=1} \bm{b_j}^{\top} \bm{C}^{\top} \bm{C} \bm{b_j} = \tr \left( \bm{B}^{\top} \bm{C}^{\top} \bm{C} \bm{B} \right) = \Vert \bm{C} \bm{B} \Vert^2_F,
\end{equation*}
which implies the final result.
\end{proof}


\section{Wall-Clock Runtime Comparison}
\label{sec:wall-clock}

{\newText

Table~\ref{tab:roc.runtime} documents the wall-clock runtime required for each method to generate a Receiver Operating Characteristic (ROC) curve. We execute each method on $100$ distinct values of $\lambda_n$. Table~\ref{tab:scv.runtime} logs the wall-clock runtime of the FPCA-$g_X$ method to generate an estimated graph via the Selective Cross-Validation (SCV) algorithm as described in Algorithm~\ref{Alg:scv.grid}, over a two-dimensional grid of $(\lambda_n, t_\epsilon)$. This grid encompasses $100$ unique values of $\lambda_n$ and $7$ different values of $t_\epsilon$. All tasks are executed in parallel using R, utilizing 28 CPU cores on Chicago Booth's Mercury Computing Cluster.

The runtime analysis indicates that the FPCA-$g_X$ method is marginally faster than the other two FPCA methods and significantly outpaces PSKL in most instances. Even though FPCA-$g_X$ can occasionally be slower than FGLasso, it delivers more accurate results while maintaining comparable runtime. Notably, FPCA methods outperform FGLasso or PSKL in terms of speed. Despite the SCV process operating over a two-dimensional grid with $100$ different values of $\lambda_n$ and $7$ different values of $t_\epsilon$, its runtime is akin to that of the ROC process, which operates on a one-dimensional grid of $\lambda_n$. This can be attributed to the fact that the most time-intensive step of the SCV process, the ADMM algorithm, is executed only once for each value of $\lambda_n$.

\begin{table}[t]
    \centering
    \begin{tabular}{|p{0.9cm}|p{0.5cm}||p{2.1cm}|p{2.1cm}|p{2.15cm}|p{2.15cm}|p{2.1cm}|}
    \hline
    Model & $p$  & FPCA-$g_X$ & FPCA-$g_Y$ & FGLasso & PSKL & FPCA-PSKL\\
    \hhline{|=|=||=|=|=|=|=|}
    \multirow{3}{4em}{A} & 
    50 & 24.8 (3.0) & 47.9 (6.3) & {\bf 11.4} (2.5) & 126.6 (9.9) & \\
    &100 & 137.7 (16.7) & 186.4 (20.7) & {\bf 35.1} (21.4) & 241.3 (21.2) & N/A \\
    &150 & 305.9 (58.2) & 377.9 (63.8) & {\bf 69.9} (10.3) & 418.4 (24.0) & \\
    \hline
    \multirow{3}{4em}{B} & 
    50 & 46.6 (2.6) & 75.1 (3.9) & {\bf 33.5} (15.5) & 108.1 (7.8) & \\
    &100 & 263.3 (29.2) & 345.0 (33.4) & {\bf 111.0} (72.1) & 226.8 (13.4) & N/A \\
    &150 & 807.7 (125.3) & 973.7 (136.0) & {\bf 212.5} (188.2) & 437.3 (27.9) & \\
    \hline
    \multirow{3}{4em}{C} & 
    50 & {\bf 14.6} (4.1) & 29.2 (7.1) & 91.3 (27.7) & 343.3 (43.5) &\\
    &100 & 174.1 (125.2) & 218.9 (113.1) & {\bf 87.8} (78.8) & 695.6 (73.3) & N/A\\
    &150 & 692.3 (684.0) & 754.9 (517.1) & {\bf 139.8} (132.7) & 1037.9 (200.2) &\\
    \hline
    \multirow{3}{4em}{D} & 
    50 & {\bf 11.2} (16.6) & 26.7 (22.3) & 72.3 (37) & 259.6 (74.8) & 14.8 (23.0)\\
    &100 & {\bf 92.0} (99.1) & 129.9 (121.4) & 390.2 (333.8) & 603.3 (151.0) & 112.3 (121.8)\\
    &150 & {\bf 92.3} (111.5) & 126 (128.2) & 1228.2 (325.7) & 1108.4 (103.6) & 100.0 (122.3)\\
    \hline
    \end{tabular}
    \caption{\label{tab:roc.runtime}The average wall-clock running time for each method to obtain a ROC under each model. 
    All of results are recorded in seconds.
    The fastest algorithm in each setting is marked bold. The standard deviation is given in parenthesis.}
\end{table}

\begin{table}[t]
    \centering
    \begin{tabular}{|p{1.2cm}||p{2.1cm}|p{2.1cm}|p{2.15cm}|p{2.15cm}|}
    \hline
    Model & A & B & C & D \\
    \hhline{|=||=|=|=|=|}
    $p=50$ & 22.3 (2.9) & 45.0 (6.9) & 12.2 (1.2) & 17.0 (16.2) \\
    $p=100$ & 141.1 (24.1) & 218.9 (42.8) & 111.7 (12.9) & 108.6 (106.8) \\
    $p=150$ & 326.0 (70.5) & 707.1 (174.0) & 463.8 (322.0) & 220.6 (60.0) \\
    \hline
    \end{tabular}
    \caption{The average wall-clock running time of FPCA-$g_X$ to obtain an estimated graph under each model, where $(\lambda_n,\epsilon_n)$ is chosen by SCV algorithm stated in Algorithm~\ref{Alg:scv.grid}.
    All of results are recorded in seconds.
    The standard deviation is given in parenthesis.}
    \label{tab:scv.runtime}
\end{table}

}

\clearpage

\section{Labels of ROIs in the AAL Atlas}
\label{sec:roi-labels}

\begin{longtable}{|p{0.85cm}|p{3.4cm}||p{0.85cm}|p{3.4cm}||p{0.85cm}|p{3.4cm}|}
\hline
Label & Name & Label & Name & Label & Name \\
\hhline{|=|=||=|=||=|=|}
2001  & Precentral\_L           & 4112  & ParaHippocampal\_R     & 8101 & Heschl\_L              \\
2002  & Precentral\_R           & 4201  & Amygdala\_L            & 8102 & Heschl\_R              \\
2101  & Frontal\_Sup\_L         & 4202  & Amygdala\_R            & 8111 & Temporal\_Sup\_L       \\
2102  & Frontal\_Sup\_R         & 5001  & Calcarine\_L           & 8112 & Temporal\_Sup\_R       \\
2111  & Frontal\_Sup\_Orb\_L    & 5002  & Calcarine\_R           & 8121 & {\small Temporal\_Pole\_Sup\_L} \\
2112  & Frontal\_Sup\_Orb\_R    & 5011  & Cuneus\_L              & 8122 & {\small Temporal\_Pole\_Sup\_R} \\
2201  & Frontal\_Mid\_L         & 5012  & Cuneus\_R              & 8201 & Temporal\_Mid\_L       \\
2202  & Frontal\_Mid\_R         & 5021  & Lingual\_L             & 8202 & Temporal\_Mid\_R       \\
2211  & Frontal\_Mid\_Orb\_L    & 5022  & Lingual\_R             & 8211 & {\small Temporal\_Pole\_Mid\_L} \\
2212  & Frontal\_Mid\_Orb\_R    & 5101  & Occipital\_Sup\_L      & 8212 & {\small Temporal\_Pole\_Mid\_R} \\
2301  & Frontal\_Inf\_Oper\_L   & 5102  & Occipital\_Sup\_R      & 8301 & Temporal\_Inf\_L       \\
2302  & Frontal\_Inf\_Oper\_R   & 5201  & Occipital\_Mid\_L      & 8302 & Temporal\_Inf\_R       \\
2311  & Frontal\_Inf\_Tri\_L    & 5202  & Occipital\_Mid\_R      & 9001 & Cerebelum\_Crus1\_L    \\
2312  & Frontal\_Inf\_Tri\_R    & 5301  & Occipital\_Inf\_L      & 9002 & Cerebelum\_Crus1\_R    \\
2321  & Frontal\_Inf\_Orb\_L    & 5302  & Occipital\_Inf\_R      & 9011 & Cerebelum\_Crus2\_L    \\
2322  & Frontal\_Inf\_Orb\_R    & 5401  & Fusiform\_L            & 9012 & Cerebelum\_Crus2\_R    \\
2331  & Rolandic\_Oper\_L       & 5402  & Fusiform\_R            & 9021 & Cerebelum\_3\_L        \\
2332  & Rolandic\_Oper\_R       & 6001  & Postcentral\_L         & 9022 & Cerebelum\_3\_R        \\
2401  & {\small Supp\_Motor\_Area\_L}    & 6002  & Postcentral\_R         & 9031 & Cerebelum\_4\_5\_L     \\
2402  & {\small Supp\_Motor\_Area\_R}    & 6101  & Parietal\_Sup\_L       & 9032 & Cerebelum\_4\_5\_R     \\
2501  & Olfactory\_L            & 6102  & Parietal\_Sup\_R       & 9041 & Cerebelum\_6\_L        \\
2502  & Olfactory\_R            & 6201  & Parietal\_Inf\_L       & 9042 & Cerebelum\_6\_R        \\
2601  & {\small Frontal\_Sup\_Medial\_L} & 6202  & Parietal\_Inf\_R       & 9051 & Cerebelum\_7b\_L       \\
2602  & {\small Frontal\_Sup\_Medial\_R} & 6211  & SupraMarginal\_L       & 9052 & Cerebelum\_7b\_R       \\
2611  & Frontal\_Med\_Orb\_L    & 6212  & SupraMarginal\_R       & 9061 & Cerebelum\_8\_L        \\
2612  & Frontal\_Med\_Orb\_R    & 6221  & Angular\_L             & 9062 & Cerebelum\_8\_R        \\
2701  & Rectus\_L               & 6222  & Angular\_R             & 9071 & Cerebelum\_9\_L        \\
2702  & Rectus\_R               & 6301  & Precuneus\_L           & 9072 & Cerebelum\_9\_R        \\
3001  & Insula\_L               & 6302  & Precuneus\_R           & 9081 & Cerebelum\_10\_L       \\
3002  & Insula\_R               & 6401  & Paracentral\_Lobule\_L & 9082 & Cerebelum\_10\_R       \\
4001  & Cingulum\_Ant\_L        & 6402  & Paracentral\_Lobule\_R & 9100 & Vermis\_1\_2           \\
4002  & Cingulum\_Ant\_R        & 7001  & Caudate\_L             & 9110 & Vermis\_3              \\
4011  & Cingulum\_Mid\_L        & 7002  & Caudate\_R             & 9120 & Vermis\_4\_5           \\
4012  & Cingulum\_Mid\_R        & 7011  & Putamen\_L             & 9130 & Vermis\_6              \\
4021  & Cingulum\_Post\_L       & 7012  & Putamen\_R             & 9140 & Vermis\_7              \\
4022  & Cingulum\_Post\_R       & 7021  & Pallidum\_L            & 9150 & Vermis\_8              \\
4101  & Hippocampus\_L          & 7022  & Pallidum\_R            & 9160 & Vermis\_9              \\
4102  & Hippocampus\_R          & 7101  & Thalamus\_L            & 9170 & Vermis\_10             \\
4111  & ParaHippocampal\_L      & 7102  & Thalamus\_R            &      &\\
\hline

\caption{Labels and names of each ROI in the AAL atlas}

\label{tab:aal}

\end{longtable}

\clearpage

\section{Table of Notations}
\label{sec:notations}

\begin{longtable}{ |p{2.8cm}||p{9.5cm}|p{0.9cm}|  }
 \hline
 Notation & Meaning & Page\\
 \hline
 $G = (V,E)$ & undirected graph, $V$ is set of vertices, $E$ is set of edges & 2 \\
 $\bm{X}$ & $p$-dimensional random variables & 2\\
 $M$ & number of basis functions we used to do dimension reduction & 3\\
 $\bm{g}(\cdot)$ & $p$-dimensional multivariate Gaussian process & 5\\
 $\mathcal{T}$ & domain of multivariate Gaussian process & 5\\
 $C_{jl}(t,t')$ & conditional cross-covariance function & 5\\
 $\{\bm{g_i}(\cdot)\}_{i=1}^n$ & $\bm{g_i}(\cdot)=(g_{i1}(\cdot), \cdots, g_{ip}(\cdot))^\top$, random copies of $\bm{g}(\cdot)$ & 5\\
 $\beta_{jk}(t,t')$ & coefficient on ${g}_{ij}$ from ${g}_{ik}$ & 5\\
 $\mathscr{N}_j, \hat{\mathscr{N}}_j$ & (estimated) neighborhood set of node $j$ & 5\\
 $e_{ij}(\cdot)$ & error of $g_{ij}(\cdot)$ & 5\\
 $\bm{\phi_j}$ & $=\{\phi_{jm}(\cdot)\}_{m=1}^\infty$, orthonormal functional basis on $\mathbb{H}$ & 6\\
 $\bm{a_{i,k,M}}$ & $=(a_{ik1},\cdots,a_{ikM})^\top$, vector of projection scores & 6\\
 $g_i^Y(\cdot), g_i^{X_k}(\cdot)$ & random functions of the target node and the other random functions & 6\\
 $\bm{a_{i,M}^Y}, \bm{a_{i,M}^{X_k}}$, $\bm{a_{i,M}^X}$ & vectors of scores projected on known bases $\bm{\phi_j}$ & 6\\
 $\bm{B^{\ast}_{k,M}}$ & regression matrix parameter & 7\\
 $\bm{w_{i,M}}, \bm{r_{i,M}}$ & noise vector and bias term for $M$-truncation & 7\\
 $\bm{\hat{B}_{k,M}}$ & Estimator of $\bm{B^{\ast}_{k,M}}$ & 7\\
 $\lambda_n$ & penalty parameter for group Lasso regression & 7\\
 $\epsilon_n$ & threshold parameter of neighborhood recognition & 7\\
 $\hat{E}$ & estimated edge set &7\\
 $\bm{\hat{a}_{i,M}^Y}, \bm{\hat{a}_{i,M}^{X_k}}$, $\bm{\hat{a}_{i,M}^X}$ & vectors of scores projected on estimated bases $\bm{\hat{\phi}_j}$ & 7-8\\
 $K_{jj}(t',t)$ & functional covariance of $g_{ij}(\cdot)$ & 8\\
 $\mathscr{K}_j(f)(t)$ & Hilbert-Schmidt covariance operator & 8\\
 $\{\sigma_{jm}\}_{m\in\mathbb{N}}$ & eigenvalues of $\mathscr{K}_j$ & 8\\
 $g_{ij}^M(t)$ & $\mathcal{L}^2$ projection of $g_{ij}(t)$ onto the basis spanned by the first $M$ FPCA functions & 8\\
 $\hat{K}_{jj}(t',t)$ & empirical functional covariance of $g_{ij}(\cdot)$ & 8\\
 $\{\hat{\sigma}_{jm},\hat{\phi}_{jm}(t)\}_{m=1}^M$ & eigenpairs of $\hat{K}_{jj}(t',t)$ & 8\\
 $\hat{a}_{ijm}$ & $=\int_\mathcal{T} g_{ij}(t)\hat{\phi}_{jm}(t)dt$, estimated FPCA scores & 8\\
 $\bm{\hat{B}}_{\lambda_n}$ & $=(\bm{\hat{B}_1},\cdots,\bm{\hat{B}_{p-1}})$, group Lasso estimates under a fixed $\lambda_n$ & 9\\
 $\bm{\Tilde{B}_k}$ & estimates of $\bm{B_k}$ in selective cross-validation process & 10\\
 $\bm{\hat{\epsilon}_i}$ & $=\bm{a^Y_{i,M}} - \sum^{p-1}_{k=1} \bm{\tilde{B}_{k}} \bm{a^{X_k}_{i,M}}$, residuals of $\bm{\Tilde{B}_k}$ & 10\\
 $\bm{A^Y}, \bm{A^{X_k}}, \bm{A^X}$ & matrices of FPCA scores & 10\\
 $\rho, \rho^h$ & penalty parameter for ADMM subproblem (on iteration $h$) & 10\\
 $b^{\ast}_{k,m m^{\prime}}$ & $= \int_{\mathcal{T}\times \mathcal{T}} \beta_k (t^{\prime},t) \phi_m (t) \phi_{m^{\prime}} (t^{\prime}) dt^{\prime} dt$ & 12\\
 $\beta_{k,M} (t^{\prime},t)$ & $ =  \sum^{M}_{m,m^{\prime}=1} b^{\ast}_{k,m m^{\prime}} \phi_m (t) \phi_{m^{\prime}} (t^{\prime})$ & 13\\
 $\beta_{k, >M} (t^{\prime},t)$ & $=  \sum^{\infty}_{m>M \,\text{or}\, m^{\prime}>M} b^{\ast}_{k,m m^{\prime}} \phi_m (t) \phi_{m^{\prime}} (t^{\prime})$ & 13\\
 $\bm{\Sigma^{X_k,r}}$ & $=\Cov(\bm{a_{i,M}^{X_k}}, \bm{r_{i,M}})$, and $\bm{\Sigma^{r, X_k}} = (\bm{\Sigma^{X_k,r}})^\top$ & 13\\
 $\bm{\Sigma^{X_k,X_l}}$ & $=\Cov(\bm{a_{i,M}^{X_k}}, \bm{a_{i,M}^{X_l}})$ &13\\
 $\Xi_1(M)$ & 
 \begin{equation*}
 =  \max_{k \in [p-1]} \left\{ \rho_{\max}(\bm{\Sigma^w}+\bm{\Sigma^r} - \bm{\Sigma^{r,X_k}} \left( \bm{\Sigma^{X_k, X_k}} \right)^{-1} \bm{\Sigma^{{X_k},r}}) \right\}
 \end{equation*}
 & 13\\
 $\Xi_2(M)$ & $=  \max_{k \in[p-1]}\rho_{\max}(\bm{\Sigma^{X_k, X_k}})$ & 13\\
 $\Xi_3(M)$ & $=  \max_{k \in[p-1]}\tr(\bm{\Sigma^{X_k, X_k}})$ & 13\\
 $\Xi_4(M)$  &$=  \tr \left\{ \bm{\Sigma^r} + \bm{\Sigma^w} + \bm{\Sigma^{r,w}} + \bm{\Sigma^{w,r}} \right\}$ &13\\
 $\omega(M)$ &$=  \max_{k \in [p-1]} \left\Vert \bm{\Sigma^{r,X_k}}\right\Vert_{\text{F}}$ & 13\\
 $K_0$ & $=\max_{k\in[p-1], m\in M} \mathbb{E}[(a_{i,m}^{X_k})^2]$, used to derive an upper bound for the estimation error of the covariance matrix of $\bm{a_{i,M}^X}$ & 13\\
 $\kappa(M)$ & $=\rho_{\min}\left((\bm{\Sigma}^X)_{\mathscr{N}_j, \mathscr{N}_j}\right)$ & 14\\
 $\tau(M)$ & $=\min_{k \in \mathscr{N}_j} \left\Vert \bm{B^{\ast}_k} \right\Vert_{\text{F}} = \min_{k \in \mathscr{N}_j} \left\Vert \beta_{k, M} (t^{\prime},t) \right\Vert_{\text{HS}}$, relevant signal strength &14\\
 $\chi(n,p,M,\delta)$ & $=\frac{6\sqrt{s}}{\sqrt{\kappa(M)}} \tilde{\lambda} (n, p, M, \delta)$ & 14\\
 $\tilde{\lambda} (n, p, M, \delta)$ & exact form can be found in \eqref{eq:lambda_tilde_exact} &14\\
 $\Phi(M)$ & $=\sqrt{ \sum^{p-1}_{k=1} \sum^M_{m=1} \sum^{\infty}_{m^{\prime}=M+1} \left( b^{\ast}_{k,m m^{\prime}} \right)^2 }$ & 16\\
 $\check{\chi}(n,p,M,\delta)$ & $= \frac{6\sqrt{s}}{\sqrt{\kappa(M)}} \check{\lambda} (n, p, M, \delta)$ & 16\\
 $\check{\lambda} (n, p, M, \delta)$ & exact form given in \eqref{eq:lambda_check} &16\\
 $\Lambda(M,\bm{\phi})$ & $\frac{\omega(M)}{\sqrt{\kappa(M)}\tau(M)}$ & 17\\
 $(\bm{a_{i1}},\cdots, \bm{a_{ip}})^\top$ & generated functional scores from a mean zero Gaussian distribution & 18\\
 $\bm{f}(\cdot)$ & vector of Fourier basis functions & 18\\
 $M^*$ & number of basis functions contained in $\bm{f}(\cdot)$ & 18\\
 $\bm{\Sigma}$ & covariance matrix of $(\bm{a_{i1}},\cdots, \bm{a_{ip}})^\top$ & 18\\
 $\bm{\Theta}$ & $=\bm{\Sigma}^{-1}$, precision matrix & 18\\
 $T$  & number of observation time points & 19\\
 $\epsilon_{ijk}$ & observation error of $g_{ij}(t_k)$ & 19\\
 $\sigma$ & variance of $\epsilon_{ijk}$ & 19\\
 \hline

\caption{Summary of notations used in the paper.}
 
\end{longtable}

\end{document}